\newif\ifauthors
\authorstrue

\documentclass[11pt, oneside]{article}
\usepackage{geometry}
\usepackage[utf8]{inputenc} 

\usepackage{geometry}[1in]

\usepackage[T1]{fontenc}    
\usepackage{hyperref}       

\usepackage{url}            
\usepackage{booktabs}       
\usepackage{amsfonts}       
\usepackage{latexsym,amssymb,amsthm,amscd,amsopn,amsmath}
\usepackage{bbm}
\usepackage{mathrsfs}
\usepackage{nicefrac}       
\usepackage{microtype}      
\usepackage{xcolor}         
\usepackage{algorithm}
\usepackage{verbatim}
\usepackage[noend]{algpseudocode}

\usepackage{xspace}
\usepackage{mathtools}
\usepackage{svg}
\usepackage[nameinlink,capitalize]{cleveref}
\usepackage{makecell}

\definecolor{dgreen}{rgb}{0,0.5,0}
\hypersetup{
  colorlinks=true,
  linkcolor=blue,
  filecolor=blue,
  citecolor = dgreen,      
  urlcolor=cyan,
}

\crefformat{equation}{#2(#1)#3}
\Crefformat{equation}{#2(#1)#3}
\Crefname{construction}{Construction}{Constructions}

\Crefformat{figure}{#2Figure #1#3}
\Crefname{assumption}{Assumption}{Assumptions}
\Crefformat{assumption}{#2Assumption #1#3}
\Crefname{subsubsection}{Section}{Sections}
\crefformat{subsubsection}{#2Section #1#3}
\Crefformat{subsubsection}{#2Section #1#3}

\theoremstyle{plain}
\newtheorem{theorem}{Theorem}
\newtheorem{lemma}[theorem]{Lemma}

\newtheorem{proposition}[theorem]{Proposition}

\newtheorem{assumption}[theorem]{Assumption}
\theoremstyle{definition}
\newtheorem{definition}[theorem]{Definition}
\newtheorem{defn}[theorem]{Definition}

\newtheorem{remark}[theorem]{Remark}
\newtheorem{question}[theorem]{Question}

\newtheoremstyle{named}%
    {}{}{\itshape}{}{\bfseries}{.}{.5em}{\thmnote{#3}}
\theoremstyle{named}
\newtheorem*{namedproblem}{Problem}

\newcommand\bref[3][blue]{%
    \begingroup%
    \hypersetup{linkcolor=#1}%
    \hyperlink{#2}{\textup{#3}}%
    \endgroup}

\numberwithin{theorem}{section}

\newenvironment{namedproof}[1]{\paragraph{Proof of #1.}\hspace{-1em}}{\hfill$\blacksquare$\vspace{1em}}

\newcommand{\nc}{\newcommand}
\nc{\DMO}{\DeclareMathOperator}
\newcount\Comments  
\DeclareMathOperator*{\argmin}{arg\,min} 
\DeclareMathOperator*{\argmax}{arg\,max}

\nc{\Moracle}{\MM^{\mathsf{oracle}}}
\nc{\tilMoracle}{\til{\MM}^{\mathsf{oracle}}}
\nc{\SimulateReduction}{\texttt{SimulateReduction}\xspace}
\nc{\TestReduction}{\texttt{DistinguishReduction}\xspace}
\nc{\SimulateSampling}{\texttt{SimulateSampling}\xspace}
\nc{\SimulateRegression}{\texttt{SimulateRegression}\xspace}
\nc{\Osample}{{\MO_{\mathsf{samp}}}}
\nc{\Oregress}{{\MO_{\mathsf{regress}}}}
\nc{\Oregressp}{{\MO'_{\mathsf{regress}}}}
\nc{\Obandits}{{\MO_{\mathsf{bandits}}}}
\nc{\dom}{\mathsf{dom}}
\nc{\SF}{\mathscr{F}}
\nc{\Fchernoff}{\MF^{\mathsf{chernoff}}}

\DMO{\prox}{prox}
\DMO{\Span}{span}
\DMO{\UCB}{UCB}
\DMO{\LCB}{LCB}
\nc{\expl}[2]{\ME^{#1}_{#2}}
\nc{\tilmdp}[1]{\til \MM({#1})}
\nc{\barpdp}[2]{\ol \MP_{#1}({#2})}
\nc{\barmdp}[1]{\ol \MM({#1})}
\nc{\hatmdp}[1]{\wh \MM({#1})}
\nc{\rem}[2]{\MR_{#1}({#2})}
\nc{\Pigen}{\Pi^{\rm gen}}
\nc{\Pidet}{\Pi^{\rm det}}
\nc{\PiZ}{\Pi_{\SZ}^{\rm markov}}
\nc{\und}[3]{\MU_{{#1}}^{{#2}}({#3})}
\nc{\zlow}[2]{\MZ_{{#1}}^\lowv({#2})}
\nc{\dg}{\dagger}
\nc{\bB}{\mathbf{B}}
\nc{\unif}{\mu_{\rm unif}}
\nc{\indsig}[2]{\mathcal{I}_{#1}({#2})}
\nc{\total}{{\rm fin}}
\nc{\early}{{\rm pre}}
\nc{\zsink}{z_{\rm sink}}
\nc{\lowv}{{\rm low}}
\nc{\oo}[1]{\texttt{o}({#1})}
\nc{\posnrm}[1]{\left[ {#1} \right]_+}
\nc{\negnrm}[1]{\left[ {#1} \right]_-}
\nc{\tvnrm}[1]{\left\| {#1} \right\|_1}
\nc{\absval}[1]{\left| {#1} \right|}
\nc{\normalize}[1]{\mathfrak{n}\left({#1}\right)}

\nc{\SZ}{\textsf{Z}}
\nc{\SO}{\textsf{O}}
\nc{\suff}[2]{{\rm suff}_{#1}({#2})}
\nc{\UPhi}{\mathscr{U}_{X,H}}
\nc{\UPhis}{\til{\mathscr{U}}_{X,H,\MF}}
\nc{\SV}{\mathscr{V}}
\nc{\Phiset}{\Phi_{X,H}}
\nc{\Phisets}{\til{\Phi}_{X,H,\MF}}
\nc{\Lyu}{{\mathtt{Lyu}}}
\nc{\wAlg}{{\widetilde \Alg}}

\nc{\ApproxMDP}{\texttt{ConstructMDP}\xspace}
\nc{\mainalg}{\texttt{BaSeCAMP}\xspace} 
\nc{\bspanner}{\texttt{BarySpannerPolicy}\xspace}

\nc{\gamvec}{\gamma}
\nc{\til}{\widetilde}
\nc{\td}{\tilde}
\nc{\wh}{\widehat}
\nc{\todo}[1]{\ifnum\Comments=1 {\color{red}  [TODO: #1]}\fi}
\nc{\old}[1]{\ifnum\Comments=1 {\color{brown}  [COPIED: #1]}\fi}
\definecolor{darkgreen}{rgb}{0.0, 0.5, 0.0}
\nc{\noah}[1]{\ifnum\Comments=1 {\color{darkgreen} [ng: #1]}\fi}
\nc{\dhruv}[1]{\ifnum\Comments=1 {\color{purple} [dr: #1]}\fi}
\nc{\BP}{\mathbb{P}}
\nc{\BM}{\mathbb{M}}
\nc{\bbapx}{\bb^{\rm apx}}
\nc{\bbapxs}[1]{\bb^{\rm apx, {#1}}}

\nc{\fools}[3]{\MF_{#3}({#1}, {#2})}
\nc{\fool}[2]{\MF({#1},{#2})}
\nc{\clip}[2]{{\rm clip}\left[ \left. {#1} \right| {#2} \right]}
\nc{\imax}{\omega}
\DMO{\conv}{conv}
\nc{\MH}{\mathcal{H}}
\nc{\CH}{\mathscr{H}}
\nc{\CB}{\mathscr{B}}
\nc{\cD}{\mathscr{D}}
\nc{\MC}{\mathcal{C}}
\nc{\st}{\star}
\nc{\lng}{\langle}
\nc{\rng}{\rangle}
\DMO{\OOPT}{opt}
\nc{\dopt}[2]{\ell_{\OOPT}({#1},{#2})}
\nc{\grad}{\nabla}
\nc{\MG}{\mathcal{G}}
\nc{\MP}{\mathcal{P}}
\nc{\PP}{\mathbb{P}}
\nc{\TT}{\mathbb{T}}
\nc{\TTmax}{\TT_{\max}}
\DMO{\Ham}{Ham}
\DMO{\Gap}{Gap}
\DMO{\GD}{GD}
\DMO{\GDA}{GDA}
\DMO{\EG}{EG}
\DMO{\OGDA}{OGDA}
\DMO{\Unif}{Unif}
\DMO{\Tr}{Tr}
\nc{\ul}{\underline}
\nc{\ol}{\overline}
\nc{\Qu}{\ul{Q}}
\nc{\Qo}{\ol{Q}}
\nc{\Ro}{\ol{R}}
\nc{\Vu}{\ul{V}}
\nc{\Vo}{\ol{V}}
\nc{\RanQ}{\Delta Q}
\nc{\RanV}{\Delta V}
\nc{\clipQ}{\Delta \breve{Q}}
\nc{\frzQ}{\Delta \mathring{Q}}
\nc{\clipV}{\Delta \breve{V}}
\nc{\clipdelta}{\breve{\delta}}
\nc{\cliptheta}{\breve{\theta}}
\nc{\delmin}{\Delta_{{\rm min}}}
\nc{\delmins}[1]{\Delta_{{\rm min},{#1}}}
\nc{\gapfinal}[1]{\max \left\{ \frac{\frzQ_{{#1}}^{k^\st}(x,a)}{2H}, \frac{\delmin}{4H} \right\}}
\nc{\post}[2]{R({#1}; {#2})}
\nc{\posts}[3]{R_{#3}({#1}; {#2})}
\nc{\MAJ}{\mathsf{MAJ}}
\nc{\Dnull}{D^{\circ}}
\DeclareMathOperator{\CBer}{CBer}
\nc{\BKW}{\mathtt{BKW}}
\nc{\Dec}{\mathtt{Dec}}
\nc{\delreg}{\delta_{\mathsf{reg}}}
\nc{\delreal}{\delta_{\mathsf{real}}}
\nc{\Sreg}{S_{\mathsf{Reg}}}
\nc{\Treg}{T_{\mathsf{Reg}}}
\nc{\PC}{\mathtt{PC}}
\nc{\alphaPC}{\alpha_{\mathsf{PC}}}
\nc{\TPC}{T_{\mathsf{PC}}}
\nc{\SPC}{S_{\mathsf{PC}}}
\nc{\delsmall}{{\delta_{\mathsf{small}}}}
\nc{\CL}{\mathtt{ContrastLearn}}
\nc{\Select}{\mathtt{Select}}
\nc{\piunif}{{\pi_{\mathsf{unif}}}}
\nc{\picov}{{\pi_{\mathsf{cov}}}}
\nc{\ZZ}{\mathbb{Z}}
\nc{\sk}{{\mathsf{sk}}}
\nc{\Enc}{\mathtt{Enc}}
\nc{\EntLPN}{\mathtt{EntangleLPN}}
\nc{\mureg}{{\mu_{\mathsf{reg}}}}
\nc{\PPE}{\mathtt{PPE}}
\nc{\FQI}{\mathtt{FQI}}
\nc{\False}{\mathtt{False}}
\nc{\True}{\mathtt{True}}
\nc{\epreg}{\epsilon_{\mathsf{reg}}}
\nc{\algnst}[1]{\begin{align*}#1\end{align*}}
\nc{\algn}[1]{\begin{align}#1\end{align}}
\nc{\matx}[1]{\left(\begin{matrix}#1\end{matrix}\right)}

\nc{\pimix}{{\pi_{\mathsf{mix}}}}
\nc{\BPC}{{B_{\mathsf{PC}}}}
\nc{\size}{\mathrm{size}}
\nc{\OLIVE}{\texttt{OLIVE}}
\nc{\NP}{\textsf{NP}}
\nc{\RP}{\textsf{RP}}
\nc{\cprp}{c_{\mathsf{PRP}}}
\nc{\Mtoy}{\MM_{\mathsf{toy}}}
\nc{\Brute}{\mathtt{Brute}}

\nc{\nuu}{\nu}

\nc{\bel}[1]{\mathbf{b}({#1})}
\nc{\nbel}[1]{\bar{\mathbf{b}}({#1})}
\nc{\sbel}[2]{\mathbf{b}'_{#1}({#2})}
\nc{\nsbel}[2]{\bar{\mathbf{b}}'_{#1}({#2})}

\nc{\bv}{\mathbf{v}}
\nc{\bone}{\mathbf{1}}
\nc{\bX}{\mathbf{X}}
\nc{\be}{\mathbf{e}}
\nc{\bY}{\mathbf{Y}}
\nc{\bG}{\mathbf{G}}
\nc{\bz}{\mathbf{z}}
\nc{\bw}{\mathbf{w}}
\nc{\bA}{\mathbf{A}}
\nc{\bJ}{\mathbf{J}}
\nc{\bK}{\mathbf{K}}
\nc{\bb}{\mathbf{b}}
\nc{\ba}{\mathbf{a}}
\nc{\bs}{\mathbf{s}}
\nc{\bzero}{\mathbf{0}}
\nc{\bi}{\mathbf{i}}
\nc{\Edistinct}{\ME^{\mathsf{distinct}}}
\nc{\bc}{\mathbf{c}}
\nc{\bC}{\mathbf{C}}
\nc{\BR}{\mathbb R}
\nc{\BA}{\mathbb{A}}
\nc{\SA}{\mathscr{A}}
\nc{\BC}{\mathbb C}
\nc{\bx}{\mathbf{x}}
\nc{\bS}{\mathbf{S}}
\nc{\bM}{\mathbf{M}}
\nc{\bR}{\mathbf{R}}
\nc{\bN}{\mathbf{N}}
\nc{\by}{\mathbf{y}}
\nc{\sy}{y}
\nc{\sx}{x}

\nc{\MO}{\mathcal O}
\nc{\MQ}{\mathcal{Q}}
\nc{\CO}{\mathscr{O}}
\nc{\MU}{\mathcal{U}}
\nc{\ME}{\mathcal{E}}
\nc{\MN}{\mathcal{N}}
\nc{\MK}{\mathcal{K}}
\nc{\MM}{\mathcal{M}}
\nc{\MS}{\mathcal{S}}
\nc{\MT}{\mathcal{T}}
\nc{\BF}{\mathbb F}
\nc{\BQ}{\mathbb Q}
\nc{\MX}{\mathcal{X}}
\nc{\MA}{\mathcal{A}}
\nc{\MD}{\mathcal{D}}
\nc{\MB}{\mathcal{B}}
\nc{\MZ}{\mathcal{Z}}
\nc{\MJ}{\mathcal{J}}
\nc{\MW}{\mathcal{W}}
\nc{\MR}{\mathcal{R}}
\nc{\MY}{\mathcal{Y}}
\nc{\ML}{\mathcal{L}}
\nc{\BZ}{\mathbb Z}
\nc{\BN}{\mathbb N}
\nc{\ep}{\epsilon}
\nc{\gapfn}[1]{\varepsilon_{#1}}
\nc{\ggapfn}[2]{\varphi_{#1}({#2})}
\nc{\epsahk}{\gapfn{0}}
\nc{\BH}{\mathbb H}
\nc{\BG}{\mathbb{G}}
\nc{\D}{\Delta}
\nc{\MF}{\mathcal{F}}
\nc{\One}{\mathbbm{1}}
\nc{\bOne}{\mathbf{1}}
\nc{\Aopt}{\mathcal{A}^{\rm opt}}
\nc{\Amul}{\mathcal{A}^{\rm mul}}

\nc{\SP}{\mathsf P}
\nc{\SQ}{\mathsf Q}

\nc{\DO}{\accentset{\circ}{\D}}
\nc{\mf}{\mathfrak}
\nc{\mfp}{\mathfrak{p}}
\nc{\mfq}{\mf{q}}
\nc{\Sp}{\mbox{Spec}}
\nc{\Spm}{\mbox{Specm}}
\nc{\hookuparrow}{\mathrel{\rotatebox[origin=c]{90}{$\hookrightarrow$}}}
\nc{\hookdownarrow}{\mathrel{\rotatebox[origin=c]{-90}{$\hookrightarrow$}}}
\nc{\hra}{\hookrightarrow}
\nc{\tra}{\twoheadrightarrow}
\nc{\sgn}{{\rm sgn}}
\nc{\aut}{{\rm Aut}}
\nc{\Hom}{{\rm Hom}}
\nc{\img}{{\rm Im}}
\DMO{\id}{Id}
\DMO{\supp}{supp}
\DMO{\KL}{KL}
\nc{\kld}[2]{\KL({#1}||{#2})}
\nc{\ren}[2]{D_2({#1}||{#2})}
\nc{\chisq}[2]{\chi^2({#1}||{#2})}
\nc{\tvd}[2]{D_{\mathsf{TV}}\left({#1}, {#2}\right)}
\nc{\hell}[2]{H^2({#1}, {#2})}
\DMO{\BSS}{BSS}
\DMO{\BES}{BES}
\DMO{\BGS}{BGS}
\DMO{\poly}{poly}
\nc{\indep}{\perp}
\DMO{\sink}{sink}
\DMO{\nosink}{nosink}
\nc{\sinks}{s^{\sink}}
\nc{\sinkobs}{o^{\sink}}
\nc{\fp}[1]{\MP_1({#1})}
\nc{\BO}{\mathbb{O}}
\nc{\BT}{\mathbb{T}}

\nc{\RR}{\mathbb{R}}
\nc{\NN}{\mathbb{N}}
\nc{\Gradient}{\nabla}
\DMO{\diag}{diag}
\nc{\norm}[1]{\left \lVert #1 \right \rVert}
\DMO*{\EE}{\mathbb{E}}
\nc{\LPN}{\mathsf{LPN}}
\DMO{\Ber}{Ber}
\nc{\Regress}{\mathtt{Regress}}
\nc{\LFC}{\mathtt{LearnFromCorr}}
\nc{\RegressAlg}{\mathtt{RegressAlg}}
\nc{\DrawTraj}{\mathtt{DrawTrajectory}}
\nc{\pizero}{{\pi_{\mathsf{zero}}}}
\nc{\Tred}{{T_{\mathsf{red}}}}
\nc{\epred}{{\epsilon_{\mathsf{red}}}}
\nc{\TriAlg}{\mathtt{GenerateTriangleLPN}}
\nc{\Alg}{\mathtt{Alg}}
\nc{\AffSample}{\mathtt{AffSample}}
\newcommand{\Var}{\mathbf{Var}}
\nc{\br}{\mathbf{r}}
\nc{\TV}{{\mathsf{TV}}}
\DMO{\Law}{Law}
\DMO{\Sym}{Sym}
\nc{\bu}{\mathbf{u}}
\nc{\Reg}{\mathtt{Reg}}
\nc{\Breg}{B_{\mathsf{Reg}}}
\DMO{\dc}{dc}

\DMO{\PR}{Pr}
\renewcommand{\Pr}{\PR}
\DMO*{\Prr}{Pr}
\nc{\E}{\mathbb{E}}
\nc{\ra}{\rightarrow}
\renewcommand{\t}{\top}

\title{Exploration is Harder than Prediction: Cryptographically Separating Reinforcement Learning from Supervised Learning}

\ifauthors
\author{Noah Golowich\thanks{Email: \texttt{nzg@mit.edu}. Supported by a Fannie \& John Hertz Foundation Fellowship and an NSF Graduate Fellowship.} \\ MIT \and Ankur Moitra\thanks{Email: \texttt{moitra@mit.edu}. Supported in part by a Microsoft Trustworthy AI Grant, an ONR grant and a David and Lucile Packard Fellowship.} \\ MIT \and Dhruv Rohatgi\thanks{Email: \texttt{drohatgi@mit.edu}. Supported by a U.S. DoD NDSEG Fellowship.} \\ MIT}
\fi
\date{\today}

\begin{document}

\maketitle

\begin{abstract}
  Supervised learning is often computationally easy in practice. But to what extent does this mean that other modes of learning, such as reinforcement learning (RL), ought to be computationally easy by extension? In this work we show the first cryptographic separation between RL and supervised learning, by exhibiting a class of block MDPs and associated decoding functions where reward-free exploration is provably computationally harder than the associated regression problem. We also show that there is no computationally efficient algorithm for reward-directed RL in block MDPs, even when given access to an oracle for this regression problem.
  
  It is known that being able to perform regression in block MDPs is necessary for finding a good policy; our results suggest that it is not sufficient. Our separation lower bound uses a new robustness property of the {\em Learning Parities with Noise (LPN)} hardness assumption, which is crucial in handling the dependent nature of RL data. 
We argue that separations and oracle lower bounds, such as ours, are a more meaningful way to prove hardness of learning because the constructions better reflect the practical reality that supervised learning by itself is often not the computational bottleneck.

\end{abstract}

\newpage
\setcounter{tocdepth}{2}
\tableofcontents
\newpage

\section{Introduction}
Supervised learning is often computationally easy in practice. Fueled by advances in deep learning, we can now achieve, and sometimes even surpass, human-level performance in a variety of tasks from speech recognition \cite{radford2023robust} to protein folding \cite{jumper2021highly} and beyond. Unfortunately, these success stories have largely eluded our attempts to theoretically explain them. We know that the ability of deep neural networks to generalize beyond their training data revolves around properties of natural data \cite{zhang2021understanding}. And yet it is hard to articulate precisely what structures natural data has that make supervised learning tractable, and a growing literature of computational lower bounds has shown that standard assumptions are not enough \cite{goel2020superpolynomial,daniely2021local,chen2022hardness,daniely2024computational}. 


Even worse, modern machine learning is about much more than predicting the labels of data points. For instance, in \emph{reinforcement learning}, an agent interacts with its environment over a sequence of episodes and receives rewards. It seeks to maximize its reward. The main challenge is that the agent's actions affect the environment it operates in. Thus reinforcement learning is not merely about prediction \--- e.g., of the cumulative rewards of a given policy \--- but also about how to use these predictions to guide exploration \cite{kaelbling1996reinforcement, sutton2018reinforcement}. Other examples abound, including online learning \cite{littlestone1988learning} and private learning \cite{dwork2006differential}. In each case, supervised learning is a key building block, but is far from sufficient to solve the entire problem. 

Given the gaps in our understanding of the assumptions that make supervised learning tractable, what can we say about richer modes of learning? We can still work in models where we \emph{assume} that supervised learning is easy, and then explore the consequences for other downstream learning tasks. 
This perspective is akin to approaches in cryptography and complexity theory. We may not know which of Impagliazzo's five worlds \cite{impagliazzo1995personal} we live in, but we can still conditionally explore the relationships between different average-case assumptions and their implications for cryptography. 


In the context of learning theory, this perspective underlies \emph{oracle-efficient} algorithm design, where the goal is to design a computationally efficient algorithm, except that the algorithm is allowed to make a polynomial number of queries to an oracle for solving certain optimization problems. This design paradigm is pervasive throughout modern machine learning \cite{hazan2016computational,foster2020beyond,vietri2020new} and particularly in reinforcement learning \cite{du2019provably} \--- see \cref{sec:related} for further references. It sidesteps the ``supervised learning barrier'', as we wanted, but it has its own shortcomings. The main drawback, in the context of reinforcement learning, is that the \emph{choice of oracle} is typically ad hoc. When every algorithm uses a different set of oracles, how do we compare different algorithms? What is the ``right'' oracle for a problem? See \cref{sec:oracle-intro} for further discussion. 

In this work we are motivated by these questions in the setting of reinforcement learning, as well as the loftier goal of developing a complexity theory for reinforcement learning \emph{relative to supervised learning}. Our first main contribution is to use cryptographic techniques to give the first \emph{separation} between \emph{reward-free} reinforcement learning and supervised learning. Our second main contribution is to show that even in \emph{reward-directed} reinforcement learning, the natural regression oracle is not sufficient for oracle-efficient learning.

\subsection{Background: Supervised learning and reinforcement learning}
\paragraph{Supervised learning.} First let us discuss \emph{prediction}: given a series of independent and identically distributed samples $(x^i,y^i)$, where $x^i$ is a covariate and $y^i$ is a label, the goal is to learn how to estimate $y$ for a fresh, unlabelled covariate $x$. 
When the covariate space is larger than the number of samples, sample-efficient prediction requires incorporating some inductive bias about $y|x$. The PAC model for classification \cite{valiant1984theory} formalizes inductive bias via a class $\Phi$ of binary-valued functions. Under the assumption that $\Pr[y=\phi^\st(x)|x] = 1/2+\eta$ for some $\phi^\st \in \Phi$ and constant $\eta>0$, the theory of PAC learning with noise asserts that near-optimal prediction is possible with only $O(\log|\Phi|)$ samples (see e.g. \cite{shalev2014understanding}). Moreover, this theory can be generalized to real-valued settings, i.e., regression. For reasons that will become clear when we introduce the reinforcement learning models that we care about (namely, \emph{block MDPs}; see \cref{def:block-mdp-informal}), we focus on regression problems where the hypothesis class is induced by a class of multiclass predictors. 
Concretely, fix a covariate space $\MX$, a finite \emph{latent state space} $\MS$, and any class $\Phi$ of multiclass predictors $\phi : \MX\to \MS$, which we will also refer to as \emph{decoding functions}. We define 
\emph{regression over $\Phi$} as follows:

\begin{definition}\label{def:phi-regression}
Let $\epsilon>0$. Given independent and identically distributed samples $(x^i, y^i)$ where $x^i \in \MX$, $y^i \in \{0,1\}$, and $\EE[y^i|x^i] = f(\phi^\st(x^i))$ for some $\phi^\st\in\Phi$ and $f:\MS\to[0,1]$, the goal is to produce a predictor $\MR:\MX\to[0,1]$ such that
\[\EE_x \left(\MR(x) - f(\phi^\st(x))\right)^2 \leq \epsilon.\]
\end{definition}

The key assumption is that $y^i|x^i$ is well-specified with respect to $\phi^\st$, i.e. the law of $y^i|x^i$ only depends on $\phi^\st(x^i)$. Note that this dependence is specified by an arbitrary function $f : \MS \to [0,1]$, but we think of $\MS$ as small, so the space of all such $f$ has bounded complexity. 
Thus, standard arguments imply that the statistical complexity of regression over $\Phi$ is at most $\poly(|\MS|,\log|\Phi|)$ for any constant $\epsilon>0$. Tight bounds are attainable via appropriate generalizations of VC dimension \cite{natarajan1989learning, shalev2014understanding,alon1997scale}. But the computational complexity of learning is a much thornier issue. Empirical Risk Minimization (ERM) has time complexity $O(|\Phi|)$, and for many expressive function classes of interest, substantial improvements seem unlikely \cite{blum2003noise, chen2022hardness}.

\paragraph{Episodic reinforcement learning (RL).} The field of RL formalizes the algorithmic tasks faced by an \emph{agent} that must learn how to interact with an unknown \emph{environment}. The agent learns by doing: over a series of independent episodes of interaction, the agent plays some \emph{policy}, observes how the environment responds, and repeats. The ultimate goal is to discover a ``good'' policy or set of policies. This basic framework is central to modern machine learning pipelines in applications ranging from robotics \cite{kober2013reinforcement} and healthcare \cite{liao2020personalized, yu2021reinforcement} to games \cite{silver2018general,paquette2019no}.

In \emph{reward-directed} RL, a good policy is one that approximately maximizes some reward function. In the closely related problem of \emph{reward-free} RL, a good set of policies is one that explores the entire feasible \emph{state space} of the environment. In either case, it is typical to model the environment as a \emph{Markov decision process} (MDP), which is defined by a set of states, a set of \emph{actions}, and a unknown transition function that describes the dynamics of the environment: if the environment is in a given, observed state and the agent takes a given action, it specifies the distribution over the next observed state. A policy is a description of which action the agent should take for any given observation history. A \emph{trajectory} is the sequence of observations and actions across an entire episode.

In many applications, the main challenge is that the state space of the environment is far too large to even write down \cite{kober2013reinforcement, silver2018general}. Taming the complexity of exploration in such environments necessitates making structural assumptions about the state space and dynamics. One assumption that has received intense interest in theoretical reinforcement learning is the \emph{block MDP} assumption \cite{du2019provably}, which informally asserts that the ``states'' observed by the agent are in fact stochastic emissions from a much smaller \emph{latent} MDP. This assumption is motivated by applications where the agent has access to rich observations such as images, but the underlying dynamics of the environment are simple:

\begin{definition}[Informal; see \cref{sec:block-rl}]
  \label{def:block-mdp-informal}
For sets $\MX$, $\MS$, let $\Phi$ be a set of functions $\phi: \MX\to\MS$. An MDP with state space $\MX$ and action space $\MA$ is a \emph{$\Phi$-decodable block MDP} with \emph{latent state space} $\MS$ if there is some \emph{decoding function} $\phi^\st \in \Phi$ so that the transition probability between any two states $x,x' \in \MX$ under action $a \in \MA$ only depends on $(\phi^\st(x),\phi^\st(x'),a)$, and the reward at state $x$ under action $a$ only depends on $(\phi^\st(x),a)$.
\end{definition}

Equivalently, a block MDP $M$ with latent state space $\MS$ and decoding function $\phi^\st$ can be described by a standard MDP $M_\mathsf{latent}$ on state space $\MS$, together with an \emph{emission distribution} $\til \BO(\cdot|s) \in \Delta(\MX)$ for each state $s \in \MS$, such that the support of $\til \BO(\cdot|s)$ is contained in $(\phi^\st)^{-1}(s)$ for all $s$. 

To distinguish the observed states from the latent states, we will refer to elements $x \in \MX$ as \emph{emissions} or \emph{observations} \--- or \emph{covariates}, in analogy with the supervised learning setting. Unlike in generic partially observable MDPs (POMDPs) \cite{jaakkola1994reinforcement}, every emission $x$ in a block MDP uniquely determines the underlying latent state, via $\phi^\st$. If $\phi^\st$ were known, then RL in the block MDP would reduce to RL in the latent MDP. The challenge comes from $\phi^\st$ being unknown: thus, the standard RL problem is intertwined with the problem of learning a good representation for the emissions in order to exploit the block structure.

Mirroring the story of regression over $\Phi$, the \emph{sample complexity} of RL in $\Phi$-decodable block MDPs, i.e. the number of episodes of interaction needed by the learner, is known to be polynomial in $\log|\Phi|$ and the size of the latent MDP. Note that there is no dependence on the size of the observed state space $\MX$. However, absent additional assumptions on $\Phi$, the \emph{time complexity} of the learning algorithm scales polynomially with $|\Phi|$, which is prohibitively expensive (for example, consider the set of neural networks with $n$ neurons, which has size $\exp(n)$ under any reasonable discretization). See \cref{sec:related} for details and references.


\paragraph{Is RL in block MDPs harder than regression?} It's known that reward-directed RL in $\Phi$-decodable block MDPs is provably no easier than regression over $\Phi$ \cite[Appendix F]{golowich2023exploring}. Since PAC learning \--- and, by extension, regression \--- is believed to be computationally intractable for many practically useful concept classes such as neural networks \cite{chen2022hardness}, it follows that RL in block MDPs is intractable for such classes as well.

But the above argument suggests that the source of hardness in RL is supervised learning. This is contradicted by the widespread empirical successes of machine learning heuristics for prediction tasks \cite{krizhevsky2012imagenet, taigman2014deepface, jumper2021highly}.
Such mismatches between theory and practice have motivated seminal methodologies such as smoothed analysis \cite{spielman2004smoothed}, which was an early example of \emph{beyond worst-case analysis} \cite{roughgarden2019beyond}. Broadly, this paradigm seeks modes of analysis that are better correlated with different algorithms' empirical performance. In our context, we seek a theory where computational lower bounds better reflect the source of hardness faced by empirical approaches or heuristics.
To this end, we believe it is more useful to ask whether RL is \emph{harder} than supervised learning. More concretely, we ask:

\begin{question}\label{question:first-take}
Is there a concept class $\Phi$ for which RL in $\Phi$-decodable block MDPs is computationally harder than regression over $\Phi$?
\end{question}

Out of technical necessity, we will refine this question further in \cref{sec:separation-intro}, but broadly this is the first question that our paper seeks to address.



\paragraph{Discussion: exploration versus prediction.} At one level, RL in $\Phi$-decodable block MDPs and regression over $\Phi$ appear similar. Both involve learning the decoding function $\phi^\st$ in some appropriate, implicit sense. As the concept class $\Phi$ becomes more expressive, both problems should become correspondingly more challenging. Additionally, prediction is intuitively a very valuable primitive for RL: after observing a set of independent emissions, the RL algorithm may compute a real-valued label for each emission, representing e.g. some simulated value function, and if these labels have discriminative power for the underlying latent state, then it would be useful to predict the label of a fresh emission. Indeed, in \cref{sec:special-cases}, we will see that this intuition can be made formal for several natural special cases of RL in block MDPs. 


But there also appear to be deep differences between RL and regression. The latent states are never observed in RL, so besides the observed rewards, it's not clear what labels to predict. Also, while regression is a static problem, RL is dynamic: the data collected by the learning algorithm in any episode intimately depends on the policy chosen by the algorithm. Moreover, the algorithm must necessarily update its policy over time \--- or otherwise it would likely never visit most of the state space. At its heart, RL is fundamentally about \emph{exploration} \--- a goal made explicit in the reward-free formulation \--- rather than prediction.



\subsection{Main result: a cryptographic separation}\label{sec:separation-intro}


Our first main result is that in fact exploration is strictly harder than prediction. In particular, we construct a family of block MDPs $\MM$ so that \emph{reward-free RL} in $\MM$ (\cref{def:rf-rl}) requires more computation than \emph{$\MM$-realizable regression} (\cref{def:regression-algorithm}), under a plausible cryptographic hardness assumption.




\begin{theorem}[Informal version of \cref{thm:main-separation}]\label{thm:separation-intro}
  Under \cref{assm:fine-lpn}, for any constant $C>0$, there is a block MDP family $\MM$ for which the time complexity of reward-free reinforcement learning (\cref{def:rf-rl}) is larger than that of $1/(HAS)^C$-accurate $\MM$-realizable regression (\cref{def:regression-algorithm}), by a multiplicative factor of at least $(HAS\log|\Phi|)^C$, where $H$ is the horizon, $A$ is the number of actions, $S$ is the number of states, and $\Phi$ is the decoding function class.\footnote{With an additional reasonable restriction on the sample complexities of the respective algorithms; see \cref{thm:main-separation}.}
\end{theorem}

\begin{remark}\label{rmk:m-vs-phi}
Compared to \cref{question:first-take}, the above result is framed in terms of a family of block MDPs $\MM$ rather than a concept class $\Phi$; in the proof, $\MM$ will be a family of $\Phi$-decodable block MDPs for some $\Phi$, but it will not be \emph{all} $\Phi$-decodable block MDPs. This reframing is necessary for technical reasons: in our construction, regression over $\Phi$ (\cref{def:phi-regression}) is only computationally tractable under an additional distributional assumption on the regression samples $(x^i,y^i)$, which we term $\MM$-realizability. Specifically, the distribution of the covariate $x^i$ must be expressible as the visitation distribution of some policy in some block MDP $M \in \MM$, and $\EE[y^i|x^i]$ must be well-specified with respect to the decoding function of $M$ specifically (not just an arbitrary function in $\Phi$). 

This assumption is fair since we are separating $\MM$-realizable regression from RL in $\MM$, and any natural dataset that an RL algorithm might construct during interaction with an MDP $M \in \MM$ will likely be $\MM$-realizable. Indeed, as we discuss in \cref{sec:special-cases}, there are a number of special cases of RL in block MDPs that \emph{are} efficiently reducible to $\MM$-realizable regression. Still, one might ask whether it is possible to remove this assumption and directly answer \cref{question:first-take}. We discuss the technical obstacles to doing so in \cref{sec:discussion}.
\end{remark}

\paragraph{Our cryptographic toolbox.} \cref{assm:fine-lpn} is a variant of the \emph{Learning Parities with Noise (LPN)} hardness assumption; essentially, it asserts that at high noise levels $1/2-\delta$, the time complexity of learning scales super-polynomially with $\delta^{-1}$. There is a long history of using cryptographic assumptions to prove computational lower bounds for learning \cite{valiant1984theory, kearns1994cryptographic, daniely2014average, daniely2021local}. Our result is one of comparatively fewer cryptographic \emph{separations} \--- see e.g. \cite{bun2020computational,bun2024private}. As a key technical lemma that may be of independent interest, we prove that the LPN hardness assumption is robust to weak dependencies in the noise distribution among small batches of samples:

\begin{lemma}[Informal statement of \cref{lemma:construct-corr-lpn}]\label{lemma:batch-lpn}
Let $k\in\NN$ be a constant and let $\delta \in (0, 1/2^{k+3})$. Let $p \in \Delta(\BF_2^k)$ be a \emph{$\delta$-Santha-Vazirani source} (\cref{def:sv-source}) and let $n \in \NN$. Then LPN with noise level $1/2 - 2^{k+2}\delta$ is polynomial-time reducible to \emph{batch LPN} with batch size $k$ and joint noise distribution $p$. 
\end{lemma}

As a preview, \cref{lemma:batch-lpn} is needed because the data observed by an RL agent comes in trajectories. The emissions in each trajectory are dependent due to the underlying state. See \cref{sec:techniques} for a high-level overview of the other techniques involved in proving \cref{thm:separation-intro}, and \cref{sec:overview} for a comprehensive development of the construction of $\MM$ (as well as the various technical challenges that arise).



\subsection{Oracle-efficiency in reinforcement learning}\label{sec:oracle-intro}


As a byproduct, our main result also helps clarify what sorts of oracles can be the basis for \emph{oracle-efficient algorithms} for RL. 
Formally, an oracle $\mathcal{O}$ is a solver for some optimization problem on a set of trajectories, and an RL algorithm with oracle access to $\mathcal{O}$ is termed ``oracle-efficient'' with respect to $\MO$ if it is computationally efficient aside from the oracle calls. Due to the dearth of end-to-end computationally efficient RL algorithms, much of the theoretical RL literature focuses on oracle-efficient algorithms instead. Recently there has been particular interest in designing oracle-efficient algorithms for RL in block MDPs \cite{dann2018oracle,du2019provably, misra2020kinematic,modi2021model, zhang2022efficient, mhammedi2023representation}. 

Unfortunately, almost every new algorithm for RL in block MDPs uses a new set of oracles. The particular choice of oracle could have a tremendous impact on the empirical performance of a particular algorithm. For instance, many of the aforementioned algorithms utilise min-max optimization oracles; however, min-max optimization faces both theoretical \cite{daskalakis2021complexity} and practical \cite{razaviyayn2020nonconvex} difficulties not encountered by pure minimization. Even among minimization problems, convergence of gradient descent to a good solution is highly dependent on subtle features of the optimization landscape \cite{wen2024sharpness}. Yet little basis for comparison between oracles has been proposed. 

\paragraph{An oracle lower bound.} 

Motivated by the above considerations, we ask: {\em Is there an oracle that is both necessary and sufficient for RL in block MDPs?} We call such an oracle {\em minimal}. Prior to our work, the natural candidate was $\MM$-realizable regression, which was at least known to be necessary for RL in block MDPs \cite{golowich2023exploring}. \cref{thm:separation-intro} gives a \emph{separation} between regression and reward-free RL, which implies that regression is an insufficient oracle for reward-free RL. But there is also a direct argument for proving oracle lower bounds. This argument applies to reward-directed RL, does not require the regression labels to be realizable (see \cref{def:regression-oracle}), and yields quantitatively stronger bounds:


\begin{theorem}[Informal statement of \cref{thm:reduction-prp}]\label{thm:no-reduction-intro}
Suppose that a pseudorandom permutation family with sub-exponential hardness exists (\cref{asm:prf}). Then there is a constant $c>0$, a function class $\Phi$, and a family of $\Phi$-decodable block MDPs $\MM$ with succinctly describable optimal policies so that any reinforcement learning algorithm for $\MM$, with time complexity $T$ and access to an $\epsilon$-accurate regression oracle, must satisfy either $\epsilon \leq 2^{-(HAS\log|\Phi|)^c}$ or $T \geq 2^{(HAS\log|\Phi|)^c}$, where $H$ is the horizon of the MDP, $A$ is the number of actions, and $S$ is the number of latent states.
\end{theorem}

The main idea behind \cref{thm:no-reduction-intro} is to design a family of block MDPs where RL and regression are \emph{both} intractable, but no oracle-efficient algorithm can even pose a non-trivial query to the regression oracle. Thus the oracle can be implemented efficiently for all intents and purposes, so an oracle-efficient algorithm would contradict the hardness of RL. 
The challenge is in ensuring the second property: that no efficient algorithm can pose a non-trivial query. This is not immediate; as we discuss in \cref{sec:overview,sec:special-cases}, there are some natural families of block MDPs where RL is intractable without an oracle but it is possible to construct non-trivial queries (and in fact there is an efficient reduction to regression). Our lower bound relies on a simple structural property of the latent MDPs that we call \emph{open-loop indistinguishability} (\cref{def:open-loop-indist}). The key insight is the following. If the block MDP has two different actions (say, at the first step) that induce different latent state visitation distributions, then the algorithm may construct a non-trivial regression query by contrasting these two actions. However, this is essentially \emph{all it can do}, if the decoding function class is sufficiently rich. Thus, to show that no efficient algorithm can make non-trivial queries to the regression oracle, it essentially suffices to choose a latent MDP where any two action sequences (i.e. \emph{open-loop} policies) induce the same latent state visitation distributions at each step, and to choose emission distributions that are intractable to decode.



\paragraph{So is there a minimal oracle?} 
The moral of \cref{thm:no-reduction-intro} is that regression over $\Phi$ is an insufficient oracle because the algorithm cannot efficiently construct non-trivial labels for its regression data. The algorithms proposed by \cite{modi2021model, zhang2022efficient,mhammedi2023efficient} avoid this issue by refining the oracle to also maximize over label functions. Similarly, the algorithm proposed by \cite{mhammedi2023representation} refines the oracle to condition on two latent states from different timesteps (rather than just one). This gives the algorithm greater flexibility in generating its own labels. Our results show that some such refinement is necessary. The oracle used in \cite{mhammedi2023representation} is conceptually particularly similar to regression over $\Phi$, lending credence to the possibility that it is the ``right'' oracle for reinforcement learning in block MDPs. See \cref{tab:oracle-overview} for an overview of the different oracles that have been studied. 

\subsection{Discussion: what makes RL tractable?}\label{sec:special-cases}


The main results of this paper give concrete evidence \--- modulo technical restrictions and hardness assumptions \--- that RL in $\Phi$-decodable block MDPs is likely computationally harder than regression over $\Phi$. As discussed in \cref{sec:oracle-intro}, this motivates comparison with more intricate variants of regression, to understand if there is a minimal oracle for RL in block MDPs. But it also suggests more instance-dependent questions: assuming access to only the basic regression oracle, what is the most general family of block MDPs in which RL is tractable? From the complexity-theoretic perspective, which structural assumptions cause RL to be harder than regression, and which don't?


Some partial answers are already known. On the side of lower bounds, the proof of \cref{thm:no-reduction-intro} illustrated that a key source of hardness in block MDPs is open-loop indistinguishability (\cref{def:open-loop-indist}) of the latent MDP. On the side of algorithms, below are several special cases of RL that do reduce to realizable regression:

\begin{itemize}
    \item Offline reinforcement learning in $\Phi$-decodable block MDPs, under all-policy concentrability: in this setting, exploration is a non-issue, because the given dataset is assumed to be already exploratory (this setting is also roughly equivalent to online RL assuming that an exploratory policy is known). Thus, using an algorithm such as Fitted $Q$-Iteration \cite{ernst2005tree, chen2019information}, RL in this setting reduces to regression (\cref{sec:offline}).
    \item Reinforcement learning in $\Phi$-decodable block contextual bandits (i.e. block MDPs with horizon one): in this setting, the uniformly random policy is exploratory, so again the only computational challenge is prediction. In fact, this special case of RL is computationally \emph{equivalent} to regression (\cref{sec:bandits}).
    \item Reinforcement learning in $\Phi$-decodable block MDPs with deterministic dynamics: such block MDPs are, in a sense, the opposite of those satisfying open-loop indistinguishability, because they can be explored using only open-loop policies (i.e. fixed action sequences). Moreover, although the set of such policies is still exponentially large, it can be pruned to a succinct policy cover via contrastive learning with the regression oracle (\cref{sec:deterministic}).
\end{itemize}

All of these results are either folklore or essentially known from prior work (though not explicitly stated in our notation); for completeness, we have included both proofs as well as the relevant references in the appendices. But in combination with \cref{thm:separation-intro,thm:no-reduction-intro}, these results have interesting implications. Offline RL with all-policy concentrability, RL with horizon one, and RL with deterministic dynamics are all widely-studied special cases/easier models that could be considered ``stepping stones'' towards full-blown RL. Our results imply some of the first concrete separations between these cases and the general problem of online RL.

We also remark that all of these nuances would have been lost by an oracle-free computational analysis: e.g. in that model, both block MDPs with deterministic dynamics and block MDPs with open-loop indistinguishability would appear equally intractable, due to the lower bound from \cite{golowich2023exploring}. At the other extreme, with a strong enough oracle all of these problems would appear equally tractable \cite{jiang2017contextual}. Regression over $\Phi$ is the simplest oracle that fully encompasses the known difficulties of supervised learning, and hence may be the right baseline for understanding what assumptions make RL tractable. Of course, there is still a considerable gap between our lower bound constructions and the above special cases. We see narrowing this gap (e.g. via more general algorithms) as a compelling open problem for future research.

\subsection*{Outline of the paper}\label{sec:outline}

In \cref{sec:techniques} we give a high-level overview of the main techniques involved in the proof of \cref{thm:separation-intro}. In \cref{sec:related} we survey related work on RL for block MDPs, as well as the current landscape of computationally efficient RL and computational lower bounds for RL. In \cref{sec:prelim} we formally define block MDPs, the episodic RL model, and the oracles, computational problems, and hardness assumptions studied in this paper. In \cref{sec:overview} we give a detailed technical overview for both \cref{thm:separation-intro,thm:no-reduction-intro}, and discuss some directions for future technical improvement.

In \cref{sec:construction} we formally state and prove \cref{thm:separation-intro}, drawing on the results of \cref{sec:regression-to-lpn,sec:lpn-to-rl}. In \cref{sec:oracle-lb} we formally state and prove \cref{thm:no-reduction-intro}. 

\section{Technical overview}\label{sec:techniques}

In this section we give an overview of the proof of \cref{thm:separation-intro}. See \cref{sec:overview} for a more comprehensive treatment, as well as the proof overview for \cref{thm:no-reduction-intro}.

\cref{assm:fine-lpn} is a variant of the classical \emph{Learning Parities with Noise (LPN)} hardness assumption 
\cite{blum2003noise,pietrzak2012cryptography}. Essentially, it asserts that when the noise level is $1/2-\delta$ for very small $\delta$, the time complexity of learning must scale super-polynomially in $\delta^{-1}$.\footnote{In comparison, the standard LPN hardness assumption typically operates in a regime with less noise (i.e. $\delta$ is bounded away from $0$, or is even near $1/2$) but asserts a quantitatively weaker computational lower bound.} In particular, this means that there is some function $\delta = \delta(n)$ so that learning $n$-variable parities with noise level $1/2-O(\delta^2)$ is strictly harder than learning $n$-variable parities with noise level $1/2-O(\delta)$. We prove \cref{thm:separation-intro} by constructing a block MDP family where on the one hand, reward-free RL is as hard as learning parities with noise level $1/2-O(\delta^2)$, and on the other hand, realizable regression is as easy as learning parities with noise level $1/2 - O(\delta)$. Ensuring that both of these properties hold requires careful design of both the latent MDP dynamics as well as the emission distributions.

\subsection{Proof techniques I: a warm-up separation}\label{sec:intro-warmup-overview}

We start by sketching a separation between realizable regression and \emph{strong} reward-free RL, which is substantially simpler than \cref{thm:separation-intro} but illustrates many of the key ideas. In the standard formulation of reward-free RL, the goal is to find an \emph{$(\alpha,\gamma)$-policy cover} (\cref{def:pc}), which visits every state with approximately maximal probability, allowing for both multiplicative approximation error $\alpha$ and additive approximation error $\gamma$. In the strong formulation, the goal is to find a $(1,\gamma)$-policy cover. Notably, in an MDP where some state $s$ is reachable by some policy with probability $1$, a $(1,\gamma)$-policy cover must contain a policy that reaches $s$ with probability at least $1-\gamma$. Throughout this overview, we think of both $\gamma \in (0,1/2)$ and the regression accuracy parameter $\epsilon>0$ as constants.

\paragraph{A horizon-two block MDP.} We design a simple family of block MDPs $\Mtoy = \{M^\sk: \sk\in\BF_2^n\}$ indexed by vectors $\sk \in \BF_2^n$. Each MDP $M^\sk$ has the same latent structure, with latent state space $\BF_2$, action space $\BF_2$, and episodes of length two. The initial latent distribution is uniform over $\BF_2$, and the latent transition from the first step to the second is defined by addition of the latent state and action over $\BF_2$ (see \cref{fig:horizon-two}). Finally, for the MDP $M^\sk$, the emission from a state $b \in \BF_2$ is
\[(u, \langle u, \sk\rangle + e + b, \Enc_\sk(b))\]
where $u \sim \Unif(\BF_2^n)$ and $e \sim \Ber(1/2-\delta)$, and $\Enc$ is the (randomized) encryption function for some private-key encryption scheme $(\Enc,\Dec)$. The decoding function is $(u,y,c)\mapsto \Dec_\sk(c)$.

\begin{figure}[t]
\centering
\includegraphics[width=0.5\textwidth]{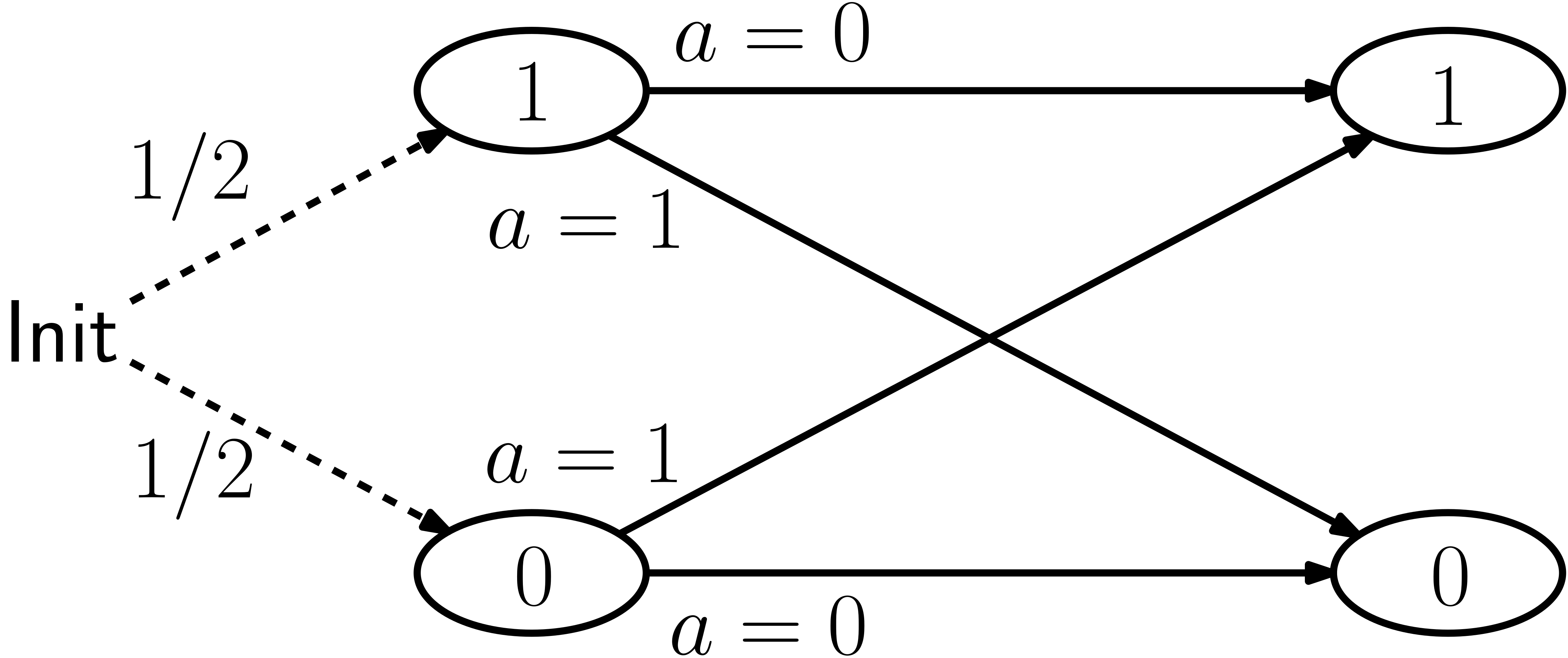}
\caption{The horizon-two latent MDP}
\label{fig:horizon-two}
\end{figure}

\paragraph{Intuition.} Since there exists a policy that reaches state $0$ deterministically, any strong reward-free RL algorithm must construct a policy where the action strongly correlates with the latent state. Similarly, any realizable regression algorithm must learn to predict the latent state (if the labels are not too noisy). Thus, both problems have essentially the same goal; the difference is in the structure of the available data. Regression is a \emph{supervised} problem, where the learner has access to a label for each emission. In contrast, reward-free RL is somehow \emph{self-supervised}: in each episode of interaction, the learner observes two correlated emissions, and must learn by contrasting these emissions. Since each emission is highly noisy, contrasting compounds the noise. Intuitively, this compounding is the reason why RL corresponds to LPN with noise level $1/2-O(\delta^2)$ rather than $1/2-O(\delta)$.

To formalize this intuition, we define two toy problems that are variants of LPN \--- one ``supervised'' and one ``self-supervised''. For both problems, fix some unknown vector $\sk \in \BF_2^n$:
\nc{\problemAname}{Supervised LPN}
\begin{namedproblem}[\hypertarget{problem:supervised}{\problemAname}]
Consider the random variables $x = (u, \langle u, \sk\rangle + e + b)$ and $y = b + \xi$ where $u \sim \Unif(\BF_2^n)$, $e \sim \Ber(1/2 - \delta)$, $b \sim \Ber(1/2)$, and $\xi \sim \Ber(1/2-\epsilon)$ are independent. Given access to independent samples distributed as $(x,y)$, we would like to recover $\sk$.
\end{namedproblem}
\nc{\problemBname}{Self-supervised LPN}
\begin{namedproblem}[\hypertarget{problem:self-supervised}{\problemBname}]
Consider the random variables $x_1 = (u_1, \langle u_1, \sk\rangle + e_1 + b)$ and $x_2 = (u_2, \langle u_2, \sk\rangle + e_2 + b)$ where $u_1,u_2 \sim \Unif(\BF_2^n)$, $e_1, e_2 \sim \Ber(1/2-\delta)$, and $b \sim \Ber(1/2)$ are independent. Given access to independent samples distributed as $(x_1,x_2)$, we would like to recover $\sk$.
\end{namedproblem}

\begin{figure}[t]
    \centering
    \includegraphics[width=\textwidth]{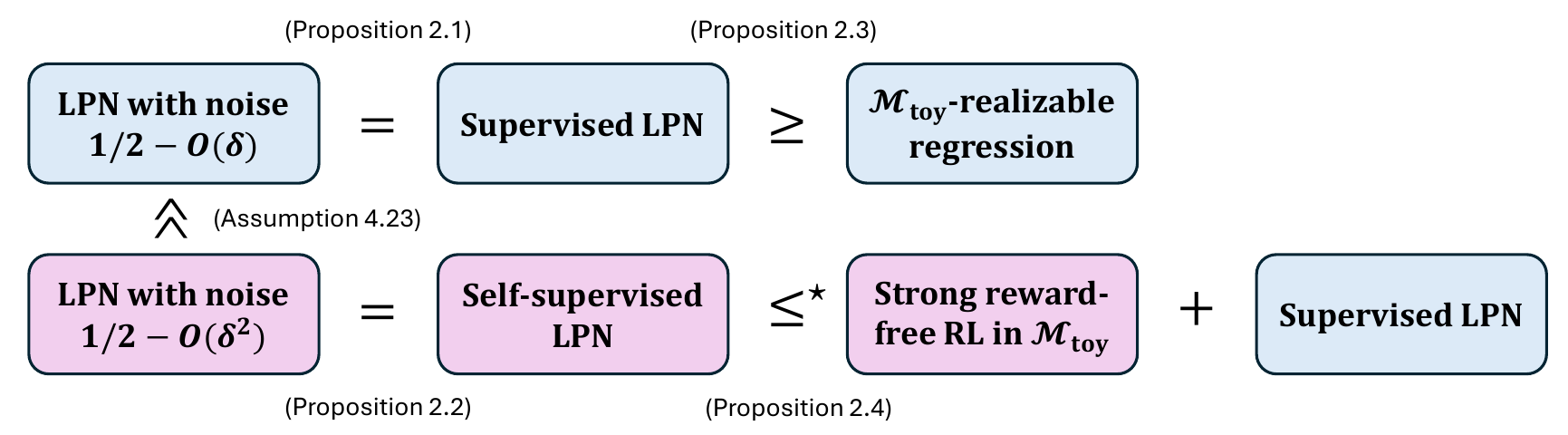}
    \caption{Diagram of the separation between realizable regression and strong reward-free RL under \cref{assm:fine-lpn}. Here, the inequalities refer to polynomial-time reducibility. The starred inequality is only true for an idealized encryption scheme $(\Enc,\Dec)$, but can be made rigorous with an explicit LPN-based encryption scheme and a modification of \bref{problem:self-supervised}{\problemBname}. See \cref{sec:warmup-overview}.}
    \label{fig:toy-flowchart}
\end{figure}

We relate these problems to standard LPN with noise level $1/2-O(\delta)$ and $1/2-O(\delta^2)$, and then explain how they relate back to $\Mtoy$-realizable regression and strong reward-free RL in $\Mtoy$ respectively (the full chain of reductions is diagrammed in \cref{fig:toy-flowchart}). The first two reductions are straightforward:
\begin{proposition} \bref{problem:supervised}{\problemAname} is efficiently reducible to learning parities with noise level $1/2-2\delta\epsilon$:\footnote{In fact, the problems are equivalent, but we only need one direction.} 
\end{proposition}
\begin{proof}[Proof sketch]
For any sample $(x, y)$ from \bref{problem:supervised}{\problemAname}, adding the label $y$ to the second part of $x$ yields the tuple $(u, \langle u,\sk\rangle + e + \xi)$. Since $e$ and $\xi$ are independent, the distribution of $e+\xi$ is precisely $\Ber(1/2-2\delta\epsilon)$.
\end{proof}

\begin{proposition}
Learning parities with noise level $1/2-2\delta^2$ is efficiently reducible to \bref{problem:self-supervised}{\problemBname}.
\end{proposition}
\begin{proof}[Proof sketch]
Intuitively, this is because for any sample $(x_1,x_2)$ from \bref{problem:self-supervised}{\problemBname}, the marginal distributions of $x_1$ and $x_2$ possess no information about $\sk$, and it appears that the only thing a learning algorithm can do is add $x_1$ and $x_2$ element-wise, which yields an LPN sample \[(u_1+u_2, \langle u_1+u_2,\sk\rangle + e_1 + e_2)\] with noise level $1/2-2\delta^2$. This is of course only intuition, but it can be formalized into a simple average-case reduction: given an LPN sample $(u,b)$ with noise level $1/2-2\delta^2$, let $(u',b')$ be an independent random variable with $u' \sim \Unif(\BF_2^n)$ and $b' \sim \Unif(\BF_2)$. Then the joint distribution of $x_1 := (u',b')$ and $x_2 := (u+u',b+b')$ is exactly that of \bref{problem:self-supervised}{\problemBname}.
\end{proof} 

It remains to relate \bref{problem:supervised}{\problemAname} and \bref{problem:self-supervised}{\problemBname} to $\Mtoy$-realizable regression and strong reward-free RL in $\Mtoy$ respectively.
\begin{proposition}
$\Mtoy$-realizable regression with error tolerance $O(\epsilon)$ can be reduced to \bref{problem:supervised}{\problemAname}.
\end{proposition}
\begin{proof}[Proof sketch]
This reduction leverages the fact that there are only two latent states. Either some constant function is a near-optimal regressor, or the regression labels are at least $\epsilon$-correlated with the latent state, and hence each label $y$ can be written as $b+\xi$ where $b$ is the latent state and $\xi \sim \Ber(1/2-\epsilon')$ for some $\epsilon' \geq \epsilon$. This is precisely the setting of \bref{problem:supervised}{\problemAname}. Moreover, after computing $\sk$, it's easy to construct a near-optimal regressor.
\end{proof}

\begin{proposition}
\bref{problem:self-supervised}{\problemBname} can be reduced to strong reward-free RL in $\Mtoy$, if the reduction is given access to an oracle for \bref{problem:supervised}{\problemAname}.
\end{proposition}
\begin{proof}[Proof sketch]
For the purposes of this overview, we'll assume that the encryptions are simply random noise, i.e. the emission from latent state $b$ is $(u,\langle u,\sk\rangle+e+b, W)$ for some uniform noise vector $W$.\footnote{Obviously this no longer corresponds to a valid block MDP emission distribution, since the latent state $b$ is not perfectly decodable from the emission. Intuitively, for an ``ideal'' encryption scheme, this simplification might seem without loss of generality, since $\Enc_\sk(b)$ may be \emph{computationally indistinguishable} from random noise $W$. The issue with formalizing this intuition (under any standard cryptographic assumption on $(\Enc,\Dec)$) is that the RL algorithm has access to not just $\Enc_\sk(b)$ but also side-information in each emission that depends on $\sk$.

Of course, the side-information in our setting is computationally hard to invert. Designing encryption schemes that are robust to such side-information is an active area of research known as cryptography with auxiliary input; see e.g. \cite{dodis2009cryptography}. Unfortunately, these general-purposes results do not directly apply to our setting, and in any case would likely require strong additional cryptographic assumptions beyond \cref{assm:fine-lpn}. Instead, we leverage the fact that the side-information essentially consists of LPN samples to show that an \emph{LPN-based} encryption scheme is secure even in the presence of this side-information, all using only \cref{assm:fine-lpn}. We defer the overview of this part of the reduction to \cref{sec:warmup-overview}.
} With this caveat, given a sample $(x_1, x_2)$ from \bref{problem:self-supervised}{\problemBname}, it is possible to efficient simulate an episode of interaction with $M^\sk$. Indeed, suppose $x_1 = (u_1, y_1)$ and $x_2 = (u_2, y_2)$. Then the first simulated emission is $(u_1, y_1, W_1)$ for random noise $W_1$, and the second simulated emission is $(u_2, y_2+a, W_2)$ for random noise $W_2$, where $a$ is the action produced by the RL agent after seeing the first emission. If $b \in \BF_2$ was the latent random variable in $x_1$ and $x_2$, then the first emission is exactly a random emission from state $b$, and the second emission is exactly a random emission from state $b+a$, as desired. Note that this argument crucially uses the \emph{additive} structure of the latent MDP and the emissions over $\BF_2$.

It follows that any strong reward-free RL algorithm can be simulated on $M^\sk$ using samples from \bref{problem:self-supervised}{\problemBname}. Say that the RL algorithm produces a $(1,1/3)$-policy cover. Then there is some policy $\hat\pi$ in the cover that visits state $0$ with probability at least $2/3$. Hence, the output of $\hat\pi$ on an emission $x$ is non-trivially correlated with the latent state of $x$, so it can be used to \emph{label} emissions and thereby generate samples from \bref{problem:supervised}{\problemAname}. Invoking the oracle for \bref{problem:supervised}{\problemAname} then enables recovering $\sk$.
\end{proof}

Combining the above propositions, we see that if strong reward-free RL in $\Mtoy$ were as easy as $\Mtoy$-realizable regression, then LPN with noise $1/2-O(\delta^2)$ would be roughly as easy as LPN with noise $1/2-O(\delta)$, which would imply that \cref{assm:fine-lpn} is false. This completes our high-level overview of a separation between strong reward-free RL and realizable regression. We give a more detailed overview in \cref{sec:warmup-overview}.

\subsection{Proof techniques II: the full separation}\label{sec:intro-full-overview}

The block MDP family $\Mtoy$, constructed above, exhibits a separation between regression and finding a $(1,\gamma)$-policy cover, but it cannot separate regression from the more standard problem of finding an $(\alpha,\gamma)$-policy cover, where the multiplicative approximation factor $\alpha$ is typically allowed to be polynomially small in the size of the latent MDP. The reason is that each MDP in $\Mtoy$ has constant horizon, and thus the policy that plays uniformly random actions already has good coverage. Thus, proving \cref{thm:separation-intro} requires replacing $\Mtoy$ by a family of block MDPs with super-constant horizon \--- and where the random policy has bad coverage. The overall proof structure remains similar to \cref{fig:toy-flowchart}, but each step becomes significantly more technically involved. Below, we enumerate some of the high-level obstacles and discuss how we circumvent them. We give a more detailed overview of the construction and proof in \cref{sec:full-argument-overview}.


\paragraph{Static-to-dynamic reduction.} Cryptographic assumptions have long been employed to prove hardness of learning problems \cite{valiant1984theory, kharitonov1993cryptographic, klivans2009cryptographic}. Indeed, cryptographic assumptions and classical learning and testing problems are two sides of the same coin. Reinforcement learning is fundamentally different, in that it is a \emph{dynamic} problem. In each episode of interaction, the distribution of the sample trajectory is jointly determined by the environment and the arbitrary learning algorithm that we are trying to rule out. Thus, to reduce a standard, \emph{static} learning problem like LPN to RL in some family of block MDPs, we need to be able to simulate any distribution over trajectories that might arise from any policy, given only samples from a single distribution. Moreover, when generating a trajectory we have to implicitly manipulate the latent state without ever explicitly ``knowing'' it, since if we had even a noisy estimate of the latent state, we could reduce to supervised learning. 

One might ask: why start with a static hardness assumption? There is a wide array of cryptographic primitives with \emph{dynamic} security guarantees that may, at first glance, seem useful to design an emission distribution around. For instance, with a pseudorandom function family (PRF) \cite{goldreich1986construct}, one could simulate a trajectory by directly encrypting a sequence of latent states. Alternatively, fully homomorphic encryption (FHE) \cite{gentry2009fully,brakerski2014efficient} enables implicit, arbitrary manipulations of an encrypted state. As we discuss further in \cref{sec:overview-dynamicity}, all of these approaches run into a fundamental obstacle stemming from the fact that we are try to prove a computational \emph{separation}, not just computational hardness. Indeed, with a PRF, it is straightforward to construct a block MDP family for which RL is hard, and in fact we will use such a construction to give a direct proof that there is no reduction from RL to regression (\cref{thm:no-reduction-intro}). However, regression is then equally intractable, so there is no computational separation. In broad strokes, general-purpose primitives either lack the flexibility needed to simulate RL, or are so secure that regression is completely intractable. 

The LPN problem occupies a sweet spot where security satisfies non-trivial robustness guarantees (see e.g. \cref{lemma:batch-lpn}) but there are also non-trivial algorithms. In the sketch above, we reduced the static \bref{problem:self-supervised}{\problemBname} to strong reward-free RL in $\Mtoy$ by leveraging both the additive latent dynamics and the additively homomorphic nature of LPN samples. In the full proof, the latent structure will necessarily be more complex, and the idea of adding the action to the latent state (over a finite field) will no longer be sufficient.

\paragraph{Batch LPN.} Notice that \bref{problem:self-supervised}{\problemBname} can be interpreted as a variant of LPN with batches of samples, where the samples in each batch have correlated noise terms. Essentially, this was a consequence of the static-to-dynamic reduction, and the fact that the emissions in a single episode of RL are correlated via the latent state. A key piece of the toy separation sketched above was a tight reduction to this \emph{batch LPN} problem from standard LPN.\footnote{In particular, while it's trivial to show that \bref{problem:self-supervised}{\problemBname} is as hard as LPN with noise level $1/2-\delta$, this would fail to establish the claimed separation between strong reward-free RL and regression.} Fortunately, there was a simple equivalence between \bref{problem:self-supervised}{\problemBname} and LPN with noise level $1/2-2\delta^2$.

In the full proof, the analogue of \bref{problem:self-supervised}{\problemBname} has larger batches and a more complex correlation structure, so there is no longer an evident equivalence. Moreover, prior work gives some reason to be skeptical of hardness: for seemingly innocuous variants such as batch LPN with one-out-of-three noise,\footnote{Formally, each batch has size three, and the noise vector in each batch is uniformly random subject to having Hamming weight one.} the Arora-Ge linearization attack recovers the parity function in polynomial time \cite{arora2011new}. As a key step in the proof of \cref{thm:separation-intro}, we identify a general condition on the joint noise distribution under which such attacks can be avoided, and in fact batch LPN is provably hard under standard LPN:

\begin{definition}[c.f. \cite{santha1986generating}]\label{def:sv-source}
Let $k\in\NN$, $\delta \in (0,1/2)$, and $p \in \Delta(\BF_2^k)$. We say $p$ is a \emph{$\delta$-Santha-Vazirani source} if for all $i \in [k]$ and $x \in \BF_2^k$ it holds that
\[\Prr_{X\sim p}[X_i=1|X_{<i}=x_{<i}] \in [1/2-\delta, 1/2+\delta].\]
\end{definition}

For example, the joint noise distribution in \bref{problem:self-supervised}{\problemBname} is an $O(\delta^2)$-Santha-Vazirani source, as is its analogue in the full proof. On the other hand, the one-out-of-three noise distribution is \emph{not} a $\gamma$-Santha-Vazirani source for any $\gamma < 1/2$, since fixing the first two noise terms determines the third.


Our reduction from batch LPN with Santha-Vazirani noise to standard LPN is stated informally as \cref{lemma:batch-lpn}, and formally as \cref{lemma:construct-corr-lpn}. For context, the LPN hardness assumption is generally regarded as robust to non-uniformity in the covariates and the secret \cite{pietrzak2012cryptography, dodis2009cryptography}, but little was previously known about its robustness to dependent noise, besides the negative result of \cite{arora2011new} and a positive result for some specific structured noise distributions \cite{bartusek2019new}. \cref{lemma:batch-lpn} sheds further light on this question, and may be thought of as a partial converse to \cite{arora2011new}. See \cref{sec:dependent-lpn} for the proof. 

\section{Related work}\label{sec:related}


\paragraph{RL for block MDPs.} Since any block MDP with $S$ latent states has Bellman rank at most $S$, the seminal algorithm \OLIVE{} for learning in contextual decision processes \cite{jiang2017contextual} is statistically efficient. In particular, for any (finite) decoding function class $\Phi$, \OLIVE{} learns an $\epsilon$-suboptimal policy in a $\Phi$-decodable block MDP with sample complexity $\poly(H,|\MA|,|\MS|,\log|\Phi|,\epsilon^{-1})$, where $H$ is the horizon, $\MA$ is the set of actions, and $\MS$ is the set of latent states (see \cref{sec:block-rl} for formal definitions of these parameters). However, \OLIVE{} is generally considered computationally impractical \cite{du2019provably}, since it relies on \emph{global optimism}, i.e. explicitly maintaining the set of all value functions consistent with data collected thus far. It has been shown that the steps comprising \OLIVE{} cannot be implemented in polynomial time for even tabular MDPs \cite{dann2018oracle}, implying that it cannot be made oracle-efficient with respect to any oracles that are efficiently implementable for tabular MDPs.

Hence, subsequent works have sought to match the statistical performance of \OLIVE{} on block MDPs, with more practical algorithms \--- i.e., algorithms that are oracle-efficient with respect to oracles that are commonly implemented by machine learning heuristics. The first attempts in this direction \cite{dann2018oracle, du2019provably} required additional assumptions on the dynamics of the latent MDP. In particular, \cite{dann2018oracle} studied block MDPs with deterministic dynamics (\cref{def:det-dynamics}), and \cite{du2019provably} made a reachability assumption as well as a ``backwards separability'' assumption which generalizes determinism but excludes many natural scenarios. In the former work, the algorithms require a cost-sensitive classification oracle, among others. In the latter work, the most natural instantiation of the algorithm requires a proper regression oracle over a class $\MG$ that consists of decoding functions $\phi \in \Phi$ composed with maps from latent states to real vectors. This oracle is similar though slightly more complex than the regression problem we study. 

A more recent line of work has developed oracle-efficient\footnote{With plausible oracles, as discussed above, in contrast to \OLIVE{}.} algorithms for block MDPs with only the (largely technical) assumption of reachability \cite{misra2020kinematic, modi2021model} or even with no additional assumptions \cite{zhang2022efficient, mhammedi2023representation,mhammedi2023efficient}. Among these, \cite{modi2021model, zhang2022efficient,mhammedi2023efficient} require a regression oracle, as well as a min-max oracle that finds a discriminator label function inducing the maximum regression error with respect to the current estimated decoding function. In contrast, \cite{misra2020kinematic} uses a contextual bandits / cost-sensitive classification oracle over the policy space, as well as a regression oracle over \emph{pairs} of emissions. Finally, \cite{mhammedi2023representation} uses only a maximum likelihood oracle over pairs of emissions. The regression oracle over pairs of emissions would also work with their algorithm \cite[Footnote 5]{mhammedi2023representation}. See \cref{tab:oracle-overview} for an informal comparison between the oracles that suffice for RL in block MDPs (with no further assumptions) versus the regression oracle considered in our work.

\begin{table}[t]

\centering

\begin{tabular}{|c|c|c|}\hline
\thead{Oracle} & \thead{Necessary for \\ oracle-efficiency?} & \thead{Sufficient for \\ oracle-efficiency?} \\\hline

\makecell{No oracle} & 
\makecell{Yes (trivial)} & \makecell{Likely not \\ \cite{golowich2023exploring}} \\\hline

\makecell{
\small{$\argmin\limits_{\substack{\hat f: \MS \to [0,1] \\ \hat \phi \in \Phi}} \sum\limits_{(x,y) \in \MD} (\hat f(\hat \phi(x)) - y)^2$}
}& 
Yes \cite{golowich2023exploring}
 &
\makecell{Likely not \\ \textbf{(this paper)}} \\ \hline

\makecell{
\small{$\argmin\limits_{\hat \phi\in\Phi} \max\limits_{\substack{f: \MS \to [0,1] \\ \phi\in\Phi}} \min\limits_{\hat f: \MS\times\MA \to [0,1]}$} \\ \small{$ \sum\limits_{(x,a)\in\MD} (\hat f(\hat\phi(x),a) - \EE_{x'|x,a} f(\phi(x')))^2$}
}
&
\makecell{?}
& 
\makecell{Yes \\  \cite{modi2021model,zhang2022efficient}\\\cite{mhammedi2023efficient}}
\\\hline

\makecell{
\small{$\argmax\limits_{\substack{\mu: \MS^2 \to \Delta(\MA \times\MS)\\ \hat \phi\in\Phi}} \sum\limits_{(j,a,x,x') \in \MD} \log \mu((a,j)|\hat \phi(x),\hat \phi(x'))$}
}& 
\makecell{?} & 
\makecell{Yes \cite{mhammedi2023representation}} 
\\\hline

    \end{tabular}
    \caption{Overview of the oracles that have been studied for RL in $\Phi$-decodable block MDPs. The definitions of the oracles are given informally in the table; see the respective papers for formal definitions. We remark that \cite{modi2021model,zhang2022efficient,mhammedi2023efficient} also show sufficiency of a slightly weaker oracle than the one shown in the third row, via their \texttt{RepLearn} algorithm. Also, the oracle in the fourth row can be replaced by an analogous squared-loss minimization oracle \cite[Footnote 5]{mhammedi2023representation}.
    }
    \label{tab:oracle-overview}
\end{table}

See also \cite{feng2020provably}, which solves RL in block MDPs using only an unsupervised clustering oracle for the emission distributions, but hence requires the additional assumption that the emissions are clusterable; and \cite{foster2021instance}, which solves RL in block MDPs using the regression oracle, under the additional assumption that the optimal $Q$-function exhibits a gap.





\paragraph{Computationally efficient RL.} Most of the literature in theoretical reinforcement learning is focused on developing statistically efficient or oracle-efficient algorithms under as broad structural assumptions as possible. A complementary paradigm is to develop end-to-end computationally efficient algorithms under more restrictive \--- but hopefully still plausible \--- assumptions. Our work draws motivation from both paradigms: we are fine with using oracles, but we are interested in using the least computationally burdensome oracles.

Unfortunately, beyond tabular MDPs (i.e. those with small state space) \cite{kearns2002near, brafman2002r} and linear MDPs \cite{jin2020provably}, few positive results are known for computationally efficient RL. It is known to be possible for $\Phi$-decodable block MDPs when $\Phi$ is small (i.e. the time complexity scales polynomially in $|\Phi|$) \cite{modi2021model, zhang2022efficient,golowich2023exploring},\footnote{It's essentially immediate to get time and sample complexity both polynomial in $|\Phi|$, but these works (implicitly) show how to achieve time complexity $\poly(|\Phi|)$ while the sample complexity still only scales with $\poly(\log|\Phi|)$ \--- unlike in the PAC learning setting, this is non-trivial. See \cite[Section 8.6]{modi2021model}, which can also be used to implement the algorithm of \cite{zhang2022efficient}. Alternatively, the same guarantee, albeit with a larger polynomial, follows as an immediate consequence of the main result of \cite{golowich2023exploring}.} which essentially corresponds to brute-force computation of the empirical risk minimizer in PAC learning. This can be improved when $\Phi$ is the class of decoding functions induced by low-depth decision trees \cite{golowich2023exploring}.

For a broader discussion on computationally efficient reinforcement learning, see e.g. \cite{golowich2023exploring} and references.



\paragraph{Computational lower bounds in RL.} Most known computational lower bounds in reinforcement learning are simply those inherited from corresponding statistical lower bounds. Exceptions (i.e. lower bounds exceeding the achievable sample complexity) include hardness for learning POMDPs with polynomially small observability \cite{jin2020sample, golowich2023planning} and hardness for learning MDPs with linear $Q^\st$ and $V^\st$ \cite{kane2022computational, liu2023exponential}, all of which hold under some version of the Exponential Time Hypothesis.

Notably, this is a worst-case hardness assumption. In classical PAC learning theory, there is a clear distinction between \emph{proper}/\emph{semi-proper} learning, where intractability can often be based on worst-case assumptions such as $\NP \neq \RP$ \cite{pitt1988computational}, and \emph{improper} learning, where intractability via $\NP$-hardness is generally believed to be unlikely \cite{applebaum2008basing}, and all known lower bounds are based on average-case and/or cryptographic assumptions \--- see e.g. \cite{valiant1984theory,klivans2009cryptographic,daniely2021local}. For reinforcement learning, no analogous distinction is evident. However, to date, there is no known computational hardness for reinforcement learning in block MDPs based on a worst-case hardness assumption.\footnote{Indeed, it seems plausible that there is a complexity-theoretic obstruction to any natural reduction from an $\NP$-hard problem, as is known for improper learning \cite{applebaum2008basing}, but formalizing why there may be such an obstruction for block MDPs and not e.g. general POMDPs is unclear.}

In fact, until the present work, the only known computational lower bound for block MDPs was the previously-discussed reduction stated in \cite{golowich2023exploring}, which implies that reinforcement learning in a family of block MDPs inherits the hardness of the corresponding (improper) supervised learning problem. The lower bound therefore holds under average-case assumptions such as hardness of learning noisy parities. Our work is in the same vein. However, we emphasize that the prior result equally applied to contextual bandits, whereas our computational separations fundamentally use the full complexity of reinforcement learning \--- as we show in \cref{sec:bandits}, there is no such separation for contextual bandits.

\section{Preliminaries}\label{sec:prelim}
For a finite set $S$, we write $\Delta(S)$ to denote the set of distributions over $S$, and $\Unif(S)$ to denote the uniform distribution over $S$. For $\gamma \in [0,1]$ we write $\Ber(\gamma)$ to denote the distribution of a Bernoulli random variable $X$ with $\Pr[X=1]=\gamma$. For a distribution $p$ and integer $n \in \NN$ we write $p^{\otimes n}$ to denote the distribution of $(X_1,\dots,X_n)$ where all $X_i \sim p$ are independent.

\subsection{Block MDPs and episodic RL}\label{sec:block-rl}

We work with the finite-horizon episodic reinforcement learning model, as is standard for recent work on RL in block MDPs. We start by formally defining block MDPs in this model. A \emph{block MDP} \cite{du2019provably} is a tuple
\begin{equation}
M = (H, \MS, \MX, \MA, \til \BP_0, (\til \BP_h)_{h \in [H]}, (\til \BO_h)_{h \in [H]}, (\til \br_h)_{h \in [H]}, \phi^\st)
\label{eq:mdp-tuple}
\end{equation}
where $H \in \NN$ is the \emph{horizon}, $\MS$ is the \emph{latent state space}, $\MX$ is the \emph{emission space}, $\MA$ is the \emph{action set}, $\til \BP_0 \in \Delta(\MS)$ is the \emph{latent initial distribution}, $\til \BP_h: \MS \times \MA \to \MS$ is the \emph{latent transition distribution} at step $h$, $\til \BO_h: \MS \to \Delta(\MX)$ is the \emph{emission distribution} at step $h$, $\til \br_h: \MS \times \MA \to [0,1]$ is the \emph{latent reward function}, and $\phi^\st: \MX \to \MS$ is the \emph{decoding function}. 
It is required that $\phi^\st(x_h) = s_h$ with probability $1$ over $x_h \sim \til \BO_h(\cdot|s_h)$, for all $h \in [H]$ and $s_h \in \MS$. Note that this implies that $\til \BO_h(\cdot | s_h), \til \BO_h(\cdot | s_h')$ have disjoint supports for all $s_h \neq s_h'$. 

For any function class $\Phi$ that contains $\phi^\st$, we say that $M$ is $\Phi$-decodable. Also, for any $x,x'\in\MX$ and $a \in \MA$, we write $\BP_h(x'|x,a)$ to denote $\til\BP_h(\phi^\st(x')|\phi^\st(x),a) \til\BO_{h+1}(x'|\phi^\st(x'))$. We similarly define $\BP_0(x) = \til\BP_0(\phi^\st(x))\til\BO_1(x|\phi^\st(x))$ and $\br_h(x,a) = \til\br_h(\phi^\st(x),a)$. Observe that $(H, \MX, \MA, \BP_0, (\BP_h)_h, (\br_h)_h)$ is an MDP (with the potentially large state space $\MX$).

\paragraph{Episodic RL access model.} Fix a block MDP $M$ specified as in \cref{eq:mdp-tuple}. We say that an algorithm $\Alg$ has interactive, episodic access to $M$ to mean that $\Alg$ is executed in the following model. First, $\Alg$ is given $H$ and $\MA$ as input. At any time, $\Alg$ can request a new \emph{episode}. The model then draws $s_1 \sim \til \BP_0$ and $x_1 \sim \til\BO_1(\cdot|s_1)$, and sends $x_1$ to $\Alg$. The timestep of the episode is set to $h=1$. So long as $h \leq H$, the algorithm $\Alg$ can at any time play an action $a_h \in \MA$, at which point the model draws $r_h \sim \Ber(\til \br_h(s_h, a_h))$, $s_{h+1} \sim \til \BP_h(\cdot|s_h, a_h)$, and $x_{h+1} \sim \til\BO_{h+1}(\cdot|s_{h+1})$ (the latter two only if $h < H$). The model sends $(r_h, x_{h+1})$ to $\Alg$ (or just $r_h$, if $h=H$) and increments $h$. The episode concludes once $h = H+1$. Note that $\Alg$ never observes the latent states $s_{1:H}$.

\paragraph{Layered state spaces.} For simplicity, we will assume that the latent state space $\MS$ is \emph{layered}, meaning that $\MS$ is the disjoint union of sets $\MS[1],\dots,\MS[H]$, where $\MS[h]$ is the set of states that are reachable at step $h$. This means that for any $h \in [H]$ and reachable $s_h \in \MS[h]$, the step $h$ is fully determined by any given emission $x_h \sim \til\BO_h(\cdot|s_h)$. We also assume that $h$ can be computed efficiently from $x_h$. These assumptions are without loss of generality up to a factor of $H$ in the size of the latent state space and emission space: simply redefine the state space $\MS$ to $\MS\times[H]$, the emission space $\MX$ to $\MX\times[H]$, and the decoding function class $\Phi$ to the set of maps $(x,h) \mapsto (\phi(x),h)$ for $\phi\in\Phi$. Any algorithm in the episodic RL access model for an arbitrary block MDP can simulate access to this layered block MDP by tracking $h$ and appending it to each emission. 
Accordingly, at times we will drop the (superfluous) subscript $h$ from the quantities $\til \BP_h, \til \BO_h, \til \br_h$, and write e.g. $\til \BO(\cdot|s)$ to denote $\til \BO_h(\cdot|s)$ for the unique $h \in [H]$ such that $s \in \MS[h]$.

\subsection{Policies, trajectories, and visitation distributions}

Fix a block MDP with horizon $H$, emission space $\MX$, and action space $\MA$. For $h \in [H]$, the space of \emph{histories} at step $h$ is $\MH_h := (\MX \times \MA \times \{0,1\})^{h-1} \times \MX$. A (randomized, general) \emph{policy} $\pi = (\pi_h)_{h=1}^H$ is a collection of mappings $\pi_h: \MH_h \to \Delta(\MA)$; we let $\Pi$ denote the space of policies. A \emph{trajectory} is a sequence $(s_1, x_1, a_1, r_1,\dots, s_H, x_H, a_H, r_H)$, which we abbreviate as $(s_{1:H}, x_{1:H}, a_{1:H}, r_{1:H})$, where each $s_h$ is a latent state, $x_h$ is an emission, $a_h$ is an action, and $r_h \in \{0,1\}$ is a reward. 

Any block MDP $M$ and policy $\pi$ together define a distribution $\BP^{M,\pi}$ over trajectories. Specifically, $\BP^{M,\pi}$ is the distribution of the random trajectory drawn during the interaction of an algorithm $\Alg$ with $M$, where at step $h$ the algorithm $\Alg$ plays an action $a_h \sim \pi_h(x_{1:h}, a_{1:h-1}, r_{1:h-1})$. For any event $\ME(s_{1:H},x_{1:H},a_{1:H},r_{1:H})$ on the set of trajectories, we write $\BP^{M,\pi}[\ME]$ to denote $\Pr_{\tau \sim \BP^{M,\pi}}[\tau \in \ME]$, and we similarly define expectations $\EE^{M,\pi}$. For example, in the below definition, $\BP^{M,\pi}[s_h=s]$ denotes the probability that a trajectory $(s_{1:H},x_{1:H},a_{1:H},r_{1:H}) \sim \BP^{M,\pi}$ satisfies $s_h = s$.

\begin{definition}[State visitation distribution]
For an MDP $M$ (with parameters as specified above), policy $\pi$, and step $h \in [H]$, the \emph{state visitation distribution} $d^{M,\pi}_h \in \Delta(\MS)$ is defined by $d^{M,\pi}_h(s) := \BP^{M,\pi}[s_h = s]$.
\end{definition}

\subsection{Block MDP families and complexity measures}

To be concrete about computational complexity, we must be able to discuss asymptotics of learning algorithms as the size of the block MDP grows. Thus, we make the following definition of a family of block MDPs, where each MDP in the family is parametrized by a positive integer $n$ that determines the latent state space, action space, horizon, and emission space. Formally:

\begin{definition}\label{def:block-family}
A \emph{block MDP family indexed by $n$} is a tuple \[\MM = ((\MS_n)_n,(\MA_n)_n,(H_n)_n,(\ell_n)_n,(\Phi_n)_n,(\MM_n)_n)\] consisting of the following data: 
\begin{itemize} 
\item sequences of sets $(\MS_n)_n$, $(\MA_n)_n$, and a sequence of positive integers $(H_n)_n$,
\item a sequence of positive integers $(\ell_n)_n$ (the ``emission lengths''),
\item a sequence of function classes $(\Phi_n)_n$ where each element of $\Phi_n$ is a function $\phi: \{0,1\}^{\ell_n} \to \MS_n$, and
\item a sequence of sets $(\MM_n)_n$, where each $M \in \MM_n$ is a $\Phi_n$-decodable block MDP with horizon $H_n$, latent state space $\MS_n$, emission space $\{0,1\}^{\ell_n}$, and action space $\MA_n$. 
\end{itemize}
\end{definition}

\begin{remark}[Booleanity and circuits]
We explicitly require the emission spaces to be over binary strings since all of our constructions have that form. We will also implicitly assume that the latent states and actions have succinct binary representations (i.e. states $s \in \MS_n$ can be efficiently mapped to/from strings of length $\log|\MS_n|$, and so forth), which will again be evident for our constructions. This ensures that the decoding functions and policies are expressible as Boolean circuits, and when we discuss the circuit size of a decoding function or policy, it will be with respect to these implicit mappings. 
\end{remark}

\begin{remark}[Randomized circuits]
For maximum generality, we will allow circuits describing policies (and regression label functions, as discussed below) to be randomized. Formally, a randomized circuit is a circuit that takes some number of extra Boolean inputs, and the output is defined to be the random variable obtained by setting these extra inputs to be independent with distribution $\Ber(1/2)$.
\end{remark}

We will have theorem statements that e.g. assume that there exists an algorithm for learning in a particular block MDP family with some given, unspecified, time or sample complexities. To make such statements more clear, we will use the following simple terminology.

\begin{definition}\label{def:complexity-measure}
A \emph{complexity measure} for a block MDP family indexed by $n$ is any positive, real-valued function of $n$.
\end{definition}

To formally state \cref{thm:no-reduction-intro}, we will also need the following definition.

\begin{defn}[Computable block MDP family]
  \label{def:computable-family}
  For a sequence $B = (B_n)_{n \in \BN}$ of natural numbers, a block MDP family $\MM = ((\MS_n)_n, (\MA_n)_n, (H_n)_n, (\ell_n)_n, (\Phi_n)_n, (\MM_n)_n)$ (\cref{def:block-family}) is said to be \emph{$B$-computable} if the following conditions hold for all $n \in \NN$:
  \begin{itemize}
  \item $\max \{ \ell_n, |\MA_n|, |\MS_n|, H_n,  \log |\Phi_n|\} \leq B_n$.
  \item There is a circuit $\MC_{\Phi_n}$ of size at most $B_n$ that takes as input a pair $(\phi, x) \in \Phi_n \times \MX_n$ (where $\phi$ is represented by an integer in $[|\Phi_n|]$ in binary) and outputs $\phi(x)$. 
  \end{itemize}
  Furthermore, we say that the block MDP family is \emph{polynomially horizon-computable} if it is $\poly(H_n)$-computable. 
\end{defn}
\begin{remark}[Succinct optimal policies]
Note that for any $B$-computable block MDP family $\MM$ and $n \in \NN$, any MDP $M \in \MM_n$ has an optimal policy that can be computed by a circuit of size $\poly(B_n)$: 
let $\pi^\st : \MS_n \to \MA_n$ denote an optimal policy for its underlying latent MDP, which is efficiently computable since $|\MS_n| \leq \poly(B_n)$. Also let $\phi^\st \in \Phi_n$ be the decoding function for $M$. Then the policy $\bar \pi^\st(x) = \pi^\st(\phi^\st(x))$ is an optimal policy for the block MDP $M$ and, by the second guarantee of \cref{def:computable-family}, can be computed by a circuit of size $\poly(B_n)$.
\end{remark} 

\subsection{Computational problems}

Our computational separation result (\cref{thm:separation-intro}) is between reward-free reinforcement learning and realizable regression. Below, we formally define what it means for an algorithm to solve each of these problems, for a given block MDP family, with given resource constraints.

A \emph{policy cover} is a natural solution concept for reward-free RL: a set of policies that (on average) explore the entire state space almost as well as possible. Many algorithms for reward-directed RL learn a policy cover as an intermediate step \cite{du2019provably, misra2020kinematic, mhammedi2023representation}, and subsequently optimize for the rewards using \emph{Fitted $Q$-Iteration} \cite{ernst2005tree, chen2019information} or \emph{Policy Search by Dynamic Programming} \cite{bagnell2003policy} in conjunction with trajectories drawn via the policy cover.

\begin{definition}[Policy cover]\label{def:pc}
Let $M$ be a block MDP and let $\Psi$ be a set of policies for $M$. For $\alpha,\gamma > 0$, we say that $\Psi$ is an $(\alpha,\gamma)$-\emph{policy cover} for $M$ if for every latent state $s$ of $M$ it holds that
that
\begin{equation}\EE_{\pi\sim\Unif(\Psi)} d^{M,\pi}(s) \geq \alpha \cdot \left(\max_{\pi'\in\Pi} d^{M,\pi'}(s) - \gamma\right).\label{eq:pc}\end{equation}
\end{definition}

To be clear, a more common (weaker) definition replaces the expectation over $\pi \sim \Unif(\Psi)$ by a maximum over $\pi \in \Psi$. For technical reasons, we cannot prove our separation under such a definition without introducing an upper bound $|\Psi| \leq P := \poly(H,|\MA|,|\MS|)$ (note that the two definitions are then equivalent up to a factor of $P$ in $\alpha$). But we do not believe this to be a substantive shortcoming, since such a bound does hold for the aforementioned RL algorithms that learn policy covers.


We now define what it means to learn a policy cover for a family of block MDPs.

\begin{definition}[Policy cover learning algorithm]\label{def:rf-rl}
Let $\MM$ be a block MDP family. Let $\alpha: \NN \to \RR_{>0}$, $S, T, B: \NN \to \NN$ 
be complexity measures. A Turing Machine $\PC$ with episodic access to an MDP is a $(S,T,\alpha,B)$-\emph{policy cover learning algorithm} for $\MM$ if the following holds. For every $n \in \NN$ and $M \in \MM_n$, 
with probability at least $2/3$, the output of $\PC(n)$ on interaction with $M$ is a set of policies $\Psi \subseteq \Pi$ where each $\pi \in \Psi$ is represented as a circuit $\MC_\pi$ of size at most $B(n)$, and $\Psi$ is an $(\alpha(n),1/4)$-policy cover (\cref{def:pc}) for $M$. 
Moreover, the time complexity of $\PC$ is at most $T(n)$ and the sample complexity is at most $S(n)$.
\end{definition}


Throughout this paper, we informally consider (standard) reward-free RL to be the problem of policy cover learning with $\alpha \geq 1/\poly(|\MS_n|,|\MA_n|,H_n)$, and strong reward-free RL to be the problem of policy cover learning with $\alpha = 1$. Note that we are fixing the additive approximation error of the policy covers to be the constant $\gamma := 1/4$. In practice, one would like $\gamma = o(1)$, but since we are proving \emph{hardness} of policy cover learning, our definition only makes our results stronger. 

Next, the following key definition describes the solution concept for regression with respect to a conditional distribution $s \mapsto \MD(\cdot|s)$.

\begin{definition}[Accurate regression predictor]
  \label{def:reg-predictor}
Let $\MS$ be a set with an associated conditional distribution $s \mapsto \MD(\cdot|s) \in \Delta(\MX)$. Let $\beta \in \Delta(\MS)$, $f: \MS \to [0,1]$, and $\epsilon>0$. We say that a circuit $\MR$ is a \emph{$(\beta,f,\epsilon)$-predictor with respect to $\MD$} if it defines a mapping $\MX \to [0,1]$ such that
\[\EE_{\substack{s \sim \beta \\ x \sim \MD(\cdot|s) \\ y \sim \Ber(f(s))}} (\MR(x) - y)^2 \leq \epsilon + \EE_{\substack{s \sim \beta \\ y \sim \Ber(f(s))}} (f(s) - y)^2.\]
The quantity $\ep$ is referred to as (an upper bound on) the \emph{excess risk} of the predictor. Note that the above inequality can be equivalently stated as \[\EE_{\substack{s\sim\beta \\ x \sim \MD(\cdot|s)}} (\MR(x) - f(s))^2 \leq \epsilon.\]
\end{definition}

A realizable regression algorithm for $\MM$ is an algorithm that, given samples where the covariate distribution is realizable by $\MM$ (i.e. obtainable as a mixture of emission distributions for some MDP $M$ in the family) and the labels are realizable with respect to the decoding function (i.e. only depend on the latent state), produces an accurate predictor:

\begin{definition}[Realizable regression algorithm]\label{def:regression-algorithm}
Let $\MM$ be a block MDP family. Let $\epsilon: \NN \to \RR_{>0}$, $S, T, B: \NN \to \NN$ 
be complexity measures. A Turing Machine $\Reg$ is a $(S,T,\epsilon,B)$-\emph{realizable regression algorithm} for $\MM$ if the following holds. 

Fix $n \in \NN$ and $M \in \MM_n$, and let $H$, $\MX$, $\MS$, $(\til\BO_h)_h$, and $\phi^\st: \MX\to\MS$ denote the horizon, emission space, latent state space, emission distributions, and decoding function of $M$ respectively. Let $h \in [H]$, and $\beta\in\Delta(\MS[h])$, where $\MS[h]$ denotes the set of states reachable at step $h$. Let $(x^i, y^i)_{i=1}^{S(n)}$ be i.i.d. samples where $x^i \sim \sum_{s\in\MS[h]} \beta(s) \til\BO_h(\cdot|s)$ and $y^i \in \{0,1\}$ satisfies $y^i \perp x^i | \phi^\st(x^i)$. Then with probability at least $2/3$, the output $\MR \gets \Reg((x^i,y^i)_{i=1}^{S(n)},n)$ is a circuit of size at most $B(n)$, and a $(\beta, f, \epsilon(n))$-predictor with respect to $\til\BO_h(\cdot|s)$, where $f: \MS[h] \to [0,1]$ is the function $f(s) := \EE[y^i|\phi^\st(x^i)=s]$.

Moreover, the time complexity of $\Reg$ on this input is at most $T(n)$.
\end{definition}

\begin{remark}
We are restricting the latent state distribution $\beta$ to be supported on the states reachable at a particular step $h$ (rather than across all steps). This simplifies the proof somewhat; moreover, it is without loss of generality up to factors of $\poly(H_n, \epsilon(n)^{-1})$ in the sample complexity and runtime, since the algorithm can partition the samples by step and solve a regression individually at each step.
\end{remark}

For the most part, low-level details of the model of computation for our algorithms will be unimportant. However, to be precise, we will consider a uniform algorithm to be a one-tape Turing Machine with alphabet $\{0,1\}$, and we will define its description complexity as follows.

\begin{definition}\label{def:desc-complexity}
The \emph{description complexity} of a uniform algorithm $\Alg$, which we denote by $\dc(\Alg)$, is the size of the state set of the corresponding Turing Machine.
\end{definition}

Note that the Turing Machine of a uniform algorithm $\Alg$ can be described as a string of length $\poly(\dc(\Alg))$. Moreover, given this description as input, a Turing Machine can simulate $\Alg$ with multiplicative overhead $\poly(\dc(\Alg))$. See e.g. Sections 1.2.1 and 1.3.1 of \cite{arora2009computational} for a reference.

\subsection{Regression oracle and reductions}

We now formally define the oracles that we use in our lower bound against oracle-efficient algorithms (\cref{thm:no-reduction-intro}). To make the proof cleaner, instead of considering algorithms in the episodic RL access model, we instead give the algorithm access to a \emph{sampling oracle} that draws a trajectory from the block MDP $M$ given a succinct description of a policy. 
All episodic RL algorithms that we are aware of can easily be expressed using a sampling oracle.

\begin{defn}[Sampling oracle]
  \label{def:sampling-oracle}
Let $M$ be a block MDP with horizon $H$. A \emph{sampling oracle} $\Osample$ for $M$ takes as input a circuit $\MB_\pi$ representing a general policy $\pi$, and outputs a trajectory $(x_{1:H}, a_{1:H}, r_{1:H}) \sim \BP^{M, \pi}$ consisting of emissions, actions, and rewards drawn from $M$ under policy $\pi$. 
\end{defn}

\begin{defn}[Regression oracle]
  \label{def:regression-oracle}
  Let $M = (H, \MS, \MX, \MA, \til\BP_0, (\til\BP_h)_h,(\til\BO_h)_h, (\til\br_h)_h, \phi^\st)$ be a block MDP. For $B \in \NN$, a \emph{$B$-bounded regression oracle} for $M$ is a nondeterministic function $\Oregress$ which takes as input a step $h \in [H]$ and (randomized) circuits $\MB_\pi, \MB_L$ describing a general policy $\pi$ and a \emph{labeling function} $L : (\MX \times \MA \times \{0,1\})^H \to \Delta(\{0,1\})$, and which outputs a circuit $\MC_\MR$ describing a mapping $\MR : \MX \to [0,1]$, where $\size(\MC_\MR) \leq B$. 
  A \emph{regression oracle} for $M$ is one which is $\infty$-bounded, i.e., for which there is no constraint on $\size(\MC_\MR)$. 
  
  Furthermore, we say that $\Oregress$ is \emph{$\ep$-accurate} for $M$ if for each tuple $(\MB_\pi, \MB_L, h)$, the output $\MR := \Oregress(\MB_\pi, \MB_L, h)$ is a $(d_h^{M,\pi}, f,  \ep)$-accurate predictor (\cref{def:reg-predictor}) with respect to $\til\BO_h(\cdot \mid s)$, where $f: \MS \to [0,1]$ is the function $f(s) := \E^{M,\pi}[L(x_{1:H}, a_{1:H}, r_{1:H}) \mid \phi^\st(x_h) = s]$.

\end{defn}

\begin{remark}[Existence of regression oracle]\label{rmk:reg-oracle-existence}
Let $M$ be a block MDP with latent state space $\MS$, emission space $\{0,1\}^\ell$ and decoding function $\phi^\st$. If $\phi^\st$ can be represented by a circuit of size $B$, then for any $\epsilon \in (0,1/2)$ there exists a $O(B + |\MS| \log(1/\epsilon))$-bounded, $\epsilon$-accurate regression oracle for $M$. Given input $(\MB_\pi,\MB_L,h)$, the output predictor is a circuit for $x \mapsto \epsilon \lfloor f(\phi^\st(x))/\epsilon\rfloor$, where $f$ is as defined above.

Thus, in particular, if $\MM$ is a $K(n)$-computable block MDP family indexed by $n$ (\cref{def:computable-family}), then for any function $\epsilon:\NN\to(0,1/2)$ there is a $\poly(K(n), \log(1/\epsilon(n)))$-bounded, $\epsilon(n)$-accurate regression oracle for each $n \in \NN$ and $M \in \MM_n$.
\end{remark}

\begin{remark}
Note that an $\ep$-accurate regression oracle is modeled as a \emph{nondeterministic} function, meaning that, on each regression oracle call, the oracle may return an \emph{arbitrary} predictor $\MR : \MX \to [0,1]$ subject to the accuracy condition in \cref{def:regression-oracle} (in particular, two identical oracle calls may return different outputs). This definition greatly eases the proof of \cref{thm:no-reduction-intro}, but we do not believe that the non-determinism is essential, and in any case it seems unlikely that the success of any natural learning algorithm would be contingent on determinism of the oracle.

Also, like in \cref{def:regression-algorithm}, we defined the regression oracle so that the covariate distribution is always emitted from a single step $h$. One could again imagine requiring the oracle to perform regression on mixture distributions across steps. However, for the same reasons as above, this is essentially without loss of generality. 
\end{remark}

\begin{remark}
Note that \cref{def:regression-oracle} does not require that the label distribution $L(x_{1:H},a_{1:H},r_{1:H})$ is independent of $x_h$ given $\phi^\st(x_h)$, as we required in \cref{def:regression-algorithm}. However, label functions satisfying independence are particularly natural in the context of \cref{def:regression-oracle} since when independence holds, we have that $\E^{M,\pi}[L(x_{1:H}, a_{1:H}, r_{1:H}) \mid x_h=x] = \E^{M,\pi}[L(x_{1:H}, a_{1:H}, r_{1:H}) \mid s_h = \phi^\st(x)] = f(\phi^\st(x))$, which implies that the requirement of $\MR$ being a $(d_h^\pi, f, \ep)$-accurate predictor is equivalent to
  \begin{align}
\EE_{\substack{s_h \sim d_h^\pi\\ x_h \sim \til\BO_h(\cdot \mid s_h) \\ y \sim \Ber(f(s_h))}}\left[ (\MR(x_h) - y)^2 \right] \leq \ep + \inf_{\substack{\phi' \in \Phi \\ f' : \MS \to \BR}} \EE_{\substack{s_h \sim d_h^\pi\\ x_h \sim \til\BO_h(\cdot \mid s_h) \\ y \sim \Ber(f(s_h))}} \left[ (f'(\phi'(x_h)) - y)^2 \right]\nonumber,
  \end{align}
  i.e., the mapping $x_h \mapsto \MR(x_h)$ is approximately as good as the best mapping $x_h \mapsto f'(\phi'(x_h))$, for $\phi' \in \Phi$ and $f':\MS\to[0,1]$. 
\end{remark}


In \cref{def:rl-reg-reduction-2}, we formally define the notion of a reduction from RL to regression: it is an algorithm for the online RL setting which has access to a sampling oracle (\cref{def:sampling-oracle}) and a regression oracle (\cref{def:regression-oracle}).

\begin{definition}[Reduction from RL to regression] 
  \label{def:rl-reg-reduction-2}
Let $\MM$ be a block MDP family. Let $\epsilon: \NN \to \RR_{>0}$ and $T, B: \NN \to \NN$ be complexity measures. We say that an oracle Turing Machine $\Alg^{\Oregress, \Osample}$, that takes as input a natural number $n$ and has access to oracles $\Oregress, \Osample$, is a \emph{\emph{$(T, \ep)$}-reduction from RL to regression} for $\MM$ if the following holds.

Let $n \in \NN$ and $M \in \MM_n$. The number of oracle calls made by $\Alg$ is at most $T(n)$. Additionally, if $\Oregress$ is $\epsilon(n)$-accurate for $M$ (\cref{def:regression-oracle}) and $\Osample$ is a sampling oracle for $M$ (\cref{def:sampling-oracle}), then with probability at least $1/2$, $\Alg^{\Oregress,\Osample}(n)$ produces a circuit $\MC_{\hat\pi}$ describing a (general) policy $\hat \pi$ with suboptimality at most $1/2$. 

We say that $\Alg$ is a \emph{computational $(T,\epsilon, B)$-reduction} if, in addition to the above, the following property holds. For each $n \in \NN$ and $M\in \MM_n$, suppose that $\Oregress$ is $B(n)$-bounded for $M$ (\cref{def:regression-oracle}). 
Then the running time of $\Alg(n)$ is at most $T(n)$.
\end{definition}

\begin{remark}[Proper vs. improper]
In the above definitions, we let the output of a regression algorithm or oracle be an arbitrary, bounded-size circuit. This corresponds to \emph{improper} PAC learning, and it suffices for the reductions described in \cref{sec:offline,sec:bandits,sec:deterministic}. Moreover, it can be checked that \cref{thm:separation-intro,thm:no-reduction-intro} still hold when the regression algorithms/oracles are required to be proper, i.e. to output the composition of a decoding function $\phi \in \Phi$ with a map $f: \MS \to [0,1]$. This is more in line with the oracles used in theoretical reinforcement learning for block MDPs.
\end{remark}


\subsection{Learning parities with noise}

We formally introduce the Learning Parities with Noise (LPN) problem and the hardness assumption on which \cref{thm:separation-intro} is based.

\begin{definition}\label{def:lpn}
Fix $n \in \NN$, $\delta \in [-1/2,1/2]$, and $\sk \in \BF_2^n$. We define $\LPN_{n,\delta}(\sk)$ to be the distribution of the pair $(u, y)$ where $u \sim \Unif(\BF_2^n)$ and $y = \langle u, \sk\rangle + e \in \BF_2$, where $e \sim \Ber(1/2-\delta)$ is independent of $a$.
\end{definition}

The noisy parity learning problem is the algorithmic task of recovering $\sk$ from independent samples from $\LPN_{n,\delta}(\sk)$. Since the noise is drawn from $\Ber(1/2 - \delta)$, \emph{smaller} values of $\delta$ (in absolute value) corresponding to \emph{harder} instances of learning parity with noise. 

\begin{definition}[Learning noisy parities]\label{def:lpn-alg-advice}
For any algorithm $\Alg$, we say that it learns noisy parities with time complexity $T(n,\delta,\eta)$ and sample complexity $S(n,\delta,\eta)$\footnote{We may assume without loss of generality that any algorithm for learning noisy parities must read its entire input, and thus $S(n,\delta,\eta)$ is at most $T(n,\delta,\eta)$.} if the following holds. For every $n \in \NN$, $\delta \in (0,1/2)$, and $\eta \in (0,1)$, for all $\sk \in \BF_2^n$, if $(u_i,y_i)_{i=1}^{S(n,\delta,\eta)}$ are independent draws from $\LPN_{n,\delta}(\sk)$, then 
\[\Pr[\Alg((u_i,y_i)_{i=1}^{S(n,\delta,\eta)},\delta,\eta) = \sk] \geq 1-\eta,\]
and the time complexity of $\Alg$ on this input is $T(n,\delta,\eta)$.
\end{definition}

With this notation, we can formally state our assumption.

\begin{assumption}\label{assm:fine-lpn}
For every constant $c>0$, there is no non-uniform algorithm for learning noisy parities with $\poly(n)$ advice and time complexity $T(n,\delta,\eta)$ that satisfies $T(n, 2^{-n/\log \log n}, 1/2) \leq 2^{cn/\log \log n}$. 
\end{assumption}

For context, note that there is an algorithm for learning noisy parities with \emph{statistical} complexity $S(n,\delta,\eta)$ satisfying $S(n,2^{-n/\log \log n}, 1/2) \leq 2^{O(n/\log \log n)}$ (\cref{lemma:brute-force-lpn}), but the best-known bound on \emph{time} complexity is $T(n,2^{-n/\log \log n},1/2) \leq 2^{O(n/\log \log \log n)}$ (see \cref{thm:bkw} due to \cite{blum2003noise}). Improving the noise tolerance is mentioned as an open problem in \cite{blum2003noise} and more recently in \cite{reyzin2020statistical}. 

A \emph{(non-uniform) algorithm with $\poly(n)$ advice} is a Turing Machine where for each $n$ (in this case, corresponding to the number of variables in the LPN instance), the input is augmented with a binary string of length $\poly(n)$, which may depend on $n$ but \emph{not} the input (or $\sk$); see e.g. \cite[Definition 6.9]{arora2009computational}. While the problem of learning noisy parities has thus far only been studied in the uniform model of computation, we are not aware of any natural learning problems where access to polynomial advice is known to decrease the asymptotic computational complexity. We discuss the technical reason why non-uniformity is needed for \cref{assm:fine-lpn} in \cref{sec:overview-chaining}.

\subsubsection{Algorithms for LPN} 

The following seminal result remains the best-known bound on the time complexity of learning noisy parities (when $1/2-\delta$ is bounded away from $0$).

\begin{theorem}[\cite{blum2003noise}]
  \label{thm:bkw}
There is a universal constant $C>0$ and an algorithm $\BKW$ that learns noisy parities with time complexity $T_{\BKW}(n,\delta,\eta)$ and sample complexity $S_{\BKW}(n,\delta,\eta)$ satisfying
\[S_{\BKW}(n,\delta,\eta) \leq T_{\BKW}(n,\delta,\eta) \leq \min_{1 \leq a \leq n} \delta^{-C \cdot 2^a} 2^{C \cdot n/a} \log \frac{1}{\eta}\]
for all $n \in \NN$, $\delta \in (0,1/2)$, and $\eta \in (0,1/2)$.
\end{theorem}
In particular, choosing $a = \log n - 2 \log \log n$, we see from \cref{thm:bkw} that $T(n, 2^{-n^{1-\gamma}}, 2^{-n}) \leq 2^{O(n/\log n)}$ for any constant $\gamma \in (0,1)$. However, for $\delta(n) = 2^{-n/\log \log n}$, the $\BKW$ algorithm requires time complexity $2^{\Omega(n/\log \log \log n)}$, which is consistent with \cref{assm:fine-lpn}.

To prove \cref{thm:separation-intro}, we will actually make use of the following result, which improves upon the sample complexity of $\BKW$, at the cost of somewhat worse (but still better than brute-force) time complexity:

\begin{theorem}[\cite{lyubashevsky2005parity}]\label{theorem:lyu}
Let $\epsilon, c\in(0,1)$ be constants. There is an algorithm $\Lyu$ that learns noisy parities with time complexity $T(n,\delta,\eta)$ and sample complexity $S(n,\delta,\eta)$ satisfying $T(n, 2^{-\log^c n}, e^{-n}) \leq 2^{O(n/\log \log n)}$ and $S(n, 2^{-\log^c n}, e^{-n}) \leq n^{1+\epsilon}$.
\end{theorem}


Finally, we recall that there is a brute-force estimator for learning noisy parities, which picks the parity function with minimal empirical labelling error. It satisfies the following guarantee, which can be deduced from standard concentration bounds:

\begin{lemma}\label{lemma:brute-force-lpn}
There is a universal constant $C>0$ and an algorithm $\Brute$ that learns noisy parities with time complexity $T_\Brute(n,\delta,\eta) \leq C\delta^{-2}n 2^n \log(1/\eta)$ and sample complexity $S_\Brute(n,\delta,\eta) \leq C\delta^{-2}n \log(1/\eta)$ for all $n \in \NN$, $\delta \in (0,1/2)$, and $\eta \in (0,1/2)$.
\end{lemma}

\subsubsection{Technical lemmas for LPN}

Recall that in \cref{def:lpn-alg-advice}, the algorithm is given the noise level $\delta$ as part of the input. We will also need the following definition about algorithms that are not given the true noise level, but rather an upper bound on the true noise level. The performance is measured as a function of the given upper bound.

\begin{definition}[Learning noisy parities with unknown noise level]\label{def:lpn-unknown-noise}
For any algorithm $\Alg$, we say that it learns noisy parities \emph{with unknown noise level} with time complexity $T(n,\delta,\eta)$ and sample complexity $S(n,\delta,\eta)$ if the following holds. For every $n \in \NN$, $0 < \delta \leq 1/2$, and $\delreal \in [-1/2,1/2]$ with $|\delreal| \geq \delta$, and $\eta \in (0,1)$, for all $\sk \in \BF_2^n$, if $(u_i,y_i)_{i=1}^{S(n,\delta,\eta)}$ are independent draws from $\LPN_{n,\delreal}(\sk)$, then 
\[\Pr[\Alg((u_i,y_i)_{i=1}^{S(n,\delta,\eta)},\delta,\eta) = \sk] \geq 1-\eta,\]
and the time complexity of $\Alg$ on this input is $T(n,\delta,\eta)$.
\end{definition}

While natural algorithms such as $\BKW$ achieve identical guarantees for learning noisy parities with unknown noise level as for learning noisy parities with known noise level, it is not a priori clear that these two problems have the same computational complexity. Thus, we will need the following lemma which states that there is a way to ``guess'' the noise level that multiplicatively blows up the time complexity by a factor of roughly the sample complexity.

\begin{lemma}[Guessing the noise level]\label{lemma:lpn-unknown-delta}
Let $\Alg$ be an algorithm for learning noisy parities with time complexity $T(n,\delta,\eta)$ and sample complexity $S(n,\delta,\eta)$, where $T$ and $S$ are non-increasing in $\delta$. Then there is an algorithm $\widetilde\Alg$ for learning parities with unknown noise level, with time complexity $\til T(n,\delta,\eta)$ and sample complexity $\til S(n,\delta,\eta)$ satisfying
\[\til T(n,\delta,\eta) \leq O(dn + T(n,\delta,1/2))\cdot S(n,\delta,1/2) \log(2/\eta)\]
\[\til S(n,\delta,\eta) = 4\cdot S(n, \delta, 1/2)\log(2/\eta) + 9\delta^{-2}\log \frac{32\cdot S(n,\delta,1/2)\log(2/\eta)}{\eta}\]
for all $n \in \NN$, $\delta \in (0,1/2]$, and $\eta \in (0,1)$.

\end{lemma}

\begin{proof}
The algorithm $\widetilde\Alg$ proceeds as follows. For notational simplicity, let $m := S(n,\delta,1/2)$, let $d := \til S(n,\delta,\eta)$ as defined in the lemma statement, and let $d' := 4m\log(2/\eta)$. Let $\MR \subset [\delta, 1/2]$ be the set of $2m$ real numbers evenly spaced from $\delta$ to $1/2$ inclusive. For each $0 \leq j < 4\log(2/\eta)$ and $r \in \MR$, compute
\[\wh\sk^{j,r} := \Alg((u_i,y_i)_{i=1+jm}^{(1+j)m}, r, 1/2)\]
\[\wh\sk^{j,-r} := \Alg((u_i,1-y_i)_{i=1+jm}^{(1+j)m}, r, 1/2).\]
Finally, compute and return
\[\wh\sk \gets \Select((u_i,y_i)_{i=d'+1}^d, \{\wh \sk^{j,r}: 0 \leq j < 4 \log(2/\eta) \land r \in \MR\cup -\MR\}),\]
where $\Select$ is defined in \cref{lemma:select-alg} below. 

\paragraph{Analysis.} Since $|\delreal| \in [\delta,1/2]$, there is some $r^\st \in \MR$ such that $|r^\st - |\delreal|| \leq 1/(4m)$. If $\delreal>0$, then $|r^\st-\delreal| \leq 1/(4m)$, so the distribution of each sample $(u_i,y_i)$ has total variation distance at most $1/(4m)$ from $\LPN_{n,r^\st}(\sk)$; otherwise, the distribution of $(u_i,1-y_i)$ has total variation distance at most $1/(4m)$ from $\LPN_{n,r^\st}(\sk)$. Consider the case $\delreal>0$. It follows that for each $0 \leq j < \log(1/\eta)$, the samples $(u_i,y_i)_{i=1+jm}^{(1+j)m}$ have distribution within total variation distance $1/4$ of $\LPN_{n,r^\st}(\sk)^{\otimes m}$. Since $r^\st \geq \delta$ and we have assumed that $S$ is monotonic non-increasing in $\delta$, we have $m = S(n,\delta,1/2) \geq S(n, r^\st, 1/2)$, so it holds with probability at least $1/4$ that $\wh\sk^{j,r^\st} = \sk$. By independence of the samples $(u_i,y_i)_{i=1+jm}^{(1+j)m}$ as $j$ varies, we get
\[\Pr[\exists j: \wh\sk^{j,r^\st} = \sk] \geq 1 - (3/4)^{4\log(2/\eta)} \geq 1 - \frac{\eta}{2}.\]
In the case $\delreal < 0$, we similarly get
\[\Pr[\exists j: \wh\sk^{j, -r^\st} = \sk] \geq 1 - \frac{\eta}{2}.\]
Let $\ME$ be the event $\wh\sk^{j,r} = \sk$ for some $0 \leq j < 4\log(4/\eta)$ and $r \in \MR \cup -\MR$. Condition on $\ME$, which by the above argument occurs with probability at least $1-\eta/2$. We now apply \cref{lemma:select-alg} with noise level $\delreal$, sample size $d-d'$, and hypothesis set size $8m\log(2/\eta)$. Since $d-d' \geq 9\delreal^{-2} \log \frac{16m\log(2/\eta)}{\eta/2}$, it follows that $\wh \sk = \sk$ with probability at least $1-\eta/2$. The union bound completes the correctness analysis.

\paragraph{Time complexity.} Immediate from the algorithm description, the assumption that $T$ is monotonic non-increasing in $\delta$, and the time complexity guarantee of $\Select$ (\cref{lemma:select-alg}).
\end{proof}

\begin{lemma}[Hypothesis selection for LPN]\label{lemma:select-alg}
There is an algorithm $\Select$ with the following property. Let $n,d \in \NN$, $\delta \in [-1/2,1/2]\setminus\{0\}$, $\eta \in (0,1)$, $\sk \in \BF_2^n$, and $\MH \subseteq \BF_2^n$. If $d \geq 9\delta^{-2} \log \frac{2|\MH|}{\eta}$ and $\sk \in \MH$, then
\[\Pr[\Select((u_i,y_i)_{i=1}^d, \MH) = \sk] \geq 1-\eta\]
where $(u_i,y_i)_{i=1}^d$ are independent draws from $\LPN_{n,\delta}(\sk)$. Moreover, the time complexity is $O((d+|\MH|)n)$.
\end{lemma}

\begin{proof}
For each $t \in \MH$, the algorithm computes 
\[\hat{E}^t := \frac{1}{d}\sum_{i=1}^d \mathbbm{1}[y_i \neq \langle u_i, t\rangle].\]
Finally, the algorithm returns any $\wh \sk \in \argmax_{t \in \MH} |\hat{E}^t - 1/2|$.

By Hoeffding's inequality, for each $t \in \MH$ we have with probability at least $1 - \eta/|\MH|$ that  
\[\left|\hat E^t - \Pr_{(u,y)\sim\LPN_{n,\delta}}[y \neq \langle u,t\rangle]\right| \leq \frac{1}{\sqrt{d}} \log \frac{2|\MH|}{\eta} \leq \frac{\delta}{3}\]
where the last inequality is by choice of $d$. Suppose that all of these events hold simultaneously, which by the union bound holds with probability at least $1-\eta$. Then for all $t \in \MH$ with $t \neq \sk$, we have
$|\hat{E}^t - 1/2| \leq \frac{\delta}{3}$. On the other hand, recalling that $\sk \in \MH$, we have $|\hat{E}^\sk - 1/2| \geq 2\delta/3$. It follows that $\wh\sk = \sk$.
\end{proof}

\begin{lemma}
  \label{lem:bernoulli-convolve}
Let $\delta_1, \delta_2 \in [-1/2,1/2]$. If $Z_1 \sim \Ber(1/2 -\delta_1)$ and $Z_2 \sim \Ber(1/2 - \delta_2)$ are independent, and $Z_3 := Z_1 + Z_2 \mod{2}$, then $Z_3 \sim \Ber(1/2 - 2\delta_1\delta_2)$.
\end{lemma}
\begin{proof}
We can check that
\[\Pr[Z_3=1] = \left(\frac{1}{2}-
\delta_1\right)\left(\frac{1}{2}+\delta_2\right)+\left(\frac{1}{2}+\delta_1\right)\left(\frac{1}{2}- \delta_2\right) = \frac{1}{2}-2\delta_1\delta_2\]
as needed.
\end{proof}

\subsection{Pseudorandom permutations}




We formally introduce pseudorandom permutations and the hardness assumption on which \cref{thm:no-reduction-intro} is based.

\begin{definition}[Pseudorandom Permutation (i.e., block cipher), see e.g. \cite{luby1988construct}]\label{def:prp}
  For each $\ell \in \NN$, write $\MX_\ell := \{0,1\}^\ell$. Let $t,q : \BN \to \BN$ be functions. An ensemble of functions $F_\ell : \{0,1\}^\ell \times \MX_\ell \to \MX_\ell$ (indexed by $\ell \in \BN$) is a \emph{$(t,q)$-pseudorandom permutation (PRP)} if the following conditions hold:
\begin{enumerate}
\item \label{it:prp-indist}  Consider any $t(\ell)$-time probabilistic oracle Turing machine $\SA$. The algorithm $\SA$ is passed as input $\ell \in \BN$ and has oracle access to a function $J : \MX_\ell \to \MX_\ell$. At each step of its computation, it is allowed to choose a value $i \in \MX_\ell$ and make a query to $J(i)$, and at termination, it outputs a single bit. 
We require that for all such $\SA$ and all sufficiently large $\ell \in \NN$,
  \begin{align}
\left| \E_{\rho \sim \{0,1\}^\ell} \left[\SA^{F_\ell(\rho, \cdot)}(\ell)\right] - \E_{J \sim \Unif(\MX_\ell^{\MX_\ell})} \left[\SA^J(\ell)\right] \right| &\leq \frac{1}{q(\ell)}\label{eq:prp-distinguish}.
  \end{align}
\item \label{it:prp-invertible} For all $\ell \in \NN$ and $\rho \in \{0,1\}^\ell$, the function $F_\ell(\rho, \cdot)$ is a bijection. 
\item \label{it:prp-eff} There is a polynomial-time Turing machine that, on input $(\rho, x) \in \{0,1\}^\ell \times \MX_\ell$, returns $F_\ell(\rho, x)$. Moreover, there is a polynomial-time Turing machine that on input $(\rho, x) \in \{0,1\}^\ell \times \MX_\ell$, returns the (unique) $x' \in \MX_\ell$ for which $F_\ell(\rho, x') = x$. With slight abuse of notation, we denote this unique $x'$ by $x' := F_\ell^{-1}(\rho, x)$. \label{it:eff-decode}
\end{enumerate}
\end{definition}

Pseudorandom permutations are a fundamental cryptographic primitive, and can be constructed from pseudorandom functions via the Luby-Rackoff transformation \cite{luby1988construct}. Efficient invertibility is a consequence of the fact that the transformation is a composition of Feistel rounds. Moreover, the transformation gives at most $O(t(\ell)^2/2^\ell)$ additional advantage to any time-$t(\ell)$ adversary.

As stated below, we will need a pseudorandom permutation family that is secure against sub-exponential time distinguishers. Since a pseudorandom function family with sub-exponential security can be constructed from a one-way function family with sub-exponential security (see e.g. \cite[Definition 2.3]{kalai2017obfuscation} and discussion), and such families exist under concrete assumptions such as sub-exponential hardness of factoring (see e.g. \cite[Section 8.4]{katz2007introduction}), the below assumption is well-founded.

\begin{assumption}[PRPs with sub-exponential hardness]
  \label{asm:prf}
For some constant $c > 0$, a $(t,q)$-pseudorandom permutation exists with $t(\ell) = 2^{\ell^c}$ and $q(\ell) = 2^{\ell^c}$. 
\end{assumption}

We also remark that the sub-exponential growth of $t(\ell), q(\ell)$ in \cref{asm:prf} is not crucial for our application: if we were to make the weaker and more standard assumption that a $(t, q)$-PRP exists for functions $t,q$ growing faster than any polynomial, then our lower bound in \cref{thm:reduction-prp} would continue to hold with a weaker quantitative bound (namely, $\Tred(n), 1/\epred(n)$ would grow faster  than any polynomial in $H_n$).



\section{Detailed technical overview}\label{sec:overview}

To construct a family of block MDPs where reinforcement learning is harder than regression, two design choices need to be made: the family of tabular MDPs describing the latent structure, and the family of emission distributions. The computational difficulty of learning depends on both. In all of our constructions, the latent MDP will be fixed, and the source of hardness will be via the unknown emission distributions. We start with some basic observations about the needed latent structure, which will be useful for both \cref{thm:separation-intro} and \cref{thm:no-reduction-intro}. 

\paragraph{Exploration must be hard.} To prove a separation against reward-free RL, where exploration is itself the goal, this is obvious. However, exploration must be hard even to prove a separation against \--- or rule out a reduction from \--- reward-directed RL. For example, if the policy that plays uniformly random actions is exploratory for the latent tabular MDP, then a near-optimal policy in the corresponding block MDP can be found by Fitted $Q$-Iteration (FQI), a simple dynamic programming algorithm that can be implemented with $H$ regressions on a horizon-$H$ block MDP (\cref{sec:offline}). Thus, for reinforcement learning to be harder than regression in a block MDP, it must be that the corresponding latent MDP requires directed exploration. But even this is not sufficient for showing lower bounds.

\paragraph{Actions must be non-discriminative.} The canonical example of a tabular MDP that requires directed exploration is the \emph{combinatorial lock}. This MDP has two states $\{\mathsf{success},\mathsf{failure}\}$ and two actions $\{0,1\}$ at each of $H$ steps, and is parametrized by an unknown action sequence $a^\st \in \{0,1\}^H$. The transitions are designed so that the final state of a trajectory is $\mathsf{success}$ if and only if the agent exactly follows $a^\st$. There is a reward of $1$ for reaching the $\mathsf{success}$ state at step $H$, and all other rewards are $0$.

In reward-directed RL, the obvious way to use a regression oracle is to construct a dataset where each emission is labelled by the reward at some future step. In a block MDP where the latent structure is the combinatorial lock, directed exploration is necessary for finding a near-optimal policy. Moreover, it's highly unlikely to see non-zero rewards without performing exploration. So one might expect that access to a regression oracle does not help, since the rewards do not provide useful regression targets until the MDP has already been explored. However, there is another approach to construct labels for a regression problem (and it applies to reward-free RL as well): consider the mixture dataset where each sample is obtained by playing a uniformly random action $a$ at the first step and observing the subsequent emission $x$. Let the sample be $(x,a)$, i.e. use the action as a label. Since action $a^\st_1$ leads to latent state $\mathsf{success}$ whereas action $1-a^\st_1$ leads to $\mathsf{failure}$, regression on this dataset produces a predictor that can approximately decode any given emission (up to a global renaming of the states). This contrastive learning approach leads to an oracle-efficient reinforcement learning algorithm when the latent structure is a combinatorial lock, or more generally any MDP with deterministic dynamics (\cref{sec:deterministic}).

In both of our main results, the first key idea is to disable this algorithmic approach by designing a latent MDP satisfying the following property:
\begin{definition}\label{def:open-loop-indist}
An MDP $M$ with action space $\MA$ and horizon $H$ satisfies \emph{open-loop indistinguishability} if at every step $h \in [H]$, for all action sequences $\ba,\ba' \in \MA^H$, it holds that $d_h^{M,\ba} \equiv d_h^{M,\ba'}$. That is, every fixed action sequence induces the same state visitation distribution at that step.
\end{definition}

In an MDP with open-loop indistinguishability, exploration requires policies that actually depend on the state. There are a variety of MDPs that satisfy this property; for our purposes, we use a \emph{counter MDP}, which has an additive structure that will be crucial later. At each step $h$, the counter MDP has state space $\{0,\dots,h\} \times \BF_2$ and action space $\BF_2$. In a state $(k,b)$, $k$ denotes the number of correct ``guesses'' that the agent has made so far, and $b$ denotes a fresh bit sampled uniformly at random from $\BF_2$, which the agent is trying to guess with its next action. In the reward-directed formulation, there is a reward at the final step if all guesses were correct. Due to the underlying symmetries, any two action sequences indeed induce the same state visitation distributions, so intuitively a reinforcement learning algorithm on a block MDP with this latent structure cannot use the regression oracle to ``gain a toehold'' in learning the decoding function.

\paragraph{Outline of the section.} In \cref{sec:no-reduction-overview} we give an overview of the proof of \cref{thm:no-reduction-intro}, which combines the above basic observations with a simple cryptographic primitive. In \cref{sec:overview-dynamicity} we explain the core reasons why standard cryptographic primitives do not appear to be useful for proving \cref{thm:separation-intro}. In \cref{sec:warmup-overview} we give an overview of the separation between strong reward-free RL and realizable regression, expanding upon \cref{sec:intro-warmup-overview}. In \cref{sec:full-argument-overview} we complete the overview of \cref{thm:separation-intro}. In \cref{sec:discussion} we discuss technical directions for improvement.

\subsection{Ruling out an oracle reduction}\label{sec:no-reduction-overview}


Fix a decoding function class $\Phi$, and consider the family of $\Phi$-decodable block MDPs where the latent structure is the counter MDP. The non-existence of a computationally efficient \emph{reduction} from reinforcement learning to regression (namely, \cref{thm:no-reduction-intro})  follows fairly straightforwardly from the above observations, so long as $\Phi$ is sufficiently hard in a sense that we will explain below. There are two elements to the proof:
\begin{enumerate}
    \item First, in the absence of a regression oracle, reinforcement learning in $\Phi$-decodable block MDPs with the counter MDP latent structure is computationally hard. This uses hardness of $\Phi$ and the fact that the counter MDP requires non-trivial exploration. 
    \item Second, for any efficiently computable set of samples, the optimal predictor is (with high probability over the decoding function) nearly constant. Thus, the regression oracle calls made by any efficient reduction can be replaced by an efficient algorithm that simply returns an appropriate constant function. This crucially uses that the counter MDP satisfies open-loop indistinguishability.\footnote{In fact, any latent structure that (a) satisfies open-loop indistinguishability, and (b) requires directed exploration to see non-zero rewards with non-trivial probability will suffice. Hence, we actually use a slightly simpler latent MDP \--- see \cref{sec:oracle-lb-first}.}
    \end{enumerate}
    We now expand on the two points above. Let us first suppose that, for fixed horizon $H$, the emission space $\MX$ is given by $\MX = \{0,1\}^H$, that the decoding function $\phi^\st$ is a \emph{uniformly random function} from $\MX$ to $\MS$, and that the emission distribution of each state $s \in \MS$ is chosen to be uniform over $(\phi^\st)^{-1}(s)$, i.e. the set of emissions that decode to $s$. (Recall that for the counter MDP, we have $|\MS| = O(H)$, so that each state's emission distribution will be uniform over roughly $\Theta(2^H/H)$ contexts.) In order for the resulting block MDP to be $\Phi$-decodable with probability 1, we need to take $\Phi$ to be the class of all functions from $\MX$ to $\MS$; let us assume this is so for now.

    It is fairly straightforward to see that both bullet points above are satisfied with high probability over the draw of $\phi^\st$. In particular, the first point holds since if the algorithm has observed  $2^{o(H)}$ trajectories, each observed emission gives no information about the latent state. Therefore, the algorithm cannot find a policy that exhibits a different state visitation distribution than any fixed action sequence, meaning that any policy the algorithm chooses leads to a final latent state of the form $(H,0)$ or $(H,1)$ with probability roughly $O'(2^{-H})$ (\cref{lem:pq-approx}). Similar reasoning establishes the second point above: in particular, for any procedure the algorithm uses to choose the labels of the regression problem, even ignoring computational efficiency, the resulting labels cannot be correlated with the latent state (\cref{lem:constant-fn-predictor}). 

    The above argument does not quite suffice to prove \cref{thm:no-reduction-intro}, though: the class $\Phi$ has size $O(H)^{2^H}$, meaning that $2^{\log |\Phi|} \geq 2^{2^H}$. Thus \cref{thm:no-reduction-intro} would require a lower bound of $2^{2^H}$ on the time complexity, whereas the above argument only shows that we need samples (and thus time) at least $2^H$. Moreover, there is no optimal policy that can be efficiently evaluated, since to store the decoding function we need $2^{\Omega(H)}$ space. We can fix both of these issues by replacing the choice of a random decoder $\phi^\st$ with  a \emph{pseudorandom permutation (PRP)} (\cref{def:prp}). Now, the class of decoders $\Phi$ is indexed by the choice of a seed of the PRP, and so $\log |\Phi| = \poly(H)$. Moreover, given the seed, one could efficiently compute the value of the PRP at any input, which implies that the optimal policy can be efficiently evaluated. To show that no reduction succeeds on the resulting modified block MDP, we use the argument discussed above \--- that no reduction can succeed when $\phi^\st$ is a \emph{uniformly random function} \--- to show that a successful reduction would distinguish between a PRP and a uniformly random function, thus compromising the security of the PRP in violation of \cref{asm:prf}.


\subsection{Separating RL and regression: the challenge of dynamicity}\label{sec:overview-dynamicity}

Unfortunately, the above argument does not give a construction where reinforcement learning is computationally \emph{harder} than regression (except relative to a regression oracle), because there is no non-trivial regression algorithm for that choice of $\Phi$. In fact, this is an inherent consequence of the proof structure: to compromise security of the PRP, the distinguisher simulates the hypothesized oracle-efficient RL algorithm. Hence, it must generate trajectories. It does so by generating trajectories in the known latent MDP and then using the random function oracle to generate an independent emission from each latent state. In particular, this uses the following key property $(\dagger)$:
\begin{quotation}
\begin{center}
$(\dagger)$: for any given latent state, the emission distribution can be efficiently sampled.
\end{center}
\end{quotation}
But now fix any two latent states, and for the purpose of exposition let us label them by $0$ and $1$. One can construct a regression dataset where the covariates are emissions drawn uniformly from these two latent states, and the label of an emission is its latent state. If there were a non-trivial regression algorithm, it would output a predictor that clusters the emissions according to latent state (i.e., either $0$ or $1$). Moreover, this generalizes to more than two latent states. But clustering the emissions by latent state reduces the problem of RL in the block MDP to RL in the latent MDP (up to renamings), which is tractable. Thus, any family where $(\dagger)$ holds cannot exhibit a computational separation between RL and regression. 

\paragraph{Algorithms versus hardness.} As mentioned in \cref{sec:intro-full-overview}, the challenge with proving \cref{thm:separation-intro} is that we need to construct a family of block MDPs $\MM$ satisfying two opposing goals. First, learning from trajectories should be hard, even though the algorithm is able to adaptively shape the trajectory distribution through its choice of policy in each episode. Second, learning from emissions labelled by their latent states should be easy. The tension between these goals makes it unclear how to employ standard cryptographic primitives, at least in a black-box manner. If we design an emission distribution based on an encryption scheme, the first goal may be immediate from the scheme's security guarantees, but the second goal would seem impossible due to the lack of non-trivial algorithms. With powerful primitives such as public-key encryption or fully homomorphic encryption, this problem only becomes worse: the issue is that $(\dagger)$ then holds, yielding a concrete obstruction to proving a separation.\footnote{In fact, using fully homomorphic encryption to generate trajectories from a block MDP has a second issue: the correctness guarantee of the homomorphic manipulations is insufficient. In cryptographic applications, all that matters is that after a homomorphic manipulation, the emission has the correct decoding. But for our purposes we need it to also have the correct distribution (which in FHE constructions will depend on the sequence of past manipulations, not just the latent state \--- see e.g. \cite{brakerski2014efficient}).}


For these reasons, to prove a computational separation we must open up the box and work directly with a hard learning problem. Of course, this brings us back to the challenge that reinforcement learning is a dynamic problem, whereas most hardness assumptions in computational learning theory are about static learning problems. Concretely, proving \cref{thm:separation-intro} under some hardness assumption requires identifying computational problems $\MP$ and $\MQ$, where $\MP$ is conjectured to be strictly harder than $\MQ$, and a class $\MM$ of block MDPs, where we can prove the following implications:

\begin{itemize}
    \item $\MP$ is reducible to reward-free RL in $\MM$, and
    \item $\MM$-realizable regression is reducible to $\MQ$.
\end{itemize}
Say that $\MP$ is the problem of learning a
parameter $\theta$ from independent samples from some parametric distribution $\MD_\theta$. Then, on the one hand, a reduction from $\MP$ to reinforcement learning must be able to take a sample (or several samples) from $\MD_\theta$ and use them to simulate an episode of interactive access with some MDP $M^\theta \in \MM$. On the other hand, it must not be able to simulate emissions from arbitrary latent states. Thus, simulating a trajectory in the block MDP cannot be as simple as simulating a trajectory in the latent MDP and using a fresh sample from $\MD_\theta$ to draw an emission for each latent state. Given an observation history and a new action, it must be possible to generate a new emission from the correct distribution, independent of the previous emissions (conditioned on their latent states), all without knowing the latent states.

\paragraph{Solution: additive homomorphicity.} Fully homomorphic encryption is a tempting solution to the above problem of simulated trajectories. As discussed above, it is too strong, but \emph{partial} homomorphicity turns out to be just right. In particular, we will define $\MP$ using the Learning Parity with Noise (LPN) distribution (see \cref{def:lpn}). On the one hand, this distribution has additive structure over $\BF_2$ that will prove crucial to simulating trajectories. On the other hand, while LPN is believed to be very hard, non-trivial algorithmic improvements over brute-force search are possible \cite{blum2003noise,lyubashevsky2005parity},and this is crucial to designing a non-trivial regression algorithm.

\subsection{A warm-up separation: strong reward-free RL vs regression}\label{sec:warmup-overview}

We now sketch the proof of a simpler version of \cref{thm:separation-intro}, expanding on the brief overview given in \cref{sec:intro-warmup-overview}. We explain how we can use a variant of the LPN hardness assumption to separate regression from a \emph{strong} version of reward-free RL: specifically, the problem of finding a $(1,1/4)$-policy cover (see \cref{def:pc}). For this problem, unlike standard reward-free RL, one can hope for a separation even when the horizon $H$ is a constant, so we do not need to use the counter MDP as the latent structure. Instead, we take $H = 2$ and define the latent MDP to have state space $\MS = \BF_2$ and action space $\MA = \BF_2$ at each step. The initial state distribution is $\Unif(\BF_2)$, and the dynamics are defined as
\[\BP(s_2|s_1,a) := \mathbbm{1}[s_2 = s_1 + a].\]
Note that this MDP indeed satisfies open-loop indistinguishability, since for each fixed action, the visitation distribution at each step is $\Unif(\BF_2)$. Next, for a noise level $\delta \in (0,1/2)$ and a secret parameter $\sk \in \BF_2^n$, we define the emission distribution $\MD_\sk(\cdot|s)$ at state $s \in \BF_2$ to be the distribution of the random variable
\begin{align}(u, \langle u, \sk\rangle + s + e, \Enc_{\sk}(s))\label{eq:decsk-informal}
\end{align}
where $(\Enc_{\sk},\Dec_{\sk})$ is a private-key encryption scheme that we will specify later, and $u \sim \Unif(\BF_2^n)$ and $e \sim \Ber(1/2 - \delta)$ are independent. Above, the addition is over $\BF_2$ (as will be the case throughout the remainder of this section, unless specified otherwise). Up to the choice of encryption scheme, this specifies a family of block MDPs $\MM = \{M^{\sk}: \sk \in \BF_2^n\}$ parametrized by $\sk$, where the decoding function for an emission simply applies $\Dec_{\sk}$ to the third piece of the emission. The intuition behind this construction is that $\epsilon$-accurate $\MM$-realizable regression is roughly as easy as solving LPN with noise level $1/2 - 2\epsilon \delta$, whereas strong reward-free RL in $\MM$ is as hard as solving LPN with noise level $1/2 - 2\delta^2$. Taking $\delta \ll \epsilon$, we achieve a computational separation under current beliefs about hardness of LPN. We now sketch both implications, referring to $1/2-2\epsilon\delta$ as ``low-noise'' and $1/2-2\delta^2$ as ``high-noise''.

\paragraph{Reducing regression to low-noise LPN.} From \cref{def:regression-algorithm}, any realizable regression dataset $(x^i,y^i)_{i=1}^m$ has an associated ground truth label function $f: \MS \to [0,1]$, where $f(s) := \EE[y^i|\phi^\st(x^i) = s]$ and $\phi^\st$ is the true decoding function. For regression when there are only two states (as above), there are two cases: 

\begin{enumerate}
\item On the one hand, the constant function $x \mapsto \EE[y]$ may already be an $\epsilon$-accurate predictor, in which case the constant function that outputs the sample average $x\mapsto \frac{1}{m}\sum_{i=1}^m y^i$ is $O(\epsilon)$-accurate with high probability, for sufficiently large $m$.

\item On the other hand, if the constant function is not $\epsilon$-accurate, then the label $y^i$ necessarily has correlation at least $\epsilon$ with the underlying state $\phi^\st(x^i)$. Thus, writing each emission $x^i$ as $(u^i, \langle u^i, \sk\rangle + s^i + e^i, \Enc_{\sk}(s^i))$, we can construct a dataset
\begin{align}
  \label{eq:intro-reg-lpn-dataset}(u^i, \langle u^i, \sk\rangle + s^i + e^i + y^i)_{i=1}^m.
\end{align}
Since $y^i$ is $\epsilon$-correlated with $\rho(x^i) = s^i$, the Bernoulli random variable $s^i + y^i$ is biased at least $\epsilon$ away from uniform. By the composition law for independent Bernoulli random variables (\cref{lem:bernoulli-convolve}) and the fact that $e^i$ has bias $\delta$, the bias of $s^i + e^i + y^i$ is at least $2\epsilon\delta$, so we can apply the \emph{low-noise} LPN algorithm to recover $\sk$. Using $\sk$ and $\Dec_{\sk}$, we can then estimate the true label function $f$ and define a predictor.
\end{enumerate}
Note that we may not be able to explicitly check which case holds. However, we can simply compute both predictors and choose the predictor with lower empirical error on a validation set.

\paragraph{Reducing high-noise LPN to strong reward-free RL + low-noise LPN.} For this implication, we crucially need the emissions to be additively homomorphic in the following sense: given an emission of state $s$ and a bit $b \in \BF_2$, we can efficiently construct an emission of state $s+b$. Note that the first part of the emission $(u, \langle u,\sk\rangle + s + e)$ has this property; for now we will assume that $\Enc_{\sk}$ has this property as well.

To begin, we claim that a $(1,1/4)$-policy cover enables recovering $\sk$. The reason is that it contains a policy with state visitation distribution non-trivially distinct from that of the constant policy $\pi(x) = 0$. For any two policies $\pi^0, \pi^1$ with non-trivially distinct state visitation distributions, one can generate a set of emissions $(x^i)_{i=1}^m$ from both policies, and label each emission $x^i$ by the index (in $\{0,1\}$) of the policy from which it was generated. This produces a dataset where the label correlates with the latent state. In particular, if we denote the label of emission $x^i$ by $y^i$, then we can construct a dataset of the form in \cref{eq:intro-reg-lpn-dataset}, and can use the low-noise LPN algorithm to recover $\sk$. 

Thus, it remains to argue that interactive access to $M^{\sk}$ can be efficiently simulated using samples from high-noise LPN. The first key idea is that additive homomorphicity, together with the additive structure of the latent MDP, lets us simplify this challenge of dynamic simulation to simulation of static but correlated samples. In particular, given two independent emissions from the same (unknown) state $s$, we can simulate a trajectory by passing the first emission to the RL algorithm, receiving an action $a$, and homomorphically adding $a$ to the second emission:

\[\begin{tabular}{l} $(u^1, \langle u^1, \sk\rangle + s + e^1, \Enc_{\sk}(s))$ \\  $(u^2, \langle u^2, \sk\rangle + s + e^2, \Enc_{\sk}(s))$ \end{tabular} \implies \begin{tabular}{l} $(u^1, \langle u^1, \sk\rangle + s + e^1, \Enc_{\sk}(s))$ \\  $(u^2, \langle u^2, \sk\rangle + s + a + e^2, \Enc_{\sk}(s+a))$ \end{tabular}.\]
Thus, we only need to show that we can use high-noise LPN samples to construct two independent emissions from the same state $s$, where $s \sim \Unif(\BF_2)$ (recall that $\Unif(\BF_2)$ is the initial state distribution). Let us temporarily ignore the term $\Enc_{\sk}(s)$; we will return to it once we define it properly. Then it's obvious that we could construct a pair of samples $x^1 = (u^1, \langle u^1,\sk\rangle + s + e^1)$ and $x^2 = (u^2, \langle u^2,\sk\rangle + s + e^2)$ by drawing two independent samples from $\LPN_{n,\delta}(\sk)$, sampling $s \sim \Unif(\BF_2)$, and adding $s$ to both. But we only have access to high-noise LPN samples \--- i.e. samples from $\LPN_{n,2\delta^2}(\sk)$, not samples from $\LPN_{n,\delta}(\sk)$. 

Nonetheless, a reduction is possible. 
Given a sample $(u, b)$ from $\LPN_{n,2\delta^2}(\sk)$, it can be checked that the pair
\[\begin{tabular}{c} $(u',b')$ \\  $(u' + u, b' + b)$ \end{tabular}\]
where $u' \sim \Unif(\BF_2^n)$ and $b' \sim \Unif(\BF_2)$ are independent, has exactly the same distribution as $(x^1, x^2)$.

\begin{remark}
Another piece of intuition for why it should be possible to construct $(x^1,x^2)$ from a high-noise LPN sample is the following: $s+e^1$ has bias $0$, and even after conditioning on any realization of $s+e^1$, the second noise term $s+e^2$ has bias not $\delta$ but rather $2\delta^2$. That is, the joint noise distribution $(s+e^1, s+e^2)$ is an $O(\delta^2)$-Santha-Vazirani source (\cref{def:sv-source}). This perspective will be valuable later when we discuss the proof of the separation between regression and standard reward-free RL.
\end{remark}

\paragraph{Implementing the encryption scheme.} It remains to specify the encryption scheme, which is necessary because otherwise the emissions would not satisfy the unique decodability property of block MDPs. In order for the reduction from high-noise LPN to strong reward-free RL to still go through, the scheme must satisfy additive homomorphicity as defined above, and it must be possible to generate two independent copies of $\Enc_{\sk}(s)$ jointly with $(u^1,\langle u^1,\sk\rangle+s+e^1)$ and $(u^2,\langle u^2,\sk\rangle+s+e^2)$, using high-noise LPN samples. We accomplish these goals by using an LPN-based private-key encryption scheme \cite{gilbert2008encrypt}, which for binary messages is very simple:
\begin{align*}
\Enc_{\sk}(s) &:= \begin{bmatrix} (u'_1, \langle u'_1, \sk\rangle + s + e'_1) \\ \vdots \\ (u'_N, \langle u'_N, \sk\rangle + s + e'_N)\end{bmatrix} \\ 
\Dec_{\sk}(Z) &:= \MAJ\{y'_i(Z) - \langle u'_i(Z), \sk\rangle: i \in [N]\}
\end{align*}
where $u'_1,\dots,u'_N \sim \Unif(\BF_2^n)$ and $e'_1,\dots,e'_N \sim \Ber(1/2 - 2\delta^2)$ are independent, and for an encryption $Z$, $(u'_i(Z),y'_i(Z))$ refers to the $i$-th row of $Z$. By a standard Chernoff bound, so long as $N \gg \delta^{-2} n$, the decoding error is exponentially small.\footnote{Technically, the presence of error means that $M^\sk$ is no longer a block MDP, but its trajectories are exponentially close to a block MDP for any policy, and the deviation can be handled. We defer dealing with this issue to the formal proof.}

As defined, it's clear that $\Enc_{\sk}$ satisfies the desired additive homomorphicity property. The final piece is to generalize the argument above (which ignored the encryptions) to show that we can indeed generate two independent emissions from the same state $s$, where $s \sim \Unif(\BF_2)$. Formally, using high-noise LPN samples, we need to generate a random variable of the form
\[R = \begin{bmatrix} (u_1, \langle u_1,\sk\rangle + s + e_1) \\ (u_2, \langle u_2,\sk\rangle + s + e_2) \\ (u_3, \langle u_3,\sk\rangle + s + e_3) \\ \vdots \\ (u_{2N+2}, \langle u_{2N+2}, \sk\rangle + s + e_{2N+2})\end{bmatrix}\]
where $u_1,\dots,u_{2N+2} \sim \Unif(\BF_2^n)$, $s \sim \Ber(1/2)$, $e_1,e_2 \sim \Ber(1/2-\delta)$, and $e_3,\dots,e_{2N+2} \sim \Ber(1/2-2\delta^2)$. Notice that without the first two rows, it would be trivial to generate the remaining rows using $2N$ independent samples from $\LPN_{n,2\delta^2}$, by simply drawing $s \sim \Unif(\BF_2)$ and adding it to all samples. Additionally, without the last $2N$ rows, we would know how to generate the first two rows, as argued above. The problem is that there are unknown correlations: conditioned on the first two rows, the last $2N$ rows can be written as $\Enc'_\sk(s)$ (where $\Enc'$ is the variant of $\Enc$ with $2N$ rows), but the conditional distribution of $s$ depends on $\sk$. The solution is the combination of two facts:
\begin{itemize}
    \item For any known $\alpha,\beta$, a sample $\Enc'_{\sk}(\langle \alpha,\sk\rangle + \beta)$ can be constructed from a sample $\Enc'_{\sk}(0)$ (which is just $2N$ samples from $\LPN_{n,2\delta^2}$). This fact has been previously employed to show that LPN-based encryption is secure against a limited form of \emph{key-dependent message} encryption: specifically, affine functions of $\sk$ (see Lemma 11 of \cite{applebaum2009fast}).

    \item Let $(u_1,y_1)$ and $(u_2,y_2)$ be a realization of the first two rows. Then we can randomly construct an affine function $F: \BF_2^n \to \BF_2$ so that for any fixed $\sk$, the random variable $F(\sk)$ is distributed exactly according to the conditional distribution of $s$ given $(u_1,y_1)$ and $(u_2,y_2)$. In particular, we can construct $F$ as \[F(\til\sk) := c_0 + c_1(y_1 - \langle u_1,\sk\rangle) + c_2(y_2 - \langle u_2,\sk\rangle)\] for a carefully chosen random vector $(c_0,c_1,c_2) \in \Delta(\BF_2^3)$. 
\end{itemize}
Thus, to generate $R$, we generate the first two rows $(u_1,y_1)$ and $(u_2,y_2)$, compute the function $\til \sk \mapsto F(\til\sk)$, and then generate the last $2N$ rows by drawing $2N$ additional samples from $\LPN_{n,2\delta^2}$ and using them to construct $\Enc'_{\sk}(F(\sk))$. The same ideas, albeit in greater generality, are used to prove \cref{lemma:batch-lpn}, which is crucial for the full proof of \cref{thm:separation-intro} as discussed below. This completes the proof sketch for reducing high-noise LPN to strong reward-free RL.

\subsection{The full argument: reward-free RL vs regression}\label{sec:full-argument-overview}

We now give an overview of the proof of \cref{thm:separation-intro}, which separates regression from a more standard notion of reward-free reinforcement learning: producing an $(\alpha,1/4)$-policy cover (\cref{def:pc}) with $\alpha \geq 1/\poly(H,|\MA|,|\MS|)$. Note that the construction described in \cref{sec:warmup-overview} cannot be used to prove such a separation because the block MDPs had constant horizon, and the policy that plays uniformly random actions always constitutes an $(\alpha,0)$-policy cover with $\alpha := 1/|\MA|^H$. Thus, we need a family of MDPs with longer horizon. This significantly complicates the argument, but many of the key insights from \cref{sec:warmup-overview} carry over. As before, we will reduce ``high-noise'' LPN to reinforcement learning in this family, and we will reduce regression in this family to ``low-noise'' LPN. This implies a separation under the assumption that high-noise LPN is sufficiently harder than low-noise LPN; in \cref{sec:overview-chaining} we will discuss how to obtain a separation under \cref{assm:fine-lpn}.

We define a family of block MDPs where the latent MDP is the horizon-$H$ counter MDP. To expand on the informal description given earlier, the counter MDP has latent state space $\{0,\dots,h-1\}\times\BF_2$ and action space $\BF_2$ at step $h$. The initial state distribution is $\Unif(\{(0,0),(0,1)\})$. The transitions at latent state $(k,b_h)$ and action $a_h \in \BF_2$ are defined as
\[\BP_h((k+\mathbbm{1}[b_h=a_h], 0) | (k,b_h), a_h) = \BP_h((k + \mathbbm{1}[b_h=a_h], 1) | (k,b_h), a_h) = \frac{1}{2}\]
and there are no rewards. It remains to define the emission distribution. A key building block for the emission from state $(k,b_h)$ will be the distribution $\MD_\sk(\cdot|b_h)$ from \cref{sec:warmup-overview}. Of course, an emission from this distribution would enable decoding of $b_h$ but not $k$, so additional elements are needed. The precise construction requires some care.

\paragraph{The emission distribution: a first attempt.} One of the main challenges in \cref{sec:warmup-overview} was constructing an emission distribution for each state so that the dynamic problem of simulating a trajectory could be reduced to a static problem. The key insight was that given a sample from the static distribution $\frac{1}{2}\sum_{s \in \BF_2} \MD_\sk(\cdot|s) \times \MD_\sk(\cdot|s)$ (i.e. two independent emissions of the first latent state), the given action can be incorporated into the second emission via additive homomorphicity. It's necessary to start with a pair of samples from $\MD_\sk(\cdot|s)$ (and not just one sample) since the emissions in the trajectory need to be independent conditioned on the latent states and action.

We generalize this idea to the counter MDP, starting with a plausible static distribution $\mu_\sk$ and deducing what the emission distributions should look like. In the horizon-$H$ counter MDP, a trajectory is determined by $H$ independent random bits $b_1,\dots,b_H \sim \Unif(\BF_2)$ and the agent's actions $a_1,\dots,a_H$: the state at step $h$ is then $(\sum_{i=1}^{h-1} \mathbbm{1}[b_i=a_i], b_h)$. Thus, for an emission at step $h$ to uniquely determine the latent state, it must incorporate information about all of $b_1,\dots,b_h$. This suggests that the static distribution $\mu_\sk$, conditioned on $b \sim \Unif(\BF_2^H)$, should consist of $H-h+1$ independent draws $Z_{hh},\dots,Z_{hH} \sim \MD_\sk(\cdot|b_h)$ for all $h \in [H]$:
\[\begin{array}{cccc}
Z_{11} & Z_{12} & \dots & Z_{1H} \\
& Z_{22} & \dots & Z_{2H} \\ 
& & \ddots & \vdots \\ 
& & & Z_{HH}
\end{array}\]
How do we simulate a trajectory using this data? For each $h \in [H]$, we want to use $Z_{1h},\dots,Z_{hh}$ and given actions $a_1,\dots,a_{h-1}$ to construct an emission that (a) uniquely determines the latent state $(\sum_{i=1}^{h-1} \mathbbm{1}[b_i=a_i], b_h)$, and (b) contains no other information about $b_1,\dots,b_h$. We accomplish this converting $Z_{ih}$ into a sample $Z'_{ih} \sim \MD_\sk(\cdot|b_i+a_i)$ for each $i \in [h-1]$ (using additive homomorphicity) and then \emph{randomly permuting} $Z'_{1h},\dots,Z'_{h-1,h}$. The simulated emission at step $h$ is then $(Z'_{\sigma_h(1),h},\dots,Z'_{\sigma_h(h-1),h},Z_{hh})$ where $\sigma_h: [h-1] \to [h-1]$ is a uniformly random permutation.

Since each $Z'_{\sigma_h(i),h}$ can be decoded to $b_i+a_i$, it can be seen that this emission uniquely determines the latent state $(\sum_{i=1}^{h-1} \mathbbm{1}[b_i=a_i], b_h)$. Moreover, due to the random permutation, it does not ``leak'' any other information. In all, we have described an efficient reduction that takes a sample $Z \sim \mu_\sk$ and simulates interactive access with a block MDP $M_H^\sk$ where the latent dynamics are described by the counter MDP, and the emission distribution of state $(k,b_h)$ at step $h$ is defined as follows:
\begin{enumerate}
    \item Let $b_1,\dots,b_{h-1} \in \BF_2$ be uniformly random bits conditioned on the event that $\sum_{i=1}^{h-1} \mathbbm{1}[b_i=1] = k$.
    \item Emit $(Z_1,\dots,Z_h)$ where the samples $Z_i \sim \MD_\sk(\cdot|b_i)$ are independent (conditioned on $b_1,\dots,b_{h-1}$).
\end{enumerate}

This construction does not quite work. As we will see below, reward-free RL is indeed hard for this family of block MDPs, but to ensure that regression is easy we will have to make a small but crucial modification to the construction.

\paragraph{Reducing high-noise LPN to reward-free RL.} Note that the uniformly random policy reaches latent state $(H-1,0)$ with probability $1/2^H$, but there exists a policy that reaches that state with probability $1/2 - o(1)$ (by decoding each emission and playing the appropriate action). Thus, in an $(\alpha,1/4)$-policy cover with coverage coefficient $\alpha \geq 1/\poly(H,|\MA|,|\MS|)$, for sufficiently large $H$, there is some policy that reaches the state $(H-1,0)$ with probability significantly greater than $1/2^H$. By a similar argument as in \cref{sec:warmup-overview}, solving reward-free RL in $M_H^\sk$ therefore enables recovering $\sk$. Since we have shown above that interactive access to $M_H^\sk$ can be simulated using samples from $\mu_\sk$, it only remains to show that a sample from $\mu_\sk$ can be simulated with samples from high-noise LPN. We can accomplish this using \cref{lemma:batch-lpn}, which essentially generalizes the analogous arguments from \cref{sec:warmup-overview}. Note that $\mu_\sk$ is the distribution of a list of LPN samples where the covariates are independent and uniform, but the noise terms are correlated. We can show that the joint noise distribution is a $2^{O(H^2)} \delta^2$-Santha-Vazirani source (\cref{def:sv-source}), which by \cref{lemma:batch-lpn} (and a slight variant \--- see \cref{lemma:large-batch-lpn}) implies a generic reduction from standard LPN. This completes the proof sketch of hardness of reward-free RL in $\{M^\sk_H: \sk\in\BF_2^n\}$.

\paragraph{Reducing regression to low-noise LPN?} When there are only two latent states $\{0,1\}$ (as in the warm-up construction from \cref{sec:warmup-overview}), the labels in a regression either (a) are accurately predicted by a constant function, or (b) correlate with the latent state. One could similarly hope that in a regression on the counter MDP, the optimal predictor $f(k,b_h) := \EE[y|(k,b_h)]$ is either nearly constant or correlates with either $k$ or $b_h$. In the latter two cases (i.e., correlated with either $k$ or $b_h$), given an emission $Z$ and a label $y$, one could then generate a low-noise LPN sample by adding $y$ to either the first row $Z_1$ of $Z$, which would induce the random variable $(u, \langle u,\sk\rangle + b_1 + e + y)$ where $u \sim \Unif(\BF_2^n)$, $e \sim \Ber(1/2-\delta)$, and $b_1 \sim \Ber(k/(h-1))$; or the last row $Z_h$ of $Z$, which would induce the random variable $(u, \langle u,\sk\rangle + b_h + e + y)$. As before, this would enable recovering $\sk$ and solving the regression that way. Unfortunately, since $f$ has $2h$ degrees of freedom, for any $h>1$ it's possible that $f$ could be non-constant but also uncorrelated with $k$ and $b_h$. Thus, we need more test variables.

Fix any distribution $\beta$ over latent states and consider drawing $(k,b_h) \sim \beta$. Let $b_1,\dots,b_{h-1} \in \BF_2$ be uniform conditioned on the event that $\sum_{i=1}^{h-1} \mathbbm{1}[b_i=1] = k$. It turns out that for every function $f: \{0,\dots,h-1\}\times\BF_2 \to [0,1]$, the random variable $f(k,b_h)$ is either near-constant or has non-trivial correlation with $\sum_{i=1}^m b_i + rb_h$ for some $(m,r) \in \{0,\dots,h-1\} \times \{0,1\}$ with $(m,r) \neq (0,0)$ (see \cref{lemma:f-corr} for the precise statement). Thus, one might hope to construct a low-noise LPN sample from an emission $Z$ and label $y$ by guessing a ``good'' tuple $(m,r)$, and then adding the label $y$ to the aggregated sample $\sum_{i=1}^m Z_i + rZ_h$. Formally, if we write $Z_i = (u_i, \langle u_i, \sk\rangle + b_i + e_i, \Enc_\sk(b_i))$ for each $i$, then we could construct the LPN sample
\[\left(\sum_{i=1}^m u_i + ru_h, \left\langle \sum_{i=1}^m u_i + ru_h, \sk\right\rangle + \sum_{i=1}^m b_i + rb_h + y + \sum_{i=1}^m e_i + re_h\right).\]
Notably, if $(m,r)$ is a ``good'' tuple, then $\sum_{i=1}^m b_i + rb_h + y$ has non-trivial bias. The problem is that we are now adding multiple independent noise terms $e_i$. The bias of $\sum_{i=1}^m e_i + re_h$ will be roughly $O(\delta^{m+r})$, so the above sample will not be low-noise unless $m+r=1$. This is a fatal flaw of the emission distribution as defined above.

\paragraph{Fixing the emission distribution.} The solution is to redefine the emission distribution so that the noise terms are slightly \emph{correlated}. Concretely, for any $n \in \NN$ and $\delta \in (0,1/2)$, there is a random variable $X \sim \CBer(n,\delta)$ on $\BF_2^n$ so that for any non-empty set $S \subseteq [n]$, the sum $\sum_{i \in S} X_i$ has distribution $\Ber(1/2-\delta)$ (see \cref{def:cber}). Now, we redefine the emission distribution of latent state $(k,b_h)$ at step $h$ as follows:
\begin{enumerate}
    \item Let $b_1,\dots,b_{h-1} \in \BF_2$ be uniformly random bits conditioned on the event that $\sum_{i=1}^{h-1}\mathbbm{1}[b_i=1]=k$. Also let $e \sim \CBer(h, \delta)$ and $u_1,\dots,u_h \sim \Unif(\BF_2^n)$.
    \item For each $i \in [h]$, define 
    \[Z_i = (u_i, \langle u_i,\sk\rangle + b_i + e_i, \Enc_\sk(x_i)),\]
    and emit $(Z_1,\dots,Z_h)$.
\end{enumerate}
Compared to the previous definition of the emissions, the only change is in the distribution of $e$. This change enables aggregating rows of $Z$ without blowing up the noise. Additionally, the reduction from high-noise LPN to reward-free RL still works. Simulating a trajectory can still be reduced to drawing a sample from a (slightly redefined) static distribution $\mu_\sk$, where the columns $(Z_{1h},\dots,Z_{hh})$ now have correlated noise (\cref{def:mu-defn}). The key point is that after averaging over the choice of $b_1,\dots,b_H \sim \Unif(\BF_2)$, the correlations between noise terms can still be bounded by $2^{O(H^2)} \delta^2$ (\cref{lemma:z-near-unif}), so our generic reductions from standard LPN still apply. This completes the proof of a separation between reward-free RL and regression, assuming a separation between low-noise and high-noise LPN.

\subsubsection{Parameter choices, chaining, and technical caveats.}\label{sec:overview-chaining}

We now make the above discussion more quantitative. For any functions $\epsilon,\delta: \NN \to (0,1/2)$ and $H,N: \NN \to \NN$ with $N \gg \delta(n)^{-2}n$, the above arguments show that there is a family of block MDPs, indexed by $n \in \NN$, with the following two guarantees:
\begin{itemize}
    \item First, $\epsilon(n)$-accurate realizable regression is as easy as learning $n$-variable parities with noise level $1/2 - \delta(n)\epsilon(n)/2^{O(H(n)^2)}$ (\cref{lemma:realizable-regression-alg}).
    \item Second, learning $n$-variable parities with noise level $1/2 - 2^{O(H(n)^2)} \delta(n)^2$ is as easy as learning an $(\alpha(n),1/4)$-policy cover where $\alpha(n) \geq 1/\poly(H(n))$, in conjunction with learning $n$-variable parities with noise level $1/2 - \delta(n)/2^{O(H(n)^2)}$, up to a multiplicative factor of $\poly(n,\delta(n)^{-1})$ (\cref{lemma:policy-cover-to-lpn}).
\end{itemize}
For what parameter choices do these guarantees yield a plausible separation? In both, there is a factor exponential in $H(n)^2$, which arises from the lower bound on the correlation of any non-constant regression label function (\cref{lemma:f-corr}) as well as the reduction from standard LPN to batch LPN (\cref{lemma:construct-corr-lpn}). In the first guarantee, there is a factor of $\epsilon(n)$, but one should think of $\epsilon(n)$ as inverse-polynomial in the size of the latent MDP, so this term is dominated by the factor $2^{-O(H(n)^2)}$. To summarize, let us suppose for the sake of contradiction that reward-free RL is as easy as regression. Then the above arguments, in the regime $2^{O(H(n)^2)} \ll 1/\delta(n)$, suggest that learning parities with noise level $1/2-\delta(n)^{2-o(1)}$ would be as easy as learning parities with noise level $1/2-\delta(n)$, up to a multiplicative factor of $\poly(n,\delta^{-1})$.

A priori this might actually seem plausible, since for many statistical problems, the time and sample complexity of learning with noise level $1/2-\delta$ scales polynomially in $\delta^{-1}$. But if the above were true for LPN, it would actually imply an improvement over the seminal $\BKW$ algorithm \cite{blum2003noise} in the regime $\delta \leq 2^{-cn/\log n}$: starting with $\delta_1(n) := 2^{-n^{3/4}}$, a regime where $\BKW$ has time complexity $2^{O(n/\log n)}$, we could chain together the reductions to get an improved algorithm for $\delta_2(n) := \delta_1(n)^{2-o(1)}$, and then $\delta_3(n) := \delta_1(n)^{(2-o(1))^2}$, and subsequently $\delta_k(n) := \delta_1(n)^{(2-o(1))^k}$ for each $k \in \NN$. Since the time complexity increases multiplicatively by only $\poly(n,\delta^{-1})$ at step $k$, eventually we would end up with an algorithm for learning parities with noise level $1/2 - 2^{O(n/\log n)}$ in time $2^{O(n/\log n)}$. But $\BKW$ requires time complexity $2^{O(n/\log \log n)}$ in this regime, and such an improvement seems unlikely.

This is the rough blueprint of the proof of \cref{thm:separation-intro}, which we formalize as \cref{thm:main-separation}. However, the formal proof must also address two subtle technical caveats:

\paragraph{Unknown noise level.} The reduction from high-noise LPN to reward-free RL requires knowing the noise level, in order to construct the appropriate batch LPN distribution (\cref{lemma:triangle-lpn}). However, the reduction from regression to low-noise LPN produces an instance of LPN where \emph{the noise level is unknown}. We know that it is at most $1/2-\delta(n)\epsilon(n)/2^{O(H(n)^2)}$, but we cannot compute it exactly. For natural learning algorithms, this does not seem like an obstacle \--- e.g., $\BKW$ performs just as well when we know that the noise has bias at least $\delta$ as when it has bias exactly $\delta$ \--- but there is no obvious tight (generic) \emph{reduction} from the unknown noise level case to the known noise level case.

There is a simple reduction that blows up the time complexity of learning by a factor linear in the \emph{sample complexity} of learning (\cref{lemma:lpn-unknown-delta}). Unfortunately, this breaks the above chaining argument, since the sample complexity of $\BKW$ is roughly as large as its time complexity, so step $i$ of chaining multiplies the time complexity by $2^{O(n/\log n)}$ rather than $\poly(n,\delta_i^{-1})$. We fix this issue by using a sample-efficient algorithm for learning noisy parities, due to Lyubashevsky \cite{lyubashevsky2005parity}. This algorithm uses only $\poly(n)$ samples and achieves time complexity $2^{O(n/\log \log n)}$ whenever $\delta \geq 2^{-(\log n)^{1-\Omega(1)}}$. Suppose that the hypothesized time complexity equivalence between regression and reward-free RL extends to an approximate sample complexity equivalence (as made formal in the statement of \cref{thm:main-separation}). Then if we redefine $\delta_1(n) := 2^{-(\log n)^{3/4}}$ (and $H(n) := (\log n)^{3/4}$), the above chaining argument gives an algorithm for learning parities with noise level $1/2-2^{O(n/\log \log n)}$ in time $2^{O(n/\log \log n)}$. This again would improve upon the best-known algorithm, and \cref{assm:fine-lpn} codifies the belief that such an improvement is impossible.

\paragraph{Non-uniform computation.} The second caveat with the chaining argument is that it produces a sequence of algorithms (i.e. Turing Machines) $(\Alg_k)_{k \in \NN}$ where for each $k$, we know that $\Alg_k$ efficiently learns $n$-variable parities with noise level $1/2 - \delta_1(n)^{(2-o(1))^k}$. But there is no fixed $k$ for which $\Alg_k$ learns $n$-variable parities with noise level $1/2-2^{O(n/\log \log n)}$; for this, we must take $k$ to grow with $n$. Thus, the outcome of the chaining argument is a \emph{non-uniform} algorithm $\Alg^\st$ for LPN. Under the additional assumption that the hypothesized equivalence between regression and reward-free RL extends to an approximate equivalence between the description complexities of the respective algorithms (see \cref{thm:main-separation} for the formal statement), $\Alg^\st$ can be shown to be implementable with $\poly(n)$ advice. This then violates \cref{assm:fine-lpn}. Although we had to strengthen the assumption to rule out non-uniform algorithms, we do not see this as a significant modification.

\subsection{Technical limitations and future directions}\label{sec:discussion}

While \cref{thm:separation-intro} provides the first computational separation between regression and RL, there are a variety of quantitative and qualitative ways in which the result could be strengthened. In this section we enumerate the most salient directions for improvement, and discuss the technical limitations of the current approach.

\paragraph{Reward-directed vs reward-free.} \cref{thm:separation-intro} only provides a separation against reward-free RL. The reason is that the horizon $H(n)$, while super-constant, is fairly small compared to the overall time complexities at play. In particular, reward-directed RL is always possible with roughly $|\MA_n|^{O(H(n))}$ time and samples, and for our construction $|\MA_n|^{H(n)} = 2^{(\log n)^{1/3}}$ is much smaller than $2^{O(n/\log n)}$, which is the best-case time complexity of LPN (and hence our regression algorithm). Extending our approach to larger horizons would seem to require removing the factors exponential in $H(n)$ in e.g. \cref{lemma:realizable-regression-alg,lemma:policy-cover-to-lpn}. 

As a reminder, \cref{thm:no-reduction-intro} does rule out a \emph{reduction} from reward-directed RL to regression, but it does not establish a separation.

\paragraph{Distributional assumptions.} In our definition of realizable regression (\cref{def:regression-algorithm}), the covariate distribution must be a mixture of the emission distributions of latent states. One could ask whether this restriction could be removed, which would allow phrasing the separation purely in terms of a concept class $\Phi$ rather than a family of block MDPs (\cref{rmk:m-vs-phi}). Unfortunately, our regression algorithm seems to rely strongly on this distributional assumption, even in the toy construction (\cref{sec:warmup-overview}): if the covariate distribution were arbitrary, and all we knew was that the label $y$ for an emission $x$ has law only depending on $\phi(x)$ for some $\phi \in \Phi$, then by definition of $\Phi$ as $\{\Dec_\sk:\sk\in\BF_2^n\}$ we would know that the third piece of $x$ (in, e.g., \cref{eq:decsk-informal}) is $\Enc_\sk(\phi(x))$, but we would have no guarantees about the first two pieces of $x$, which in the current approach contain crucial information for our regression algorithm \--- namely, an LPN sample masked with $\phi(x)$.

\paragraph{Quantitative improvements.} Finally, another limitation of our construction is that both regression and RL are quite hard, and the gap is relatively quite small. Thus, it does not rule out the possibility that e.g. the time complexity of RL is at most quadratic in the time complexity of regression. Perhaps the strongest result one could hope for is a construction where regression is solvable in polynomial time, but RL is cryptographically hard. One of the obstacles to proving such a result is that it significantly limits the applicable algorithmic toolkit: there are very few rich concept classes $\Phi$ where regression over $\Phi$ is tractable in polynomial time.


\section{The block counter MDP construction}\label{sec:construction}
Fix $n,N,H \in \NN$, $\delta \in (0,1/2)$, and $\sk \in \BF_2^n$. We define a block MDP $M^\sk_{n,N,H,\delta}$ as follows. The horizon is $H$ and the latent state space is $\MS := \sqcup_{1 \leq h \leq H} \MS[h]$ where 
\[\MS[h] := \{h\} \times \{0,\dots,h-1\} \times \BF_2.\]
The action space is $\MA := \BF_2$, the initial latent state distribution is $\Unif(\MS[1])$, and for each $h \in [H-1]$ the latent transition distribution at state $(h,k,b_h) \in \MS[h]$ with action $a \in \MA$ is defined by
\begin{align}
  \label{eq:latent-transitions}\til \BP_h((h+1,k + \mathbbm{1}[b_h=a], 0)|(h,k,b_h),a) = \til \BP_h((h+1,k + \mathbbm{1}[b_h=a], 1)|(h,k,b_h),a) = \frac{1}{2}.
\end{align}
The emission distribution at state $s = (h,k,b_h) \in \MS$ is $Z \sim \MD^\sk_{n,N,h,\delta}(\cdot|h,k,b_h)$ defined below in \cref{def:real-emission}. Note that the set of reachable latent states at step $h$ is exactly $\MS[h]$, and the step $h$ is efficiently decodable from (the dimension of) any emission in the support of $\MD^\sk_{n,N,h,\delta}(\cdot|k,b_h)$. There is a natural identification of the emission space $\sqcup_{1 \leq h \leq H} \MZ_{n,N,h}$ (see \cref{def:emission}) within $\{0,1\}^{H^2(n+1)(N+1)}$.

\begin{definition}\label{def:nu}
Fix $h \in \NN$, $k \in \{0,\dots,h-1\}$, and $b_h \in \BF_2$. Let $\nu_{h,k,b_h}$ be the distribution of a uniformly random vector $B$ from $\BF_2^h$ subject to the constraints $\sum_{i=1}^{h-1}\mathbbm{1}[B_i=1] = k$ and $B_h = b_h$.
\end{definition}

\begin{definition}
  \label{def:cber}
For any $n \in \NN$ and $\delta \in [0,1/2]$ we define the $(n,\delta)$-\emph{correlated Bernoulli} distribution $\CBer(n,\delta) \in \Delta(\BF_2^n)$ by $2\delta \Ber(0)^{\otimes n} + (1-2\delta) \Ber(1/2)^{\otimes n}$. 
\end{definition}

\begin{definition}
  \label{def:emission}
Fix $n,N,h \in \NN$, $\sk \in \BF_2^n$, and $\delta \in(0,1/2)$. For $k \in \{0,1,\dots,h-1\}$ and $b_h \in \BF_2$, let $\til\MD_{n,h,\delta}^{\sk,0}(\cdot|h,k,b_h)$ 
be the distribution of the random matrix
\[
\begin{bmatrix} u_1 & \langle u_1,\sk\rangle + e_1 + b_1 \\
\vdots & \vdots \\
u_h& \langle u_h, \sk\rangle + e_h + b_h
\end{bmatrix} \in \BF_2^{h \times (n+1)}
\]
where $u_1,\dots,u_h \sim \Unif(\BF_2^n)$, $e = (e_1,\dots,e_h) \sim \CBer(h, \delta)$, and $b = (b_1,\dots,b_h) \sim \nu_{h,k,b_h}$ are independent.\footnote{There is a slight abuse of notation in writing $b \sim \nu_{h,k,b_h}$ since $b_h$ has two meanings. However, the $h$-th entry of the random variable $b$ is always equal to $b_h$ under the distribution $\nu_{h,k,b_h}$, so the notation is consistent.} Also let $\til\MD_{n,N,h,\delta}^\sk(\cdot|h,k,b_h)$ be the distribution of the random object
\[
\Xi := \begin{bmatrix} u_1 & \langle u_1,\sk\rangle + e_1 + b_1 & (u'_{1i}, \langle u'_{1i},\sk\rangle + e'_{1i} + b_1)_{i=1}^N \\
\vdots & \vdots & \vdots \\
u_h & \langle u_h, \sk\rangle + e_h + b_h &(u'_{hi}, \langle u'_{hi},\sk\rangle + e'_{hi} + b_h)_{i=1}^N \end{bmatrix}
\]
where $u_1, \ldots, u_h \sim \Unif(\BF_2^n), e = (e_1,\dots,e_h) \sim \CBer(h, \delta), b = (b_1,\dots,b_h) \sim \nu_{h,k,b_h}$, and $e_{ji}' \sim \Ber(1/2 - 2\delta^2)$ (for $1 \leq j \leq h, 1 \leq i \leq N$) are all independent. Let $\MZ_{n,N,h}$ be the set of objects of the same shape as $\Xi$ (note that $\MZ_{n,N,h}$ can be identified with $\BF_2^{h \times (n+1) \times (N+1)}$, although we have written $\Xi$ in the above form to highlight that the first slice along the third axis differs from the rest).
\end{definition}

\begin{definition}[Emission components]\label{def:emission-components}
Fix $n,N,h \in \NN$. We define functions $u_j, u'_{ji}: \MZ_{n,N,h} \to \BF_2^n$ and $y_j, y'_{ji}: \MZ_{n,N,h} \to \BF_2$, for all $j \in [h]$ and $i \in [N]$, as follows. For each $Z \in \MZ_{n,N,h}$, we write
\[Z =: \begin{bmatrix} u_1(Z) & y_1(Z) & (u'_{1i}(Z), y'_{1i}(Z))_{i=1}^N \\
\vdots & \vdots & \vdots \\ 
u_h(Z) & y_h(Z) & (u'_{hi}(Z), y'_{hi}(Z))_{i=1}^N
\end{bmatrix}. \]
\end{definition}

\begin{definition}[Decoding function]
  \label{def:decoder}
Let $n,N,h \in \NN$ and $t \in \BF_2^n$. We define the decoding function $\Dec_{n,N,h}^t: \MZ_{n,N,h} \to \{h\} \times \{0,\dots,h-1\}\times\BF_2$ by
\begin{align*}
\Dec_{n,N,h}^t(Z) := \Bigg(h, \,\,&\sum_{i=1}^{h-1} \mathbbm{1}[\MAJ\{y'_{ij}(Z) - \langle u'_{ij}(Z), t\rangle: j \in [N]\} \equiv 1 \bmod{2}], \\ 
&\MAJ\{y'_{h+1,j}(Z) - \langle u'_{h+1,j}(Z), t\rangle: j \in [N]\}\Bigg)
\end{align*}
where $\MAJ: \BF_2^N \to \BF_2$ is the majority function (breaking ties to $0$ for concreteness). We define $\Dec^t_{n,N}$ to be the function with domain $\sqcup_{1 \leq h \leq H} \MZ_{n,N,h}$ that applies $\Dec^t_{n,N,h}$ to any element of $\MZ_{n,N,h}$ (where $H$ will be clear from context). 
\end{definition}

\begin{definition}[Emission distribution]\label{def:real-emission}
Let $n,N,h \in \NN$, $\delta \in (0,1/2)$, and $\sk \in \BF_2^n$. For $k \in \{0,\dots,h-1\}$ and $b_h \in \BF_2$, let $\MD^\sk_{n,N,h,\delta}(\cdot|h,k,b_h)$ be the distribution on $\MZ_{n,N,h}$ with probability mass function 
\[\MD^\sk_{n,N,h,\delta}(z | h,k,b_h) := \Pr_{Z \sim \til\MD^\sk_{n,N,h,\delta}(\cdot|h,k,b_h)}\left( Z =z \mid \Dec^\sk_{n,N,h}(Z) = (h,k,b_h)\right),\]
recalling the definition of $\Dec^\sk_{n,N,h}(\cdot)$ in \cref{def:decoder}.
\end{definition}

For notational convenience, also define an MDP $\til M^\sk_{n,N,H,\delta}$ identically as above, except that the emission distribution at step $h$ and state $s \in \MS[h]$ is $Z \sim \til\MD^\sk_{n,N,h,\delta}(\cdot|s)$ (with no conditioning). Note that $\til M^\sk_{n,N,H,\delta}$ is not a block MDP, since it does not have a decoding function (i.e., distinct states do not have emission distributions with disjoint supports). However, as formalized in the following lemma, its trajectories are close in total variation distance to those of $M^\sk_{n,N,H,\delta}$, so long as $N$ exceeds a (fixed) polynomial in $\delta^{-1}$.

\begin{lemma}[Ground truth decoding error]\label{lemma:decoding-error-prob}
Let $n,N,H \in\NN$, $\delta\in(0,1/2)$, and $\sk \in \BF_2^n$. For any latent state $(h,k,b_h) \in \MS$, the unconditioned emission $Z \sim \til\MD^\sk_{n,N,h,\delta}(\cdot|h,k,b_h)$ satisfies
\begin{align}\Pr[\Dec^\sk_{n,N,h}(Z) \neq (h,k,b_h)] \leq h\exp(-\delta^4 N).\label{eq:decoding-close}
\end{align}
Thus, for any policy $\pi$, for trajectories $\tau, \til\tau$ drawn with policy $\pi$ from $M^\sk_{n,N,H,\delta}$ and $\til M^\sk_{n,N,H,\delta}$ respectively, it holds that
\[\TV(\Law(\tau), \Law(\til\tau)) \leq H^2\exp(-\delta^4 N).\]
\end{lemma}

\begin{proof}
To prove the first claim, recall from \cref{def:emission} that there are independent random variables $b \sim \nu_{h,k,b_h}$ and $e'_{ji} \sim \Ber(1/2-2\delta^2)$ (for $(j,i) \in [h]\times[N]$) so that $y'_{ji}(Z) = \langle u'_{ji}(Z),\sk\rangle + e'_{ji} + b_j$ for all $(j,i) \in [h]\times[N]$. Now for each $j \in [h]$ consider the event $\ME_j$ that
\[\#\{i \in [N]: y'_{ji}(Z) \neq \langle u'_{ji}(Z), \sk\rangle + b_j\} \geq \frac{N}{2}.\]
This is exactly the event that $\#\{i \in [N]: e'_{ji} = 1\} \geq N/2$, which by Hoeffding's inequality occurs with probability at most $\exp(-\delta^4 N)$. By the union bound, the event $\cup_{i \in [h]} \ME_i$ occurs with probability at most $h\exp(-\delta^4 N)$. Moreover, in the complement of this event, using \cref{def:decoder} and the fact that $\sum_{i=1}^{h-1} b_i = k$ with probability $1$, it holds that $\Dec^\sk_{n,N,h}(Z) = (h,k,b_h)$.

To prove the second claim, note that a trajectory $\tau$ from $M^\sk_{n,N,H,\delta}$ with policy $\pi$ can be drawn by repeatedly sampling a trajectory $(s_{1:H}, x_{1:H}, a_{1:H})$ from $\til M^\sk_{n,N,H,\delta}$ with policy $\pi$ (including the latent states in the trajectory for the purposes of the argument) until it holds that $\Dec^\sk_{n,N,h}(x_h) = s_h$ for all $h \in [H]$. The total variation distance between $\Law(\tau)$ and $\Law(\til \tau)$ can be upper bounded by the probability that this process for generating $\tau$ resamples at least once. By the first claim of the lemma and a union bound over $h$, this occurs with probability at most $H^2 \exp(-\delta^4 N)$.
\end{proof}

The following (straightforward) technical lemmas will be useful later.

\begin{lemma}[Approximate realizability]\label{lemma:realizability-error-decomp}
Let $n,N,H,c \in \NN$, $h \in [H]$, $\delta\in(0,1/2)$, and $\sk \in \BF_2^n$. Let $\MS[h] = \{0,\dots,h-1\}\times\BF_2$ and let $\beta \in \Delta(\MS[h])$. Let $(\pi^i)_{i\in[c]}$ be policies. Let $P$ be a stochastic transformation, and consider trajectories $\tau^i \sim \BP^{M^\sk_{n,N,H,\delta},\pi^i}$ and $\til\tau^i \sim \BP^{\til M^\sk_{n,N,H,\delta},\pi^i}$ for $i \in [c]$. Then for any function $f: \MS[h]\to[0,1]$,
\begin{align*}
&\TV\left(\Law(P((\til\tau^i)_{i\in[c]})), \sum_{s\in\MS[h]} \beta(s) \cdot \left(\til\MD^\sk_{n,N,h,\delta}(\cdot|s) \times \Ber(f(s))\right)\right)\\
&\leq \TV\left(\Law(P((\tau^i)_{i\in[c]})), \sum_{s\in\MS[h]} \beta(s) \cdot \left(\MD^\sk_{n,N,h,\delta}(\cdot|s) \times \Ber(f(s))\right)\right) \\ 
&+ (c+1)H^2\exp(-\delta^4 N)
\end{align*}
\end{lemma}

\begin{proof}
Immediate from \cref{lemma:decoding-error-prob} together with the data processing inequality and the triangle inequality for total variation distance.
\end{proof}

\begin{lemma}[Optimal policy]\label{lemma:optimal-policy}
Let $n,N,H \in\NN$, $\delta \in (0,1/2)$, and $\sk \in \BF_2^n$. Then there is a policy $\pi^\st$ such that
\[\Pr^{M^\sk_{n,N,H,\delta},\pi^\st}[s_H=(H,H-1,0)] =  \Pr^{M^\sk_{n,N,H,\delta},\pi^\st}[s_H = (H,H-1,1)] = \frac{1}{2}.\]
\end{lemma}

\begin{proof}
The policy $\pi^\st$ on emission $Z$ at step $h$ computes $(h,k,b_h) \gets \Dec^\sk_{n,N,h}(Z)$ and plays action $b_h$. Inductively we can see that for each $h \in [H]$, the state visitation distribution at step $h$ is uniform on $(h,h-1,0)$ and $(h,h-1,1)$. 
\end{proof}

For any function $n \mapsto \delta(n)$, we define a family of block MDPs indexed by $n$, as follows.

\begin{definition}\label{def:mdp-family}
Fix $\delta: \NN \to (0,1/2)$. We define $\MM(\delta)$ as the tuple
\[\MM(\delta) := ((\MS_n)_n,(\MA_n)_n,(H_n)_n,(\ell_n)_n,(\Phi_n)_n,(\MM_n)_n)\]
where $H_n := (\log n)^{1/3}$, $N_n := 3\delta(n)^{-4} n$, $\ell_n := H_n^2(n+1)(N_n+1)$, and
\begin{itemize}
    \item $\MS_n := \sqcup_{1 \leq h \leq H_n} \{h\} \times \{0,\dots,h-1\} \times \BF_2$
    \item $\MA_n := \BF_2$
    \item $\Phi_n := \{\Dec^\sk_{n,N_n}: \sk \in \BF_2^n\}$
    \item $\MM_n := \{M^\sk_{n,N_n,H_n,\delta(n)}: \sk \in \BF_2^n\}$.
\end{itemize}
\end{definition}

The following lemma is immediate from \cref{def:mdp-family}.

\begin{lemma}\label{lemma:family-is-computable}
There is a constant $C_{\ref{lemma:family-is-computable}}>0$ with the following property. For any $\delta: \NN \to (0,1/2)$, $\MM(\delta)$ is a $(n/\delta(n))^{C_{\ref{lemma:family-is-computable}}}$-computable block MDP family indexed by $n$ (per \cref{def:computable-family}).
\end{lemma}

\begin{proof}
It's straightforward to check that $\MM(\delta)$ is a block MDP family indexed by $n$ (\cref{def:block-family}). It remains to check the two conditions of \cref{def:computable-family}. For the first condition, the dominating parameter is $\ell_n$, which is indeed at most $\poly(n,\delta(n)^{-1})$. For the second condition, observe that the map $(\sk,Z) \mapsto \Dec^\sk_{n,N_n}(Z)$ indeed has a circuit of size $\poly(n,\delta(n)^{-1})$ by definition of $\Dec$ (\cref{def:decoder}).
\end{proof}

The main results of \cref{sec:regression-to-lpn} and \cref{sec:lpn-to-rl} respectively show that for any $\delta(n) \leq O(2^{-(\log n)^{1-\gamma}})$, regression for $\MM(\delta)$ is roughly as easy as learning parities with noise $1/2 - \delta$ (\cref{lemma:realizable-regression-alg}) whereas policy cover learning for $\MM(\delta)$ is roughly as hard as learning parities with noise $1/2 - 2\delta^2$ (\cref{lemma:policy-cover-to-lpn}). If one is willing to make a fine-grained assumption that the latter LPN problem is significantly harder than the former, then this already gives a computational separation. However, by chaining together the reductions and applying Lyubashevsky's algorithm for the base case \cite{lyubashevsky2005parity}, we can alternatively prove a computational separation under the straightforward hardness assumption that learning parities with noise level $1/2 - 2^{O(n/\log \log n)}$ cannot be done in time $2^{O(n/\log \log n)}$ (\cref{assm:fine-lpn}). This is accomplished in the following theorem, which is the formal version of \cref{thm:separation-intro}.

\begin{theorem}\label{thm:main-separation}
Suppose that \cref{assm:fine-lpn} holds. Then for every constant $C$, there are functions $K,S,T,B:\NN \to \NN$; a $K(n)$-computable block MDP family (\cref{def:computable-family}) $\MM = ((\MS_n)_n,(\MA_n)_n,(H_n)_n,(\ell_n)_n,(\Phi_n)_n,(\MM_n)_n)$; and an integer $D \in \NN$ so that:
\begin{itemize}
\item There is a $(S,T,\epsilon,B)$-realizable regression algorithm (\cref{def:regression-algorithm}) for $\MM$ where $\epsilon:\NN\to\RR_{>0}$ satisfies $\epsilon(n) \leq 1/L_n^C$ for all $n$, and the algorithm has description complexity $D$, and
\item For all functions $S', T', B': \NN \to \NN$ and $\alpha: \NN \to \RR_{>0}$ that satisfy
\begin{align}
\begin{split}
S'(n) &\leq K(n)^C S(n) \\
T'(n) &\leq K(n)^C T(n) \\
\alpha(n) &\geq 1/L_n^C \\ 
B'(n) &\leq K(n)^C B(n),\label{eq:resource-ineq}
\end{split}
\end{align}
there is \emph{no} $(S',T',\alpha,B')$-policy cover learning algorithm (\cref{def:rf-rl}) for $\MM$, with description complexity at most $C \cdot D$,
\end{itemize}
where $L_n := H_n |\MA_n| |\MS_n|$.
\end{theorem}

Recall that $S$ and $T$ measure the sample and time complexity, respectively, of the regression algorithm, and $\epsilon$ measures the accuracy of the output predictor. Similarly, $S'$ and $T'$ measure the sample and time complexity of the policy cover learning algorithm, and $\alpha$ measures the approximation factor of the output policy cover. The functions $B$ and $B'$ measure the description complexity of the circuits produced by the regression algorithm and RL algorithm respectively, and $D$ measures the description complexity of the algorithms themselves (\cref{def:desc-complexity}); while imposing such bounds is necessary for complete rigor, we do not see them as morally important to the result.

Quantitatively, the important parameter in the above statement is $L_n$, which measures the size of the latent MDPs. Thus, the theorem is contrasting regression with inverse-polynomial accuracy against reward-free RL with inverse-polynomial coverage. In words, \cref{thm:main-separation} essentially asserts that there is a family of block MDPs $\MM$ where reward-free RL with inverse-polynomial coverage requires either \emph{strictly more samples} or \emph{strictly more time} than regression with inverse-polynomial accuracy.

\begin{remark}[Why is this a computational separation?]
The either/or guarantee of \cref{thm:main-separation} may seem unsatisfying: the goal was to prove a computational separation, and at first glance the guarantee of \cref{thm:main-separation} looks as though it could be proved via a purely information-theoretic argument (i.e. removing the requirement $T'(n) \leq K(n)^C T(n)$ in \cref{eq:resource-ineq}). This is not the case, since for \emph{every} $K(n)$-computable block MDP family $\MM$, there is a policy cover learning algorithm with statistical complexity $S'(n)$ at most $\poly(K(n))$, and with $\alpha(n) \geq 1/\poly(L_n)$. Thus, computation is indeed inherent to \cref{thm:main-separation}.

Nonetheless, it is natural to ask whether it's possible to remove the restriction on $S'(n)$, which only arises for a subtle technical reason \--- see \cref{sec:overview-chaining} for discussion of this reason.
\end{remark}

To prove \cref{thm:main-separation}, we will show that there is some $\delta(n)$ so that the block MDP family $\MM(\delta)$ (\cref{def:mdp-family}) satisfies the above criteria. 

\begin{namedproof}{\cref{thm:main-separation}}
Let $H_n$, $\MS_n$, $\MA_n$, and $\Phi_n$ be as defined in \cref{def:mdp-family} for the block MDP family $\MM(\delta)$ (note that none of these quantities depend on the choice of function $\delta$). Setting $L_n := H_n|\MA_n||\MS_n|$ as in the theorem statement, we have that $L_n \leq 4\log n$ for all $n \in \NN$. Also recall from \cref{lemma:family-is-computable} that for any function $\delta$, $\MM(\delta)$ is $(n/\delta(n))^{C_{\ref{lemma:family-is-computable}}}$-computable.

Suppose that there is a constant $C$ so that for every function $\delta(n)$, for any functions $S,T,B:\NN\to\NN$ so that there is a $(S, T, 1/L_n^C, B)$-realizable regression algorithm for $\MM(\delta)$, there is also a $(S',T',\alpha,B')$-policy cover learning algorithm for $\MM(\delta)$ satisfying the set of inequalities \cref{eq:resource-ineq} with $K(n) := (n/\delta(n))^{C_{\ref{lemma:family-is-computable}}}$, and with at most $C$ times the description complexity (\cref{def:desc-complexity}) of the regression algorithm. We will derive a contradiction to \cref{assm:fine-lpn}.

We define a sequence of functions $\delta_k: \NN \to (0,1/2)$ for $k \in \ZZ_{\geq 0}$ as follows. Define $\delta_0(n) := 2^{-(\log n)^{3/4}}$. For $k \geq 0$ define 
\begin{equation}
\delta_{k+1}(n) := C_{\ref{lemma:triangle-lpn}}(H_n) 2^{2H_n+4} L_n^{2C} C_{\ref{lemma:f-corr}}(H_n)^2 \delta_k(n)^2,
\label{eq:deltak-def}
\end{equation}
where $C_{\ref{lemma:triangle-lpn}} : \BN \to \BN$ and $C_{\ref{lemma:f-corr}} : \BN \to \BN$ are mappings defined in \cref{lemma:triangle-lpn} and \cref{lemma:f-corr}, respectively. 
Note that we can therefore write $\delta_k(n) = 2^{(2^k-1)f(n)-2^k(\log n)^{3/4}}$ where we have defined $f(n) = C_{\ref{lemma:triangle-lpn}}(H_n) 2^{2H_n+4} L_n^{2C} C_{\ref{lemma:f-corr}}(H_n)^2$. By inspecting the definitions of $C_{\ref{lemma:triangle-lpn}},C_{\ref{lemma:f-corr}}$, we may observe that there is some fixed $n_1 \in \NN$ so that $f(n) \leq 2^{O(H_n^2)} (4\log n)^{2C} \leq 2^{(\log n)^{3/4}/2} = 1/\sqrt{\delta_0(n)}$ for all $n \geq n_1$. Thus, for each $n \geq n_1$, we can check by induction that for all $k \geq 0$, \[\delta_{k+1}(n) \leq \frac{\delta_k(n)^2}{\sqrt{\delta_0(n)}} \leq \delta_k(n)^{3/2} \leq \delta_0(n)^{3/2} \leq \delta_0(n).\]
Without loss of generality, we can also let $n_1$ be sufficiently large so that $1/L_n^C \geq 2^{-n/8}$ and $1/L_n^C \geq 2^{3-H_n}$ and $2^{H_n} L_n^C \cdot 4C_{\ref{lemma:f-corr}}(H_n)\delta_0(n) < 1/2^{|\Gamma(H_n)|+2H_n+7}$ for all $n \geq n_1$. We can also check by induction that $\delta_k(n) \geq 2^{-2^k (\log n)^{3/4}}$ for all $k \geq 0$ and $n \in \NN$. Hence, for each $k \geq 0$, there is some minimal $n_2(k) \in \NN$ such that $\delta_k(n) \geq 2^{-n/8}$ for all $n \geq n_2(k)$. Note that $n_2(k+1) \geq n_2(k)$ for all $k \geq 0$.

We define an infinite sequence of algorithms $\Alg_0, \Alg_1,\dots,\Alg_k,\dots$, where $\Alg_k$ learns noisy parities with sample complexity $S_k(n,\delta,\eta)$ and time complexity $T_k(n,\delta,\eta)$, as follows. First, let $\Alg_0 := \Lyu$ be the algorithm guaranteed by \cref{theorem:lyu} with parameters $c = 3/4$ and $\epsilon = 1$. By \cref{theorem:lyu} and definition of $\delta_0$, we know that $\Alg_0$ learns noisy parities (per \cref{def:lpn-alg-advice}) with sample complexity $S_0$ and time complexity $T_0$ satisfying $S_0(n,\delta_0(n),1/2) \leq n^2$ and $T_0(n,\delta_0(n),1/2) \leq 2^{O(n/\log \log n)}$ for all $n \in \NN$. 

Now fix any $k \geq 0$; we will construct $\Alg_{k+1}$ using $\Alg_k$. We do this in several steps. First, by \cref{lemma:lpn-unknown-delta} applied with $\Alg_k$, there is an algorithm $\wAlg_k$ for learning noisy parities with unknown noise level (per \cref{def:lpn-unknown-noise}), with sample complexity $\til S_k(n,\delta,\eta)$ and time complexity $\til T_k(n,\delta,\eta)$ satisfying
\[\til S_k(n,\delta,1/n) = 4S_k(n,\delta,1/2) \log(2n) + 9\delta^{-2} \log\left(32n S_k(n,\delta,1/2)\log(2n)\right) \leq O(n\delta^{-2}) \cdot S_k(n,\delta,1/2)\]
\[\til T_k(n,\delta,1/n) \leq O(n^2\delta^{-2}) \cdot S_k(n,\delta,1/2) T_k(n,\delta, 1/2)\]
for all $\delta \in (0,1/2)$ and $n \geq n_2$, where $n_2 \in \NN$ is an absolute constant (independent of $k$). Also, we can bound the description complexity of $\wAlg_k$ as $\dc(\wAlg_k) \leq \dc(\Alg_k) + O(1)$.

Second, we apply \cref{lemma:realizable-regression-alg} with algorithm $\wAlg_k$, noise level function $\til \delta_k(n) := 2^{H_n}L_n^C \cdot 4C_{\ref{lemma:f-corr}}(H_n) \delta_k(n)$, and error function $\epsilon(n) := 1/L_n^C$. As shown above, we have $\til \delta_k(n) \geq \delta_k(n) \geq 2^{-n/8}$ and $\epsilon(n) \geq 2^{-n/8}$ for all $n \geq \max(n_1, n_2(k))$. Thus, we get that there is a $(\Sreg,\Treg,\epsilon,\Breg)$-realizable regression algorithm $\Reg_k$ for $\MM(\til\delta_k)$ where
\begin{align*}
\Sreg(n) 
&\leq 4 \cdot \til S_k\left(n, \frac{\til \delta_k}{L_n^C \cdot 4C_{\ref{lemma:f-corr}}(H_n)},1/n\right) + C_{\ref{lemma:realizable-regression-alg}} (L_n\log(n))^{C_{\ref{lemma:realizable-regression-alg}} \cdot C} \\ 
&\leq O(n2^{-2H_n}\delta_k^{-2}) \cdot S_k\left(n, 2^{H_n}\delta_k,1/2\right) \\ 
&\leq (n/\delta_k)^{C_2} \cdot S_k\left(n, \delta_k,1/2\right),
\end{align*}
and similarly
\[\Treg(n) \leq (n/\delta_k)^{C_2} \cdot S_k\left(n, \delta_k,1/2\right)T_k\left(n, \delta_k,1/2\right)\]
and
\[\Breg(n) \leq (n/\delta_k)^{C_2},\]
for all $n \geq \max(n_1, n_2(k), C_{\ref{lemma:realizable-regression-alg}})$, and some universal constant $C_2>0$. Also, $\dc(\Reg_k) \leq \dc(\wAlg_k) + O(1)$. In the above bounds, we are using the guarantees of \cref{lemma:realizable-regression-alg} together with the definition of $H_n$, the fact that $L_n^C \leq n$ for large $n$, and the fact that (without loss of generality) $S_k$ and $T_k$ are non-increasing in $\delta$. 

Third, by the assumption we made at the beginning of the proof, there is a $(S'_k,T'_k,\alpha_k,B'_k)$-policy cover learning algorithm $\PC_k$ for $\MM(\til\delta_k)$ where $\alpha(n) \geq 1/L_n^C$ and
\[S'_k(n) \leq (n/\delta_k)^{C_{\ref{lemma:family-is-computable}} \cdot C} \Sreg(n)\]
\[T'_k(n) \leq (n/\delta_k)^{C_{\ref{lemma:family-is-computable}} \cdot C} \Treg(n)\]
\[B'_k(n) \leq (n/\delta_k)^{C_{\ref{lemma:family-is-computable}} \cdot C} \Breg(n)\]
for all $n \in \NN$. 
Moreover, by assumption $\dc(\PC_k) \leq C \cdot \dc(\Reg_k)$. 
Define $n_3(k)$ to be the minimal positive integer such that $S'_k(n) \leq 2^n$ for all $n \geq n_3(k)$.

Fourth and lastly, we apply \cref{lemma:policy-cover-to-lpn} with base LPN algorithm $\wAlg_k$, policy cover learning algorithm $\PC_k$, and noise level function $\til\delta_k$. By definition of $n_1$, we can check that \[\til \delta_k(n) = 2^{H_n}L_n^C \cdot 4C_{\ref{lemma:f-corr}}(H_n) \delta_k(n) \leq 2^{H_n}L_n^C \cdot 4C_{\ref{lemma:f-corr}}(H_n) \delta_0(n) < 1/2^{|\Gamma(H_n)| + 2H_n+7}\] for all $n \geq n_1$. Also, $\alphaPC(n) \geq 2^{3-H_n}$ for all $n \geq n_1$, $\SPC(n) \leq 2^n$ for all $n \geq n_3(k)$, and $\til\delta_k(n) \geq \delta_k(n) \geq 2^{-n/4}$ for all $n \geq n_2(k)$. Thus, there is an algorithm $\Alg_{k+1}$ for learning noisy parities with sample complexity $S_{k+1}(n,\delta,\eta)$ and time complexity $T_{k+1}(n,\delta,\eta)$ satisfying
\begin{align*}
S_{k+1}(n,C_{\ref{lemma:triangle-lpn}}(H)\til\delta_k^2,1/2) &\leq (n/\til\delta_k)^{C_{\ref{lemma:policy-cover-to-lpn}}} \cdot \left(S'_k(n) + \til S_k(n,\til\delta_k/(2^{H}C_{\ref{lemma:f-corr}}(H)),1/n)\right) \\ 
&\leq (n/\delta_k)^{C_4} S_k(n,\delta_k,1/2)
\end{align*}
and
\begin{align*}
T_{k+1}(n,C_{\ref{lemma:triangle-lpn}}(H)\til\delta_k^2,1/2) 
&\leq (nB_k'(n)/\delta_k)^{C_{\ref{lemma:policy-cover-to-lpn}}} \cdot \left(T'_k(n) + \til T_k(n,\til\delta_k/(2^{H}C_{\ref{lemma:f-corr}}(H)),1/n)\right) \\ 
&\leq (nB_k'(n)/\delta_k)^{C_3} \cdot S_k(n,\delta_k,1/2) T_k(n,\delta_k,1/2) \\ 
&\leq (n/\delta_k)^{C_4} \cdot S_k(n,\delta_k,1/2) T_k(n,\delta_k,1/2)
\end{align*}
for all $n \geq \max(n_1, n_2(k), n_3(k), C_{\ref{lemma:realizable-regression-alg}},C_{\ref{lemma:policy-cover-to-lpn}})$, and some universal constants $C_3, C_4>0$. In the second inequality in each display we are substituting in the preceding bounds on $S'_k, T'_k, \Sreg, \Treg, \til S_k, \til T_k$; we are also using the fact that $\til \delta_k/(2^{H}C_{\ref{lemma:f-corr}}(H)) \geq \delta_k$ and again using monotonicity of $S_k, T_k$. In the third inequality in the second display we are using that $B'_k(n) \leq (n/\delta_k)^{C_{\ref{lemma:family-is-computable}} \cdot C} \Breg(n) \leq (n/\delta_k)^{C_{\ref{lemma:family-is-computable}} \cdot C+ C_2}$. Moreover, by construction, $\dc(\Alg_{k+1}) \leq \dc(\PC_k) + \dc(\wAlg_k) + O(1)$, and thus $\dc(\Alg_{k+1}) \leq (C+1)\dc(\Alg_k) + O(1) \leq \exp(O(k))$. 

By \cref{eq:deltak-def} and definition of $\til\delta_k$, we have that $C_{\ref{lemma:triangle-lpn}}(H) \til \delta_k^2 = \delta_{k+1}$. Thus, we have
\begin{align}
\begin{split}
S_{k+1}(n,\delta_{k+1},1/2) &\leq n^{C_4} \delta_k^{-C_4} S_k(n,\delta_k,1/2) \\
T_{k+1}(n,\delta_{k+1},1/2) &\leq n^{C_4} \delta_k^{-C_4} S_k(n,\delta_k,1/2)T_k(n,\delta_k,1/2)
\label{eq:st-recursion}
\end{split}
\end{align}
for all $n \geq \max(n_1, n_2(k), n_3(k), C_{\ref{lemma:realizable-regression-alg}},C_{\ref{lemma:policy-cover-to-lpn}})$. Now recall that $\delta_{k+1}(n) \leq \delta_k(n)^{3/2}$ for all $k \geq 0$ and $n \geq n_1$. Thus, for any $k \geq 0$, $n \geq n_1$, and $e > 0$, it holds that \begin{equation} \prod_{i=0}^k \delta_i(n)^{-e} \leq \delta_{k+1}(n)^{-e\sum_{i=1}^{k+1} (2/3)^i} \leq \delta_{k+1}(n)^{-2e}.\label{eq:delta-prod}\end{equation}
Iterating the first recursion in \cref{eq:st-recursion} and applying \cref{eq:delta-prod}, we get for any $k \geq 0$ that
\begin{equation} S_k(n,\delta_k,1/2) \leq S_0(n,\delta_0,1/2) \cdot n^{C_4 k} \prod_{i=0}^{k-1} \delta_i(n)^{-C_4} \leq n^{C_4 k + 2} \delta_k(n)^{-2C_4}
\label{eq:sample-bound-iterated}
\end{equation}
for all $n \geq \max(n_1, \max_{i < k} n_2(i), \max_{i < k} n_3(i), C_{\ref{lemma:realizable-regression-alg}},C_{\ref{lemma:policy-cover-to-lpn}})$. It follows that
\[S'_k(n) \leq n^{2C}\Sreg(n) \leq (n/\delta_k)^{C_2+2C} \cdot S_k(n,\delta_k,1/2) \leq n^{C_5 k} \delta_k(n)^{-C_5}\]
for all such $n$, where $C_5\geq 1$ is a universal constant. Thus, we have
\[n_3(k) \leq \max(n_1, \max_{i < k} n_2(i), \max_{i < k} n_3(i), C_{\ref{lemma:realizable-regression-alg}},C_{\ref{lemma:policy-cover-to-lpn}}, n_4(k))\]
where $n_4(k)$ is the minimal positive integer so that $n^{C_5 k} \delta_k(n)^{-C_5} \leq 2^{n/8}$ for all $n \geq n_4(k)$. By induction and the fact that $n_4(k+1) \geq n_4(k)$ for all $k \geq 0$, we get that
\[n_3(k) \leq \max(n_1, \max_{i < k} n_2(i), C_{\ref{lemma:realizable-regression-alg}},C_{\ref{lemma:policy-cover-to-lpn}}, n_4(k))\]
for all $k \geq 0$. Moreover $n_2(i) \leq n_2(k) \leq n_4(k)$ for all $0 \leq i \leq k$. Thus, \cref{eq:st-recursion,eq:sample-bound-iterated} in fact hold for all $n \geq \max(n_1, C_{\ref{lemma:realizable-regression-alg}},C_{\ref{lemma:policy-cover-to-lpn}}, n_4(k))$. In particular, applying \cref{eq:sample-bound-iterated} to the second recursion in \cref{eq:st-recursion} and iterating gives that for all $k \geq 0$ and all $n \geq \max(n_1, C_{\ref{lemma:realizable-regression-alg}},C_{\ref{lemma:policy-cover-to-lpn}}, n_4(k))$,
\begin{align}
T_k(n,\delta_k,1/2)
&\leq n^{C_4k+2}\delta_{k-1}^{-3C_4} T_{k-1}(n,\delta_{k-1},1/2) \\ 
&\leq n^{C_4k^2+2} \left(\prod_{i=0}^{k-1} \delta_i^{-3C_4} \right) \cdot T_0(n,\delta_0,1/2) \\ 
&\leq n^{C_4k^2+2} \delta_k^{-6C_4} 2^{O(n/\log \log n)}
\label{eq:time-bound-iterated}
\end{align}
where the last inequality uses \cref{eq:delta-prod}.

Now for each $n \in \NN$, let $k^\st(n)$ be the minimal positive integer so that $\delta_{k^\st(n)}(n) \leq 2^{-n/\log \log n}$. Note that $k^\st(n) \leq \log n$, so for all $m \geq n$ and all sufficiently large $n$, it holds that $m^{C_5 k^\st(n)} \leq m^{C_5 \log m} \leq 2^{m/16}$. Additionally, either $k^\st(n) = 0$ or $\delta_{k^\st(n)}(n) \geq \delta_{k^\st(n)-1}(n)^2 > 2^{-2n/\log \log n}$ by minimality of $k^\st(n)$. In both cases it holds that $\delta_{k^\st(n)}(n) \geq 2^{-2n/\log \log n}$ for sufficiently large $n$. Define $h(x) = \log \delta_{k^\st(n)}(x)$ and $g(x) = -x/(16C_5)$. Then $h(n) \geq g(n)$ for sufficiently large $n$.
Also, since $h(x) = (2^{k^\st(n)}-1)f(x)-2^{k^\st(n)}(\log x)^{3/4}$, we have for all $x \geq n$ that
\[h'(x) \geq -2^{k^\st(n)} \frac{d}{dx} (\log x)^{3/4} = -\frac{3}{4} 2^{k^\st(n)} \frac{1}{x(\log x)^{1/4}} \geq -\frac{3}{4(\log n)^{1/4}} \geq -1/(16C_5) = g'(x)\]
where the second inequality uses that $k^\st(n) \leq \log(n)$, and the last inequality holds so long as $n$ is sufficiently large. We conclude that $h(m) \geq g(m)$ for all $m \geq n$, and thus $\delta_{k^\st(n)}(m) \geq 2^{-m/(16C_5)}$ for all $m \geq n$. But then $m^{C_5 k^\st(n)} \delta_{k^\st(n)}(m)^{-C_5} \leq 2^{m/8}$ for all $m \geq n$, and hence $n_4(k^\st(n)) \leq n$. This means that $n \geq \max(n_1, C_{\ref{lemma:realizable-regression-alg}},C_{\ref{lemma:policy-cover-to-lpn}}, n_4(k^\st(n)))$, so applying \cref{eq:time-bound-iterated}, we get that for all sufficiently large $n$,
\[T_{k^\st(n)}(n,\delta_{k^\st(n)}(n),1/2) \leq n^{C_4 (k^\st(n))^2 + 2} \delta_{k^\st(n)}(n)^{-6C_4} 2^{O(n/\log \log n)} \leq 2^{O(n/\log \log n)}\]
where the final inequality uses the previously-derived lower bound on $\delta_{k^\st(n)}(n)$ and upper bound on $k^\st(n)$. 

Consider the non-uniform algorithm $\Alg^\st$ for learning parities with noise level $2^{-n/\log \log n}$, defined as follows: on input of dimension $n$, pad the noise from $1/2-2^{-n/\log \log n}$ to $1/2 - \delta_{k^\st(n)}$ and simulate the algorithm $\Alg_{k^\st(n)}$. Recall that we showed $\dc(\Alg_k) \leq \exp(O(k))$ for all $k \geq 0$. Thus, we have $\dc(\Alg_{k^\st(n)}) \leq \poly(n)$. So $\Alg^\st$ can be implemented with $\poly(n)$ advice; simulating $\Alg_{k^\st(n)}$ incurs a multiplicative overhead of only $\poly(n)$ in the time complexity (see discussion after \cref{def:desc-complexity}). So the time complexity of $\Alg^\st$ is $2^{O(n/\log \log n)}$. This contradicts \cref{assm:fine-lpn}.
\end{namedproof}


\section{Reducing regression to low-noise LPN}\label{sec:regression-to-lpn}


The main result of this section is \cref{lemma:realizable-regression-alg}, which constitutes the first half of the separation between reward-free RL and regression. It states that $\epsilon$-accurate regression in the family of MDPs $\MM(\delta)$ (\cref{def:mdp-family}) can be efficiently reduced to learning parities with noise level $1/2-\delta\epsilon/2^{O(H^2)}$.

\begin{lemma}\label{lemma:realizable-regression-alg}
There is a constant $C_{\ref{lemma:realizable-regression-alg}}$ with the following property. Let $n_0 \in \NN$. Let $\Alg$ be an algorithm for learning noisy parities with unknown noise level, with time complexity $T(n,\delta,\eta)$ and sample complexity $S(n,\delta,\eta)$ (\cref{def:lpn-unknown-noise}). Fix any $\delta: \NN \to (0, 1/2)$ and $\epsilon: \NN \to (0,1)$ with $\epsilon(n)  \geq 2^{-n/8}$ and $\delta(n) \geq 2^{-n/8}$ for all $n \geq n_0$. Then there is a $(\Sreg,\Treg,\epsilon,\Breg)$-realizable regression algorithm (\cref{def:regression-algorithm}) for $\MM(\delta)$ where 
\begin{subequations}
\begin{align}
\begin{split}
\Sreg(n) &\leq 4\cdot S\left(n,\frac{\delta\epsilon}{4C_{\ref{lemma:f-corr}}(H)},\frac{1}{n}\right) + C_{\ref{lemma:realizable-regression-alg}}\cdot\left(\frac{\log n}{\epsilon}\right)^{C_{\ref{lemma:realizable-regression-alg}}}\label{eq:sreg}
\end{split}\\
\begin{split}
\Treg(n) &\leq C_{\ref{lemma:realizable-regression-alg}}\log (n) \cdot T\left(n,\frac{\delta\epsilon}{4C_{\ref{lemma:f-corr}}(H)},\frac{1}{n}\right) \\ 
&\qquad+ C_{\ref{lemma:realizable-regression-alg}}\cdot \left(\frac{n}{\delta\epsilon}\right)^{C_{\ref{lemma:realizable-regression-alg}}} \cdot S\left(n,\frac{\delta\epsilon}{4C_{\ref{lemma:f-corr}}(H)},\frac{1}{n}\right)\label{eq:treg}
\end{split}\\
\begin{split}
\Breg(n) &\leq C_{\ref{lemma:realizable-regression-alg}}\cdot \left(\frac{n}{\delta}\right)^{C_{\ref{lemma:realizable-regression-alg}}}
\label{eq:breg}
\end{split}
\end{align}
\end{subequations}
for all $n \geq \max(n_0, C_{\ref{lemma:realizable-regression-alg}})$.
\end{lemma}

In the above display, we have written $H = H_n$ (defined in \cref{def:mdp-family}), $\delta = \delta(n)$, and $\epsilon=\epsilon(n)$ for notational simplicity. From the definition of realizable regression, we may observe that each regression sample $(Z^i, F^i)$ has distribution
\[(Z^i,F^i) \sim \sum_{s \in \MS[h]} \beta(s) \MD^\sk_{n,N,h,\delta}(\cdot|s) \times \Ber(f(s))\]
for some fixed, unknown secret $\sk \in \BF_2^n$ (corresponding to a particular block MDP in $\MM(\delta)$), step $h \in [H]$, latent state distribution $\beta \in \Delta(\MS[h])$, and link function $f: \MS[h] \to [0,1]$. The goal of regression is to learn $\EE[F|Z]$, which we accomplish by either (a) learning $\sk$ and then directly estimating $f$, or (b) learning a constant predictor. The challenge is to show that the first method is tractable whenever $f$ is non-trivially far from constant. There are two main conceptual ingredients:

\begin{enumerate}
\item First, recall that an emission $Z \sim \MD^\sk_{n,N,h,\delta}(\cdot|s)$ contains a vector of $h$ correlated LPN samples, where the responses of the LPN samples are additively ``masked'' by a random vector $b \sim \nu_s$ (\cref{def:emission}). We show that if the regression label function $f$ is far from constant, then $F$ must be correlated with some ``prefix sum'' of $b$ (\cref{lemma:f-corr})

\item Second, we show that if such a correlation exists, it can be converted into a standard LPN sample with low noise level, and hence the secret key $\sk$ can be recovered (\cref{lemma:learn-from-corr}). This crucially uses the fact that the noise terms $e_1,\dots,e_h$ in the emission $Z$ are drawn from a correlated Bernoulli distribution (\cref{def:cber}) rather than a product distribution.
\end{enumerate}

Given these ingredients, the proof of \cref{lemma:realizable-regression-alg} is straightforward via appropriate generalization bounds. For most of the section, we consider emissions sampled from $\til\MD^\sk_{n,N,H,\delta}(\cdot|s)$ rather than $\MD^\sk_{n,N,H,\delta}(\cdot|s)$, and account for the (exponentially small) distance between these distributions at the end when we prove \cref{lemma:realizable-regression-alg}.

We start by proving \cref{lemma:f-corr}. Informally, the lemma is equivalent to the statement that any label function $f: \MS[h]\to[0,1]$ that approximately satisfies a certain set of $2h$ linear constraints should be nearly constant. Since $|\MS[h]| = 2h$, it suffices to show that the constraints are (quantitatively) linearly independent. Towards this end, we need \cref{lemma:partial-sum-rep}, below, which gives an algebraic description for these linear constraints. For notational convenience, we first define $\mu_{\beta,f}$ to be the joint distribution of $(b,F)$ induced by $\beta$ and $f$:

\begin{definition}
Let $h \in \NN$ and let $\MS[h] := \{h\}\times\{0,\dots,h-1\}\times\BF_2$. For $\beta \in \Delta(\MS[h])$ and $f: \MS[h] \to [0,1]$, we define distribution $\mu_{\beta,f} \in \Delta(\BF_2^{h+1})$ by
\[\mu_{\beta,f} := \sum_{s \in \MS[h]} \beta(s) \nu_{s} \times \Ber(f(s))\]
where $\nu_s \in \Delta(\BF_2^h)$ was defined in \cref{def:nu}. For a real number $\alpha \in [0,1]$, we write $\mu_{\beta,\alpha}$ to denote $\mu_{\beta,f_\alpha}$ for constant function $f_\alpha(s) = \alpha$. We write $(B, F) \sim \mu$ to denote that $B$ is the first $h$ bits of the sample from $\mu$, and $F$ is the last bit.
\end{definition}

\begin{lemma}\label{lemma:partial-sum-rep}
Let $h \in \NN$ and set $\MS[h] := \{h\}\times\{0,\dots,h-1\} \times \BF_2$. Fix $\beta \in \Delta(\MS[h])$. Then for each integer $0 \leq m < h$, there is a degree-$m$ polynomial $P_m: \RR \to \RR$ and real numbers $c_{m,\beta}, \til{c}_{m,\beta} \in \RR$ so that for all functions $f: \MS[h] \to [0,1]$,
\[\Prr_{(B,F) \sim \mu_{\beta,f}}[B_1 + \dots + B_m + F \equiv 0 \bmod{2}] = c_{m,\beta} + \sum_{i=0}^{h-1} P_m(i) \left(\beta(h,i,0) f(h,i,0) + \beta(h,i,1) f(h,i,1)\right)\]
\[\Prr_{(B,F) \sim \mu_{\beta,f}}[B_1 + \dots + B_m + B_h + F \equiv 0 \bmod{2}] = \til{c}_{m,\beta} + \sum_{i=0}^{h-1} P_m(i) \left(\beta(h,i,0) f(h,i,0) - \beta(h,i,1) f(h,i,1)\right).\]

Moreover, every coefficient of $P_m$ is at most $2^m \cdot h!$ in magnitude, and the degree-$m$ coefficient of $P_m$ is exactly $(-1)^{m+1} 2^m (h-1-m)!/(h-1)!$.
\end{lemma}

In particular, for each $0 \leq m < h$ there is a linear constraint on the function $s \mapsto \beta(s)f(s)$ whose coefficient for state $s=(h,k,b_h)$ is the evaluation of a polynomial $P_m(k)$ of degree $m$, and there is another constraint whose coefficient is the evaluation of $(1-2b_h)P_m(k)$. 

\begin{proof}
For any nonnegative integer $i$, define $Q_i: \RR \to \RR$ by $Q_i(t) = t(t-1) \cdots (t-i+1)$. Observe that $Q_i(k) = \frac{k!}{(k-i)!}$ for any integer $k \geq i$, and moreover $Q_i(k) = 0$ for any integer $0 \leq k < i$. For any $s = (h,k,b_h) \in \MS[h]$ and integer $0 \leq m < h$, the first $h-1$ coordinates of $B\sim\nu_s$ are uniform over $\BF_2^{h-1}$ subject to having Hamming weight $k$ (\cref{def:nu}). Thus, for any integer $0 \leq i \leq m$, a combinatorial argument gives
\begin{align}
\Prr_{B \sim \nu_s}\left[\sum_{j=1}^m \mathbbm{1}[B_j=1] = i\right] \nonumber
&= \mathbbm{1}[k+m-h+1 \leq i \leq k] \frac{\binom{m}{i} \binom{h-1-m}{k-i}}{\binom{h-1}{k}} \nonumber \\ 
&= \binom{m}{i} \frac{Q_i(k) Q_{m-i}(h-1-k)}{Q_m(h-1)} \label{eq:b1bm-pmf}
\end{align}
Now, for any integer $0 \leq m < h$ and $f:\MS[h]\to[0,1]$, we can write
\allowdisplaybreaks
\begin{align*}
&\Prr_{(B,F) \sim \mu_{\beta,f}}[B_1+\dots+B_m+F \equiv 0 \bmod{2}] \\ 
&= \sum_{s \in \MS[h]} \beta(s) \left(\Prr_{B \sim \nu_{s}}[B_1+\dots+B_m \equiv 1 \bmod{2}] \cdot f(s) + \Prr_{B \sim \nu_{s}}[B_1+\dots+B_m \equiv 0 \bmod{2}] \cdot (1 - f(s))\right) \\
&= \sum_{(h,k,b_h) \in \MS[h]} \beta(h,k,b_h) \Bigg(\sum_{\substack{0 \leq i \leq m \\ i \equiv 1 \bmod{2}}} \binom{m}{i}\frac{Q_i(k)Q_{m-i}(h-1-k)}{Q_m(h-1)}f(h,k,b_h) \\
&\qquad\qquad\qquad\qquad\qquad+ \sum_{\substack{0 \leq i \leq m \\ i \equiv 0 \bmod{2}}} \binom{m}{i}\frac{Q_i(k)Q_{m-i}(h-1-k)}{Q_m(h-1)} (1-f(h,k,b_h))\Bigg) \\
&= \sum_{(h,k,b_h) \in \MS[h]} \left(\sum_{i=0}^m \binom{m}{i} \frac{Q_i(k) Q_{m-i}(h-1-k)}{Q_m(h-1)} (-1)^{i+1}\right) \beta(h,k,b_h)f(h,k,b_h) \\ 
&\qquad+ \sum_{(h,k,b_h) \in \MS[h]} \sum_{\substack{0 \leq i \leq m \\ i \equiv 0 \bmod{2}}} \binom{m}{i} \frac{Q_i(k)Q_{m-i}(h-1-k)}{Q_m(h-1)} \beta(h,k,b_h) \\
&= c_{m,\beta} + \sum_{(h,k,b_h) \in \MS[h]} P_m(k) \beta(h,k,b_h)f(h,k,b_h)
\end{align*}
where we used \cref{eq:b1bm-pmf} in the second equality, and in the final equality we have defined $P_m$ and $c_{m,\beta}$ by \[P_m(t) := \frac{1}{Q_m(h-1)}\sum_{i=0}^m \binom{m}{i}Q_i(t)Q_{m-i}(h-1-t) (-1)^{i+1}\] and \[c_{m,\beta} := \sum_{(h,k,b_h) \in \MS} \sum_{\substack{0 \leq i \leq m \\ i \equiv 0 \bmod{2}}} \binom{m}{i} \frac{Q_i(k)Q_{m-i}(h-1-k)}{Q_m(h-1)} \beta(h,k,b_h).\]
Note that neither $P_m$ nor $c_{m,\beta}$ depends on $f$.

It's clear that $P_m$ is a polynomial of degree at most $m$. Finally, note that the degree-$m$ coefficient of $Q_i(t)Q_{m-i}(h-t)$ (in $t$) is $(-1)^{m-i}$, so the degree-$m$ coefficient of $P_m(t)$ is \[\frac{1}{Q_m(h-1)} \sum_{i=0}^m \binom{m}{i} (-1)^{m-i}(-1)^{i+1} = \frac{2^m(-1)^{m+1}}{Q_m(h-1)}\] as claimed. Every coefficient of $Q_i(t)$ is at most $i!$ in absolute value, and every coefficient of $Q_{m-i}(h-1-t)$ is at most $h!$ in absolute value, so every coefficient of $P_m(t)$ is at most $\frac{2^m m! h!}{Q_m(h-1)} = 2^m m! (h-1-m)! \cdot h \leq 2^m h!$ in absolute value.

To prove the second claim of the lemma, notice that if $(B,F) \sim \mu_{\beta,f}$ then $(B, F+B_h)$ has distribution $\mu_{\beta,\til f}$ where \[\til{f}(h,k,b_h) := \begin{cases} f(h,k,b_h) & \text{ if } b_h = 0 \\ 1 - f(h,k,b_h) & \text{ if } b_h = 1 \end{cases}.\]
Thus, applying the first claim of the lemma to function $\til f$, we get
\begin{align*}
&\Prr_{(B,F)\sim\mu_{\beta, f}}[B_1+\dots+B_m+B_h+F\equiv 0 \bmod{2}]  \\
&=\Prr_{(B,F)\sim\mu_{\beta, \til f}}[B_1+\dots+B_m+F\equiv 0 \bmod{2}] \\
&= c_{m,\beta} + \sum_{k=0}^{h-1} P_m(k)\left(\beta(h,k,0)\til{f}(h,k,0) + \beta(h,k,1)\til{f}(h,k,1)\right) \\
&= c_{m,\beta} + \sum_{k=0}^{h-1} P_m(k) \beta(h,k,1) + \sum_{k=0}^{h-1} P_m(k) \left(\beta(h,k,0) f(h,k,0) - \beta(h,k,1) f(h,k,1)\right)
\end{align*} 
which proves the second claim with $\til{c}_{m,\beta} := c_{m,\beta} + \sum_{k=0}^{h-1} P_m(k) \beta(h,k,1)$.
\end{proof}

Linear independence of the constraints described in \cref{lemma:partial-sum-rep} follows from symbolic linear independence of the polynomials $\{P_m(k), (1-2b_h)P_m(k)\}_{0 \leq m < h}$. This argument can be made quantitative via properties of the appropriate Vandermonde matrix. We will need the following condition number bounds:

\begin{lemma}[see e.g. \cite{guggenheimer1995simple}]\label{lemma:gug-bound}
Let $n \in \NN$ and let $A \in \RR^{n \times n}$ be a matrix. Then
\[\frac{\sigma_{\max}(A)}{\sigma_{\min}(A)} \leq \frac{2}{|\det(A)|} \left(\frac{\norm{A}_F}{\sqrt{n}}\right)^n.\]
\end{lemma}

\begin{lemma}\label{lemma:vandermonde-lsv}
Let $h \in \NN$, and let $V \in \RR^{h\times h}$ be the Vandermonde matrix of $\{0,1,2,\dots,h-1\}$:
\[ V := \begin{bmatrix} 1 & 0^1 & \dots & 0^{h-1} \\ 
1 & 1^1 & \dots & 1^{h-1} \\ \vdots & \vdots & & \vdots \\ 
1 & (h-1)^1 & \dots & (h-1)^{h-1} \end{bmatrix}.\]
Then $\sigma_{\min}(V) \geq \frac{1}{2h^{h^2}}$.
\end{lemma}

\begin{proof}
By a standard identity, $|\det(V)| \geq 1$. Also, every entry of $V$ is at most $h^{h-1}$, so $\norm{V}_F \leq h^h$. Thus, by \cref{lemma:gug-bound},
\[\frac{\sigma_{\max}(V)}{\sigma_{\min}(V)} \leq \frac{2}{|\det(V)|} \left(\frac{\norm{V}_F}{\sqrt{h}}\right)^{h} \leq 2h^{h^2}.\]
But certainly $\sigma_{\max}(V) \geq 1$, so the lemma follows.
\end{proof}

\begin{lemma}\label{lemma:lt-lsv}
Fix $\alpha,\beta>0$ and $n \in \NN$. Let $A \in \RR^{n\times n}$ be a lower-triangular matrix with $|A_{ii}| \geq \alpha$ and $|A_{ij}| \leq \beta$ for all $i,j \in [n]$. Then $\sigma_{\min}(A) \geq \frac{\alpha}{2}(\alpha/(\beta\sqrt{n}))^n$. 
\end{lemma}

\begin{proof}
From \cref{lemma:gug-bound} and the fact that $|\det(A)| = \prod_{i=1}^n |A_{ii}| \geq \alpha^n$, we have
\[\frac{\sigma_{\max}(A)}{\sigma_{\min}(A)} \leq \frac{2}{\alpha^n} (\beta \sqrt{n})^n.\]
But $\sigma_{\max}(A) \geq \alpha$, from which the lemma follows.
\end{proof}


We can now formally state and prove \cref{lemma:f-corr}.

\begin{lemma}\label{lemma:f-corr}
Let $h \in \NN$, $\epsilon > 0$, and set $\MS[h] = \{h\}\times\{0,\dots,h-1\} \times \BF_2$. Fix $\beta \in \Delta(\MS[h])$ and $f: \MS[h] \to [0,1]$, and let $\alpha := \sum_{s \in \MS[h]} \beta(s) f(s)$. Suppose that for all $1 \leq m < h$ we have that
\begin{align}\begin{split}
\left|\Prr_{(B,F)\sim \mu_{\beta,f}}[B_1+\dots+B_m+F\equiv 0\bmod{2}] - \frac{1}{2}\right| &\leq \epsilon \\
\left|\Prr_{(B,F)\sim \mu_{\beta,\alpha}}[B_1+\dots+B_m+F\equiv 0\bmod{2}] - \frac{1}{2}\right| &\leq \epsilon 
\label{eq:bf-unbiased-1}
\end{split}\end{align}
and for all $0 \leq m < h$ we have that
\begin{align}\begin{split}
\left|\Prr_{(B,F)\sim \mu_{\beta,f}}[B_1+\dots+B_m+B_h+F\equiv 0\bmod{2}] - \frac{1}{2}\right| \leq \epsilon \\
\left|\Prr_{(B,F)\sim \mu_{\beta,\alpha}}[B_1+\dots+B_m+B_h+F\equiv 0\bmod{2}] - \frac{1}{2}\right| \leq \epsilon.
\label{eq:bf-unbiased-2}
\end{split}\end{align}
Then \[\sum_{s\in\MS[h]} \beta(s)\left|f(s) - \alpha\right| \leq C_{\ref{lemma:f-corr}}(h)\cdot \epsilon\]
where $C_{\ref{lemma:f-corr}}(h) := 4\sqrt{2}(2h)^{3(h+1)^2}$.
\end{lemma}

\begin{proof}
Define $g: \MS[h] \to [0,1]$ by $g(s) = \beta(s)(f(s) - \alpha)$. For every $0 \leq m < h$, using \cref{eq:bf-unbiased-1} and \cref{eq:bf-unbiased-2} respectively, we have
\begin{align}
\left|\Prr_{(B,F)\sim\mu_{\beta,f}}\left[\sum_{i=1}^m B_i +F\equiv 0\bmod{2}\right] - \Prr_{(B,F)\sim\mu_{\beta,\alpha}}\left[\sum_{i=1}^m B_i +F\equiv 0\bmod{2}\right]\right| &\leq 2\epsilon,
\label{eq:bf-close-1}\\
\left|\Prr_{(B,F)\sim\mu_{\beta,f}}\left[\sum_{i=1}^m B_i +B_h+F\equiv 0\bmod{2}\right] - \Prr_{(B,F)\sim\mu_{\beta,\alpha}}\left[\sum_{i=1}^m B_i +B_h+F\equiv 0\bmod{2}\right]\right| &\leq 2\epsilon.
\label{eq:bf-close-2}
\end{align}
Note that to establish \cref{eq:bf-close-1} for $m=0$, we have used that
\begin{align}
  \label{eq:marginal-alpha}
  \Pr_{(B,F) \sim \mu_{\beta,f}}[F\equiv 0\bmod{2}] = \Pr_{(B,F) \sim \mu_{\beta,\alpha}}[F\equiv 0\bmod{2}] = 1-\alpha.
\end{align}
Applying \cref{lemma:partial-sum-rep} with functions $f$ and $(h,k,b_h)\mapsto\alpha$, and using the definition of $g$, it follows from \cref{eq:bf-close-1} that
\[
\norm{
\begin{bmatrix} 
P_0(0) & P_0(1) & \dots & P_0(h-1) \\
P_1(0) & P_1(1) & \dots & P_1(h-1) \\
\vdots & \vdots & & \vdots \\
P_{h-1}(0) & P_{h-1}(1) & \dots & P_{h-1}(h-1)
\end{bmatrix}
\begin{bmatrix} 
g(h,0,0) + g(h,0,1) \\ 
g(h,1,0) + g(h,1,1) \\
\vdots \\
g(h,h-1,0) + g(h,h-1,1)
\end{bmatrix}
}_\infty 
\leq 2\epsilon
\]
and from \cref{eq:bf-close-2} that
\[
\norm{
\begin{bmatrix} 
P_0(0) & P_0(1) & \dots & P_0(h-1) \\
P_1(0) & P_1(1) & \dots & P_1(h-1) \\
\vdots & \vdots & & \vdots \\
P_{h-1}(0) & P_{h-1}(1) & \dots & P_{h-1}(h-1)
\end{bmatrix}
\begin{bmatrix} 
g(h,0,0) - g(h,0,1) \\ 
g(h,1,0) - g(h,1,1) \\
\vdots \\
g(h,h-1,0) - g(h,h-1,1)
\end{bmatrix}
}_\infty 
\leq 2\epsilon.
\]
Now let $V$ be the Vandermonde matrix of $\{0,1,\dots,h-1\}$, and for any polynomial $P \in \RR[t]$ and $d \geq 0$ let $\MC_d(P)$ denote the degree-$d$ coefficient of $P$. Then we can equivalently write
\begin{align*}
&\norm{
\begin{bmatrix} 
\MC_0(P_0) & 0  & \dots & 0 \\
\MC_0(P_1) & \MC_1(P_1)  & \dots & 0 \\
\vdots & \vdots & \ddots  & \vdots \\
\MC_0(P_{h-1}) & \MC_1(P_{h-1}) & \dots & \MC_{h-1}(P_{h-1})
\end{bmatrix}
V^\t
\begin{bmatrix} 
g(h,0,0) + g(h,0,1) \\ 
g(h,1,0) + g(h,1,1) \\
\vdots \\
g(h,h-1,0) + g(h,h-1,1)
\end{bmatrix}
}_\infty
\leq 2\epsilon
\end{align*}
and
\begin{align*}
&\norm{
\begin{bmatrix} 
\MC_0(P_0) & 0  & \dots & 0 \\
\MC_0(P_1) & \MC_1(P_1)  & \dots & 0 \\
\vdots & \vdots & \ddots  & \vdots \\
\MC_0(P_{h-1}) & \MC_1(P_{h-1}) & \dots & \MC_{h-1}(P_{h-1})
\end{bmatrix}
V^\t
\begin{bmatrix} 
g(h,0,0) - g(h,0,1) \\ 
g(h,1,0) - g(h,1,1) \\
\vdots \\
g(h,h-1,0) - g(h,h-1,1)
\end{bmatrix}
}_\infty
\leq 2\epsilon
\end{align*}
where the coefficient matrix is lower-triangular since each polynomial $P_i$ has degree $i$. Now from \cref{lemma:partial-sum-rep} we also know that the diagonal entries of the coefficient matrix are all at least $1/h^h$ in absolute value, and all entries are at most $2^{h-1}h^h$ in absolute value, so the least singular value is at least $(2h)^{-2h^2-2h}$ by \cref{lemma:lt-lsv}. From \cref{lemma:vandermonde-lsv}, we have $\sigma_{\min}(V^\t) \geq 1/(2h^{h^2})$. Thus, we get 
\[
\norm{
\begin{bmatrix} 
g(h,0,0) + g(h,0,1) \\ 
g(h,1,0) + g(h,1,1) \\
\vdots \\
g(h,h-1,0) + g(h,h-1,1)
\end{bmatrix}
}_2
\leq 4h^{h^2}(2h)^{2h(h+1)}\epsilon \sqrt{h}
\]
\[
\norm{
\begin{bmatrix} 
g(h,0,0) - g(h,0,1) \\ 
g(h,1,0) - g(h,1,1) \\
\vdots \\
g(h,h-1,0) - g(h,h-1,1)
\end{bmatrix}
}_2
\leq 4h^{h^2}(2h)^{2h(h+1)}\epsilon \sqrt{h}.
\]
The triangle inequality gives
\[\sqrt{\sum_{s\in\MS[h]} g(s)^2} \leq 4h^{h^2}(2h)^{2h(h+1)}\epsilon \sqrt{h}\]
from which we get
\[\sum_{s\in\MS[h]} |g(s)| \leq 4\sqrt{2}hh^{h^2}(2h)^{2h(h+1)}\epsilon\]
which proves the lemma.
\end{proof}

Next, we prove that if any of the constraints in \cref{eq:bf-unbiased-1} or \cref{eq:bf-unbiased-2} are violated, then it is possible to efficiently construct an LPN sample with noise $1/2-\Omega(\delta\epsilon)$ from each regression sample, and thereby efficiently learn $\sk$.

\begin{lemma}\label{lemma:learn-from-corr}
Let $\Alg$ be an algorithm for learning noisy parities with unknown noise level with time complexity $T(n,\delta,\eta)$ and sample complexity $S(n,\delta,\eta)$ (\cref{def:lpn-unknown-noise}). Then there is an algorithm $\LFC$ with the following property. Let $n,N,h,d \in \NN$ and $\delta,\eta,\epsilon>0$. Let $\MS[h] := \{h\}\times\{0,1,\dots,h-1\}\times\BF_2$ and $I := (\{0,\dots,h-1\}\times\BF_2 \setminus (0,0)) \times \BF_2$. Let $\sk \in \BF_2^n$, $\beta\in\Delta(\MS[h])$, and $f: \MS[h]\to[0,1]$. For notational convenience, write $\mu^1 := \mu_{\beta,f}$ and $\mu^0 := \mu_{\beta,\alpha}$
where $\alpha := \sum_{s\in\MS[h]} \beta(s)f(s)$. 

Let $(m,r,r') \in I$. If 
\begin{align}
  \label{eq:yxf-asm}
\left|\Pr_{(B,F)\sim\mu^{r'}}\left[\sum_{i=1}^m B_i + rB_h + F \equiv 0 \bmod{2}\right] - \frac{1}{2}\right| \geq \epsilon
\end{align}
and $d \geq S(n,2\delta\epsilon,\eta)$, then
\[\Pr[\LFC((Z^i,F^i)_{i=1}^d, m,r,r',\delta,\epsilon,\eta) = \sk] \geq 1 -\eta]\]
where $(Z^i,F^i)_{i=1}^d$ are independent samples $(Z^i,F^i) \sim \sum_{s\in\MS[h]} \beta(s)\til\MD^\sk_{n,N,h,\delta}(\cdot|s) \times \Ber(f(s))$. Moreover, the time complexity of $\LFC$ is $T(n,2\delta\epsilon,\eta) + d \cdot \poly(n,N,h)$.
\end{lemma}

\begin{proof}
The algorithm $\LFC$ outputs
\[\wh\sk := \Alg\left(\left(r \cdot u_h(Z^i) + \sum_{j=1}^m u_j(Z^i), r'F^i + r \cdot y_h(Z^i) + \sum_{j=1}^m y_j(Z^i) \right)_{i=1}^{S(n,2\delta\epsilon,\eta)}, 2\delta\epsilon,\eta\right)\]
where the maps $Z \mapsto u_j(Z)$ and $Z \mapsto y_j(Z)$ are as defined in \cref{def:emission-components}.
\paragraph{Analysis.} Observe that
\begin{align*}
\left|\Pr_{(B,F) \sim \mu^1}\left[\sum_{i=1}^m B_i + rB_h \equiv 0 \bmod{2}\right] - \frac{1}{2}\right|
&= \left|\Pr_{(B,F) \sim \mu^0}\left[\sum_{i=1}^m B_i + rB_h \equiv 0 \bmod{2}\right] - \frac{1}{2}\right| \\ 
&= \frac{1}{2|\alpha-\frac{1}{2}|}\left|\Pr_{(B,F) \sim \mu^0}\left[\sum_{i=1}^m B_i + rB_h + F \equiv 0 \bmod{2}\right] - \frac{1}{2}\right| \\ 
&\geq \left|\Pr_{(B,F) \sim \mu^0}\left[\sum_{i=1}^m B_i + rB_h + F \equiv 0 \bmod{2}\right] - \frac{1}{2}\right|
\end{align*}
where the first equality uses that $\mu^0, \mu^1$ have the same marginal distribution over $B$, and the second equality uses that $F$ is independent of $B$ under $\mu^0$. Thus, regardless of whether $r'=0$ or $r'=1$, \cref{eq:yxf-asm} implies the following bound on $\mu^1=\mu_{\beta,f}$:
\[\left|\Pr_{(B,F) \sim \mu^1}\left[\sum_{i=1}^m B_i + rB_h + r'F \equiv 1 \bmod{2}\right] - \frac{1}{2}\right| \geq \epsilon.\]
Let $p \in [0,1]$ denote the probability in the above display. The samples $(Z^i,F^i)_{i=1}^d$ are independent and identically distributed; fix any $i \in [d]$. Since $u_1(Z^i),\dots,u_h(Z^i) \sim \Unif(\BF_2^n)$ are independent, and at least one of $r$ or $m$ is nonzero (by definition of $I$), it's clear that $r\cdot u_h(Z^i) + \sum_{j=1}^m u_j(Z^i) \sim \Unif(\BF_2^n)$. Now condition on $u_1(Z^i),\dots,u_h(Z^i)$. By \cref{def:emission}, we have that
\begin{align*}
&r'F^i + r\cdot y_h(Z^i) + \sum_{j=1}^m y_j(Z^i)  - \left\langle r\cdot u_h(Z^i) + \sum_{j=1}^m u_j(Z^i), \sk\right\rangle \\ 
&= r'F^i + r\cdot (e_h+B_h) + \sum_{j=1}^m (e_j+B_j) \\
&= (re_h+e_1+\dots+e_m) + r'F^i + rB_h + \sum_{j=1}^m B_j 
\end{align*}
where $(e_1,\dots,e_h) \sim \CBer(h,\delta)$ and $(B,F) \sim \sum_{s\in\MS[h]} \beta(s)\nu_{s} \times \Ber(f(s)) = \mu^1$ are independent random variables that satisfy $e_h+B_h = y_h(Z^i) - \lng u_h(Z^i), \sk \rng$ and $e_j + B_j = y_j(Z^i) - \lng u_j(Z^i), \sk \rng$ for each $j \in [m]$. From the fact that $(e_1,\dots,e_h)\sim\CBer(h,\delta)$ and \cref{lemma:cber-sum}, we have that $re_h+e_1+\dots+e_m \sim \Ber(1/2-\delta)$. From the definition of $p$, we have that $r'F^i + rB_h + \sum_{j=1}^m B_j \sim \Ber(p)$. Thus, by \cref{lem:bernoulli-convolve}, the above sum has distribution $\Ber(1/2-2\delta(1/2-p))$. This means that the inputs to $\Alg$ are independent LPN samples with noise level $1/2-2\delta(1/2-p)$. Since $|1/2-p| \geq \epsilon$, we have \[1/2-2\delta(1/2-p) \in [0, 1/2-2\delta\epsilon] \cup [1/2+2\delta\epsilon, 1].\] Hence, by the guarantee of $\Alg$, we get that $\wh \sk = \sk$ with probability at least $1 - \eta$ over the samples $(Z^i,F^i)_{i=1}^d$. This proves correctness of $\LFC$. The time complexity bound is also immediate.
\end{proof}

Above, we crucially used the fact that the noise terms $e_1,\dots,e_h$ in the definition of $\til \MD^\sk_{n,N,h,\delta}(\cdot|k,b_h)$ (\cref{def:emission}) are drawn from the correlated Bernoulli distribution $\CBer(h,\delta)$, and not the product distribution $\Ber(1/2-\delta)^{\otimes h}$, so that the sum of two or more noise terms (over $\BF_2$) still has bias $\delta$ rather than $O(\delta^2)$. This property is formalized below: 

\begin{lemma}\label{lemma:cber-sum}
Fix $n \in \NN$ and $\delta \in (0,1/2)$. Let $X\sim \CBer(n,\delta)$. For any nonempty set $S\subseteq [n]$, it holds that $\sum_{i\in S} X_i \sim \Ber(1/2-\delta)$.
\end{lemma}
\begin{proof}
Note that $\CBer(n,\delta)$ is a mixture of the distribution $\Ber(0)^{\otimes n}$ and the distribution $\Ber(1/2)^{\otimes n}$. Consider the event $\ME$ that $X$ is drawn from the latter distribution, which occurs with probability $1-2\delta$. Then $\sum_{i\in S} X_i | \ME$ has distribution $\Ber(1/2)$, whereas $\sum_{i \in S} X_i | \overline{\ME}$ has distribution $\Ber(0)$. The lemma follows.   
\end{proof}

We now combine \cref{lemma:f-corr,lemma:learn-from-corr} to give a regression algorithm. The idea is to compute the best constant predictor $\MR^0$, as well as one predictor $\MR^{m,r,r'}$ for each of the possible tuples $(m,r,r') \in I$ identified in \cref{lemma:learn-from-corr}. Each of these latter predictors $\MR^{m,r,r'}$ is constructed by using $\LFC$ to estimate a candidate $\hat \sk^{m,r,r'}$ for $\sk$ and then estimating $f$ under the assumption that $\Dec^{\hat\sk^{m,r,r'}}_{n,N,h}$ is the true decoding function. After constructing all of these predictors, the algorithm then uses fresh samples to select the best predictor. In \cref{lemma:regression-alg}, this algorithm is formalized and analyzed in the setting where the emissions are drawn from $\til\MD^{\sk}_{n,N,h,\delta}$ rather than $\MD^\sk_{n,N,h,\delta}$.

\begin{lemma}\label{lemma:regression-alg}
Let $\Alg$ be an algorithm for learning noisy parities with unknown noise level, with time complexity $T(n,\delta,\eta)$ and sample complexity $S(n,\delta,\eta)$ (\cref{def:lpn-unknown-noise}). Then there is an algorithm $\RegressAlg$ with the following property. Let $n,N,h,d \in \NN$ and $\delta,\eta,\epsilon > 0$. Let $\MS[h] := \{h\}\times\{0,1,\dots,h-1\}\times\BF_2$ and $\delreg := \delta\epsilon/(4C_{\ref{lemma:f-corr}}(h))$.

For any $\sk \in \BF_2^n$, distribution $\beta \in \Delta(\MS[h])$, and function $f: \MS[h] \to [0,1]$, given $d$ independent samples $(Z^i, F^i)_{i=1}^d$ with \[(Z^i,F^i) \sim \mureg := \sum_{s \in \MS[h]} \beta(s) \cdot \left( \til\MD^\sk_{n,N,h,\delta}(\cdot|s) \times \Ber(f(s))\right),\] the output of $\hat\MR \gets \RegressAlg((Z^i,F^i)_{i=1}^d, \delta, \epsilon)$, with probability at least $1-\eta$, is a $(\beta,f,\epsilon)$-predictor (\cref{def:reg-predictor}) with respect to $\til\MD^\sk_{n,N,h,\delta}$, so long as
\begin{equation} d \geq \max(4S(n, \delreg,\eta/4), 2^{21}h^4\epsilon^{-4}\log(32h/\eta))\label{eq:regression-alg-d-lb}\end{equation}
and $N \geq \delta^{-4}\log(512h^3/\epsilon^2)$.

Moreover, $\hat\MR$ is a circuit of size $\poly(n,N,h,\log d)$, and the time complexity of $\RegressAlg$ is $O(h) \cdot T(n, \delreg, \eta/4) + d\cdot \poly(n,h,N)$.


\end{lemma}

\begin{proof}
For notational convenience, define $\alpha := \sum_{s \in \MS[h]} \beta(s) \cdot f(s)$. Also note that the algorithm has access to the parameter $h$ (as it is determined by any sample $Z^i$). 
  

\paragraph{Algorithm description.} First, $\RegressAlg$ computes $\hat\alpha := \frac{4}{d}\sum_{i=1}^{d/4} F^i$. The first candidate predictor $\MR^0$ is the constant function $\MR^0(Z) := \hat\alpha$. The remaining predictors are indexed by $I := (\{0,\dots,h-1\}\times\BF_2\setminus(0,0))\times\BF_2$. In particular, for each $(m,r,r') \in I$, $\RegressAlg$ computes
\[\wh\sk^{m,r,r'} \gets \LFC((Z^i,F^i)_{i=d/4+1}^{d/2}, m, r, r', \delta, \epsilon/(8C_{\ref{lemma:f-corr}}(h)), \eta/4),\]
where $\LFC$ was defined in \cref{lemma:learn-from-corr}.
Next, for each $(m,r,r') \in I$, $\RegressAlg$ computes a decoding function $\Dec^{m,r,r'}$ defined by $\Dec^{m,r,r'}(Z) := \Dec^{\wh\sk^{m,r,r'}}_{n,N,h}(Z)$ (\cref{def:decoder})
and a regressor $f^{m,r,r'}: \MS[h] \to \RR$ defined by
\[f^{m,r,r'}(s) := \frac{1}{\#\{d/2 < i \leq 3d/4: \Dec^{m,r,r'}(Z^i) = s\}} \sum_{i=d/2+1}^{3d/4} F^i \cdot \mathbbm{1}[\Dec^{m,r,r'}(Z^i) = s].\]
Finally, for each $(m,r,r') \in I$, the candidate predictor $\MR^{m,r,r'}$ is defined as $\MR^{m,r,r'}(Z) := f^{m,r,r'}(\Dec^{m,r,r'}(Z))$.

The output of $\RegressAlg$ is the predictor
\[\hat\MR \gets \argmin_{\MR \in \mathfrak{R}} \frac{2}{d}\sum_{i=3d/4+1}^d \left(\MR(Z^i) - F^i\right)^2\] 
where $\mathfrak{R} := \{\MR^0\}\cup\{\MR^{m,r,r'}: (m,r,r') \in I\}$.

\paragraph{Analysis.} The circuit size bound on $\hat\MR$ is immediate from the algorithm description; indeed, for any $t \in \BF_2^n$, the function $\Dec^t_{n,N,h}$ has circuit size $\poly(n,N,h)$, and each constant $f^{m,r,r'}(s)$ (as well as $\hat \alpha$) is a rational number in $[0,1]$ with denominator at most $d$, so it has circuit size $\poly(\log d)$.

We next prove that $\hat\MR$ is a $(\beta,f,\epsilon)$-predictor with respect to $\til\MD^\sk_{n,N,h,\delta}$, with probability at least $1-\eta$. By Hoeffding's inequality and the fact that all of the predictors $\MR\in\mathfrak{R}$ are independent of $(Z^i,F^i)_{i=3d/4+1}^d$, there is some event $\ME_1$ occurring with probability at least $1-\eta/8$ over the samples $(Z^i,F^i)_{i=3d/4+1}^d$ so that, under $\ME_1$, 
\[\EE_{(Z,F)\sim\mureg}(\hat\MR(Z) - F)^2 \leq \frac{2\log(8(2h+1)/\eta)}{\sqrt{d/4}} + \min_{\MR\in\mathfrak{R}} \EE_{(Z,F)\sim\mureg}(\MR(Z) - F)^2.\]
So long as $d \geq 64\epsilon^{-2} \log(8(2h+1)/\eta)$, it thus suffices to show that at least one of the predictors in $\mathfrak{R}$ has at most $\epsilon/2$ excess risk (with high probability over the first $3d/4$ samples). We proceed to show that this is the case.

By Hoeffding's inequality, since $d/4 \geq 128\epsilon^{-2}\log(16/\eta)$, we have, under some event $\ME_2$ occurring with probability at least $1-\eta/8$ over the samples $(Z^i,F^i)_{i=1}^{d/4}$, that $|\hat\alpha - \alpha| \leq \epsilon/8$. We proceed to condition the first $d/4$ samples, which fixes $\hat\alpha$ and thus the estimator $\MR^0$, and we restrict to the event that $\ME_2$ holds. Now we consider two cases. In the first case, $\MR^0$ itself has low excess risk, i.e.
\begin{align*}
\EE_{(Z,F)\sim \mureg} (\MR^0(Z) - F)^2 
&= \sum_{s \in \MS[h]} \beta(s) \left(f(s)(1-\hat\alpha)^2 + (1-f(s))\hat\alpha^2\right)\\
&\leq \frac{\epsilon}{2} + \sum_{s\in\MS} \beta(s) \Var(\Ber(f(s))).
\end{align*}
Next we consider the second case, in which the reverse inequality holds. Thus, 
\[\sum_{s\in\MS[h]} \beta(s)\left|f(s) - \hat\alpha\right| \geq \sum_{s\in\MS[h]} \beta(s)(f(s) - \hat\alpha)^2 > \frac{\epsilon}{2}.\]
 By the triangle inequality, under $\ME_2$, we then get
\[\sum_{s\in\MS[h]} \beta(s)|f(s)-\alpha| > \frac{\epsilon}{8}.\]
By \cref{lemma:f-corr}, it follows that there is some $(m,r,r') \in I$ such that 
\[\left|\Prr_{(B,F) \sim \mu^{r'}}\left[\sum_{i=1}^m B_i + rB_h + F \equiv 0 \bmod{2}\right] - \frac{1}{2}\right| > \frac{\epsilon}{8C_{\ref{lemma:f-corr}}(h)} = \frac{\delreg}{2\delta}\]
where $\mu^0 := \sum_{s \in \MS[h]} \beta(s)\nu_s \times \Ber(\alpha)$ and $\mu^1 := \sum_{s \in \MS[h]} \beta(s)\nu_{s} \times \Ber(f(s))$, and the final equality is by definition of $\delreg$. By the guarantee of \cref{lemma:learn-from-corr}, since $d/4 \geq S(n,\delreg, \eta/4)$, we have for this particular $(m,r,r')$ that \[\Pr[\wh\sk^{m,r,r'} = \sk] \geq 1-\eta/4\] where the probability is over the samples $(Z^i,F^i)_{i=d/4+1}^{d/2}$. Let the event $\{\wh\sk^{m,r,r'}=\sk\}$ be denoted by $\ME_3$. Now condition additionally on the samples $(Z^i,F^i)_{i=d/4+1}^{d/2}$, which determines $\wh\sk^{m,r,r'}$ and thus $\Dec^{m,r,r'}$. Restrict to the event that $\ME_3$ holds. Then $\Dec^{m,r,r'} \equiv \Dec^\sk_{n,N,h}$, so
\begin{align}
\Prr_{\substack{s\sim\beta \\ Z \sim \til\MD^\sk_{n,N,h,\delta}(\cdot|s)}}[\Dec^{m,r,r'}(Z) \neq s] 
&\leq h \cdot \exp(-\delta^4 N) \nonumber\\ 
&\leq \frac{\epsilon^2}{128|\MS[h]|^2} \label{eq:decoding-error}
\end{align}
where the first inequality is by \cref{lemma:decoding-error-prob} and the second inequality uses that $N \geq \delta^{-4}\log(512h^3/\epsilon^2)$. Thus, we can bound the excess risk of $\MR^{m,r,r'}$ as follows:
\begin{align}
&\EE_{(Z,F)\sim \mureg} (\MR^{m,r,r'}(Z) - F)^2 - \sum_{s \in \MS[h]} \beta(s) \Var(\Ber(f(s)))\nonumber\\
&= \sum_{s \in \MS[h]} \beta(s) \EE_{Z \sim \til\MD^\sk_{n,N,h,\delta}(\cdot|s)} (\MR^{m,r,r'}(Z) - f(s))^2 \nonumber\\ 
&= \sum_{s \in \MS[h]} \beta(s) \EE_{Z \sim \til\MD^\sk_{n,N,h,\delta}(\cdot|s)} (f^{m,r,r'}(\Dec^{m,r,r'}(Z)) - f(s))^2 \nonumber\\ 
&\leq \frac{\epsilon^2}{128|\MS[h]|^2} + \sum_{s \in \MS[h]} \beta(s)(f^{m,r,r'}(s) - f(s))^2 \nonumber\\
&\leq \frac{\epsilon}{4} + |\MS[h]|\max_{s\in\MS[h]: \beta(s) \geq \epsilon/(8|\MS[h]|)} \beta(s)(f^{m,r,r'}(s) - f(s))^2 \label{eq:mbc-excess-risk}
\end{align}
 where the first inequality uses \cref{eq:decoding-error} and holds under $\ME_3$, and the second inequality bounds each term of the summation for which $\beta(s) \leq \epsilon/(8|\MS[h]|)$ by $\epsilon/(8|\MS[h]|)$, so that in total these terms contribute at most $\epsilon/8$. But now fix any $s\in\MS[h]$ with $\beta(s) \geq \epsilon/(8|\MS[h]|)$. By Hoeffding's inequality, since $d/4 \geq 2^{15}|\MS[h]|^4\epsilon^{-4}\log(16|\MS[h]|/\eta)$, under some event $\ME_{4}^{s}$ occurring with probability at least $1-\eta/(8|\MS[h]|)$, it holds that 
\[\left|\frac{4}{d}\#\{d/2 < i \leq 3d/4: \Dec^{m,r,r'}(Z^i) = s\} - \Prr_{\substack{s'\sim\beta \\ Z\sim \til\MD^\sk_{n,N,h,\delta}(\cdot|s')}}[\Dec^{m,r,r'}(Z) = s]\right| \leq \frac{\epsilon^2}{128|\MS[h]|^2}.\]
Moreover, from \cref{eq:decoding-error}, we know that under $\ME_3$, 
\[\left|\Prr_{\substack{s'\sim\beta \\ Z\sim \til\MD^\sk_{n,N,h,\delta}(\cdot|s')}}[\Dec^{m,r,r'}(Z) = s] - \beta(s)\right| \leq \frac{\epsilon^2}{128|\MS[h]|^2}.\]
Thus, under $\ME_3 \cap \ME_4^s$,
\begin{align}
  \label{eq:e34-bound}
  \left|\frac{4}{d}\#\{d/2 < i \leq 3d/4: \Dec^{m,r,r'}(Z^i) = s\} - \beta(s)\right| \leq \frac{\epsilon^2}{128|\MS[h]|^2} \leq \frac{\epsilon}{16|\MS[h]|}\beta(s).
\end{align}
Similarly, by Hoeffding's inequality, we have that, under some event $\ME_5^s$ that occurs with probability at least $1-\eta/(8|\MS[h]|)$, that
\[\left|\frac{4}{d}\sum_{i=d/2+1}^{3d/4}F^i \cdot \mathbbm{1}[\Dec^{m,r,r'}(Z^i) = s] - \EE_{\substack{s'\sim\beta \\ Z \sim \til\MD^\sk_{n,N,h,\delta}(\cdot|s') \\ F \sim \Ber(f(s'))}} F \cdot \mathbbm{1}[\Dec^{m,r,r'}(Z) = s]\right| \leq \frac{\epsilon}{128|\MS[h]|}.\]
From \cref{eq:decoding-error}, under $\ME_3$, 
\[\left|\EE_{\substack{s'\sim\beta \\ Z \sim \til\MD^\sk_{n,N,h,\delta}(\cdot|s') \\ F \sim \Ber(f(s'))}} F \cdot \mathbbm{1}[\Dec^{m,r,r'}(Z) = s] - \beta(s) f(s) \right| \leq \frac{\epsilon^2}{128|\MS[h]|^2}.\]
Thus, under $\ME_3 \cap \ME_5^{s}$, 
\begin{align}
  \label{eq:e35-bound}
  \left|\frac{4}{d}\sum_{i=d/2+1}^{3d/4}F^i \cdot \mathbbm{1}[\Dec^{m,r,r'}(Z^i) = s] - \beta(s)f(s)\right| \leq \frac{\epsilon}{64|\MS[h]|}.
\end{align}
From the definition of $f^{m,r,r'}$, we conclude that, under $\ME_3 \cap \ME_4^{s} \cap \ME_5^{s}$, \allowdisplaybreaks
\begin{align*}
&|f^{m,r,r'}(s) - f(s)| \\
&\leq \left|\frac{\frac{4}{d}\sum_{i=d/2+1}^{3d/4}F^i\cdot\mathbbm{1}[\Dec^{m,r,r'}(Z^i) = s]}{\frac{4}{d}\#\{d/2<i\leq 3d/4: \Dec^{m,r,r'}(Z^i)=s\}} - \frac{\frac{4}{d}\sum_{i=d/2+1}^{3d/4}F^i\cdot\mathbbm{1}[\Dec^{m,r,r'}(Z^i) = s]}{\beta(s)}\right| \\ 
&\qquad+ \frac{\left|\frac{4}{d}\sum_{i=d/2+1}^{3d/4}F^i\cdot\mathbbm{1}[\Dec^{m,r,r'}(Z^i) = s] - \beta(s)f(s)\right|}{\beta(s)} \\ 
&\leq \left|\frac{1}{\frac{4}{d}\#\{d/2<i\leq 3d/4: \Dec^{m,r,r'}(Z^i)=(s)\}} - \frac{1}{\beta(s)}\right| \\ 
&\qquad+ \frac{1}{\beta(s)}\left|\frac{4}{d}\sum_{i=d/2+1}^{3d/4}F^i\cdot\mathbbm{1}[\Dec^{m,r,r'}(Z^i) = s] - \beta(s)f(s)\right| \\ 
&\leq \frac{(\epsilon/16|\MS[h]|)\beta(s)}{(1/2)\beta(s)^2} + \frac{\epsilon^2/(64|\MS[h]|)}{\beta(s)} \\ 
&\leq \frac{\epsilon}{4|\MS[h]|\beta(s)},
\end{align*}
where the third inequality uses \cref{eq:e34-bound,eq:e35-bound}. 
Let us write $\ME_4 := \bigcup_{s :\ \beta(s) \geq \ep/(8|\MS[h]|)} \ME_4^s$, and define $\ME_5$ similarly. Each of $\ME_4, \ME_5$ occurs with probability at least $1-\eta/8$. By a union bound, the above display holds for all $s\in\MS$ with $\beta(s) \geq \epsilon/(8|\MS[h]|)$ under the event $\ME_3 \cap \ME_4 \cap \ME_5$. 
So, by \cref{eq:mbc-excess-risk}, under the event $\ME_3 \cap \ME_4 \cap \ME_5$,  we have that
\[\EE_{(Z,F)\sim \mureg} (\MR^{m,r,r'}(Z) - F)^2 - \sum_{s \in \MS} \beta(s) \Var(\Ber(f(s))) \leq \frac{\epsilon}{2}.\]
This completes the second case: under event $\ME_2$, in either case, with probability at least $1-3\eta/4$, at least one of the predictors has excess risk at most $\epsilon/2$. We now apply the union bound with events $\ME_2$ and $\ME_1$, which each occur with probability at least $1-\eta/8$. Hence, with probability at least $1-\eta$, the excess risk of $\hat\MR$ is at most $\epsilon$, i.e. $\hat\MR$ is a $(\beta,f,\epsilon)$-predictor with respect to $\til \MD^\sk_{n,N,h,\delta}(\cdot|s)$, as defined in \cref{def:reg-predictor}.

\paragraph{Time complexity.} The first part of the algorithm consists of $O(h)$ calls to $\LFC$, for overall time complexity of $O(h) \cdot T(n,\delreg,\eta/4) + d \cdot \poly(n,N,h)$ (by \cref{lemma:learn-from-corr}). The second part requires computing the functions $f^{m,r,r'}$, for overall time complexity of $d \cdot \poly(n,N,h)$. The third and final part requires selecting $\hat \MR$, which again has time complexity $d \cdot \poly(n,N,h)$.
\end{proof}

The proof of \cref{lemma:realizable-regression-alg} follows almost directly by applying \cref{lemma:regression-alg} with the parameters $n,N,H$ set in \cref{def:mdp-family}, and unpacking the definition of \cref{def:regression-algorithm}. The only remaining detail is to handle the discrepancy between $\MD^\sk_{n,N,H,\delta}$ (the true emission distribution of $M^\sk_{n,N,H,\delta}$) and $\til \MD^\sk_{n,N,H,\delta}$ (the ``unconditioned'' version of the emission distribution, that shows up in \cref{lemma:regression-alg}). 

\begin{proof}[\textbf{Proof of \cref{lemma:realizable-regression-alg}}]
Let $\RegressAlg$ be the algorithm guaranteed by \cref{lemma:regression-alg}, instantiated with the algorithm $\Alg$ for learning noisy parities. Then the regression algorithm on input $(x^i,y^i)_{i=1}^d$ and $n \geq \max(n_0, C_{\ref{lemma:realizable-regression-alg}})$ simply computes and outputs $\hat \MR \gets \RegressAlg((x^i,y^i)_{i=1}^d, \delta(n),\epsilon(n))$. When $n < \max(n_0, C_{\ref{lemma:realizable-regression-alg}})$, the regression algorithm iterates over all decoding functions $\Dec^t_{n,N}$ (for $t \in \BF_2^n$), computes the corresponding optimal regressor for each decoding function, and outputs the regressor with minimal empirical loss. 

\paragraph{Analysis.} When $n < \max(n_0,C_{\ref{lemma:realizable-regression-alg}})$, it's clear that for sufficiently large sample complexity, the output regressor will be an accurate predictor per \cref{def:regression-algorithm}. The lemma statement requires no bounds on $\Sreg(n)$ and $\Treg(n)$ in this case, so we are done with this case. Now fix $n \geq \max(n_0,C_{\ref{lemma:realizable-regression-alg}})$ and $M := M^\sk_{n,N_n,H_n,\delta(n)} \in \MM(\delta)$. For notational convenience write $N=N_n$, $H=H_n$, $\delta=\delta(n)$, and $\epsilon=\epsilon(n)$. Set 
\begin{equation} \Sreg(n) := \max(4S(n,\delta\epsilon/(4C_{\ref{lemma:f-corr}}(H)),1/n), 2^{21}H^4\epsilon^{-4}\log(32Hn))
\label{eq:sreg-def}
\end{equation}
and note that it satisfies the claimed asymptotic \cref{eq:sreg}. We may assume without loss of generality that $S(n,\delta\epsilon/(4C_{\ref{lemma:f-corr}}(H)),1/n) \leq O(C_{\ref{lemma:f-corr}}(H)^2 \delta^{-2}\epsilon^{-2} n 2^n)$ (since we can interleave $\Alg$ with $\Brute$, which has time complexity bounded as in \cref{lemma:brute-force-lpn}). It follows from \cref{eq:sreg-def}, the assumption that $\delta,\epsilon \geq 2^{-n/8}$, the definition of $H=H(n)$, and the assumption that $n \geq C_{\ref{lemma:realizable-regression-alg}}$ that $\Sreg(n) \leq 2^{2n}$ so long as $C_{\ref{lemma:realizable-regression-alg}}$ is a sufficiently large absolute constant.

Fix $h \in [H]$ and $\beta\in\Delta(\MS[h])$, where $\MS[h] := \{h\}\times\{0,\dots,h-1\}\times\BF_2$ is the set of states reachable at step $h$ in $M$. By construction, the emission distribution of $M$ for any state $s \in \MS[h]$ is $\BO_h(\cdot|s) := \MD^\sk_{n,N,h,\delta}(\cdot|s)$ and the decoding function is $\Dec^\sk_{n,N}(\cdot)$. Let $(Z^i,F^i)_{i=1}^{\Sreg(n)}$ be i.i.d. samples where $Z^i \sim \sum_{s\in\MS[h]} \beta(s)\BO_h(\cdot|s)$ and $F^i \in \{0,1\}$ satisfies $F^i \perp Z^i | \Dec^\sk_{n,N}(Z^i)$. Then
\[(Z^i,F^i) \sim \sum_{s\in\MS[h]} \beta(s)\MD^\sk_{n,N,h,\delta}(\cdot|s) \times \Ber(f(s))\]
where $f(s) := \E[F^i|\Dec^\sk_{n,N,h}(Z^i)=s]$. By \cref{lemma:decoding-error-prob} and choice of $N$, it follows that the distribution of $(Z^i,F^i)_{i=1}^{\Sreg(n)}$ is within total variation distance $H\exp(-3n) \cdot \Sreg(n)$ of
\[\sum_{s\in\MS[h]} \beta(s)\til\MD^\sk_{n,N,h,\delta}(\cdot|s) \times \Ber(f(s)).\]
By definition of $H$, the previously derived bound $\Sreg(n) \leq 2^{2n}$, and the assumption that $n \geq C_{\ref{lemma:realizable-regression-alg}}$, we have that $H\exp(-3n)\cdot\Sreg(n) \leq \exp(-n)$. Thus, applying \cref{lemma:regression-alg}, and using the fact that \cref{eq:regression-alg-d-lb} is satisfied (by definition of $\Sreg(n)$) and $N \geq \delta^{-4}\log(512H^3/\epsilon^2)$ (by definition of $N, H$, the assumption that $\epsilon \geq 2^{-n/8}$, and the assumption that $n \geq C_{\ref{lemma:realizable-regression-alg}}$), we have with probability at least $1 - 4/n - \exp(-n) \geq 2/3$ that $\MR$ is a $(\beta,f,\epsilon/2)$-predictor with respect to $\til\MD^\sk_{n,N,h,\delta}$. In this event, by \cref{lemma:decoding-error-prob} and choice of $N$, $H$, $\epsilon$, we get that $\MR$ is a $(\beta,f,\epsilon)$-predictor with respect to $\MD^\sk_{n,N,h,\delta}$.

This proves correctness of the algorithm (with respect to \cref{def:regression-algorithm}). The time complexity bound \cref{eq:treg} and circuit size bound \cref{eq:breg} follow immediately from \cref{lemma:regression-alg} along with the parameter choices of $N$, $H$, and the fact that $S(n,\delta\epsilon/(4C_{\ref{lemma:f-corr}}(H)),1/n) \leq 2^{O(n)}$.
\end{proof}

\section{Reducing high-noise LPN to reward-free RL}\label{sec:lpn-to-rl}
The main result of this section is the following lemma, which shows that there is an efficient reduction from learning parities with noise level $2^{O(H^2)} \cdot \delta^2$ to reward-free reinforcement learning (as defined in \cref{def:rf-rl}) in the family of block MDPs $\MM(\delta)$ (\cref{def:mdp-family}), together with learning parities with noise level $2^{-O(H^2)} \cdot \delta$.

\begin{lemma}\label{lemma:policy-cover-to-lpn}
There is a constant $C_{\ref{lemma:policy-cover-to-lpn}}$ with the following property. Fix $n_0 \in \NN$. Let $\Alg$ be an algorithm for learning noisy parities with unknown noise level with time complexity $T(n,\delta,\eta)$ and sample complexity $S(n,\delta,\eta)$ (\cref{def:lpn-unknown-noise}). Fix a function $\delta: \NN \to (0,1/2)$ and complexity measures $\SPC, \TPC, \BPC: \NN \to \NN$ and $\alphaPC: \NN \to \RR_{>0}$ for $\MM(\delta)$. Suppose that $\PC$ is a $(\SPC,\TPC,\alphaPC,\BPC)$-policy cover learning algorithm (\cref{def:rf-rl}) for $\MM(\delta)$, and that the following bounds hold for all $n \geq n_0$: $\alphaPC(n) \geq 2^{3-H_n}$, $\SPC(n) \leq 2^n$, and $\delta(n) \in (2^{-n/4}, 1/2^{H_n^2+2H_n+7})$. 

Then there is an algorithm $\Alg'$ for learning noisy parities with sample complexity $S'(n,\delta,\eta)$ and time complexity $T'(n,\delta,\eta)$ satisfying
\[S'(n,\delsmall,1/2) \leq (n/\delta)^{C_{\ref{lemma:policy-cover-to-lpn}}} \cdot \left(\SPC(n) + S(n, \delta/(2^{H}C_{\ref{lemma:f-corr}}(H)), 1/n)\right)\]
\[T'(n,\delsmall,1/2) \leq (n\BPC(n)/\delta)^{C_{\ref{lemma:policy-cover-to-lpn}}} \cdot \left(\TPC(n) + T(n,\delta/(2^{H}C_{\ref{lemma:f-corr}}(H)),1/n)\right)\]
for all $n \geq \max(n_0, C_{\ref{lemma:policy-cover-to-lpn}})$, where $\delsmall(n) := C_{\ref{lemma:triangle-lpn}}(H)\delta(n)^2$.
\end{lemma}
In the above display, we are writing $\delsmall = \delsmall(n)$, $H=H_n$, and $\delta=\delta(n)$ for notational simplicity. Recall from \cref{def:mdp-family} that $H_n$ is the horizon of each MDP in $\MM(\delta)_n$. The proof of \cref{lemma:policy-cover-to-lpn} consists of four main ingredients:

\begin{enumerate}
    \item First, we show that the \emph{dynamic} problem of simulating interaction with the MDP $\til M^\sk_{n,N,H,\delta}$ (which is very close to $M^\sk_{n,N,H,\delta}$) can be reduced to a \emph{static} problem of simulating a variant of the learning parities with noise distribution where batches of samples have correlated noise terms (\cref{lemma:counter-trajectory}). We call this distribution ``triangle LPN'' and formally denote it by $\LPN^{\mathsf{tri}}_{n,N,H,\delta}(\sk)$ (\cref{def:mu-defn}).
    \item Second, we give a generic reduction showing that a batch of samples from any LPN variant with \emph{weakly} correlated noise terms can be simulated using standard LPN samples, so long as the batch size is small (\cref{lemma:construct-corr-lpn}) or there is certain additional structure to the joint distribution (\cref{lemma:large-batch-lpn}). Roughly speaking, if the correlations can be bounded by $\gamma$, then the batch can be simulated using samples from $\LPN_{n,O(\gamma)}(\sk)$.
    \item Third, we show that the triangle LPN distribution satisfies the above desiderata with correlation level $O(\delta^2)$ (up to a constant depending only on $H$), which means that a sample from triangle LPN can be efficiently simulated using samples from $\LPN_{n,O(\delta^2)}(\sk)$ (\cref{lemma:triangle-lpn}).
    \item Fourth, we show that given any policy $\pi$ on $M^\sk_{n,N,H,\delta}$ that visits state $(H, H-1,0)$ with probability significantly more than that of the uniform policy, we can recover $\sk$ using an algorithm for learning parities with noise level $1/2-O(\delta)$. This is achieved by a contrastive learning approach, where we draw samples from the uniform policy as well as from $\pi$ and label each sample by its origin (\cref{lemma:constrast-learn}). By the guarantee on $\pi$, the labels must be correlated with the latent state, which enables recovering $\sk$ using the $\LFC$ algorithm (which invokes the algorithm for learning parities with noise) developed in \cref{sec:regression-to-lpn}. 
\end{enumerate}

Given these ingredients, the proof of \cref{lemma:policy-cover-to-lpn} is straightforward. Given samples from $\LPN_{n,\delsmall}(\sk)$, we can approximately simulate the execution of $\PC$ on $M^\sk_{n,N,H,\delta}$, generating a policy cover. By the definition of a policy cover \cref{eq:pc}, the mixture policy that follows a random policy from the cover must visit state $(H,H-1,0)$ with much higher probability than the policy that takes uniformly random actions. Thus, we can recover $\sk$. This outline is formalized below.

\begin{proof}[\textbf{Proof of \cref{lemma:policy-cover-to-lpn}}]
For notational simplicity, we again omit the various parameters' dependencies on $n$. On input $(u_i,y_i)_{i=1}^{S'(n,\delsmall,1/2)}$ and $\delsmall$, the algorithm $\Alg'$ does the following. Note that it knows $n$ (from the length of the vectors $u_i$) and thus $N$, $H$, and $\delta$. First, it simulates the execution of $\PC$. For each episode $t$ of interaction, let $I^{(t)}$ be the index set of $|\Gamma(H)|\cdot (N+1)$ unused samples from $(u_i,y_i)_i$. The algorithm $\Alg'$ computes a triangle LPN batch
\[W^{(t)} := \TriAlg((u_i,y_i)_{i \in I^{(t)}}, H,N,\delta),\]
where $\TriAlg$ is defined in \cref{lemma:triangle-lpn}. It then implements episode $t$ using the algorithm $\DrawTraj$ (defined in \cref{lemma:counter-trajectory}) with parameters $H$, $N$, $\delta$, and input $W^{(t)}$. Eventually, $\PC$ produces a set of policies $\Psi$, where each policy $\pi \in\Psi$ is represented as a circuit $\MC_\pi$. Set $D := 2S(n,\delta/(2^H C_{\ref{lemma:f-corr}}(H)),1/n)$. Let $J^{(1)},\dots,J^{(D)}$ be (disjoint) index sets each containing $|\Gamma(H)| \cdot (N+1)$ unused samples from $(u_i,b_i)_i$. For $1 \leq d \leq D$, the algorithm computes a triangle LPN batch 
\[\overline W^{(d)} := \TriAlg((u_i,y_i)_{i \in J^{(d)}}, H,N,\delta).\]
The algorithm then computes the set $\hat{\mathcal{L}} := \CL((\overline W^{(d)})_{d=1}^D, (\MC_\pi)_{\pi \in \Psi}, N, H, \delta)$. Finally, let $K$ be the index set of $9\delsmall^{-2}\log 8(H+1)n$ unused samples from $(u_i,y_i)_i$. The algorithm computes and outputs $\hat \sk := \Select((u_i,b_i)_{i \in K}, \hat{\mathcal{L}})$.

\paragraph{Correctness analysis.} Fix $n \geq n_0$. Suppose that the samples $(u_i,y_i)_{i=1}^{S'(n,\delsmall,1/2)}$ are drawn independently from $\LPN_{n,\delsmall}(\sk)$ for some $\sk \in \BF_2^n$. We claim that $\Alg'$ (approximately) simulates $\PC$ on the MDP $M^\sk_{n,N,H,\delta}$. Indeed, for each episode $t$ of simulated interaction, \cref{lemma:triangle-lpn} (using the assumption on $\delta(n)$) guarantees that the batch $W^{(t)}$ produced by $\TriAlg$ is distributed according to $\LPN^{\mathsf{tri}}_{n,N,H,\delta}(\sk)$ (\cref{def:mu-defn}). Similarly, it guarantees that $(\overline W^{(d)})_{d=1}^D$ are independent samples from $\LPN^{\mathsf{tri}}_{n,N,H,\delta}(\sk)$, which we will use later.

By the guarantee on $\DrawTraj$ (\cref{lemma:counter-trajectory}), the $t$-th episode of simulated interaction with $\PC$ is distributed according to the MDP $\til M^\sk_{n,N,H,\delta}$ (meaning that each new emission has the correct distribution according to the episodic RL model for $\til M^\sk_{n,N,H,\delta}$, conditioned on the episode's history). By \cref{lemma:decoding-error-prob}, the distribution of each simulated emission, conditioned on the emission history, is within total variation distance $H\exp(-\delta^4 N)$ of the distribution under $M^\sk_{n,N,H,\delta}$. Summing over all steps and all episodes, the total variation error of the simulation of $\PC$ (compared to an execution of $\PC$ with interactive access to $M^\sk_{n,N,H,\delta}$) can be bounded by $H^2\exp(-\delta^4 N) \cdot \SPC(n)$, which is at most $\exp(-n)$ by choice of $H$ and $N$, and the assumption that $\SPC(n) \leq 2^n$. Thus, unpacking the guarantee on $\PC$ (\cref{def:rf-rl}), there is an event $\ME_1$ that occurs with probability at least $2/3 - \exp(-n)$, under which the set of policies $\Psi$ produced by the simulation is a $(\alphaPC(n),1/4)$-policy cover (\cref{def:pc}) for $M^\sk_{n,N,H,\delta}$. In particular, by the fact that the latent state $(H,H-1,0)$ in $M^\sk_{n,N,H,\delta}$ is reachable at step $H$ with probability $1/2$ (\cref{lemma:optimal-policy}), we have under $\ME_1$ that \begin{equation}\frac{1}{|\Psi|} \sum_{\pi \in \Psi} d^{M^\sk_{n,N,H,\delta},\pi}_H(H,H-1,0) \geq \alphaPC(n) \cdot \left(\frac{1}{2}-\frac{1}{4}\right) \geq 2^{1-H},\label{eq:exploratory-policy-exists}\end{equation}
where the first inequality is by \cref{eq:pc} and the second inequality is by the lemma assumption on $\alphaPC$.

From now on, condition on $\ME_1$, so that \cref{eq:exploratory-policy-exists} holds. Then we can apply \cref{lemma:constrast-learn} to the output $\hat \ML$ of $\CL$. Using also the fact that $(\overline W^{(d)})_{d=1}^D$ are independent samples from $\LPN^{\mathsf{tri}}_{n,N,H,\delta}(\sk)$ and the choices of $N$ and $D$, we get that $|\hat \ML| \leq 4(H+1)$, and that there is an event $\ME_2$ under which $\sk \in \hat \ML$, and moreover
\[\Pr[\ME_2] \geq 1 - 1/n - O(\delta^{-2}2^{2H}C_{\ref{lemma:f-corr}}(H)^2 H^2 \exp(-2n)).\] Now condition additionally on $\ME_2$, so that $\sk \in \hat\ML$. Then the guarantee of $\Select$ (\cref{lemma:select-alg}) applies, so there is an event $\ME_3$ which occurs with probability at least $1-1/n$, under which the output of $\Select$ is $\sk$. In this event, the output of $\Alg'$ is $\sk$ as well. Thus, $\Alg'$ succeeds in the event $\ME_1\cap\ME_2\cap\ME_3$, which by the union bound occurs with probability at least
\[2/3 - \exp(-n) - 2/n - O(\delta^{-2}2^{2H}C_{\ref{lemma:f-corr}}(H)^2H^2\exp(-2n)) \geq 1/2.\]
The final inequality holds by assumption that $\delta \geq 2^{-n/4}$ and definition of $H=H(n)$, so long as $n \geq n_0$ and $n \geq C_{\ref{lemma:policy-cover-to-lpn}}$ where $C_{\ref{lemma:policy-cover-to-lpn}}$ is chosen to be a sufficiently large (absolute) constant.

\paragraph{Time complexity.} For notational simplicity let $M$ refer to $M^\sk_{n,N,H,\delta}$. By assumption, $\PC$ uses at most $\SPC(M)$ episodes of interaction. By \cref{lemma:triangle-lpn}, constructing the triangle LPN sample for each episode takes time $\poly(n,N,2^{H^2})$, and implementing the trajectory takes time $\poly(n,N,H)$. Thus, simulating $\PC$ takes time $\TPC(M) + \poly(n,N,2^{H^2}) \cdot \SPC(M)$. Next, again by \cref{lemma:triangle-lpn}, computing $(\overline W^{(d)})_{d=1}^D$ takes time $\poly(n,N,2^{H^2}) \cdot S(n,\delta/(2^H C_{\ref{lemma:f-corr}}(H)),1/n)$. By assumption, the circuit representation of each $\pi \in \Psi$ has size at most $\BPC(M)$, and $|\Psi| \leq \TPC(M)$, so by \cref{lemma:constrast-learn}, the invocation of $\CL$ takes time 
\begin{align*}
&O(\BPC(M) \cdot \TPC(M)) + O(H) \cdot T(n,\delta/(2^H C_{\ref{lemma:f-corr}}(H)), 1/n) \\&+ \poly(n,N,\BPC(M),H) \cdot S(n,\delta/(2^H C_{\ref{lemma:f-corr}}(H)), 1/n).
\end{align*}
Finally, since $|K| = \poly(\delta^{-1}, \log(Hn))$ (note that $\delsmall^{-1} \leq \delta^{-2}$), and $|\hat\ML| \leq O(H)$, the invocation of $\Select$ takes time $\poly(\delta^{-1}, H, N)$. Combining the above bounds and using that $H = (\log n)^{1/3}$ and $N = 3\delta^{-4} n$ (\cref{def:mdp-family}) and also (without loss of generality) $S \leq T$ pointwise, we get that the overall time complexity of $\Alg'$ is
\[\poly(\delta^{-1},n,\BPC(M)) \cdot \left(\TPC(M) + T(n,\delta/(2^HC_{\ref{lemma:f-corr}}(H)),1/n)\right)\]
as claimed.

\paragraph{Sample complexity.} Constructing $(W^{(t)})_t$ requires $|\Gamma(H)| \cdot (N+1) \cdot \SPC(M)$ samples, and constructing $(\overline W^{(d)})_{d=1}^D$ requires $|\Gamma(H)| \cdot (N+1) \cdot 2S(n,\delta/(2^H C_{\ref{lemma:f-corr}}(H)),1/n)$ samples. The invocation of $\Select$ uses $\poly(\delta^{-1},\log(Hn))$ samples. Since $|\Gamma(H)| = O(H^2)$ and by definition of $H$ and $N$, the claimed sample complexity bound follows.
\end{proof}

\subsection{From static to dynamic: simulating a trajectory}

We next show that an episode of interaction with the MDP $\til M^\sk_{n,N,H,\delta}$ (defined in \cref{sec:construction}) can be simulated using a sample from a \emph{static} distribution that we call ``triangle LPN'', defined below. A sample from the triangle LPN distribution consists of a batch of LPN samples where the noise terms are correlated, and in particular are jointly drawn from the following distribution $\mu_{H,N,\delta}$.


\begin{definition}
  \label{def:mu-defn}
For $H,N \in \NN$ and $\delta \in (0,1/2)$, let $\Gamma(H)$ denote the set $\{i,j: 1 \leq j \leq i \leq H\}$. We define $\mu_{H,N,\delta} \in \Delta(\BF_2^{\Gamma(H)\times[N+1]})$ to be the distribution of the (partial) random tensor 
$\Xi$ defined by
\[\Xi^{i,j} := \begin{bmatrix} e_{ij} + b_j \\ e'_{ij1} + b_j \\ e'_{ij2} + b_j \\ \vdots \\ e'_{ijN} + b_j \end{bmatrix} \in \BF_2^{N+1} \qquad \forall\,(i,j) \in \Gamma(H)\]
where $b_1,\dots,b_H \sim \Ber(1/2)$ and $e_i \sim \CBer(i, \delta)$ (for all $i \in [H]$) and $e'_{ijk} \sim \Ber(1/2 - 2\delta^2)$ (for all $1 \leq j \leq i \leq H$ and $k \in [N]$) are all independent. We define the \emph{triangle LPN} distribution $\LPN^{\mathsf{tri}}_{n,N,H,\delta}(\sk)$ with secret key $\sk \in \BF_2^n$ and parameters $H,N,\delta$ to be the distribution of the partial random tensor $(u_{ijk},y_{ijk})_{(i,j,k) \in \Gamma(H)\times[N+1]}$ where
\[(u_{ijk},y_{ijk}-\langle u_{ijk}, \sk\rangle)_{(i,j,k) \in \Gamma(H)\times[N+1]} \sim \Unif(\BF_2^n)^{\otimes \Gamma(H)\times[N+1]} \times \mu_{H,N,\delta}.\]


With the above notation, we also let $\mu^0_{H,\delta}$ be the distribution of the random variable $Z = (Z_{ij})_{1 \leq j \leq i \leq H}$ where $Z_{ij} = e_{ij} + b_j$ for $1 \leq j \leq i \leq H$.
\end{definition}


\begin{lemma}[Drawing a trajectory]\label{lemma:counter-trajectory}
There is an interactive algorithm $\DrawTraj$ with the following property. Let $n,N,H \in \NN$ and $\sk \in \BF_2^n$, and $\delta \in (0, 1/2^{|\Gamma(H)|+2H+7})$. After initialization, the algorithm outputs an emission $x_1$ and sets $h \gets 2$. Then, so long as $h \leq H$, when the algorithm next receives an action $a_{h-1}$, it outputs an emission $x_h$ and sets $h \gets h+1$. 

The correctness guarantee of $\DrawTraj$ is the following. Suppose that $(u_{ijk},y_{ijk})_{(i,j,k) \in \Gamma(H)\times[N+1]}$ is a random variable distributed according to $\LPN^{\mathsf{tri}}_{n,N,H,\delta}(\sk)$. Then $\DrawTraj$ with initial input $(u_{ijk},y_{ijk})_{(i,j,k) \in \Gamma(H)\times[N+1]}$ implements an episode of interaction with $\til M^\sk_{n,N,H,\delta}$ under the episodic RL access model.\footnote{We only defined this model explicitly for block MDPs (\cref{sec:block-rl}); however, the definition does not use the block assumption in any way, and thus makes sense also for general partially observable MDPs such as $\til M^\sk_{n,N,H,\delta}$.} 

Moreover, the time complexity of $\DrawTraj$ is $\poly(n,N,H)$.
\end{lemma}

\begin{proof}
On input $(u_{ijk},y_{ijk})_{(i,j,k) \in \Gamma(H)\times[N+1]}$, we define $\DrawTraj$ as follows. The initial emission is 
\[x_1 := \begin{bmatrix} u_{111} & y_{111} & (u_{11i}, y_{11i})_{i=2}^{N+1} \end{bmatrix}.\]
Consider any step $2 \leq h \leq H$ and suppose that the actions received so far are $a_1,\dots,a_{h-1}$. Then the emission at step $h$ is
\[x_h := \begin{bmatrix}
u_{h,\sigma(1),1} & y_{h,\sigma(1),1} + a_{\sigma(1)} + 1 & (u_{h,\sigma(1),i}, y_{h,\sigma(1),i} + a_{\sigma(1)} + 1)_{i=2}^{N+1} \\ 
\vdots & \vdots & \vdots \\ 
u_{h,\sigma(h-1),1} & y_{h,\sigma(h-1),1} + a_{\sigma(h-1)} +1 & (u_{h,\sigma(h-1),i}, y_{h,\sigma(h-1),i} + a_{\sigma(h-1)} + 1)_{i=2}^{N+1} \\
u_{h,h,1} & y_{h,h,1} & (u_{h,h,i}, y_{h,h,i})_{i=2}^{N+1}
\end{bmatrix}\]
for an independent and uniformly random permutation $\sigma: [h-1]\to[h-1]$ (re-sampled at each step).

\paragraph{Analysis.} The claimed time complexity bound is immediate from the algorithm definition. We now prove correctness by arguing that $\DrawTraj$ is \emph{implicitly} simulating a latent state trajectory. Consider the first emission. By definition of $\LPN^{\mathsf{tri}}_{n,N,H,\delta}(\sk)$, we know that $(u_{11k},y_{11k})_{k\in[N+1]}$ has the same distribution as
\[\begin{bmatrix}
u_{111} & \langle u_{111}, \sk\rangle + e_{11} + b_1 \\ 
u_{112} & \langle u_{112}, \sk\rangle + e_{111} + b_1 \\ 
\vdots \\ 
u_{1,1,N+1} & \langle u_{1,1,N+1} , \sk \rangle + e_{1,1,N} + b_1
\end{bmatrix}\]
where $b_1 \sim \Ber(1/2)$, $e_{11} \sim \CBer(1,\delta) \equiv \Ber(1/2-\delta)$ and $e_{11k} \sim \Ber(1/2-2\delta^2)$ (for $1 \leq k \leq N$) are all independent. This is exactly distributed as $\frac{1}{2}\til\MD^\sk_{n,N,1,\delta}(\cdot|1,0,0) + \frac{1}{2}\til\MD^\sk_{n,N,1,\delta}(\cdot|1,0,1)$ (per \cref{def:emission}), which is the emission distribution of the random state $(1,0, b_1)$ at step $1$ in $\til M^\sk_{n,N,H,\delta}$.

Now consider any step $2 \leq h \leq H$, and condition on the history $x_{1:h-1},a_{1:h-1}$ and realizations of previously-defined random variables $b_1,\dots,b_{h-1}$. We claim that $x_h$ is distributed according to $\til \MD^\sk_{n,N,h,\delta}(\cdot|h,k,b_h)$ where $k := \sum_{i=1}^{h-1}\mathbbm{1}[b_i=a_i]$ and $b_h \sim \Ber(1/2)$. Indeed, by definition of $\LPN^{\mathsf{tri}}_{n,N,H,\delta}(\sk)$, the emission $x_h$ has the same distribution as
{\renewcommand*{\arraystretch}{2}

\[\begin{bmatrix}
u_{h,\sigma(1),1} & \substack{\langle u_{h,\sigma(1),1}, \sk\rangle + e_{h,\sigma(1)} \\+ b_{\sigma(1)} + a_{\sigma(1)} + 1} & \left(u_{h,\sigma(1),i}, \substack{\langle u_{h,\sigma(1),i},\sk\rangle + e'_{h,\sigma(1),i} \\+ b_{\sigma(1)} + a_{\sigma(1)} + 1}\right)_{i=2}^{N+1} \\ 
\vdots & \vdots & \vdots \\ 
u_{h,\sigma(h-1),1} & \substack{\langle u_{h,\sigma(h-1),1},\sk\rangle + e_{h,\sigma(h-1)} \\+ b_{\sigma(h-1)} + a_{\sigma(h-1)} + 1} & \left(u_{h,\sigma(h-1),i}, \substack{\langle u_{h,\sigma(h-1),i},\sk\rangle + e'_{h,\sigma(h-1),i} \\+ b_{\sigma(h-1)} + a_{\sigma(h-1)} + 1}\right)_{i=2}^{N+1} \\
u_{h,h,1} & \langle u_{h,h,1}, \sk\rangle + e_{h,h} + b_h & (u_{h,h,i}, \langle u_{h,h,i}, \sk\rangle + e'_{h,h,i} + b_h)_{i=2}^{N+1}
\end{bmatrix}\]}where $\sigma \sim \Unif(\Sym_{h-1})$, $u_{hji} \sim \Unif(\BF_2^n)$ (for all $j,i$), $b_h \sim \Ber(1/2)$, $e_h \sim \CBer(h,\delta)$, and $e'_{hji} \sim \Ber(1/2-2\delta^2)$ (for all $j,i$) are all independent (here, $\Sym_{h-1}$ is the set of permutations on $[h-1]$). 

Let us additionally condition on $(u_{h,\sigma(j),i})_{(j,i) \in [h]\times[N+1]}$; doing so does not affect the distribution of the remaining random variables (in particular, it does not affect the distribution of $\sigma$ \--- this uses permutation invariance of the joint distribution of $(u_{hji})_{j,i}$). Since $\sigma$ is a uniformly random permutation on $[h-1]$, the vector \[(b_{\sigma(1)}+a_{\sigma(1)}+1,\dots,b_{\sigma(h-1)}+a_{\sigma(h-1)}+1)\] is uniformly random among the vectors in $\BF_2^{h-1}$ of Hamming weight $k$, so that the vector $(b_{\sigma(1)}+a_{\sigma(1)}+1,\dots,b_{\sigma(h-1)}+a_{\sigma(h-1)}+1, b_h)$ has distribution $\nu_{h,k,b_h}$. Let additionally condition on $\sigma$ and $b_h$, which determines the vector $(b_{\sigma(1)} + a_{\sigma(1)} + 1, \ldots,  b_{\sigma(h-1)} + a_{\sigma(h-1)} + 1, b_h)$. Since the distribution $\CBer(h,\delta)$ is invariant under permutations, it is evident that $(e_{h,\sigma(1)}, \ldots, e_{h,\sigma(h-1)}, e_{h,h}) \sim \CBer(h, \delta)$. Finally, if we additionally condition on $(e_{h,\sigma(1)}, \ldots, e_{h,\sigma(h-1)}, e_{h,h})$, then all the random variables $e'_{h,\sigma(j),i}$ and $e'_{h,h,i}$ are still independent and distributed according to $\Ber(1/2 - 2\delta^2)$ (for $1 \leq j \le h-1$ and $2 \leq i \leq N+1$).
Summarizing, from \cref{def:emission}, we have that the step-$h$ emission $x_h$ has distribution $\til\MD^\sk_{n,N,h,\delta}(\cdot|h,k,b_h)$, which is indeed the step-$h$ emission distribution of the (random) state $(h,k, b_h)$ in the (partially observable) MDP $\til M^\sk_{n,N,H,\delta}$.
\end{proof}

\subsection{Learning parities with weakly dependent noise}\label{sec:dependent-lpn}

In this section we present generic algorithms for constructing a batch of LPN samples with (weakly) \emph{dependent} noise from standard, fully independent LPN samples. Equivalently, these algorithms can be thought of as reductions: they prove computational hardness of learning from batches of LPN samples \--- where the joint noise distribution of each batch is near-uniform but not necessarily a product distribution \--- under a standard LPN hardness assumption. The main results of the section are \cref{lemma:construct-corr-lpn} and \cref{lemma:large-batch-lpn}; the former lemma illustrates the key techniques, and the latter lemma is a variant needed for technical reasons.

Concretely, \cref{lemma:construct-corr-lpn} formalizes \cref{lemma:batch-lpn}. It shows that for any joint distribution $p$ on $\BF_2^k$ that is a $\delta$-Santha-Vazirani source (\cref{def:sv-source}) \--- meaning that each bit is within $\delta$ of uniform conditioned on all previous bits \--- a batch of $k$ LPN samples with independent noise $\Ber(1/2 - 2^{k+2}\delta)$ can be efficiently converted into a batch of $k$ samples where the noise follows the given distribution $p$. Note that the bias needed in the input samples scales only linearly in $\delta$, though it scales exponentially in the batch size $k$. Larger bias in the input samples makes the conversion easier; in the limiting case, if the bias were $1/2$, then one could efficiently compute $\sk$ from the input samples, after which constructing the batch would be trivial. The challenge is in showing that a small bias suffices.

\begin{lemma}\label{lemma:construct-corr-lpn}
There is an algorithm $\EntLPN$ with the following property. Fix $n,k \in \NN$, $\delta\in(0,1/2^{k+3})$, and $p \in \Delta(\BF_2^k)$. Suppose that \[\Prr_{Z\sim p}[Z_i=1|Z_{<i}=z_{<i}] \in [1/2-\delta,1/2+\delta]\] for all $z \in \BF_2^k$ and $i \in [k]$. For every $\sk \in \BF_2^n$, for independent samples $(a_i,y_i)_{i=1}^k$ from $\LPN_{n,2^{k+2}\delta}(\sk)$, the output of $\EntLPN((a_i,y_i)_{i=1}^k, p, \delta)$ is $(a'_i,y'_i)_{i=1}^k$ where
\[(a'_1,\dots,a'_k,(y'_1 - \langle a'_1,\sk\rangle, y'_2 - \langle a'_2,\sk\rangle, \dots, y'_k - \langle a'_k,\sk\rangle)) \sim \Unif(\BF_2^n)^{\otimes k} \times p.\]
Moreover, the time complexity of $\EntLPN$ is $\poly(n, 2^k)$.
\end{lemma}

\paragraph{Proof sketch: adding fresh noise?} An obvious attempt at proving this lemma would be to show that there is some distribution $\til p \in \Delta(\BF_2^k)$ so that if $X \sim \Ber(1/2 - 2^{k+2}\delta)^{\otimes k}$ and $Y \sim \til p$ are independent, then $Z = X+Y$ is distributed according to $p$. If this were true, one could construct $(a'_i,y'_i)_{i=1}^k$ by simply adding some appropriate (joint) noise to the vector $y$ of responses in the input batch $(a_i,y_i)_{i=1}^k$. Unfortunately, this is false even for $k=2$, and the counterexample showed up even in the warm-up technical overview, \cref{sec:warmup-overview}. Reparametrizing that example to match the current notation, suppose that $p$ is the distribution of the random vector $Z = (e_1 + b, e_2 + b) \in \BF_2^2$ where $e_1, e_2 \sim \Ber(1/2-\sqrt{\delta})$ and $b \sim \Ber(1/2)$ are independent. Then $Z_2|Z_1=0$ and $Z_2|Z_1=1$ are both $O(\delta)$-close to $\Ber(1/2)$. However, since $\TV(Z_1+Z_2, \Ber(1/2)) = \Omega(\delta)$, the data processing inequality implies that if we have $Z=X+Y$ for independent random variables $X,Y$, then \[\TV(X_1+X_2, \Ber(1/2)) \geq \TV(X_1+Y_1+X_2+Y_2, \Ber(1/2)+Y_1+Y_2) = \TV(Z_1+Z_2, \Ber(1/2)) = \Omega(\delta).\] Thus, if the input noise distribution is $X \sim \Ber(1/2-\eta)^{\otimes 2}$, then (by e.g. \cref{lem:bernoulli-convolve}) the bias $\eta$ must be $\Omega(\sqrt{\delta})$. Unfortunately, it's crucial for the application of the reduction that the input noise distribution has bias $O(\delta)$, so this approach does not work.

For this particular choice of $p$, there is a simple transformation (described in \cref{sec:warmup-overview}) to construct the desired batch, with joint noise distribution $p$, from a single sample from $\LPN_{n,2\delta^2}(\sk)$, but it's not immediately clear how to generalize that transformation even slightly \--- e.g., if $p$ is instead the distribution of $Z = (e_1+b, e_2+b, e_3+b)$ for $e_1,e_2,e_3 \sim \Ber(1/2-\delta)$ and $b \sim \Ber(1/2)$. 

The takeaway of the above analysis seems to be that we need to add noise that is \emph{correlated} with the noise in the input LPN samples. Indeed, suppose that we have managed to construct a batch of $k-1$ LPN samples $(a'_i,y'_i)_{i=1}^{k-1}$ with joint noise distribution $p_{1:k-1}$, so all that remains is to construct a random variable $(a'_k,y'_k)$ where $a'_k \sim \Unif(\BF_2^n)$ and $y'_k$ has the appropriate conditional distribution: \[y'_k - \langle a'_k, \sk\rangle \sim \Ber(\Pr_{Z\sim p}[Z_k=1|Z_j = y'_j-\langle a'_j,\sk\rangle \, \forall j \in [k-1]]).\]
By assumption on $p$, this Bernoulli random variable has bias at most $\delta$. Thus, if we knew the bias (say, $\delta'$), then we could construct $(a'_k,y'_k)$ from a fresh input sample $(a_k,y_k) \sim \LPN_{n,2^{k+2}\delta}(\sk)$ by adding $y_k$ to an independent Bernoulli random variable with bias $\delta'/(2^{k+3}\delta)$. This gets around the above obstacle because $\delta'$ implicitly depends on the noise terms $(y_i-\langle a_i,\sk\rangle)_{i=1}^{k-1}$ of the input samples. Unfortunately, it therefore also depends on $\sk$, so we cannot hope to learn it.

\paragraph{Solution: adding affine noise (in $\sk$).} The first insight (previously employed for key-dependent message cryptography \cite{applebaum2009fast}) is to observe that we can sometimes ``simulate'' adding something to the noise term of an LPN sample $(a_k,y_k)$ that depends on $\sk$, by appropriately adjusting not just $y_k$ but also $a_k$. In particular, suppose that $(a_k, y_k) \sim \LPN_{n,\delta}(\sk)$ and $F(\sk) = \alpha + \langle \beta,\sk\rangle$ where $\alpha,\beta$ are independent of $(a_k,y_k)$. Then the random variable
\[(a'_k,y'_k) := (a_k - \beta, y_k + \alpha)\]
satisfies that $a'_k \sim \Unif(\BF_2^n)$ and moreover that, conditioned on $a'_k$,
\[y'_k - \langle a'_k, \sk\rangle = y_k + \alpha - \langle a_k - \beta, \sk\rangle = y_k - \langle a_k,\sk\rangle + F(\sk).\]
Thus, for all intents and purposes, we have ``added'' the affine function $F(\sk)$ to the noise term. It remains to argue that we can write the random variable $\Ber(1/2 - \delta'/(2^{k+3}\delta))$ as a (random) affine function in $\sk$.

\paragraph{Linearization.} Note that $1/2 - \delta'/(2^{k+3}\delta)$ is some known function (say, $q$) of the unknown vector $z := (y'_j-\langle a'_j,\sk\rangle)_{j=1}^{k-1}$. Thinking of $z$ as an affine function in $\sk$, we can write down the coefficients of any affine function that is a linear combination of $z_1,\dots,z_{k-1}$ (along with a constant term), i.e. $F_0 + F_1z_1 + \dots + F_{k-1}z_{k-1}$. Thus, it suffices to show that there is a distribution $\mu$ over such linear combinations so that
\begin{equation} F_0 + F_1 z_1 + \dots + F_{k-1} z_{k-1} \sim \Ber(q(z))\label{eq:fz-q-constraint}\end{equation}
for all fixed $z$. We accomplish this by a perturbation argument. It's easy to see that if $q$ is the constant function $q(z) = 1/2$, then we can simply define $\mu := \Ber(1/2)^{\otimes k}$. Of course $q$ may not be constant, but by assumption we know that $q$ is entry-wise \emph{close} to constant. Moreover, for each $z$ the induced constraint \cref{eq:fz-q-constraint} on $\mu$ is linear in the density function of $\mu$, and we can show that the system of constraints (across all $z$) is essentially well-conditioned. Hence, a small perturbation to $q$ preserves satisfiability of the system. Some care has to be taken to ensure that e.g. the perturbed $\mu$ is still a distribution, but this is the main idea.

We proceed to the formal proof of \cref{lemma:construct-corr-lpn}, starting with the following least singular bound that will be needed for the perturbation argument.

\begin{lemma}
  \label{lem:b-conditioning}
Fix $k \in \NN$ and let $B \in \RR^{2^k-1 \times 2^k-1}$ be the matrix with rows and columns indexed by nonzero vectors in $\{0,1\}^k$, where $B_{uv} = \langle u,v\rangle \bmod{2}$ for any nonzero vectors $u,v \in \{0,1\}^k$. Then we have the following properties (note that we are treating $B$ as a real-valued matrix, not a matrix over $\BF_2$):
\begin{itemize}
    \item $B\mathbbm{1} = 2^{k-1}\mathbbm{1}$.
    \item The least singular value of $B$ satisfies $\sigma_{\min}(B) \geq 2^{(k-2)/2}$.
\end{itemize}
\end{lemma}

\begin{proof}
For any nonzero vector $v \in \{0,1\}^k$, the number of $u \in \{0,1\}^k$ such that $\langle u,v\rangle \equiv 1 \bmod{2}$ is exactly $2^{k-1}$. The first property immediately follows. Next, for any distinct nonzero vectors $u,v \in \{0,1\}^k$, we observe that the number of $w \in \{0,1\}^k$ such that $\langle u,w\rangle \equiv \langle v,w \rangle \equiv 1 \bmod{2}$ is exactly $2^{k-2}$. Thus, $B^\t B = 2^{k-2} \mathbbm{1}\mathbbm{1}^\t + 2^{k-2} I$. It follows that $\lambda_{\min}(B^\t B) \geq 2^{k-2}$ and thus $\sigma_{\min}(B) \geq 2^{(k-2)/2}$.
\end{proof}

Using \cref{lem:b-conditioning}, we can now formally prove that \cref{eq:fz-q-constraint} is solvable for any near-uniform function $q$. Note that the distance to uniform needs to be exponentially small in the batch size $k$, hence the factor of $2^{k+2}$ in the bias of the input LPN samples in \cref{lemma:construct-corr-lpn}. It is an interesting question whether this dependence can be removed.

\begin{lemma}\label{lemma:cond-dist-linearization}
Fix $k \in \NN$, and pick any function $q: \BF_2^k \to [1/2 - 2^{-(k+3)}, 1/2+2^{-(k+3)}]$. Then there is a random variable $F = (F_0,\dots,F_k)$ on $\BF_2^{k+1}$ such that for all $z \in \BF_2^k$, it holds that
\[\Pr(F_0 + z_1 F_1 + \dots + z_k F_k \equiv 1 \bmod{2}) = q(z).\]
Moreover, there is an algorithm that takes $q$ as input and samples $F$ in time $2^{O(k)}$.
\end{lemma}

\begin{proof}
We represent $q$ as an element of $\RR^{\BF_2^k}$, and we also identify $\Delta(\BF_2^{k+1})$ with the simplex in $\RR^{\BF_2^{k+1}}$. Throughout this proof, we let $\mathbbm{1}$ denote the all-ones real vector, of appropriate dimension per the context. Set $\delta := \norm{q - \frac{1}{2}\mathbbm{1}}_\infty \leq 2^{-(k+3)}$. Let $A \in \RR^{2^{k+1} \times 2^k}$ be the matrix with rows indexed by $\BF_2^{k+1}$ and columns indexed by $\BF_2^k$, so that for any $f \in \BF_2^{k+1}$ and $z \in \BF_2^k$, we define $A_{fz} := \mathbbm{1}[f_0+f_1z_1+\dots+f_kz_k \equiv 1 \bmod{2}]$. Then for any distribution $\mu \in \Delta(\BF_2^{k+1})$, it is clear that for all $z \in \BF_2^k$,
\begin{equation}
\Prr_{F\sim \mu}(F_0+z_1F_1+\dots+z_kF_k \equiv 1 \bmod{2}) = (A^\t \mu)_z.\label{eq:fz-pr-linear}\end{equation}
We construct a distribution $\mu^\st \in \Delta(\BF_2^{k+1})$ satisfying $A^\t \mu^\st = q$ as follows. First, define $\bar \mu \in \Delta(\BF_2^{k+1})$ by 
\[\bar\mu := \Ber(q(\mathbf{0})) \times \Ber(1/2)^{\otimes k}\]
where $\mathbf{0}$ denotes the all-0s vector (taken to be of appropriate dimension per the context).  
Next, define $\MZ = \BF_2^k \setminus \mathbf{0}$ and $\MF = \{0\} \times \MZ \subseteq \BF_2^{k+1}$, and define the matrix $B := (A_{\MF\MZ})^\t$. Note that $B$ is exactly the matrix defined in \cref{lem:b-conditioning}. Let $q_\MZ \in \BR^\MZ$ denote the vector whose entry corresponding to $z \in \MZ$ is $q(z)$. We construct $\mu^\st$, viewed as a vector in $\BR^{\BF_2^{k+1}}$, by separately defining its components: $\mu^\st_\MF$ for entries in $\MF$, $\mu^\st_{\mathbf{0}}$ for the $\mathbf{0}$-entry, and all other entries: 
\begin{align} \mu^\st_\MF &:=  \bar\mu_\MF + B^{-1} \left(q_\MZ - \frac{1}{2}\mathbbm{1}\right)\label{eq:pdfmu1} \\
\mu^\st_{\mathbf{0}} &:=  \bar\mu_{\mathbf{0}} + \mathbbm{1}^\t B^{-1} \left(\frac{1}{2}\mathbbm{1} - q_\MZ\right)\label{eq:pdfmu2}\\
  \mu^\st_f &:=  \bar\mu_f \qquad \forall f \in \BF_2^{k+1} \setminus (\MF \cup \{\mathbf{0}\}).\label{eq:pdfmu3}
\end{align}
We must first show that $\mu^\st$ is a well-defined distribution on $\BF_2^{k+1}$. First, by \cref{lem:b-conditioning} and the fact that $B$ is a symmetric matrix, we have that $B$ is invertible. Moreover, 
\begin{align}\label{eq:bdiff}\norm{B^{-1}\left(q_\MZ - \frac{1}{2}\mathbbm{1}\right)}_\infty \leq \norm{B^{-1}\left(q_\MZ - \frac{1}{2}\mathbbm{1}\right)}_2 \leq 2^{(2-k)/2} \norm{q_\MZ - \frac{1}{2}\mathbbm{1}}_2 \leq 2\norm{q_\MZ - \frac{1}{2}\mathbbm{1}}_\infty \leq 2\delta\end{align}
where the second inequality is by \cref{lem:b-conditioning} and the final inequality is by definition of $\delta$. Also, by construction, we know that $\frac{\min(q(\mathbf{0}), 1-q(\mathbf{0}))}{2^k}\leq \bar\mu_f \leq \frac{1}{2^k}$ for all $f \in \BF_2^{k+1}$. It follows that for all $f \in \MF$,
\[0 \leq \frac{1/2-\delta}{2^k} - 2\delta \leq \frac{\min(q(\mathbf{0}), 1-q(\mathbf{0}))}{2^k} - 2\delta \leq \mu^\st_f \leq \frac{1}{2^k} + 2\delta \leq 1,\]
where the first inequality is because $(2^{k+2}+2)\delta \leq 1$, the second inequality is by definition of $\delta$, the third and fourth inequalities use \cref{eq:pdfmu1} and \cref{eq:bdiff}, and the final inequality is because $\delta \leq 1/4$. Next, since $B\mathbbm{1} = 2^{k-1} \mathbbm{1}$ (\cref{lem:b-conditioning}), we have $B^{-1}\mathbbm{1} = 2^{1-k}\mathbbm{1}$, so \[|\mu^\st_{\mathbf{0}}-\bar\mu_{\mathbf{0}}| = 2^{1-k}\left|\frac{\langle\mathbbm{1},\mathbbm{1}\rangle}{2} - \langle q_\MZ,\mathbbm{1}\rangle\right| \leq 2 \norm{\frac{1}{2} \mathbbm{1} - q_\MZ}_\infty \leq 2\delta\] and thus, like above, $0 \leq \mu^\st_{\mathbf{0}} \leq 1$. Finally, since $\bar\mu$ is a distribution, it's immediate from \cref{eq:pdfmu3} that $0 \leq \mu^\st_f \leq 1$ for all $f \in \BF_2^{k+1} \setminus (\MF \cup \{\mathbf{0}\})$. This shows that all entries of $\mu^\st$ are between $0$ and $1$. Next, observe that
\[\mathbbm{1}^\t \mu^\st = \mathbbm{1}^\t \bar\mu + \mathbbm{1}^\t B^{-1}\left(q_\MZ - \frac{1}{2}\mathbbm{1}\right) + \mathbbm{1}^\t B^{-1} \left(\frac{1}{2}\mathbbm{1} - q_\MZ\right) = 1\] 
where the final equality uses that $\bar\mu$ is a distribution. We conclude that $\mu^\st$ is a distribution.

Next, we must prove that $A^\t \mu^\st = q$. First recall that $\bar\mu$ is the distribution $\Ber(q(\mathbf{0})) \times \Ber(1/2)^{\otimes k}$. Thus, one can see from \cref{eq:fz-pr-linear} that $(A^\t \bar\mu)_{\mathbf{0}} = q(\mathbf{0})$, whereas for any $z \in \MZ$, $(A^\t \bar\mu)_z = 1/2$. Now since $A_{f,\mathbf{0}} = 0$ for all $f \in \MF \cup \{\mathbf{0}\}$, and $\mu^\st_{f} = \bar\mu_{f}$ for all $f \in \BF_2^{k+1}\setminus(\MF\cup\{\mathbf{0}\})$, we see that $(A^\t \mu^\st)_{\mathbf{0}} = (A^\t \bar\mu)_{\mathbf{0}} = q(\mathbf{0})$ as desired. Moreover, using that $A_{\mathbf{0},z} = 0$ for all $z \in \BF_2^k$,
\[(A^\t \mu^\st)_\MZ = (A^\t \bar\mu)_\MZ + A_{\MF\MZ}^\t (\mu^\st_\MF - \bar\mu_\MF) = \frac{1}{2}\mathbbm{1} + BB^{-1}\left(q_\MZ - \frac{1}{2}\mathbbm{1}\right) = q_\MZ\]
as needed. Since $\MZ\cup\{\mathbf{0}\} = \BF_2^k$, this proves that $A^\t\mu^\st = q$.

Finally, the claimed time complexity bound for sampling $F \sim \mu^\st$ follows from the fact that the probability mass function $\mu^\st$ (of size $2^{k+1}$) has an explicit formula (described in \cref{eq:pdfmu1,eq:pdfmu2,eq:pdfmu3}) and can be computed in time $2^{O(k)}$.
\end{proof}



We can now complete the proof of \cref{lemma:construct-corr-lpn}, following the sketch described above for generating the $k$-th sample of the batch given the first $k-1$ samples (and the target distribution $p$).

\begin{namedproof}{\cref{lemma:construct-corr-lpn}}
Throughout, $\sk \in \BF_2^n$ is fixed and unknown to the algorithm. The algorithm $\EntLPN$ constructs $(a'_i,y'_i)_{i=1}^k$ iteratively from $i=1$ to $k$. At each step $i$, we will maintain the invariant that $(a'_1,\dots,a'_i,(y'_1-\langle a'_1,\sk\rangle,\dots,y'_i-\langle a'_i,\sk\rangle))$ is distributed according to $\Unif(\BF_2^n)^{\otimes i} \times p_{1:i}$, for any $\sk \in \BF_2^n$, where $p_{1:i}$ is the marginal distribution of $Z_{1:i}$ for $Z \sim p$.

At step $i=1$, $\EntLPN$ sets
\[a'_1 := a_1\]
\[e'_1 \sim \Ber\left(\frac{1}{2} - \frac{1}{2}\frac{\frac{1}{2} - p_1(1)}{2^{k+2}\delta}\right)\]
\[y'_1 := y_1 + e'_1\]
where $p_1$ is the marginal distribution of $Z_1$ for $Z\sim p$. Note that $e'_1$ is well-defined since $p_1(1) \in [1/2-\delta,1/2+\delta]$ (by the lemma's assumption). By construction, $a'_1 \sim \Unif(\BF_2^n)$. Moreover, since $y_1 = \langle a_1,\sk\rangle + e_1$, where $e_1 \sim \Ber(1/2 - 2^{k+2} \delta)$ is independent of $a_1$, it follows that $y'_1 - \langle a'_1,\sk\rangle = e_1 + e'_1$ is independent of $a'_1$ and (by \cref{lem:bernoulli-convolve}) has distribution $\Ber(p_1(1))$. Thus, indeed $(a'_1, y'_1 - \langle a'_1,\sk\rangle) \sim \Unif(\BF_2^n) \times p_1$.

Now fix $1 < i \leq k$ and suppose that $\EntLPN$ has constructed $(a'_j,y'_j)_{j=1}^{i-1}$ with the desired distribution, using $(a_j,y_j)_{j=1}^{i-1}$. At step $i$, $\EntLPN$ computes the function $p^{(i)}: \BF_2^{i-1} \to [0,1]$ defined by
\begin{align}
  \label{eq:pi-defn}
  p^{(i)}(z_{1:i-1}) := \frac{1}{2} - \frac{\frac{1}{2} - \Pr_{Z\sim p}[Z_i=1|Z_{1:i-1}=z_{1:i-1}]}{2^{k+3}\delta}.
\end{align}
By assumption, it holds that $\Pr_{Z\sim p}[Z_i=1|Z_{1:i-1}=z_{1:i-1}] \in [1/2-\delta,1/2+\delta]$, and thus $p^{(i)}(z_{1:i-1}) \in [1/2 - 1/2^{k+3}, 1/2 + 1/2^{k+3}]$, for all $z_{1:i-1} \in \BF_2^{i-1}$. Applying the algorithm guaranteed by \cref{lemma:cond-dist-linearization}, $\EntLPN$ samples a random variable $(F^{(i)}_0,F^{(i)}_1,\dots,F^{(i)}_{i-1})$ with the property that
\begin{align}
  \label{eq:Fproperty}
  \Pr(F^{(i)}_0+F^{(i)}_1z_1+\dots+F^{(i)}_{i-1}z_{i-1} \equiv 1 \bmod{2}) = p^{(i)}(z_{1:i-1})
\end{align}
for all $z_{1:i-1} \in \BF_2^{i-1}$. 
Finally, $\EntLPN$ sets
\[a'_i := a_i + F^{(i)}_1 a'_1 + \dots + F^{(i)}_{i-1} a'_{i-1}\]
\[y'_i := y_i + F^{(i)}_0 + F^{(i)}_1 y'_1 + \dots + F^{(i)}_{i-1} y'_{i-1}.\] 
It remains to argue that $(a'_j,y'_j)_{j=1}^i$ has the desired distribution. To see this, first observe that the following four random variables are, by construction, mutually independent:
\begin{itemize}\setlength\itemsep{0.1em}
    \item $(a'_j,y'_j)_{j=1}^{i-1}$,
    \item $a_i$,
    \item $y_i - \langle a_i, \sk\rangle$,
    \item $(F_0^{(i)},\dots,F_{i-1}^{(i)})$.
\end{itemize}
Thus, suppose we condition on the first $i-1$ samples $(a'_j,y'_j)_{j=1}^{i-1}$. Then still $a_i \sim \Unif(\BF_2^n)$, and moreover $a_i$ is independent of $(F^{(i)}_0,\dots,F^{(i)}_{i-1})$. Thus, conditioned on $(a_j', y_j')_{j=1}^{i-1}$, we have that $a'_i$ is distributed according to $\Unif(\BF_2^n)$.

Next we also condition on $a'_i$. Since $a'_i$ is independent of $(F^{(i)}_0,\dots,F^{(i)}_{i-1})$ conditioned on $(a_j', y_j')_{j=1}^{i-1}$, the distribution of $(F^{(i)}_0,\dots,F^{(i)}_{i-1})$ conditioned on $\{ (a_j', y_j')_{j=1}^{i-1}, a_i' \}$ still satisfies the property \cref{eq:Fproperty}: for every $z \in \BF_2^{i-1}$, the distribution of $F^{(i)}_0+F^{(i)}_1z_1+\dots+F^{(i)}_{i-1}z_{i-1}$ is still $\Ber(p^{(i)}(z_{1:i-1}))$. We have that
\begin{align}
y'_i - \langle a'_i, \sk\rangle 
&= y_i - \langle a_i, \sk\rangle + F^{(i)}_0 + F^{(i)}_1(y'_1 - \langle a'_1, \sk\rangle) + \dots + F^{(i)}_{i-1}(y'_{i-1} - \langle a'_{i-1}, \sk\rangle).\label{eq:yprime-diff}
\end{align}
Using the property \cref{eq:Fproperty}, we have $F^{(i)}_0+\sum_{j=1}^{i-1} F^{(i)}_j (y'_j - \langle a'_j, \sk\rangle) \sim \Ber(p^{(i)}((y'_j-\langle a'_j,\sk\rangle)_{j=1}^{i-1}))$. Also $y_i - \langle a_i,\sk\rangle \sim \Ber(1/2 - 2^{k+2} \delta)$ is independent of $(F^{(i)}_j)_{j=0}^{i-1}$. 
Thus, combining \cref{eq:yprime-diff}, the definition of $p^{(i)}$ in \cref{eq:pi-defn}, and \cref{lem:bernoulli-convolve}, we get
\[y'_i - \langle a'_i, \sk\rangle \sim \Ber\left(\Pr_{Z\sim p}\left[Z_i=1|Z_j = y'_j-\langle a'_j,\sk\rangle \quad \forall j \in [i-1]\right]\right).\]
It follows that the conditional distribution of $(a'_i, y'_i-\langle a'_i,\sk\rangle)$, given $(a'_j,y'_j)_{j=1}^{i-1}$, is \[\Unif(\BF_2^n) \times \Ber\left(\Pr_{Z\sim p}\left[Z_i=1|Z_j = y'_j-\langle a'_j,\sk\rangle \quad \forall j \in [i-1]\right]\right).\]
By the inductive hypothesis, we get that the distribution of $(a'_j,y'_j-\langle a'_j,\sk\rangle)_{j=1}^i$ is 
\[\Unif(\BF_2^n)^{\otimes i} \times p_{1:i}\] as desired. Finally, we note that the claimed time complexity bound on $\EntLPN$ follows from the algorithm description and \cref{lemma:cond-dist-linearization}.
\end{namedproof}


\paragraph{Intermission: tying back to triangle LPN.} In this section we are considering generic joint noise distributions, but ultimately we are interested in the noise distribution $\mu_{H,N,\delta}$ that induces the triangle LPN distribution $\LPN_{n,N,H,\delta}^{\mathsf{tri}}(\sk)$ (see \cref{def:mu-defn}). In particular, we ultimately will want to show that a triangle LPN sample $(u_{ijk},y_{ijk})_{(i,j,k) \in \Gamma(H) \times [N+1]}$ can be efficiently generated from standard LPN samples with bias $O_H(\delta^2)$. To this end, in \cref{sec:triangle-lpn-hardness} we will use \cref{lemma:construct-corr-lpn} to show that \emph{part} of a triangle LPN sample (specifically, the part $(u_{ij1}, y_{ij1})_{(i,j) \in \Gamma(H)}$ corresponding to the noise distribution $\mu^0_{H,N,\delta}$) can be efficiently generated from standard LPN samples with bias $O_H(\delta^2)$. We pay a factor $\exp(O(H^2))$ in the bias, but this is fine since $\exp(O(H^2)) \ll \delta^{-1}$ in our parameter regime of interest (see \cref{def:mdp-family}).

Unfortunately, we cannot directly use \cref{lemma:construct-corr-lpn} to construct the entire triangle LPN sample, since it consists of $O(H^2 \cdot N)$ (correlated) LPN samples, and thus we would pay a factor of $\exp(N)$ in the bias. Since $N \asymp \poly(\delta^{-1})$ in our parameter regime of interest, this is far too large.

However, there is significant structure in a triangle LPN sample that we are not yet exploiting. In particular, suppose we have generated $(u_{ij1},y_{ij1})_{(i,j) \in \Gamma(H)}$ (according to the correct marginal distribution). Conditioning on these variables induces some altered distribution for the latent variables $(b_1,\dots,b_H)$, which depends on $\sk$. If we knew this distribution and could correctly sample the $H$ variables $b_1,\dots,b_H$, then we could generate $(u_{ijk},y_{ijk})_{(i,j,k) \in \Gamma(H) \times \{2,\dots,N+1\}}$ (according to the correct conditional distribution) by simply generating $|\Gamma(H)|\cdot N$ independent samples $(\til u_{ijk},\til y_{ijk})_{(i,j,k) \in \Gamma(H) \times \{2,\dots,N+1\}}$ from $\LPN_{n,O_H(\delta^2)}(\sk)$ and adding $b_j$ to each $\til y_{ijk}$. The key point is that we are generating these $|\Gamma(H)|\cdot N$ samples all at once rather than one-by-one, which avoids needing to condition on $\Omega(N)$ variables, and is only possible since these samples, unlike $(u_{ij1},y_{ij1})_{(i,j) \in \Gamma(H)}$, have sufficient mutually independent noise.

Of course, we do not know the conditional distribution on $(b_1,\dots,b_H)$ because we do not know $\sk$. However, we can get around this issue via the same ideas as in the proof of \cref{lemma:construct-corr-lpn}. The following lemma will let us sample an affine function (in $\sk$) which, if evaluated at $\sk$, has the same distribution as $(b_1,\dots,b_H)$.

\begin{lemma}
  \label{lemma:uv-construction}
There is an algorithm $\AffSample$ with the following property. Fix $n,k,H \in \NN$ and a function $q: \BF_2^k \to \Delta(\BF_2^H)$ such that \[\Prr_{W\sim q(z)}[W_h = 1 | W_{<h} = w_{<h}] \in [1/2 - 1/2^{H+k+2}, 1/2 + 1/2^{H+k+2}]\] for all $z \in \BF_2^k$, $h \in [H]$, and $w \in \BF_2^H$. Fix any $\sk,a_1,\dots,a_k \in \BF_2^n$ and $y_1,\dots,y_k \in \BF_2$. The output of $\AffSample(q, (a_i,y_i)_{i=1}^k)$ is $(U^i,v^i)_{i=1}^H$ such that $U^1,\dots,U^H \in \BF_2^n$ and $v^1,\dots,v^H \in \BF_2$, and the distribution of $(v^h - \langle U^h,\sk\rangle)_{h=1}^H$ is $q((y_i-\langle a_i,\sk\rangle)_{i=1}^k)$.

Moreover, the time complexity of $\AffSample$ is $\poly(n,2^k,2^H)$.
\end{lemma}

Note that \cref{lemma:uv-construction} is essentially an extension of \cref{lemma:cond-dist-linearization} where each $q(z)$ is now a distribution over $\BF_2^H$ rather than $\BF_2$. Accordingly, it is proved by iteratively applying \cref{lemma:cond-dist-linearization}. Also note that the description complexity of $q$ is $2^{k+H}$, so the claimed time complexity bound is indeed feasible.


\begin{proof}
The algorithm $\AffSample$ constructs $(U^h,v^h)_{h=1}^H$ iteratively from $h=1$ to $H$. At step $h$, we will maintain the invariant that $(v^i - \langle U^i,\sk\rangle)_{i=1}^h$ is distributed as $X_{1:h}$ where $X \sim q((y_j-\langle a_j,\sk\rangle)_{j=1}^k)$.

In particular, fix any step $h$, and suppose that $\AffSample$ has already constructed $(U^i,v^i)_{i=1}^{h-1}$ such that $(v^i-\langle U^i,\sk\rangle)_{i=1}^{h-1}$ is distributed as $X_{1:h-1}$ where $X \sim q((y_j-\langle a_j,\sk\rangle)_{j=1}^k)$. For each $z \in \BF_2^k$ and $w \in \BF_2^{h-1}$, the algorithm $\AffSample$ computes $p^{(h)}(z,w) := \Pr_{X \sim q(z)}[X_h=1|X_{1:h-1} = w]$. 
By the lemma assumption, it holds that $p^{(h)}(z,w) \in [1/2 - 1/2^{H+k+2}, 1/2 + 1/2^{H+k+2}]$ for all $z \in \BF_2^k$ and $w \in \BF_2^{h-1}$. By \cref{lemma:cond-dist-linearization} applied to $p^{(h)}$, $\AffSample$ can construct random vectors $(F^{(h)}_0,\dots,F^{(h)}_k)$ and $(G^{(h)}_1,\dots,G^{(h)}_{h-1})$ such that 
\begin{align}
  \label{eq:prfg-phza}
  \Pr\left(F^{(h)}_0 + \sum_{j=1}^k z_j F^{(h)}_j + \sum_{i=1}^{h-1} w_i G^{(h)}_i \equiv 1 \mod 2\right) = p^{(h)}(z,w)
\end{align}
for all $z \in \BF_2^k$ and $w \in \BF_2^{h-1}$. Next, $\AffSample$ computes
\[U^h := \sum_{j=1}^k F^{(h)}_j a_j + \sum_{i=1}^{h-1} G^{(h)}_i U^i\]
\[v^h := F^{(h)}_0 + \sum_{j=1}^k F^{(h)}_j y_j + \sum_{i=1}^{h-1} G^{(h)}_i v^i.\]
Observe that conditioned on $(U^i,v^i)_{i=1}^{h-1}$, we have, using \cref{eq:prfg-phza},
\begin{align*}
v^h - \langle U^h, \sk\rangle 
&= F^{(h)}_0 + \sum_{j=1}^k (y_j - \langle a_j,\sk\rangle) F^{(h)}_j + \sum_{i=1}^{h-1} (v^i - \langle U^i,\sk\rangle) G^{(h)}_i \\
&\sim \Ber(p^{(h)}((y_j-\langle a_j,\sk\rangle)_{j=1}^k, (v^i - \langle U^i,\sk\rangle)_{i=1}^{h-1})).
\end{align*}
Since this is the distribution of $X_h|X_{1:h-1} = (v^i-\langle U^i,\sk\rangle)_{i=1}^{h-1}$ where $X \sim q((y_j-\langle a_j,\sk\rangle)_{j=1}^k)$, it follows that $(v^i-\langle U^i,\sk\rangle)_{i=1}^h$ is distributed as $X_{1:h}$ as desired.

To analyze the time complexity of $\AffSample$, note that the most time-intensive computations at each round $h$ are the computation of $p^{(h)}$, which takes time $\poly(2^h, 2^k)$, and the construction of random vectors $(F_i)_i, (G_j)_j$, which take time $\poly(2^h, 2^k)$ by \cref{lemma:cond-dist-linearization}. Furthermore, computation of $U^h, v^h$ takes time $\poly(h,k,n)$. 
\end{proof}

Using \cref{lemma:uv-construction}, we can prove the following technical lemma, which describes a reduction from standard LPN to batch LPN where the exponential dependence on the batch size can be mitigated if the joint noise distribution has certain structure \--- as is the case for the triangle LPN distribution.

\begin{lemma}\label{lemma:large-batch-lpn}
There is an algorithm $\Alg$ with the following property. Fix $n,k,H,N \in \NN$, $\delta\in(0,1/2)$, and a function $q: \BF_2^k \to \Delta(\BF_2^H)$ such that \[\Prr_{W\sim q(z)}[W_h = 1 | W_{<h} = w_{<h}] \in [1/2 - 1/2^{H+k+2}, 1/2 + 1/2^{H+k+2}]\] for all $z \in \BF_2^k$, $h \in [H]$, and $w \in \BF_2^H$. For each $z \in \BF_2^k$, define $\til{q}(z) \in \Delta(\BF_2^{H\times N})$ as the distribution of $(e_{ij}+X_i)_{i,j=1}^{H,N}$ where $X \sim q(z)$ and $(e_{ij})_{i,j}^{H,N} \sim \Ber(1/2-\delta)^{\otimes (H \times N)}$ are independent. Fix any $\sk,a_1,\dots,a_k \in \BF_2^n$ and $y_1,\dots,y_k \in \BF_2$. For independent samples $(a_{ij},y_{ij})_{i,j=1}^{H,N}$ from $\LPN_{n,\delta}(\sk)$, 
the output of $\Alg(q,(a_i,y_i)_{i=1}^k, (a_{ij},y_{ij})_{i,j=1}^{H,N})$ is $(a'_{ij},y'_{ij})_{i,j=1}^{H,N}$ where 
\[(a'_{ij},y'_{ij} - \langle a'_{ij}, \sk\rangle)_{i,j=1}^{H,N} \sim \Unif(\BF_2^n)^{\otimes (H \times N)} \times \til{q}((y_i-\langle a_i,\sk\rangle)_{i=1}^k).\]
Moreover, the time complexity of $\Alg'$ is $\poly(n,N,2^k,2^H)$.
\end{lemma}

\begin{proof}
The algorithm $\Alg$ first applies \cref{lemma:uv-construction} with inputs $q$ and $(a_i,y_i)_{i=1}^k$, to construct random vectors $U^1,\dots,U^H \in \BF_2^n$ and random variables $v^1,\dots, v^H \in \BF_2$ such that the distribution of $(v^h - \langle U^h, \sk\rangle)_{h=1}^H$, conditioned on $(a_i,y_i)_{i=1}^k$, is precisely $q((y_i-\langle a_i,\sk\rangle)_{i=1}^k)$. This is possible in time $\poly(n, 2^k, 2^H)$ by \cref{lemma:uv-construction}. Then $\Alg$ outputs $(a'_{ij},y'_{ij})_{i,j=1}^{H,N}$ where 
\[ a'_{ij} := a_{ij} + U^i\] 
\[ y'_{ij} := y_{ij} + v^i\]
for $(i,j) \in [H]\times [N]$. To see why this suffices, note that $(a'_{ij})_{i,j=1}^{H,N} \sim \Unif(\BF_2^n)^{\otimes (H\times N)}$ since $(a_{ij})_{i,j=1}^{H,N}$ is uniform and independent of $(U^i)_{i=1}^H$. Next, for any $(i,j) \in [H]\times[N]$, we can write
\[y'_{ij}-\langle a'_{ij}, \sk\rangle = y_{ij} - \langle a_{ij},\sk\rangle + v^i - \langle U^i, \sk\rangle.\] 
We know that $(U^i, v^i)_{i=1}^H$ is independent of $(a'_{ij})_{i,j=1}^{H,N}$, so conditioned on $(a'_{ij})_{i,j=1}^{H,N}$, the distribution of $(v^i - \langle U^i,\sk\rangle)_{i=1}^H$ is still $q((y_i-\langle a_i,\sk\rangle)_{i=1}^k)$. Moreover, conditioned on all of $(a'_{ij})_{i,j=1}^{H,N}$ and $(U^i,v^i)_{i=1}^H$, the distribution of $(y_{ij} - \langle a_{ij},\sk\rangle)_{i,j=1}^{H,N}$ is still $\Ber(1/2-\delta)^{\otimes (H\times N)}$. This implies in particular that $(y_{ij} - \lng a_{ij}, \sk \rng )_{i,j=1}^{H,N}$ and $(v^i - \lng U^i, \sk \rng )_{i=1}^H$ are independent conditioned on $(a_{ij}')_{i,j=1}^{H,N}$. Thus, conditioned on $(a'_{ij})_{i,j=1}^{H,N}$, the distribution of $(y'_{ij}-\langle a'_{ij}, \sk\rangle)_{i,j=1}^{H,N}$ is $\til{q}((y_i-\langle a_i,\sk\rangle)_{i=1}^k)$, by definition of $\til{q}$. It follows that the joint distribution of $(a'_{ij},y'_{ij} - \langle a'_{ij}, \sk\rangle)_{i,j=1}^{H,N}$ is $\Unif(\BF_2^n)^{\otimes (H \times N)} \times \til{q}((y_i-\langle a_i,\sk\rangle)_{i=1}^k)$ as claimed.

\end{proof}

\subsection{Hardness of triangle LPN}\label{sec:triangle-lpn-hardness}

We now apply the generic reductions from \cref{sec:dependent-lpn} to prove the following lemma, which states that a sample from $\LPN^{\mathsf{tri}}_{n,N,H,\delta}$ can be efficiently constructed from standard LPN samples with noise level $1/2 - O_H(\delta^2)$. Recall the definitions of $\mu_{H,N,\delta}$ and $\Gamma(H)$ from \cref{def:mu-defn}.

\begin{lemma}\label{lemma:triangle-lpn}
There is an algorithm $\TriAlg$ with the following property. Fix $n,H,N \in \NN$ and $\delta \in (0, 1/2^{|\Gamma(H)|+2H+7})$. Define $C_{\ref{lemma:triangle-lpn}}(H) := 2^{|\Gamma(H)|+2H+6}$. For every $\sk \in \BF_2^n$, for independent samples $(a_i,y_i)_{i=1}^{|\Gamma(H)| \cdot (N+1)}$ from $\LPN_{n,C_{\ref{lemma:triangle-lpn}}(H)\delta^2}(\sk)$, the output \[(a'_{ijk},y'_{ijk})_{(i,j,k) \in \Gamma(H)\times [N+1]} := \TriAlg((a_i,y_i)_{i=1}^{|\Gamma(H)|\cdot (N+1)},H,N,\delta)\] satisfies
\[(a'_{ijk},y'_{ijk} - \langle a'_{ijk},\sk\rangle) \sim \Unif(\BF_2^n)^{\otimes \Gamma(H) \times [N+1]} \times \mu_{H,N,\delta}\]
and thus has distribution $\LPN^{\mathsf{tri}}_{n,N,H,\delta}$. Moreover, the time complexity of $\TriAlg$ is $\poly(n, N,2^{H^2})$.
\end{lemma}

As previously discussed in \cref{sec:dependent-lpn}, there are two steps to the construction. First, we apply \cref{lemma:construct-corr-lpn} to construct a batch LPN sample $(u_{ij1},y_{ij1})_{(i,j)\in\Gamma(H)}$ with noise distribution $\mu^0_{H,N,\delta}$. Second, we apply \cref{lemma:large-batch-lpn} to extend this to a full triangle LPN sample, by implicitly sampling the latent bits $b_1,\dots,b_H$ (\cref{def:mu-defn}) from the conditional distribution induced by $(u_{ij1},y_{ij1})_{(i,j)\in\Gamma(H)}$. 

For the first step, the main precondition that we have to verify is that the joint noise distribution $\mu^0_{H,\delta}$ is near-uniform. Quantitatively, we need that for any partial assignment of the noise variables, each unassigned noise variable has conditional distribution within $O_H(\delta^2)$ of uniform. This is the content of the following lemma.


\begin{lemma}\label{lemma:z-near-unif}
Let $H \in \NN$ and $\delta \in (0, 1/2^{H+3})$. Consider the random variable $Z \sim \mu^0_{H,\delta}$. For any $(\til i, \til j)\in\Gamma(H)$ and $(z_{ij})_{i,j} \in \BF_2^{\Gamma(H)}$, it holds that
\[\Pr[Z_{\til i \til j}=1 | Z_{ij} = z_{ij} \forall (i,j) \neq (\til i,\til j)] \in [1/2 - 2^{2H+4} \delta^2, 1/2 + 2^{2H+4} \delta^2].\]
\end{lemma}

It's crucial that the bias be $O_H(\delta^2)$ and not $O_H(\delta)$. To give some intuition for why this stronger guarantee holds, recall from \cref{def:mu-defn} that a sample from $\mu^0_{H,\delta}$ is distributed as the partial random matrix
\[\begin{matrix}
    e_{11} + B_1 & e_{21} + B_1 & \dots & e_{H1} + B_1 \\ 
    & e_{22} + B_2 & \dots & e_{H2} + B_2 \\ 
    & & \ddots & \vdots \\ 
    & & & e_{HH} + B_H
\end{matrix}\]
where $B_1,\dots,B_H \sim \Ber(1/2)$ are independent, and for each column $j$, the random vector $(e_{j1},\dots,e_{jj})$ has distribution $\CBer(j, \delta)$ (\cref{def:cber}). Suppose that we want to guess the value of e.g. $Z_{H1} := e_{H1} + B_1$, and we know the values of $Z_{ij} = e_{ij} + B_j$ for all $(i,j) \neq (H,1)$. Intuitively, one of the best things we can do is guess $e_{i1} + B_1$ for some $i \neq H$, which has error $e_{i1} + e_{H1} \sim \Ber(1/2 - 2\delta^2)$ since $e_{i1}$ and $e_{H1}$ are independent. We could improve the guess somewhat by taking the majority of $Z_{11},\dots,Z_{H-1,1}$, but the error would still be $1/2 - O_H(\delta^2)$.

Alternately, we could try to make a guess with error e.g. $e_{H1} + e_{H2}$, which has distribution $\Ber(1/2-\delta)$ since $e_{H1}$ and $e_{H2}$ are correlated. However, each entry in column $H$ has a different uniformly random bit $B_i$, so it seems impossible to ``cancel out'' these bits to make such a guess. Of course, this is just intuition; we give a formal proof below.


\begin{namedproof}{\cref{lemma:z-near-unif}}
Recall that $Z_{ij} = e_{ij} + B_j$ where all $e_i \sim \CBer(i,\delta)$ and $B_j \sim \Ber(1/2)$ are independent. Thus, for any $z \in \BF_2^{\Gamma(H)}$, we have
\begin{align*}
\Pr[Z=z]
&= \frac{1}{2^H} \sum_{b \in \BF_2^H} \Pr[Z=z|B=b] \\
&= \frac{1}{2^H} \sum_{b \in \BF_2^H} \prod_{i=1}^H \left(2\delta \mathbbm{1}[z_{ij} = b_j \forall j \leq i] + (1-2\delta)2^{-i}\right) \\ 
&= \frac{(1-2\delta)^H}{2^H 2^{\binom{H+1}{2}}} \sum_{b \in \BF_2^H} \prod_{i=1}^H \left(\frac{2\delta}{1-2\delta} 2^i \mathbbm{1}[z_{ij}=b_j \forall j \leq i] + 1 \right).
\end{align*}
Thus, we get the following bounds:
\begin{equation} \frac{(1-2\delta)^H}{2^{\binom{H+1}{2}}} \leq \Pr[Z=z] \leq \frac{(1-2\delta)^H}{2^{\binom{H+1}{2}}} \cdot \exp(4\delta 2^{H+1}) \leq \frac{e(1-2\delta)^H}{2^{\binom{H+1}{2}}}\label{eq:z-upper-lower}\end{equation}
where the last inequality uses the assumption that $\delta \leq 1/2^{H+3}$. Now fix some $1 \leq \til j \leq \til i \leq H$ and $z \in \BF_2^{\Gamma(H)}$ with $z_{\til i\til j} = 1$. Define $z' \in \BF_2^{\Gamma(H)}$ by \[z'_{ij} = \begin{cases} 0 & \text{ if } (i,j) = (\til i,\til j) \\ z_{ij} & \text{ otherwise}\end{cases}.\]
Then we can write\allowdisplaybreaks
\begin{align*}
&\left|\frac{1}{2} - \Pr[Z_{\til i\til j}=1|Z_{ij}=z_{ij}\forall (i,j) \neq (\til i,\til j)] \right| \\
&= \left|\frac{1}{2} - \frac{\Pr[Z=z]}{\Pr[Z=z] + \Pr[Z=z']}\right| \\
&= \left|\frac{\Pr[Z=z'] - \Pr[Z=z]}{2(\Pr[Z=z]+\Pr[Z=z'])}\right| \\ 
&\leq \frac{2^{\binom{H+1}{2}}}{2(1-2\delta)^H} \cdot \left|\Pr[Z=z'] - \Pr[Z=z]\right| \\ 
&\leq \frac{1}{2^{H+1}} \left|\sum_{b \in \BF_2^H} \left(\prod_{i=1}^H \left(\frac{2\delta}{1-2\delta} 2^i \mathbbm{1}[z'_{ij}=b_j \forall j \leq i] + 1 \right) -\prod_{i=1}^H \left(\frac{2\delta}{1-2\delta} 2^i \mathbbm{1}[z_{ij}=b_j \forall j \leq i] + 1 \right)\right)\right| \\ 
&= \frac{1}{2^{H+1}} \left|\sum_{b\in\BF_2^H} \left(\frac{2\delta}{1-2\delta} 2^{\til i} \mathbbm{1}[z'_{\til i j}=b_j\forall j \leq \til i] - \frac{2\delta}{1-2\delta}2^{\til i}\mathbbm{1}[z_{\til i j}=b_j\forall j \leq \til i]\right)\prod_{i \neq \til i} \left(\frac{2\delta}{1-2\delta} 2^i \mathbbm{1}[z_{ij}=b_j\forall j \leq i] + 1\right)\right|.
\end{align*}
Now for notational convenience, for each $b \in \BF_2^H$, define
\[ f(b) := \mathbbm{1}[z'_{\til i j}=b_j \forall j \leq \til i] - \mathbbm{1}[z_{\til ij}=b_j \forall j \leq \til i] \] 
\[ g(b) := \prod_{i \neq \til i} \left(\frac{2\delta}{1-2\delta} 2^i \mathbbm{1}[z_{ij}=b_j\forall j \leq i] + 1\right)\] 
so that the above bound can be restated as
\[\left|\frac{1}{2} - \Pr[Z_{\til i\til j}=1|Z_{ij}=z_{ij}\forall (i,j) \neq (\til i,\til j)] \right| \leq \frac{1}{2^{H+1}} \frac{2\delta}{1-2\delta} 2^{\til i} \left|\sum_{b\in\BF_2^H} f(b) g(b)\right|.\]
Let $S_+ := \{b \in \BF_2^H: b_j = z'_{\til i j} \forall j \leq \til i\}$ and $S_- := \{\til i \in \BF_2^H: b_j = z_{\til ij} \forall j \leq \til i\}$; observe that $S_+ = f^{-1}(1)$, and $S_- = f^{-1}(-1)$, and $f(b) = 0$ for all $b \in \BF_2^H \setminus (S_+ \cup S_-)$. Moreover, since $z$ and $z'$ differ only at coordinate $\til j$, the map $\iota: \BF_2^n \to \BF_2^n$ that flips coordinate $\til j$ is a bijection between $S_+$ and $S_-$. Finally, observe that for every $b \in \BF_2^H$,
\[ 1 \leq g(b) \leq \prod_{i \neq \til i} \exp(4\delta 2^i) \leq \exp(4\delta 2^{H+1}) \leq 1 + \delta 2^{H+3}\]
again using the assumption that $\delta \leq 1/2^{H+3}$. Combining these observations, we have
\begin{align*}
\left|\frac{1}{2} - \Pr[Z_{\til i\til j}=1|Z_{ij}=z_{ij}\forall (i,j) \neq (\til i,\til j)] \right| 
&\leq \frac{1}{2^{H+1}} \frac{2\delta}{1-2\delta} 2^{\til i} \left|\sum_{b\in S_+} g(b) - \sum_{b \in S_-} g(b) \right| \\ 
&= \frac{1}{2^{H+1}} \frac{2\delta}{1-2\delta} 2^{\til i} \left|\sum_{b\in S_+} g(b) - g(\iota(b)) \right| \\ 
&\leq \frac{1}{2^{H+1}} \frac{2\delta}{1-2\delta} 2^{\til i} \cdot |S_+| \delta 2^{H+3} \\ 
&\leq 2^{2H+4} \delta^2
\end{align*}
as desired. 
\end{namedproof}

It is immediate from the above lemma and \cref{lemma:construct-corr-lpn} that we can efficiently construct a batch LPN sample with joint noise distribution $\mu^0_{H,\delta}$ (\cref{def:mu-defn}) using standard LPN samples with noise level $1/2 - O_H(\delta^2)$:

\begin{lemma}\label{lemma:triangle-lpn-0}
 There is an algorithm $\Alg$ with the following property. Fix $n,H\in \NN$ and $\delta \in (0,1/2^{|\Gamma(H)|+H+3})$. For every $\sk \in \BF_2^n$, for independent samples $(a_i,y_i)_{i=1}^{|\Gamma(H)|}$ from $\LPN_{n,2^{|\Gamma(H)|+2H+6}\delta^2}(\sk)$, the output of $\Alg((a_i,y_i)_{i=1}^{|\Gamma(H)|}, \delta)$ is $(a'_{ij},y'_{ij})_{(i,j)\in \Gamma(H)}$ where
\[(a'_{ij}, y'_{ij} - \langle a'_{ij}, \sk\rangle)_{(i,j) \in \Gamma(H)} \sim \Unif(\BF_2^n)^{\otimes \Gamma(H)} \times \mu^0_{H,\delta}.\]
Moreover, the time complexity of $\Alg$ is $\poly(n, 2^{H^2})$.
\end{lemma}

\begin{proof}
On input $\{(a_i,y_i)_{i=1}^{\Gamma(H)}, \delta\}$, $\Alg$ simply applies the algorithm $\EntLPN$ specified in \cref{lemma:construct-corr-lpn} with inputs $(a_i,y_i)_{i=1}^{|\Gamma(H)|}$, $\mu^0_{H,\delta}$, and $2^{2H+4}\delta^2$. Note that the probability mass function of $\mu^0_{H,\delta}$ has description size $2^{|\Gamma(H)|}$ and can be explicitly computed in time $2^{O(H^2)}$. The analysis is immediate from \cref{lemma:construct-corr-lpn} and \cref{lemma:z-near-unif}. (Note that \cref{lemma:construct-corr-lpn} is being applied with the value of $\delta$ set to $2^{2H+4}\delta^2$, and with dimension $k := |\Gamma(H)|$.) 
\end{proof}

For the second step of constructing a triangle LPN sample, we need to show that the conditional distribution of $(b_1,\dots,b_H)$ is near-uniform for any realization of $(u_{ij1},y_{ij1})_{(i,j)\in\Gamma(H)}$, although the quantitative bound that we need is weaker than above. We show that this bound holds in the following lemma.

\begin{lemma}\label{lemma:x-cond-near-unif}
Let $H \in \NN$ and $\delta \in (0, 1/2^{H+3})$. Let $e_i \sim \CBer(i,\delta)$ for $i \in [H]$ and $B_j \sim \Ber(1/2)$ for $j \in [H]$ be independent, and define $Z_{ij} = e_{ij} + B_j$ for all $(i,j) \in \Gamma(H)$. Then for every $z \in \BF_2^{\Gamma(H)}$, $b \in \BF_2^H$, and $h \in [H]$, it holds that
\[\Pr[B=b|Z=z] \in [1/2^H - 16\delta, 1/2^H + 16\delta]\]
and moreover
\[\Pr[B_h=b_h|Z=z\land B_{-h}=b_{-h}] \in [1/2 - 16\delta 2^{H+1}, 1/2 + 16\delta 2^{H+1}]\]
\end{lemma}

\begin{proof}
For any $z \in \BF_2^{\Gamma(H)}$ and $b \in \BF_2^H$, we can write
\[\Pr[B=b \land Z=z] = \frac{1}{2^H} \Pr[Z=z|B=b] = \frac{1}{2^H}\prod_{i=1}^H \left(2\delta\mathbbm{1}[z_{ij}=b_j\forall j \leq i] + (1-2\delta)2^{-i}\right).\]
Thus, using the fact that for non-negative real numbers $a_1, \ldots, a_H$ with $a_1 + \cdots + a_H \leq 1$, we have $\prod_{i=1}^H (1+a_i) \leq 1 + 2 \sum_{i=1}^H a_i$, 
\[\frac{(1-2\delta)^H}{2^H 2^{\binom{H+1}{2}}} \leq \Pr[B=b \land Z=z] \leq \frac{(1-2\delta)^H}{2^H 2^{\binom{H+1}{2}}} \cdot (1 + 8\delta 2^{H+1}).\]
Summing over all $b \in \{0,1\}^H$, it follows that
\[\frac{(1-2\delta)^H}{2^{\binom{H+1}{2}}} \leq \Pr[Z=z] \leq \frac{(1-2\delta)^H}{2^{\binom{H+1}{2}}} \cdot (1 + 8\delta 2^{H+1})\]
and thus
\begin{equation} \frac{1}{2^H(1 + 8\delta 2^{H+1})} \leq \Pr[B=b|Z=z] \leq \frac{1+8\delta 2^{H+1}}{2^H}
\label{eq:b-given-z-bound}
\end{equation}
which implies the first claimed bound. Next, for any $h \in [H]$, we get
\begin{align*}
\Pr[B_h=b_h|Z=z\land B_{-h}=b_{-h}]
&= \frac{\Pr[B_h=b_h\land B_{-h}=b_{-h}|Z=z]}{\Pr[B_{-h}=b_{-h}|Z=z]} \\ 
&\leq \frac{1}{2}(1+8\delta 2^{H+1})^2 \\
&\leq \frac{1}{2} + 16\delta 2^{H+1}
\end{align*}
where the second inequality applies \cref{eq:b-given-z-bound} to both numerator and denominator, and the final inequality uses that $8\delta 2^{H+1} \leq 2$. This implies the second claimed bound.
\end{proof}

\begin{proof}[\textbf{Proof of \cref{lemma:triangle-lpn}}]
The algorithm $\TriAlg$ proceeds in two steps. First, apply \cref{lemma:triangle-lpn-0} with the first $|\Gamma(H)|$ samples $(a_i,y_i)_{i=1}^{|\Gamma(H)|}$, and parameter $\delta$, to compute $(a'_{ij1},y'_{ij1})_{(i,j)\in \Gamma(H)}$ where 
\begin{equation}
(a'_{ij1}, y'_{ij1} - \langle a'_{ij1}, \sk\rangle)_{(i,j) \in \Gamma(H)} \sim \Unif(\BF_2^n)^{\otimes \Gamma(H)} \times \mu^0_{H,\delta}.
\label{eq:part-1-dist}
\end{equation}
Second, for each $z \in \BF_2^{\Gamma(H)}$, compute the conditional distribution $X|Z=z$, where $e_i \sim \CBer(i,\delta)$ for $i \in [H]$ and $X_j \sim \Ber(1/2)$ are independent, and $Z_{ij} = e_{ij} + X_j$ for $(i,j) \in \Gamma(H)$. By \cref{lemma:x-cond-near-unif} and the assumption that $\delta \leq 1/2^{|\Gamma(H)| + 2H + 7}$, it holds that
\begin{align}\Pr[X_h=1|Z=z \land X_{-h}=x_{-h}] \in [1/2 - 2^{|\Gamma(H)| + H+2}, 1/2 + 2^{|\Gamma(H)| + H+2}]\label{eq:uniformity-xz}
\end{align}
for all $x \in \BF_2^H$, $z \in \BF_2^{\Gamma(H)}$, and $h \in [H]$. Now the algorithm applies \cref{lemma:large-batch-lpn} with inputs (1) $q$ defined by $q(z) := \textsf{Law}(X|Z=z)$, (2) $(a'_{ij1},y'_{ij1})_{i,j \in \Gamma(H)}$, and (3) the remaining LPN samples, with added noise: $(a_i,y_i+\til{e}_i)_{i=|\Gamma(H)|+1}^{|\Gamma(H)| \cdot N}$, where $\til{e}_i \sim \Ber(1/2 - 1/2^{|\Gamma(H)|+2H+6})$ are independent. Call the output $(a'_{ijk},y'_{ijk})_{(i,j,k) \in \Gamma(H) \times \{2,\dots,N+1\}}$. By the uniformity bound \cref{eq:uniformity-xz}, the fact that the samples $(a_i,y_i+\til{e}_i)$ have distribution $\LPN_{n,2\delta^2}$ (by \cref{lem:bernoulli-convolve}), and \cref{lemma:large-batch-lpn}, conditioned on $(a'_{ij1},y'_{ij1})_{(i,j)\in \Gamma(H)}$ (which are independent of $(a_i,y_i)_{i=|\Gamma(H)|+1}^{|\Gamma(H)| \cdot N}$), the output satisfies
\begin{equation}
(a'_{ijk},y'_{ijk} - \langle a'_{ijk}, \sk\rangle)_{(i,j,k) \in \Gamma(H) \times \{2,\dots,N+1\}} \sim \Unif(\BF_2^n)^{\otimes (\Gamma(H) \times [N])} \times \til{q}((y'_{ij1} - \langle a'_{ij1}, \sk\rangle)_{(i,j) \in \Gamma(H)})
\label{eq:part-2-dist}
\end{equation}
where $\til{q}(z)$ is the conditional distribution of $(e'_{ijk} + X_j)_{(i,j,k) \in \Gamma(H) \times \{2,\dots,N+1\}}$ for independent $e'_{ijk} \sim \Ber(1/2 - 2\delta^2)$, given that $Z=z$.

It follows from \cref{eq:part-1-dist}, \cref{eq:part-2-dist}, and the definition of $\mu_{H,N,\delta}$ that the joint distribution $(a'_{ijk},y'_{ijk})_{(i,j,k) \in \Gamma(H) \times [N+1]}$ satisfies
\[(a'_{ijk}, y'_{ijk}-\langle a'_{ijk}, \sk\rangle)_{(i,j,k) \in \Gamma(H)\times[N+1]} \sim \Unif(\BF_2^n)^{\otimes \Gamma(H) \times [N+1]} \times \mu_{H,N,\delta}\]
as desired.
\end{proof}

\subsection{Recovering the secret}

In this section we provide the final ingredient in the proof of \cref{lemma:policy-cover-to-lpn}: an algorithm that takes as input a policy with non-trivial state visitation distribution on $M^\sk_{n,N,H,\delta}$ (specifically one that visits the state $(H,H-1,0)$ with probability significantly better than the uniform policy, but the choice of state is not crucial) and, with the help of an algorithm for learning parities at the \emph{lower} noise level $1/2 - O_H(\delta)$, computes a small set of candidates that includes the secret key $\sk$.

\begin{lemma}\label{lemma:constrast-learn}
Let $\Alg$ be an algorithm for learning noisy parities with unknown noise level with time complexity $T(n,\delta,\eta)$ and sample complexity $S(n,\delta,\eta)$ (\cref{def:lpn-unknown-noise}). Then there is an algorithm $\CL$ with the following property. Let $n,N,H,B \in \NN$, $\sk \in \BF_2^n$, and $\delta > 0$. Let $(\MC_\pi)_{\pi \in \Psi}$ be a set of circuit representations of policies $\pi \in \Psi$ where $\size(\MC_\pi) \leq B$ for each $\pi$. Suppose that \[\EE_{\pi\sim\Unif(\Psi)} d^{M^\sk_{n,N,H,\delta},\pi}_H(H,H-1,0) \geq 2^{1-H},\] and let $\delreg := \delta/(2^H C_{\ref{lemma:f-corr}}(H))$. If $N \geq 3\delta^{-4} n$, then
\begin{equation} \Pr[\sk \in \CL((W_i)_{i=1}^{2S(n,\delreg,1/n)}, (\MC_\pi)_{\pi \in\Psi}, N,H,\delta)] \geq 1 - 1/n - O(\delreg^{-2}H^2\exp(-2n))\label{eq:cl-correctness}\end{equation}
where $(W_i)_{i=1}^{2S(n,\delreg,1/n)}$ are independent samples from $\LPN^{\mathsf{tri}}_{n,N,H,\delta}(\sk)$, and the output of $\CL$ is a set of size at most $4(H+1)$. Moreover, the time complexity of $\CL$ is 
\begin{equation} O(B \cdot |\Psi|) + O(H) \cdot T(n,\delreg,1/n) + \poly(n,N,B,H) \cdot S(n,\delreg,1/n). \label{eq:cl-time}\end{equation}
\end{lemma}

\begin{proof}
The algorithm proceeds as follows. Let $D := S(n,\delreg,1/n)$, and let $\pimix$ be the policy that chooses a uniformly random $\pi \in \Psi$ and then follows $\pi$ for the entire episode. Compute independent trajectories $(\tau^{\pimix,i})_{i=1}^{D}$ using $\DrawTraj$ (\cref{lemma:counter-trajectory}) with policy $\pimix$ and the samples $(W_i)_{i=1}^{D}$. Also let $\pizero$ be the policy that always outputs action $0$, and compute independent trajectories $(\tau^{\pizero,i})_{i=1}^{D}$ using $\DrawTraj$ with policy $\pizero$ and the samples $(W_i)_{i=D+1}^{2D}$. 

Next, let $\LFC$ be the algorithm defined in \cref{lemma:learn-from-corr}, instantiated with the noisy parity learning algorithm $\Alg$. For each $(m,r,r') \in I$, compute \[\sk^{m,r,r'} \gets \LFC((Z_i,F_i)_{i=1}^D, m, r, r', \delta, 1/(2^H C_{\ref{lemma:f-corr}}(H)))\] where for each $1 \leq i \leq D$, $(Z_i,F_i) := P(\tau^{\pizero,i},\tau^{\pimix,i})$, and $P$ is the stochastic transformation defined as follows: first, draw $F_i \sim \Ber(1/2)$. Second, if $F_i = 0$, then $Z_i$ is the emission at step $H$ in trajectory $\tau^{\pizero,i}$; otherwise, $Z_i$ is the emission at step $H$ in trajectory $\tau^{\pimix,i}$. 

Finally, output $\{\sk^{m,r,r'}: (m,r,r') \in I\}$.

\paragraph{Analysis.} For notational convenience, set $M := M^\sk_{n,N,H,\delta}$ and $\til M := \til M^\sk_{n,N,H,\delta}$. Let $\MS[H] := \{H\}\times\{0,\dots,H-1\}\times\BF_2$. Define $\beta\in\Delta(\MS[H])$ by \[\beta(s) := \frac{1}{2} d^{M,\pimix}_H(s) + \frac{1}{2} d^{M,\pizero}_H(s).\]
Observe that the samples $(Z_i,F_i)_{i=1}^D$ are independent. We analyze the distribution of each $(Z_i,F_i) = P(\tau^{\pizero,i},\tau^{\pimix,i})$ using \cref{lemma:realizability-error-decomp}; note that $\tau^{\pizero,i} \sim (\til M,\pizero)$ and $\tau^{\pi,i} \sim (\til M,\pimix)$ by \cref{lemma:counter-trajectory}. Consider hypothetical trajectories $\tau^{\pizero} \sim (M,\pizero)$ and $\tau^\pimix \sim (M,\pimix)$, and set $(Z,F) := P(\tau^\pizero,\tau^\pimix)$. Then by definition of $P$ and $\beta$ we have that $Z \sim \sum_{s\in\MS[H]} \beta(s) \MD^\sk_{n,N,H,\delta}(\cdot|s)$. Moreover, for any latent state $s \in \MS[H]$, for any $z$ such that $\Dec^\sk_{n,N,H}(z) = s$, we have
\begin{align*}\Pr[F = 1 | Z=z] 
&= \frac{\frac{1}{2}\Pr[Z=z|F=1]}{\frac{1}{2}\Pr[Z=z|F=0] + \frac{1}{2}\Pr[Z=z|F=1]} \\ 
&= \frac{\MD^\sk_{n,N,H,\delta}(z|s) d^{M,\pimix}_H(s)}{\MD^\sk_{n,N,H}(z|s) d^{M,\pizero}_H(s) + \MD^\sk_{n,N,H,\delta}(z|s) d^{M,\pimix}_H(s)} \\ 
&= \frac{d^{M,\pimix}_H(s)}{d^{M,\pizero}_H(s)+d^{M,\pimix}_H(s)} =: f(s)
\end{align*}
and hence
\[(Z,F) \sim \sum_{s\in\MS[H]} \beta(s) \MD^\sk_{n,N,H,\delta}(\cdot|s) \times \Ber(f(s)).\]
Applying \cref{lemma:realizability-error-decomp} to each $(Z_i,F_i)$ and summing over $i \in [D]$ gives
\begin{align}
  \label{eq:D-sample-tvd}
  \TV\left(\Law((Z_i,F_i)_{i=1}^D), \left(\sum_{s\in\MS[H]} \beta(s) \til\MD^\sk_{n,N,H,\delta}(\cdot|s)\times \Ber(f(s))\right)^{\otimes D}\right) \leq 3DH^2\exp(-\delta^4 N).
\end{align}
Note that $\alpha := \sum_{s\in\MS[H]} \beta(s)f(s) = \frac{1}{2}$, and so
\begin{align*}
\sum_{s\in\MS[H]} \beta(s)|f(s)-\alpha| 
&\geq \beta(H,H-1,0)|f(H,H-1,0) - 1/2| + \beta(H,H-1,1)|f(H,H-1,1)-1/2| \\ 
&= \frac{1}{4}|d^{M,\pimix}_H(H,H-1,0) - d^{M,\pizero}_H(H,H-1,0)| \\
&\qquad+ \frac{1}{4}|d^{M,\pimix}_H(H,H-1,1) - d^{M,\pizero}_H(H,H-1,1)| \\ 
&\geq \frac{1}{2^{H+1}}.
\end{align*}
Note that the equality above uses the definitions of $\beta$ and $f$. The second inequality above uses the fact that $\pizero$ visits $(H,H-1, 0)$ and $(H,H-1,1)$ with probability $2^{-H}$ each, by definition of the latent transitions \cref{eq:latent-transitions}, whereas by assumption $\pimix$ visits each of those states with probability at least $2^{1-H}$ (every policy visits $(H,H-1,0)$ and $(H,H-1,1)$ with the same probability). By \cref{lemma:f-corr}, there is some $(m,r,r') \in I$ such that 
\[\left|\Pr_{(B,F)\sim\mu^{r'}}[B_1+\dots+B_m+rB_h+F\equiv 0\bmod{2}] - \frac{1}{2}\right| \geq \frac{1}{2^{H+1} \cdot C_{\ref{lemma:f-corr}}(H)}\]
where $\mu^0 := \mu_{\beta,1/2}$ and $\mu^1 := \mu_{\beta,f}$. By \cref{lemma:learn-from-corr}, \cref{eq:D-sample-tvd}, and the choice of $D$, it holds that \[\Pr[\sk^{m,r,r'} = \sk] \geq 1 - 1/n - 3DH^2\exp(-\delta^4 N) = 1 - 1/n - O(\delreg^{-2}H^2\exp(-2n)),\] where the last inequality uses the assumption on $N$ and the fact that $D = S(n,\delreg,1/n) \leq O(n\delreg^{-2}2^n)$ (otherwise, by \cref{lemma:brute-force-lpn}, $\Alg$ could be replaced by $\Brute$). The claim \cref{eq:cl-correctness} follows.

\paragraph{Time complexity.} Since $\size(\MC_\pi) \leq B$ for all $\pi \in \Psi$, an action from $\pimix$ can be drawn in time $O(\log|\Psi|) + \poly(n,N,B,H)$, where the first term bounds the complexity of drawing a random policy from $\Psi$.\footnote{This is in a random access model; one could avoid this modelling assumption by e.g. subsampling $\Psi$ at the beginning of the algorithm and using the same subsample for each episode, but we will make the random-access modelling assumption to avoid unnecessary machine-level implementation details.} We may assume without loss of generality that $|\Psi| \leq \poly(n) \cdot 2^n$ (since otherwise there is a trivial algorithm for recovering $\sk$), so this bound simplifies to $\poly(n,N,B,H)$. Thus, by \cref{lemma:counter-trajectory} and the definition of $D$, the time complexity of drawing $D$ trajectories each from $\pimix$ and $\pizero$ is $\poly(n,N,B,H) \cdot S(n,\delreg,1/n)$. Computing the regression samples $(Z_i,F_i)_{i=1}^D$ takes time $\poly(n,N,H) \cdot S(n,\delreg,1/n)$. By \cref{lemma:learn-from-corr}, the $O(H)$ calls to $\LFC$ together take time $O(H) \cdot T(n,\delreg,1/n)$. This proves the claimed time complexity bound of the overall algorithm. 
\end{proof}


\section{Oracle lower bound}
\label{sec:oracle-lb}
The main result of this section is \cref{thm:reduction-prp}, the formal statement of \cref{thm:no-reduction-intro}. Recall from \cref{def:rl-reg-reduction-2} that a \emph{reduction from RL to regression} is one that makes few oracle calls, and a \emph{computational reduction from RL to regression} is one that is moreover oracle-efficient, i.e. computationally efficient aside from the oracle calls. 

By standard results on the statistical complexity of RL in block MDPs \cite{jiang2017contextual}, any $K(n)$-computable block MDP family (\cref{def:computable-family}) trivially admits a $(T,\epsilon)$-reduction from RL to regression with $T(n) \leq \poly(K(n))$, since the reduction has no computational restrictions and therefore only needs to call the sampling oracle, which it only needs to do $\poly(H_n,A_n,S_n,\log|\Phi_n|) \leq \poly(K(n))$ times. This raises two questions:
\begin{enumerate}
\item A priori, it's unclear whether the dependence on $\log|\Phi_n|$ is necessary, since the reduction also has access to a regression oracle with low population error, potentially obviating issues of generalization.\footnote{In typical oracle-efficient algorithms for RL, the oracles are assumed only to have low \emph{sample} error, and the dependence on $\log|\Phi|$ arises in the analysis via generalization bounds.} Could it be the case that for \emph{any} block MDP family, there is a $(T,\epsilon)$-reduction from RL to regression with $T(n) \leq \poly(H_n,A_n,S_n)$ and $\epsilon(n) \geq 1/\poly(H_n,A_n,S_n)$?
\item The above reduction is in no way oracle-efficient. For any $K(n)$-computable block MDP family, is there a $(T,\epsilon,B)$-\emph{computational} reduction from RL to regression with $T(n) \leq \poly(K(n))$, $B(n) \leq \poly(K(n))$, and $\epsilon(n) \geq 1/\poly(K(n))$?
\end{enumerate}

We provide strong negative answers to both questions, which turn out to be linked. In \cref{sec:oracle-lb-first}, we construct a block MDP family $\Moracle$ for which there is no reduction from RL to regression (\cref{thm:oracle-lb}) without exponential dependencies. This family does not admit succinct optimal policies so it is not a computable block MDP family. In \cref{sec:oracle-lb-second}, we rectify this issue and show that there is a computable block MDP family $\tilMoracle$ for which there is no computational reduction from RL to regression (\cref{thm:reduction-prp}) without exponential dependencies. This second result formally proves \cref{thm:no-reduction-intro}, and uses the first result as a key ingredient.

  \subsection{Ruling out reductions from RL to regression}
  \label{sec:oracle-lb-first}
  We begin by defining the block MDP family $\Moracle$. 
Fix $N,H \in \BN$, and set $X := 2^{\lceil \log N^5 \rceil}$ (i.e., $X$ is the smallest power of 2 which is at least $N^5$). Write $\MX := [X] = \{ 0,1\}^{\log X}$ and $\MS := ( [H] \times \{0,1,2\}) \cup \{ \perp \}$. (For convenience, we omit the dependence of $\MX, \MS$ on $X$ and $H$, respectively.) Consider a function $\phi : \MX \to \MS$.  We define an MDP $M_{X,H}^\phi$ with emission space $\MX$ and latent state space $\MS$ as follows:  the action space is $\MA = \{0,1\}$, the initial latent state distribution is $\til \BP_0:= \mathrm{Unif}\{(1,0), (1,1) \}$, and the latent transitions are defined  as follows: 
\begin{align}
  \til \BP((h+1, 2) | (h, 2), a) = 1 & \quad \forall h \in [H-1], a \in \{0,1\} \nonumber\\
  \til \BP((h+1, 2) | (h, 1-a), a) = 1 & \quad \forall h \in [H-1], a \in \{0,1\} \nonumber\\
  \til \BP((h+1, 0) | (h, a), a) = \til \BP((h+1, 1) | (h,a), a) = \frac{1}{2} & \quad \forall h \in [H-1], a \in \{0,1\}\nonumber.
\end{align}
The idea of the latent transition $\til \BP$ is that if the learner does not play action $a$ at state $(h,a)$, it transitions to a ``failure'' state $(h+1, 2)$. Otherwise, it transitions to a uniformly random state in $\{ (h+1,0), (h+1,1) \}$. Finally, the state $\perp \in \MS$ is unreachable and always transitions to itself. 

The  emission distribution for the MDP $M_{X,H}^\phi$ is defined as follows:
given a latent state $s \in \MS$, the emission distribution $\MD_\phi(\cdot|s) :=\mathrm{Unif}(\phi^{-1}(s))$ (i.e., the emission is a uniformly random element in the preimage of $s$). Let $\Phiset$ denote the set of all functions $\phi : \MX \to \MS$.

Finally, the reward function $\br:\MS\to[0,1]$ is defined as follows: $\br((H,0)) = \br((H,1)) = 1$, and all other rewards are $0$. Note that the latent state space in any MDP $M_{X,H}^\phi$ has a layered structure, with states $(h,b) \in [H] \times \{0,1,2\}$ only accessible at step $h$; accordingly, we have dropped the subscript $h$ from the latent state transitions $\til\BP$ and the reward function $\br$.

Finally, recall that for a general policy $\pi$, $\BP^{M_{X,H}^{\phi}}$ denotes the distribution of a trajectory  drawn from $M_{X,H}^\phi$, according to the policy $\pi$; for convenience, we shorten this to $\BP^{\phi, \pi}$, and denote the corresponding expectation as $\E^{\phi, \pi}$. 

\paragraph{Definition of $\UPhi$.} We now define a distribution $\phi^\st \sim \UPhi \in \Delta(\Phiset)$, according to the following process. First, for a vector $\til x  = (\til x_{s, i})_{(s,i) \in \MS \times [N]}$, we let $\Dec_{\til x}$ denote the function $\phi^\st : \MX \to \MS$, defined as follows: for $x \in \MX$, $\phi^\st(x)$ is:
\begin{itemize}
\item Equal to the smallest $s$ for which there exists $i \in [N]$ so that $\til x_{s,i} = x$;
\item Equal to $\perp$ if no such $s$ exists.
\end{itemize}

We then let $\til x_{s,i} \in \MX$ be independent and uniformly random elements of $\MX$, for $(s,i) \in \MS \times [N]$. Finally, we let $\UPhi$ be the distribution of $\phi^\st := \Dec_{\til x}$.

Let $\Edistinct \subset \Phiset$ denote the set of $\phi \in \Phiset$ so that $|\phi^{-1}(s)| = N$ for each $s \in \MS \backslash \{ \perp \}$. 
Note that, for $\phi^\st := \Dec_{\til x}$ for some $\til x \in \MX^{\MS \times [N]}$, we have $\phi \in \Edistinct$ if and only if the values of $\til x_{s,i}$ are distinct for all $s \in \MS, i \in [N]$. Moreover, if $\phi \in \Edistinct$, then $\phi(\til x_{s,i}) = s$ for each $s \in \MS \backslash \{ \perp \},\ i \in [N]$. We show in \cref{lem:edistinct-prob} below that $\UPhi(\Edistinct)$ is close to 1 (which results from our choice of $X$ to be sufficiently large compared to $N$).

\paragraph{Conditional distributions for $\UPhi$.} 
Let $\MT : \MS \times [N] \to \MX \cup \{ \perp\}$ be a partial function mapping $\MS \times [N]$ to $\MX$. We interpret points $(s,i)$ for which $\MT(s,i) \neq \perp$ as the domain of $\MT$, and write $\dom(\MT) := \{ (s,i) \ : \ \MT(s,i) \neq \perp\}$. We let $\UPhi|_\MT \in \Delta(\Phiset)$ be the distribution of $\Dec_{\til x}$, where $\til x = (\til x_{s,i})_{(s,i) \in \MS \times [N]}$ has its components independently and uniformly distributed on $\MX$ conditioned on the event $\{\til x_{s,i} = \MT(s,i) \ \forall (s,i) \in \dom(\MT)\}$.

\begin{definition}[Random trajectory \& mean label]\label{def:random-trajectory}
For any general policy $\pi$, we define $\sigma_\pi \in \Delta(\MA^H \times \MX^H)$ to be the distribution of $(a_{1:H}, x_{1:H})$ defined by:
\begin{align}
a_h = \pi_h(a_{1:h-1}, x_{1:h}), \quad \mbox{ where } \quad x_1, \ldots, x_H \sim \Unif(\MX) \mbox{ are i.i.d.}\nonumber
\end{align}
Given $L : \MX^H \times \MA^H \to [0,1]$, we then define $\mu_{L,\pi} := \E_{\sigma_\pi}[L(x_{1:H} , a_{1:H})]$.
\end{definition}

\paragraph{The block MDP family.} 
We now define a block MDP family $\Moracle$, indexed by $n \in \BN$, in \cref{def:oracle-block-mdp-family} below. Roughly speaking, the family consists of the MDPs $M_{X, H}^\phi$, for $\phi \in \Phi_{X, H}$ where the parameters $X, H$ are chosen to scale with the size parameter $n$ in a suitable way. 
\begin{defn}
  \label{def:oracle-block-mdp-family}
  We define a block MDP family $\Moracle$ as the tuple
  \begin{align}
\Moracle := ((\MS_n)_n, (\MA_n)_n, (H_n)_n, (\ell_n)_n, (\Phi_n)_n, (\MM_n)_n)\nonumber,
  \end{align}
  where, for $N_n := 2^{n}, H_n := n, X_n = 2^{\lceil \log N_n^5 \rceil}$, we have:
  \begin{itemize}
  \item $\MS_n = ([H_n] \times \{0,1,2\}) \cup \{ \perp \}$
  \item $\MA_n = \{0,1\}$
  \item $\ell_n =  \log X_n $, so that the observation space is $\MX_n := \{0,1\}^{\ell_n} \equiv [X_n]$. 
  \item $\Phi_n = \Phi_{X_n, H_n} = \MS_n^{\MX_n}$
  \item $\MM_n = \{ M_{X_n, H_n}^\phi \ : \ \phi \in \Phi_n \}$.
  \end{itemize}
\end{defn}




We are ready to state our first lower bound for oracle reductions from RL to regression: 
\begin{theorem}
  \label{thm:oracle-lb}
  There are constants $C_{\ref{thm:oracle-lb}}, C_{\ref{thm:oracle-lb}}'$ and block MDP family $\MM$ (\cref{def:block-family}) satisfying $\max\{ |\MS_n|, |\MA_n|, \ell_n \} \leq O(H_n)$ so that the following holds. Consider complexity measures $\Tred : \BN \to \BN, \epred : \BN \to (0,1)$.  
  Suppose that there is a $(\Tred,\epred)$-reduction $\Alg$ from RL to regression for the family $\MM$. Then for $n \geq C_{\ref{thm:oracle-lb}}'$, we must have either 
  \begin{align}
    \epred(n) \leq C_{\ref{thm:oracle-lb}} \cdot 2^{-H_n/8} \quad \mbox{ or } \quad \Tred(n) \geq C_{\ref{thm:oracle-lb}}^{-1} \cdot 2^{H_n/8} \label{eq:oracle-lb-conclusion}.
  \end{align}
\end{theorem}

  \begin{algorithm}[t]
    \caption{$\SimulateReduction(\Alg, X, H, \MO_{\til x}, \ep,\delta)$: simulation of $\Alg$ and regression oracle}
    \label{alg:simulate-oracle}
    \begin{algorithmic}[1]
      \Require Reduction $\Alg$ from RL to regression, parameters $X, H \in \BN$, $\delta, \ep \in (0,1)$, and an oracle $\MO_{\til x}$ which supports queries to individual elements of a vector $\til x \in \MX^{\MS \times [N]}$.
      \State Let $\phi^\st := \Dec_{\til x}$.
      \State Initialize $\MT(s,i) \gets \perp$ for all $s,i$. 
      \State Execute $\Alg$ until its first oracle call.
      \For{Each oracle call of $\Alg$}
      \If{The oracle call is a sampling oracle call with input $\MB_\pi$}
      \State \Return $\mathtt{SimulateSampling}(\MB_\pi, \MO_{\til x})$.
      \ElsIf{The oracle call is a regression oracle call with input $(\MB_\pi, \MB_L, h)$}
      \State \Return $\mathtt{SimulateRegression}(\MB_\pi, \MB_L, h, \MO_{\til x},\epsilon,\delta)$.
      \EndIf
      \EndFor
      \State \Return the output policy $\hat \pi$ of $\Alg$.

      \Function{\SimulateSampling}{$\MB_\pi, \MO$}
      \State Draw $s_1 \sim \til \BP_0$. 
      \For{$1 \leq h \leq H$}
      \State Given $s_h$, draw $i_h \sim \Unif([N])$, and set $x_h = \MO({i_h, s_h})$, $a_h := \pi_h(x_{1:h}, a_{1:h-1})$, $r_h := \br_h(s_h)$.
      \State Set $\MT(s_h,i_h) \gets \MO(s_h,i_h)$.
      \State Draw $s_{h+1} \sim \til \BP(\cdot \mid s_h, a_h)$.
      \EndFor
      \State \Return the trajectory $(s_{1:H},x_{1:H}, a_{1;H}, r_{1:H})$. 
      \EndFunction

      \Function{\SimulateRegression}{$\MB_\pi, \MB_L, h, \MO, \epsilon, \delta$}
      \State \label{line:oracle-w} Set $m:= C_0 \log(1/\delta)/\ep^2$, for a sufficiently large constant $C_0$.
      \State \label{line:sample-trajs}Draw $m$ i.i.d.~trajectories $((s_{1:H}^i, x_{1:H}^i, a_{1:H}^i, r_{1:H}^i))_{i=1}^m$ by calling $\SimulateSampling(\MB_\pi, \MO)$ $m$ times.
      \State \Return the constant function $\hat \mu :=  \frac 1m \sum_{i=1}^m L(x_{1:H}^i, a_{1:H}^i, r_{1:H}^i)$. 
      \EndFunction

      \end{algorithmic}
    \end{algorithm}

    \subsubsection{Lemmas for \cref{thm:oracle-lb}}
    We first prove several lemmas useful in the proof of \cref{thm:oracle-lb}. First, we show that $\Edistinct$ occurs with high probability over $\UPhi$. 
\begin{lemma}
  \label{lem:edistinct-prob}
Given $N, H \in \BN$ with $H \leq N/4$, for $X = 2^{\lceil \log N^5 \rceil}$, the event $\Edistinct \subset \Phi_{X, H}$ satisfies $\UPhi(\Edistinct) \geq 1-1/N$. 
\end{lemma}
\begin{proof}
Recall that $|\MS| \leq 4H$ and $|\MX| = X$. For $(s,i) \in \MS \times [N]$, let $\til x_{s,i}$ be independently and uniformly distributed elements of $\MX$.  For $s,s' \in \MS,\ i,i' \in [N]$ for which $(s,i) \neq (s',i')$, note that $\Pr(\til x_{s,i} = \til x_{s', i'}) = 1/X$. Let $P$ denote the number of distinct pairs $(s,i), (s',i')$ for which $\til x_{s,i} = \til x_{s',i'}$, so that $\E[P] \leq |\MS \times [N]|^2/X \leq  (4HN)^2/X$.  Thus, $\Pr(P \geq 1) \leq (4HN)^2/X \leq N^4/X$ by Markov's inequality. Using that $X \geq N^5$, we see that $P = 0$ with probability at least $1-1/N$; moreover, under the event that $P = 0$, we have that $\phi^\st := \Dec_{\til x}$ satisfies $\phi^\st \in \Edistinct$, so that $\UPhi(\Edistinct) \geq \Pr(P = 0) \geq 1-1/N$. 
\end{proof}

    The following lemmas show that conditioning on a small set $\MT \subset \MS \times [N] \times \MX$ does not significantly change the distribution $\BP^{\phi^\st, \pi}$, for any general policy $\pi$:
\begin{lemma}
  \label{lem:diff-bound}
  Fix $\til x, \til x' \in \MX^{\MS \times [N]}$ and define  \[Z := \#\{(s,i) \in \MS \times [N]: x_{s,i} \neq x_{s,i}'\}.\] 
  Furthermore suppose that all of the values $\til x_{s,i}$ are distinct. Write $\phi = \Dec_{\til x}$ and $\phi' = \Dec_{\til x'}$. Then for any general policy $\pi$,
  \begin{align}
\tvd{\BP^{\phi, \pi}}{\BP^{\phi', \pi}} \leq \frac{4HZ}{N}\nonumber.
  \end{align}
\end{lemma}
\begin{proof}
  If $Z > N/4$ then the statement is trivially true, so we assume henceforth that $Z \leq N/4$. For any $s \in \MS$, note that $\phi^{-1}(s)$ and $(\phi')^{-1}(s)$ differ in at most $2Z$ elements. Moreover, since all of the values $\til x_{s,i}$ are distinct, $|\phi^{-1}(s)| = N$ for each $s \in \MS$. Recalling the definition $\MD_\phi(\cdot \mid s) := \Unif(\phi^{-1}(s))$, it follows that $\tvd{\MD_{\phi}(\cdot \mid s)}{\MD_{\phi'}(\cdot \mid s)} \leq \frac{2Z}{N - 2Z} \leq \frac{4Z}{N}$ for each $s \in \MS$, since $N \geq 4Z$. Thus, for any general policy $\pi$, 
  \begin{align}
    & \tvd{\BP^{\phi, \pi}}{\BP^{\phi', \pi}} \nonumber\\
    &\leq  \sum_{h=0}^{H-1} \E^{\phi, \pi} \left[ \tvd{\BP^{\phi,\pi}(x_{h+1} = \cdot \mid s_{1:h+1}, x_{1:h}, a_{1:h})}{\BP^{{\phi',\pi}}(x_{h+1} = \cdot \mid s_{1:h+1}, x_{1:h}, a_{1:h})}\right]\nonumber\\
    &= \sum_{h=1}^H \E^{\phi, \pi} \left[ \tvd{\MD_\phi(\cdot \mid s_h)}{\MD_{\phi'}(\cdot \mid s_h)} \right]\nonumber\\
    &\leq 4ZH/N\nonumber,
  \end{align}
  where the first inequality uses \cref{lem:tvd-chain} as well as the fact that the latent state transitions of $M_{X,H}^\phi, M_{X,H}^{\phi'}$ are identical. 
\end{proof}

\begin{lemma}
  \label{lem:coupling-uphi}
  Consider any mapping $\MT : \MS \times [N] \to \MX \cup \{ \perp\}$. Then there is a coupling $\SV \in \Delta(\Phiset \times \Phiset)$ of $\UPhi$ and $\UPhi|_{\MT}$ so that, for any general policy $\pi$, for $(\phi, \phi') \sim \SV$, under the event that either $\phi$ or $\phi'$ satisfies $\Edistinct$, 
  \begin{align}
\tvd{\BP^{M_{X,H}^\phi, \pi}}{\BP^{M_{X,H}^{\phi'}, \pi}} \leq \frac{4HZ}{N}\label{eq:phi-phip-close},
  \end{align}
  where $Z = |\{(s,i) \in \MS \times [N] \ : \ \MT(s,i) \neq \perp\}|$.
\end{lemma}
\begin{proof}
  Consider a random variable $(\til x, \til x')$, where $\til x_{s,i} \sim \Unif(\MX)$ i.i.d.~for each $(s,i) \in \MS \times [N]$, and where
  \begin{align}
    \til x'_{s,i} = \begin{cases}
      \til x_{s,i} &: \MT(s,i) = \perp \\
      \MT(s,i) &: \MT(s,i) \neq \perp.
    \end{cases}\nonumber
  \end{align}
  Let $\phi := \Dec_{\til x}$ and $\phi' := \Dec_{\til x'}$, and $\SV$ be the distribution of $(\phi, \phi')$. Note that $\SV$ is a coupling of $\UPhi$ and $\UPhi|_\MT$. Moreover, under the event that either $\phi$ or $\phi'$ satisfies $\Edistinct$ (i.e., $\til x_{s,i}$ are distinct for all $s,i$, or $\til x_{s,i}'$ are distinct for all $s,i$), the conditions of \cref{lem:diff-bound} are satisfied, and thus \cref{eq:phi-phip-close} holds.
\end{proof}

The following lemma shows that the latent MDP underlying $\Moracle$ satisfies open-loop indistinguishability (\cref{def:open-loop-indist}). In fact, there is a simple expression for the probability of encountering each given latent state given that any fixed action sequence is played: 
\begin{lemma}
  \label{lem:qs-formula}
Fix $h \in [H]$, $s_h \in \MS$, and a sequence $a_{1:H} \in \MA^H$. Then it holds that
  \begin{align}
\sum_{ s_{-h} \in \MS^{H-1}} \prod_{g=1}^H \til \BP( s_g |  s_{g-1}, a_{g-1}) = \begin{cases}
      1-2^{1-h} &:  s_h = (h, 2) \\
      2^{-h} &:  s_h \in \{ (h,0), (h,1) \}
    \end{cases}\label{eq:qs-lem-formula}
    =: Q_h(s_h)
  \end{align}
where $s_{-h} \in \MS^{H-1}$ represents a vector indexed by $[H] \setminus \{h\}$, and $\til \BP(s_1|s_0,a_0)$ formally denotes $\til \BP_0(s_1)$.
\end{lemma}
\begin{proof}
  Consider any sequence $s_{1:H} \in \MS^H$ so that $\prod_{g=1}^H \til \BP(s_g | s_{g-1}, a_{g-1}) > 0$.  Let $j^\st(s_{1:H}) := \min \{ j \in [H] \ : \ s_j = (j,2) \}$, with $j^\st(s_{1:H}) = H+1$ if the minimum is over an empty set. We must have $s_g = (g,2)$ for all $g > j$. Moreover, for all $g < \min\{j-1, H\}$, we must have $s_g = (g,a_g)$, and for $g=j-1$, we must have $s_g = (g, 1-a_g)$. For such a sequence $s_{1:H}$, we therefore have $\prod_{g=1}^H \til \BP(s_g | s_{g-1}, a_{g-1}) = 2^{1-j^\st(s_{1:H})}$ if $j^\st(s_{1:H}) < H$ and $\prod_{g=1}^H \til \BP(s_g | s_{g-1}, a_{g-1}) = 2^{2-H}$ if $j^\st(s_{1:H}) = H$.

  If $s_h = (h,2)$, then we must have $j^\st(s_{1:H}) \leq h$, and so the quantity in \cref{eq:qs-lem-formula} is equal to $\sum_{j=1}^h 2^{1-j} = 2^{1-h}$. If $s_h \in \{(h,0), (h,1)\}$, then we must have $j^\st(s_{1:H}) > h$, and so the quantity in \cref{eq:qs-lem-formula} is equal to $\sum_{j=h+1}^{H-1} 2^{1-j} + 2^{2-H} = 2^{-h}$, as desired.
\end{proof}

\begin{lemma}
  \label{lem:pq-approx}
  Consider any general policy $\pi$, $h \in [H]$ and $\bar s_h = (h,b) \in \MS$, as well as $\delta \in (0,1)$. There is an event $\MF_{\pi, h, \bar s_h}$ that occurs with probability $1-\delta$ over the randomness of $\phi^\st \sim \UPhi$, so that, under $\MF_{\pi, h, \bar s_h} \cap \Edistinct$, we have
  \begin{align}
\left| \BP^{\phi^\st, \pi}[s_h = \bar s_h] - Q(\bar s_h) \right| \leq \sqrt{\frac{\log 2/\delta}{N}}\nonumber,
  \end{align}
  where $Q(\bar s_h)$ is as defined in \cref{eq:qs-lem-formula}.
\end{lemma}
\begin{proof}
Consider any $\bar s_h \in \MS$ of the form $\bar s_h = (h,b)$.  To begin, fix any $\til x \in \MX^{\MS\times [N]}$ and $\phi := \Dec_{\til x}$ with $\phi \in \Edistinct$. For any $x \in \MX$ and $s \in \MS$, we may write $\MD_\phi(x | s) := \frac 1N \sum_{i=1}^N \One\{ x = \til x_{s,i} \}$. Thus, we may compute that
  \begin{align}
    & \BP^{\phi, \pi}[s_h = \bar s_h]\nonumber\\
    &= \sum_{\bar s_{-h} \in \MS^{H-1}} \sum_{\bar x_{1:H} \in \MX^H} \BP^{\phi, \pi}[s_{1:H} = \bar s_{1:H}, x_{1:H} = \bar x_{1:H}]\nonumber\\
    &= \sum_{\bar s_{-h} \in \MS^{H-1}}\sum_{\bar x_{1:H} \in \MX^H}  \frac{1}{N^H} \sum_{\bi = (i_1, \ldots, i_H) \in [N]^H} \prod_{g=1}^H \One\{\bar x_g = \til x_{\bar s_g, i_g} \} \cdot \til \BP(\bar s_g | \bar s_{g-1}, a_{g-1}^{\bi}) \nonumber\\
    &= \frac{1}{N^H} \sum_{\bi = (i_1, \ldots, i_H) \in [N]^H} \sum_{\bar s_{-h} \in \MS^{H-1}}  \prod_{g=1}^H \til \BP(\bar s_g| \bar s_{g-1},  a_{g-1}^{\bi}) \label{eq:pphi-prods},
  \end{align}
  where above we have inductively defined $a_g^\bi \in \MA$ by $a_g^\bi := \pi_g(a_{1:g-1}^\bi, (\til x_{\bar s_1, i_1}, \ldots, \til x_{\bar s_g, i_g}))$ (note that $a_g^\bi$ depends on $\bar s_{1:H}$; we omit this dependence).

  Now recall that $\phi^\st \sim \UPhi$ is defined as $\phi^\st := \Dec_{\til x}$ where $\til x_{s,i} \sim \Unif(\MX)$ are i.i.d.~for all $s \in \MS$ and $i \in [N]$. Thus, for each $\bi = i_{1:H} \in [N]^H$, it holds from \cref{def:random-trajectory} that
  \begin{align}
\E_{\UPhi} \left[ \sum_{\bar s_{-h} \in \MS^{H-1}}  \prod_{h=1}^H \til \BP(\bar s_h| \bar s_{h-1}, a_{h-1}^{\bi}) \right] &= \E_{a_{1:H} \sim \sigma_\pi} \left[ \sum_{\bar s_{-h} \in \MS^{H-1}} \prod_{g=1}^H \til \BP(\bar s_g | \bar s_{g-1}, a_{g-1}) \right]  = Q_h(\bar s_h)\label{eq:sums-pq},
  \end{align}
  where the second equality uses \cref{lem:qs-formula}. Let us define $F : \MX^{\MS\times[N]} \to \BR$ by
  \begin{align}
F((y_{s,i})_{s\in\MS, i \in [N]}) := \frac{1}{N^H} \sum_{\bi = i_{1:H} \in [N]^H} \sum_{\bar s_{-h} \in \MS^{H-1}} \prod_{g=1}^H \til \BP(\bar s_g | \bar s_{g-1}, a_{g-1}^\bi)\nonumber,
  \end{align}
  where $a_g^\bi = \pi_g(a_{1:g-1}^\bi, (y_{\bar s_1,i_1}, \ldots, y_{\bar s_g,i_g}))$, so that $\BP^{\phi^\st,\pi}[s_h=\bar s_h] = F((\til x_{s,i})_{s\in\MS,i\in[N]})$. Note that $F$ satisfies the bounded differences property \cref{eq:bd} with $c_{s,i} = 1/N$ for all $i \in [N], s \in \MS$. Thus, by \cref{thm:mcdiarmid}, there is some event $\MF_{\pi, h, \bar s_h}$ that occurs with probability $1-\delta$ so that, under $\MF_{\pi, h, \bar s_h}$, we have that
  \begin{align}
\left| Q(\bar s_h) - \BP^{\phi^\st, \pi}[s_h = \bar s_h] \right| &= \left| F((\til x_{s,i})_{s \in \MS,i \in [N]}) - \E_{\UPhi}[F((\til x_{s,i})_{s \in \MS,i \in [N]})] \right| \leq \sqrt{\frac{\log 2/\delta}{N}}\nonumber,
  \end{align}
  where the equality above uses \cref{eq:sums-pq,eq:pphi-prods}.
\end{proof}

In the following lemma we show that for any policy $\pi$ and label function $L$ that are independent of the true decoding function $\phi^\st$, it holds with high probability over the choice of decoding function that the optimal predictor $G(s)$ (of the label given the current state) is close to constant. The proof crucially uses the fact that for any fixed action sequence, the induced state visitation distribution at each step does not depend on the action sequence.

  \begin{lemma}
    \label{lem:constant-fn-predictor}
    Fix $\delta \in (0,1)$, $h \in [H]$, a general policy $\pi$, and a function $L : \MX^H \times \MA^H \to [0,1]$. Let  $G(s) := \E^{\phi^\st, \pi}[L(x_{1:H}, a_{1:H}) \mid s_h =s]$. Then there is an event $\ME'_{\pi, L, h}$ occurring with probability at least $1-\delta$ over the choice of $\phi^\st\sim \UPhi$, so that, under $\ME'_{\pi, L, h} \cap \Edistinct$, the following holds: the value $\mu_{L,\pi} := \E_{\sigma_\pi}[L(x_{1:H}, a_{1:H})]$ satisfies  $\E^{\phi^\st, \pi}[|G(s_h) - \mu_{L, \pi}|] \leq \ep$, where $\ep = C_{\ref{lem:constant-fn-predictor}} \sqrt{\log(1/\delta)/N}$, for some constant $C_{\ref{lem:constant-fn-predictor}}$.
  \end{lemma}

  \begin{proof}
    
  Consider any $\bar s_h \in \MS$ of the form $\bar s_h = (h,b)$.  To begin, fix any $\til x \in \MX^{\MS\times [N]}$ and $\phi := \Dec_{\til x}$ with $\phi \in \Edistinct$. For any $x \in \MX$ and $s \in \MS$, we may write $\MD_\phi(x | s) := \frac 1N \sum_{i=1}^N \One\{ x = \til x_{s,i} \}$. Thus, we may compute that
  \begin{align}
    & \BP^{\phi, \pi}[s_h = \bar s_h] \cdot \E^{\phi, \pi}[L(x_{1:H}, a_{1:H}) \mid s_h = \bar s_h]\nonumber\\
    =& \sum_{\bar s_{-h} \in \MS^{H-1}} \sum_{\bar a_{1:H} \in \MA^H, \bar x_{1:H} \in \MS^H} \BP^{\phi, \pi}[s_{1:H} = \bar s_{1:H}, x_{1:H} = \bar x_{1:H}, a_{1:H} = \bar a_{1:H}] \cdot L(\bar x_{1:H}, \bar a_{1:H})\nonumber\\
    =& \sum_{\bar s_{-h} \in \MS^{H-1}} \sum_{\bar x_{1:H} \in \MX^H} \frac{1}{N^H} \sum_{\bi = i_{1:H} \in [N]^H} \left(\prod_{g=1}^H \One \{ \bar x_g = \til x_{i_g, \bar s_g} \} \cdot \til \BP(\bar s_g | \bar s_{g-1}, a_g^\bi)\right) \cdot L(\bar x_{1:H}, a_{1:H}^\bi)\nonumber\\
    =& \sum_{\bar s_{-h} \in \MS^{H-1}} \frac{1}{N^H} \sum_{\bi = i_{1:H} \in [N]^H} \left(\prod_{g=1}^H \til \BP(\bar s_g | \bar s_{g-1}, a_g^\bi)\right) \cdot L((\til x_{i_1, \bar s_1}, \ldots, \til x_{i_H, \bar s_H}), a_{1:H}^\bi)\label{eq:pphi-expand},
  \end{align}
  where above we have inductively defined $a_g^\bi \in \MA$ by $a_g^\bi = \pi_g(a_{1:g-1}^\bi, (\til x_{i_1, \bar s_1}, \ldots, \til x_{i_g, \bar s_g}))$ (note that $a_g^\bi$ depends on $\bar s_{1:H}$; we omit this dependence). 
  
Now recall that $\phi^\st \sim \UPhi$ is defined as $\phi^\st := \Dec_{\til x}$ where $\til x_{s,i} \sim \Unif(\MX)$ are i.i.d.~for all $s \in \MS$ and $i \in [N]$. This means that for any fixed sequences $\bi = i_{1:H} \in [N]^H$ and $\bar s_{1:H} \in \MS^H$, we have $((\til x_{i_1, \bar s_1}, \ldots, \til x_{i_H, \bar s_H}), a_{1:H}^\bi) \sim \sigma_\pi$ (\cref{def:random-trajectory}).  
  Thus, for each $\bi = i_{1:H} \in [N]^H$ and $\bar s_{1:H} \in \MS^H$, it holds that
  \begin{align}
    & \E_{\UPhi}\left[ \prod_{g=1}^H \til \BP(\bar s_g | \bar s_{g-1}, a_{g-1}^\bi) \cdot L((\til x_{i_1, \bar s_1}, \ldots, \til x_{i_H, \bar s_H}), a_{1:H}^\bi)\right] \nonumber\\
    &=  \E_{(a_{1:H}, x_{1:H}) \sim\sigma_\pi}\left[ \prod_{g=1}^H \til \BP(\bar s_g | \bar s_{g-1}, a_{g-1}) \cdot L(x_{1:H}, a_{1:H})\right] \label{eq:uphi-exp}
  \end{align}
  where the expectation in the first line is jointly over $(\til x, \phi^\st)$. Summing the display over $\bar s_{-h} \in \MS^{h-1}$ and applying \cref{lem:qs-formula} gives that for each $\bi = i_{1:H} \in [N]^H$ and $\bar s_h \in \MS$,
  \begin{align}
\E_{\UPhi}\left[ \sum_{\bar s_{-h} \in \MS^{H-1}} \prod_{g=1}^H \til \BP(\bar s_g | \bar s_{g-1}, a_{g-1}^\bi) \cdot L((\til x_{i_1, \bar s_1}, \ldots, \til x_{i_H, \bar s_H}), a_{1:H}^\bi)\right] = Q_h(\bar s_h) \cdot \E_{(a_{1:H}, x_{1:H}) \sim \sigma_\pi} \left[   L(x_{1:H}, a_{1:H})\right]\label{eq:F-exp}.
  \end{align}
The above equality crucially uses that $L(x_{1:H}, a_{1:H})$ does not depend on $\bar s_{1:H}$ (and thus could fail to hold if e.g. $L$ were also a function of the rewards $r_{1:H}$). Next, we define a function $F : \MX^{\MS\times[N]} \to \BR$ by 
  \begin{align}
F((y_{s,i})_{s\in\MS, i \in [N]}) := \frac{1}{N^H} \sum_{\bi = i_{1:H} \in [N]^H} \sum_{\bar s_{-h} \in \MS^{H-1}} \prod_{g=1}^H \til \BP(\bar s_g | \bar s_{g-1}, a_g^\bi) \cdot L((y_{\bar s_1, i_1}, \ldots, y_{\bar s_H, i_h}), a_{1:H}^\bi)\nonumber,
  \end{align}
  where $a_g^\bi = \pi_g(a_{1:g-1}^\bi, (y_{\bar s_1, i_1}, \ldots, y_{\bar s_g, i_g}))$ (so that the dependence of $a_g^\bi$ on $\bar s_1, \ldots, \bar s_g$ and $\{ y_{s,i}\}_{s,i}$ is suppressed). 
  Note that changing a single value in $(y_{s,i})_{s\in\MS,i\in[N]}$ only affects the value of at most $N^{H-1}$ terms in the outer summation over $\bi = i_{1:H} \in [N]^H$ (in particular, changing $y_{s,i}$ can only affect the terms indexed by $\bi \in [N]^H$ for which $\bi_k = i$, where $k \in [H]$ is the unique step of the MDP at which state $s$ is reachable). 
  Moreover, for each $\bi$, we have  $\sum_{\bar s_{-h} \in \MS^{H-1}} \prod_{g=1}^H \til \BP(\bar s_g | \bar s_{g-1}, a_g^\bi) \cdot L((y_{\bar s_1, i_1}, \ldots, y_{\bar s_H, i_H}), a_{1:H}^\bi) \in [0,1]$.  Thus, $F$ satisfies the bounded differences property \cref{eq:bd} with $c_{s,i} = 1/N$ for all $(s,i) \in \MS\times [N]$. 
  Thus, by McDiarmid's Inequality (\cref{thm:mcdiarmid}), under some event $\ME_{\bar s_h}$ that occurs with probability at least $1-\delta$ over $(\til x, \phi^\st)$ sampled as above, we have that
  \begin{align}
\left| F((\til x_{s,i})_{s\in\MS,i \in [N]}) - \E_{\UPhi}[F((\til x_{s,i})_{s\in\MS,i \in [N]})] \right| \leq \sqrt{\frac{\log 2/\delta}{N}}\label{eq:L-mcdiarmid}.
  \end{align}
  Note that averaging \cref{eq:F-exp} over all $\bi = i_{1:H} \in [N]^H$ gives that
  \begin{align}
\E_{\UPhi}[F((\til x_{i,s})_{i \in [N], s \in \MS})] = Q(\bar s_h) \cdot \E_{(a_{1:H}, x_{1:H}) \sim \sigma_\pi}[L(x_{1:H}, a_{1:H})]\label{eq:FQ-exp}.
  \end{align}
  By \cref{eq:pphi-expand,eq:FQ-exp,eq:L-mcdiarmid}, under $\ME_{\bar s_h} \cap \Edistinct$, we have 
  \begin{align}
\left| \BP^{\phi^\st, \pi}[s_h = \bar s_h] \cdot \E^{\phi^\st, \pi}[L(x_{1:H}, a_{1:H}) \mid s_h = \bar s_h] - Q(\bar s_h) \cdot \E_{\sigma_\pi}[L(x_{1:H}, a_{1:H})]\right| \leq \sqrt{\frac{\log 2/\delta}{N}}\label{eq:pql-ineq}.
  \end{align}
  Recall that by definition $\mu_{L,\pi} = \E_{\sigma_\pi}[L(x_{1:H},a_{1:H})]$ (\cref{def:random-trajectory}). It then follows that, under the event $\Edistinct \cap \bigcap_{\bar s_h \in \{h\} \times \{0,1,2\}} (\ME_{\bar s_h} \cap \MF_{\bar s_h})$ (where $\MF_{\bar s_h}$ is the event defined in \cref{lem:pq-approx}, which occurs with probability at least $1-\delta$), 
  \begin{align}
     \E^{\phi^\st, \pi}\left[ |G(s_h) - \mu_{L,\pi}|\right] =&  \sum_{\bar s_h \in \{h\} \times \{0,1,2\}} \BP^{\phi^\st, \pi}[s_h= \bar s_h] \cdot |G(s_h) - \mu_{L,\pi}|\nonumber\\
    \leq & \sum_{\bar s_h \in \{ h \} \times \{0,1,2\}} \left| \BP^{\phi^\st, \pi}[s_h = \bar s_h] \cdot  G(s_h) - Q(\bar s_h) \cdot \mu_{L,\pi} \right| + \left|Q(\bar s_h) - \BP^{\phi^\st, \pi}[s_h = \bar s_h] \right| \cdot |\mu_{L,\pi}|\nonumber\\
    \leq & \sum_{\bar s_h \in \{h \} \times \{0,1,2\}} 2 \sqrt{\frac{\log 2/\delta}{N}} = 6 \sqrt{\frac{\log 2/\delta}{N}}\nonumber,
  \end{align}
  where the second inequality uses \cref{eq:pql-ineq,lem:pq-approx} (along with the fact that $|\mu_{L,\pi}| \leq 1$). Finally, set $\ME'_{\pi,L,h} := \bigcap_{\bar s_h \in \{ h \} \times \{0,1,2\}} (\ME_{\bar s_h} \cap \MF_{\bar s_h})$ and observe that $\Pr[\ME'_{\pi,L,h}] \geq 1-6\delta$ by the union bound. Rescaling $\delta$ appropriately then completes the proof. 
\end{proof}

\begin{lemma}
  \label{lem:acc-predictor}
  Fix $\phi \in \Phiset$, $h \in [H]$, a general policy $\pi$, a function $L : (\MX \times \MA \times [0,1])^H \to [0,1]$, and some $\hat\mu \in [0,1]$. Define $G(s) := \E^{\phi, \pi}[L(x_{1:H}, a_{1:H}, r_{1:H}) \mid s_h = s]$. Then
  \begin{align}
\E^{\phi, \pi}\left[ (\hat\mu - L(x_{1:H}, a_{1:H}, r_{1:H}))^2\right] - \E^{\phi, \pi} \left[(G(s_h) - L(x_{1:H}, a_{1:H}, r_{1:H}))^2\right] &\leq  3\E^{\phi, \pi}[|G(s_h) - \hat \mu|]\nonumber.
  \end{align}
  In particular, the constant function $x \mapsto \mu$ is a $(d_h^{\phi, \pi}, G, 3 \E^{\phi, \pi}[|G(s_h) - \hat\mu|])$-accurate predictor with respect to $\MD_{\phi}$ (\cref{def:reg-predictor}). 
\end{lemma}
\begin{proof}
   Let us write $\mu_{L,\pi}^{\phi} := \E^{\phi, \pi}[L(x_{1:H}, a_{1:H}, r_{1:H})] = \EE^{\phi,\pi}[G(s_h)]$. Then 
  \begin{align}
    & \E^{\phi, \pi}\left[ (\mu - L(x_{1:H}, a_{1:H}, r_{1:H}))^2\right] - \E^{\phi, \pi} \left[(G(s_h) - L(x_{1:H}, a_{1:H}, r_{1:H}))^2\right]\nonumber\\
    &= \E^{\phi^\st, \pi}\left[   G(s_h)^2-(\mu_{L,\pi}^{\phi})^2 \right] + (\hat\mu - \mu_{L,\pi}^{\phi})^2\nonumber\\
    &= \E^{\phi,\pi}[(G(s_h) - \mu_{L,\pi}^{\phi})^2] + (\hat\mu - \mu_{L,\pi}^{\phi})^2 \nonumber\\
    &\leq \E^{\phi, \pi}[|G(s_h) - \mu_{L,\pi}^{\phi}|] + |\hat\mu - \mu_{L,\pi}^{\phi}|\nonumber\\
    &\leq \E^{\phi, \pi}[|G(s_h) - \hat\mu|] + 2 |\hat\mu - \mu_{L,\pi}^{\phi}|\nonumber\\
    &\leq 3\E^{\phi, \pi}[|G(s_h) - \hat\mu|]\nonumber, 
  \end{align}
  where the first equality uses the fact that $\E^{\phi, \pi}[G(s_h) \cdot L(x_{1:H}, a_{1:H}, r_{1:H})] = \E^{\phi, \pi}[G(s_h)^2]$ (since $\E[L(x_{1:H}, a_{1:H}, r_{1:H}) \mid s_h] = G(s_h)$), the first inequality uses the fact that all quantities are in the interval $[0,1]$, the second inequality invokes the triangle inequality, and the final inequality uses Jensen's inequality and the fact that $\mu_{L,\pi}^{\phi} = \E^{\phi, \pi}[G(s_h)]$. By \cref{def:reg-predictor}, it follows that $\mu_{L,\pi}$ is a $(d_h^{\phi, \pi}, G, 3 \E^{\phi, \pi}[|G(s_h) - \mu|])$-predictor with respect to $\MD_\phi$. 
\end{proof}

Below we state a few well-known results that were used in the above proofs. 
\begin{theorem}[McDiarmid's Inequality]
  \label{thm:mcdiarmid}
  Let $\MX$ be a set and $X_1, \ldots, X_n \in \MX$ be $\MX$-valued independent random variables. Suppose that $F : \MX^n \to \BR$ satisfies
  \begin{align}
\sup_{x_1, \ldots, x_n, x_i' \in \MX} | F(x_i', x_{-i}) - F(x_1, \ldots, x_n)| \leq c_i\label{eq:bd}  \end{align}
  for each $i \in [n]$. Then for every $t > 0$,
  \begin{align}
\Pr \left[ \left| F(X_1, \ldots, X_n) - \E[F(X_1, \ldots, X_n)]\right| \geq t \right] \leq {2}\exp\left( \frac{-2t^2}{\sum_{i=1}^n c_i^2} \right)\nonumber.
  \end{align}
\end{theorem}

\begin{lemma}[Chain rule for total variation distance]
  \label{lem:tvd-chain}
  Let $\MY$ be a set and let $\BP, \BQ$ be distributions over random variables $Y_{1:n} = (Y_1, \ldots, Y_n) \in \MY^n$, for some $n \in \BN$. Then
  \begin{align}
\tvd{\BP}{\BQ} \leq \sum_{i=0}^{n-1} \E_{y_{1:i} \sim \BP} \left[ \tvd{\BP(Y_{i+1} = \cdot  \mid Y_{1:i} = y_{1:i})}{\BQ(Y_{i+1} = \cdot  \mid Y_{1:i} = y_{1:i})}\right]\nonumber.
  \end{align}
\end{lemma}

\subsubsection{Proof of \cref{thm:oracle-lb}}
We now proceed to the proof of \cref{thm:oracle-lb}. 
  \begin{proof}[Proof of \cref{thm:oracle-lb}]
We consider the block MDP family $\MM = \Moracle$ (\cref{def:oracle-block-mdp-family}). Note that this family certainly satisfies $\max\{ |\MS_n|, |\MA_n|, \ell_n\} = O(\max\{H_n, n\}) \leq O(H_n)$. 
Fix a $(\Tred,\epred)$-reduction $\Alg$ from RL to regression for $\Moracle$ (\cref{def:rl-reg-reduction-2}). 
    
Fix $n \in \BN$. We will write $T := 1+\Tred(n)$, $\ep := \epred(n)$, $N := N_n$, $H := H_n$, $X := X_n = 2^{\ell_n}$, $\MS := \MS_n$, $\MA := \MA_n$, $\ell := \ell_n$, $\Phi := \Phi_n = \Phi_{X,H}$, and $\br := \br_n$. Set $\delta := 1/(16T)$. 

Suppose for the purpose of contradiction that \cref{eq:oracle-lb-conclusion} does not hold. Then since $N = 2^H$, we have $T < \min\{2^H, N^{1/8}\}/C_{\ref{thm:oracle-lb}}$ and $\ep \geq C_{\ref{thm:oracle-lb}} \cdot \max\{ 2^{-H}, N^{-1/8} \}$, where the constant $C_{\ref{thm:oracle-lb}}$ is sufficiently large (as specified below).  
  Choose $\til x \in \MX^{\MS \times [N]}$ by letting $\til x_{s, i} \sim \Unif(\MX)$ for each $(s,i) \in \MS \times [N]$ independently.  Let $\MO_{\til x}$ be the oracle which, upon query $\MO_{\til x}(s, i)$, returns the value $\til x_{s,i}$. 
 
  We consider the execution of $\SimulateReduction(\Alg, X, H,  \MO_{\til x}, \ep, \delta)$ (\cref{alg:simulate-oracle}), which simulates the algorithm $\Alg$ together with a particular implementation of the sampling and regression oracles. 
  In particular, given $\til x$ sampled as above, the algorithm sets $\phi^\st := \Dec_{\til x}$, so that $\phi^\st \sim \UPhi$. In turn, the choice of $\phi^\st \sim \UPhi$ induces an MDP $M_{X,H}^{\phi^\st}$. Note that the latent transitions $\til \BP_h$ and reward function $\br$ of the MDP $M_{X,H}^{\phi^\st}$ do not depend on the choice of $\phi^\st$. 
  \cref{alg:simulate-oracle} also maintains a mapping $\MT: \MS \times [N] \to \MX \cup \{ \perp\}$, which we use purely for analysis purposes. 
To simulate the execution of $\Alg$, \cref{alg:simulate-oracle} proceeds as follows:
  \begin{itemize}
  \item  Each sampling oracle call, which receives as input a circuit $\MB_\pi$ (computing a general policy $\pi$) and the oracle $\MO_{\til x}$, is implemented by the subprocedure $\SimulateSampling(\MB_\pi, \MO_{\til x})$. This procedure generates a sequence $s_{1:H}, x_{1:H}, a_{1:H}, r_{1:H}$ of latent states, contexts, actions, and rewards as follows: at step $h$, it samples a random $i_h \sim \Unif([N])$ to index the context $x_h := \til x_{s_h,i_h}$ (which is accessed by calling $\MO_{\til x}(s_h,i_h)$), chooses $a_h$ according to $\pi$, and observes $r_h, s_{h+1}$ according to the transition and reward function of the MDP $M_{X,H}^{\phi^\st}$. Note that, under the event that $\phi^\st\in \Edistinct$ (and thus all elements of $\til x$ are distinct), the resulting distribution of $s_{1:H}, x_{1:H}, a_{1:H}, r_{1:H}$ is exactly the distribution $\BP^{{\phi^\st}, \pi}$ of a trajectory drawn from $M^{\phi^\st}_{X,H}$.
  \item Each regression oracle call, which is implemented by the subprocedure \SimulateRegression, receives as input a step $h \in [H]$, oracle $\MO_{\til x}$, parameters $\epsilon,\delta>0$, and circuits $\MB_\pi, \MB_L$, representing a general policy $\pi$ and a labeling function $L : (\MX \times \MA \times [0,1])^H \to [0,1]$, respectively. For $m = C_0 \log(1/\delta)/\ep^2$ for a sufficiently large constant $C_0$ (which will be specified below), \SimulateRegression draws $m$ i.i.d.~trajectories from the distribution $\BP^{\phi^\st, \pi}$ (by invoking \SimulateSampling), and then returns the mean of the labeling function $L$ on these trajectories.
  \item At the end of its execution, \cref{alg:simulate-oracle} returns the output policy $\hat \pi$ of the algorithm $\Alg$. 
  \end{itemize}
We will show that  with high probability over the choice of $\phi^\st \sim \UPhi$ and the execution of $\Alg$: (a) each of the calls to $\SimulateSampling$ produces trajectory samples drawn from $M_{X, H}^{\phi^\st}$ (thus satisfying \cref{def:sampling-oracle}), (b) each of the calls to $\SimulateRegression$ produces a prediction function which is accurate (as per \cref{def:reg-predictor}), and (c) $\hat \pi$ has large suboptimality in $M_{X,H}^{\phi^\st}$. These statements will thus allow us to derive a contradiction to the assumption that $\Alg$ is a $(T, \ep)$-reduction from RL to regression. 

  \paragraph{Simulating the sampling oracle.} As we have remarked above, if $\phi^\st \in \Edistinct$, then for any state $s_h$, the set $(\phi^\st)^{-1}(s_h)$ is precisely the set $\{\til x_{s_h,i}: i \in [N]\}$, which moreover has size exactly $N$. Thus, $\til x_{s_h,i_h}$ (for $i_h \sim \Unif([N])$) is distributed according to $D_{\phi^\st}(\cdot|s_h)$. This means that each call to $\SimulateSampling(\MB_\pi, \MO_{\til x})$ in \cref{alg:simulate-oracle} generates a trajectory drawn exactly from the distribution $\BP^{{\phi^\st}, \pi}$.

  \paragraph{Simulating the regression oracle.} By \cref{def:rl-reg-reduction-2}, we know that $T - 1 = \Treg(n)$ is an upper bound on the number of regression oracle calls made by $\Alg$. For $t \in [T-1]$, let the input to the $t$th 
  regression oracle call of $\Alg$ be denoted by $(\MB_{\pi_t}, \MB_{L_t}, h_t)$, for a general policy $\pi_t$ and a labeling function $L_t : (\MX \times \MA \times [0,1])^H \to [0,1]$, where we are identifying $\Delta(\{0,1\})$ with $[0,1]$ in the natural way. 
  Finally, to simplify notation, let $\pi_T$ denote the output policy of $\Alg$, and choose $L_T$ arbitrarily. 
  Note that $\pi_t, L_t$ (for $t \in [T]$) are all random variables. 

For each $t \in [T]$, define $L_t': (\MX\times\MA)^H \to [0,1]$ by $L_t'(x_{1:H}, a_{1:H}) := L_t(x_{1:H}, a_{1:H}, \bzero)$, where $\bzero = (0, \ldots, 0) \in [0,1]^H$. 
For each $t \in [T]$, 
define $G_t' : \MS \to [0,1]$ by  $G_t'(s) := \E^{\phi^\st, \pi_t}[L_t'(x_{1:H}, a_{1:H}) \mid s_{h_t} =s]$ and $G_t : \MS \to [0,1]$ by $G_t(s) := \E^{\phi^\st, \pi_t}[L_t(x_{1:H}, a_{1:H}, r_{1:H}) | s_{h_t} = s]$. Let $\MT_t$ denote the value of the mapping $\MT$ maintained by \cref{alg:simulate-oracle} directly before the $t$th regression oracle call. 
The following lemma gives that $\phi^\st$ and $(\pi_t, L_t, h_t)$ are independent conditioned on $\MT_t$. 

\begin{lemma}
  \label{lem:phi-oracle-ci}
Suppose $\phi^\st \sim \UPhi$ and the algorithm $\Alg$ is executed as in \cref{alg:simulate-oracle}. Then for any $t \in [T]$, $\phi^\st$ and $(\pi_t, L_t, h_t)$ are independent conditioned on $\MT_t$.
\end{lemma}

The proof of \cref{lem:phi-oracle-ci} is given following the proof of the theorem. 
Fix any $t \in [T]$, and condition on the values of $(\pi_t, L_t, h_t)$ and of $\MT_t$; thus the value of $L_t'$ is determined as well. 
\cref{lem:constant-fn-predictor} gives that, for some subset $\ME'_{\pi_t, L'_t, h_t} \subset \Phiset$ that depends only on $\pi_t, L_t, h_t$ and satisfies $\UPhi(\ME_{\pi_t, L'_t, h_t}') \geq 1-\delta$, 
each $\phi' \in \ME_{\pi_t, L'_t, h_t}' \cap \Edistinct$ satisfies $\E^{\phi', \pi_t}[|G_t'(s_{h_t}) - \mu_{L_t', \pi_t}|] \leq C_{\ref{lem:constant-fn-predictor}} \sqrt{\log(1/\delta)/N}$. Next, \cref{lem:coupling-uphi} gives that there is a coupling $\SV_{\MT_t}$ of $\UPhi, \UPhi|_{\MT_t}$ so that, for $(\phi',\phi^\st) \sim \SV_{\MT_t}$ (where $\phi' \sim \UPhi$ and $\phi^\st \sim \UPhi|_{\MT_t}$), when $\phi^\st\in \Edistinct$, we have
\begin{align}
\left| \E^{\phi', \pi_t}[|G_t'(s_{h_t}) - \mu_{L_t', \pi_t}|] - \E^{\phi^\st, \pi_t}[|G_t'(s_{h_t}) - \mu_{L_t', \pi_t}|]\right| &\leq   \tvd{\BP^{\phi', \pi_t}}{\BP^{\phi^\st, \pi_t}} \leq \frac{4H |\dom(\MT_t)|}{N} \leq \frac{16mH^2 T}{N}.\label{eq:phiprime-phistar-diff}
\end{align}
Above, the final inequality uses that $|\dom(\MT_t)| \leq H^2 T m$, since the total number of calls to \SimulateSampling over the course of \cref{alg:simulate-oracle} is at most $Tm$ (as each call to $\SimulateRegression$ calls $\SimulateSampling$ $m$ times), and each call to \SimulateSampling increases $|\dom(\MT)|$ by at most $H$. 

Let $\ME_{\MT_t, \pi_t, L_t, h_t} \subset \Phiset$ be defined by
\begin{align}
  \ME_{\MT_t, \pi_t, L_t, h_t} := \{ \phi^\st \in \Phiset \ : \ \exists \phi' \in \ME_{\pi_t, L'_t, h_t}' \mbox{ s.t. } \SV_{\MT_t}(\phi', \phi^\st) > 0 \}.\nonumber
\end{align}
Since $\SV_{\MT_t}$ is a coupling of $\UPhi, \UPhi|_{\MT_t}$ and $\UPhi(\ME'_{\pi_t, L'_t, h_t}) \geq 1-\delta$, we have $\UPhi|_{\MT_t}(\ME_{\MT_t, \pi_t, L_t, h_t}) \geq 1-\delta$. Moreover,  for all $\phi^\st \in \Edistinct \cap \ME_{\MT_t, \pi_t, L_t, h_t}$, we have, using \cref{eq:phiprime-phistar-diff}, 
\begin{align}
  \E^{\phi^\st, \pi_t}[|G'(s_{h_t}) - \mu_{L_t', \pi_t}|] &\leq   \frac{16mH^2 T}{N} + C_{\ref{lem:constant-fn-predictor}} \sqrt{\log(1/\delta)/N}\label{eq:edistinct-et}.
\end{align}

Since $\UPhi|_{\MT_t}(\ME_{\MT_t, \pi_t, L_t, h_t}) \geq 1-\delta$, by \cref{lem:phi-oracle-ci},  with probability at least $1-\delta$ under the distribution of $\phi^\st \sim \UPhi$ conditioned on $\MT_t, \pi_t, L_t, h_t$, we have $\phi^\st \in \ME_{\MT_t, \pi_t, L_t, h_t}$. 
Let $\MF_t$ denote the event that $\phi^\st \in \ME_{\MT_t, \pi_t, L_t, h_t}$. Since $\MT_t, \pi_t, L_t, h_t$ are generated according to the draw of $\phi^\st \sim \UPhi$ and the execution of $\Alg$, it follows that $\MF_t$ occurs with probability at least $1-\delta$ under $\phi^\st \sim \UPhi$ and the execution of $\Alg$. 
Moreover, \cref{eq:edistinct-et} holds under the event $\{ \phi^\st \in \Edistinct\} \cap \MF_t$. 
Below, we will use \cref{eq:edistinct-et} to show that  with high probability, all calls to \SimulateRegression return an accurate predictor (per \cref{def:reg-predictor}). 

\paragraph{Bounding the probability of a nonzero reward.} Using a similar argument as above, we next show that each of the policies $\pi_t$ (including the output policy $\pi_T$) rarely receives nonzero reward for the MDP $M_{X,H}^{\phi^\st}$. To do so, let us again fix a particular choice of $\MT_t, \pi_t, L_t, h_t$. By \cref{lem:pq-approx}, for some event $\MF_{\pi_t}$ occurring
  with probability at least $1-\delta$ over an independent draw of $\phi' \sim \UPhi$, we have that under $\MF_{\pi_t} \cap \Edistinct$, 
  \begin{align}
\E^{\phi', \pi_t}[|L_t(x_{1:H}, a_{1:H}, r_{1:H}) - L_t'(x_{1:H}, a_{1:H})|] \leq \BP^{\phi', \pi_t}\left[ r_H \neq 0 \right] \leq 2^{1-H} + \sqrt{\frac{\log 1/\delta}{N}}\label{eq:rews-nonzero},
  \end{align}
  where the final inequality uses \cref{lem:qs-formula} to give $Q_H((H,0)) = Q_H((H,1)) = 2^{-H}$; recall that the reward is only nonzero at states $(H,0)$ and $(H,1)$. 
Using \cref{eq:rews-nonzero},  an application of \cref{lem:coupling-uphi} yields, via the same argument as above, that there is an event $\MF_t'$ that occurs with probability at least $1-\delta$ under $\phi^\st \sim \UPhi$ and the execution of $\Alg$ so that, under $\{\phi^\st \in \Edistinct\}\cap \MF_t'$, we have
  \begin{align}
\E^{\phi^\st, \pi_t}[|L_t(x_{1:H}, a_{1:H}, r_{1:H}) - L_t'(x_{1:H}, a_{1:H})|] \leq \BP^{\phi^\st, \pi_t}[r_H \neq 0] \leq \frac{16m H^2 T }{N} + 2^{1-H} + \sqrt{\frac{\log 1/\delta}{N}}\label{eq:phist-rews-nonzero},
  \end{align}
  from which it follows that 
  \begin{align}
    \E^{\phi^\st, \pi_t}[|G_t(s_{h_t}) - G_t'(s_{h_t})|] &\leq  \E^{\phi^\st, \pi_t}\left[ \E^{\phi^\st, \pi_t} [ |L_t(x_{1:H}, a_{1:H}, r_{1:H}) - L_t'(x_{1:H}, a_{1:H})|\mid s_{h_t}] \right] \nonumber\\
    &\leq \frac{16m H^2 T}{N} +   2^{1-H} + \sqrt{\frac{\log 1/\delta}{N}}\label{eq:gt-ce}.
  \end{align}
\paragraph{Wrapping up.}   Combining \cref{eq:edistinct-et,eq:gt-ce}, we see that under the event $\{ \phi^\st \in \Edistinct\} \cap \bigcap_{t \leq T} (\MF_t \cap \MF_t')$, for all $t \in [T]$,
    \begin{align}
    \label{eq:gt-mu}
    \E^{\phi^\st, \pi_t}[|G_t(s_{h_t}) - \mu_{L_t', \pi_t}|] \leq \frac{32mH^2 T}{N} + 2^{1-H} + (C_{\ref{lem:constant-fn-predictor}} + 1) \sqrt{\frac{\log 1/\delta}{N}}.
    \end{align}
    For each $t \in [T]$, let us denote the $m$ i.i.d.~trajectories sampled from $\BP^{\phi^\st, \pi}$ on \cref{line:sample-trajs} of \SimulateRegression by $(x_{1:H}^{t,i}, a_{1:H}^{t,i}, r_{1:H}^{t,i})_{i\in[m]}$, and let $\hat \mu_t := \frac 1m \sum_{i=1}^m L(x_{1:H}^{t,i}, a_{1:H}^{t,i}, r_{1:H}^{t,i})$. Since $\E^{\phi^\st, \pi_t}[G_t(s_{h_t})] = \E^{\phi^\st, \pi_t}[L(x_{1:H}, a_{1:H}, r_{1:H})]$ for each $t$,
    a Chernoff bound gives that, with probability at least $1-\delta$, we have that
    \begin{align}
\left|\hat \mu_t - \E^{\phi^\st, \pi_t}[G_t(s_{h_t})] \right|  =      \left| \hat \mu_t -\E^{\phi^\st, \pi_t}[L_t(x_{1:H}, a_{1:H}, r_{1:H})] \right| \leq \frac{1}{10} \sqrt{\frac{C_0\log 1/\delta}{m}} = \ep/10\label{eq:yi-average},
    \end{align}
    as long as the constant $C_0$ is chosen sufficiently large (recall the choice of $m$ on \cref{line:oracle-w}). Averaging over the choice of $(\pi_t, L_t, h_t)$ and taking a union bound over all $t \in [T]$ gives that there is some event $\Fchernoff$, occurring with probability at least $1-\delta T$ under the draw of $\phi^\st \sim \UPhi$ and the execution of $\Alg$, under which \cref{eq:yi-average} holds for each $t \in [T]$. Using the facts above, we may now establish the following lemma: 
  \begin{lemma}
    \label{clm:all-good-predictors}
There is some event $\ME_{\ref{clm:all-good-predictors}}$ that occurs with probability at least $3/4$ under the draw of $\phi^\st \sim \UPhi$ and the execution of $\Alg$ in \cref{alg:simulate-oracle} so that, under $\ME_{\ref{clm:all-good-predictors}}$, for each $t \in [T]$, 
    \begin{align}
\E^{\phi^\st, \pi_t} \left[\left|G_t(s_{h_t}) - \hat \mu_t\right|\right] \leq \frac{\ep}{3} \qquad \mbox{ and } \qquad V_1^{\phi^\st, \pi_T} < 1/4\nonumber.
    \end{align}
    (Note that the above expectation is over $s_{h_t}$ drawn from $M_{X,H}^{\phi^\st}$ for the policy $\pi_t$.) 
  \end{lemma}
  \begin{proof}[Proof of \cref{clm:all-good-predictors}]
        We set $\ME_{\ref{clm:all-good-predictors}}:=\Fchernoff \cap \{ \phi^\st \in \Edistinct\} \cap \bigcap_{t \in [T]} (\MF_t \cap \MF_t')$. 
Note that $\ME_{\ref{clm:all-good-predictors}}$ occurs with probability at least $1-3\delta T -1/N\geq {3/4}$ (note that we have used here \cref{lem:edistinct-prob} to bound $\UPhi(\Edistinct) \geq 1-1/N \geq 15/16$, as long as $n$ is sufficiently large).   
Under the event $\ME_{\ref{clm:all-good-predictors}}$, \cref{eq:gt-mu} holds, and thus, by Jensen's inequality, 
\begin{align}
|\mu_{L_t',\pi_t} - \E^{\phi^\st, \pi_t}[G_t(s_{h_t})]| \leq \frac{32mH^2 T}{N} + 2^{1-H} + (C_{\ref{lem:constant-fn-predictor}} + 1) \sqrt{\frac{\log 1/\delta}{N}}.\label{eq:means-diff}
\end{align}
Thus, under $\ME_{\ref{clm:all-good-predictors}}$, by \cref{eq:gt-mu,eq:yi-average,eq:means-diff}, we have
\begin{align}
\E^{\phi^\st, \pi_t} \left[\left|G_t(s_{h_t}) - \hat \mu_t\right|\right] \leq \frac{64m  H^2 T}{N} + 2^{2-H} + (2C_{\ref{lem:constant-fn-predictor}} + 2) \sqrt{\frac{\log 1/\delta}{N}} + \frac{\ep}{10}\nonumber.
\end{align}
As long as the constant $C_{\ref{thm:oracle-lb}}$ is chosen to be sufficiently large, our assumption that $\ep > C_{\ref{thm:oracle-lb}} \cdot \max\{ 2^{-H}, N^{-1/8} \}$ and $T < C_{\ref{thm:oracle-lb}}^{-1} \cdot \min \{ N^{1/8}, 2^H \}$, as well as that $N = 2^H \geq H^{C_{\ref{thm:oracle-lb}}/3}$ (which holds when $n \geq C_{\ref{thm:oracle-lb}}'$ for $C_{\ref{thm:oracle-lb}}'$ sufficiently large), 
implies that 
\begin{align}
  & \frac{64 m H^2 T}{N} + 2^{2-H} + (2C_{\ref{lem:constant-fn-predictor}} + 2) \sqrt{\frac{\log(1/\delta)}{N}} + \frac{\ep}{10}\nonumber\\
  &=  \frac{64 C_0 \log(16T) H^2 T}{N\ep^2 } + 2^{2-H} + (2C_{\ref{lem:constant-fn-predictor}} + 2) \sqrt{\frac{\log(16T)}{N}} + \frac{\ep}{10}\nonumber\\
  &\leq C \cdot \left( \frac{\log(N) \cdot N^{6/C_{\ref{thm:oracle-lb}}} \cdot C_{\ref{thm:oracle-lb}}^{-1} N^{1/8}}{N^{3/4}} + C_{\ref{thm:oracle-lb}}^{-1} \ep + \frac{\sqrt{\log(N)}}{N^{1/2}}\right) + \frac{\ep}{10} \nonumber\\
  &\leq C' \cdot \left( N^{-1/4} + C_{\ref{thm:oracle-lb}}^{-1} \ep \right) + \frac{\ep}{10} \leq \ep/3 \nonumber,
\end{align}
where $C,C'$ are some constants. 

Moreover, by \cref{eq:phist-rews-nonzero}, on the event $\ME_{\ref{clm:all-good-predictors}}$ we have that $V_1^{\phi^\st, \pi_T} \leq \frac{16mH^2 T}{N} + 2^{1-H} + \sqrt{\frac{\log 1/\delta}{N}} < 1/4$.
\end{proof}
On the event $\ME_{\ref{clm:all-good-predictors}}$, by \cref{clm:all-good-predictors,lem:acc-predictor} (where \cref{lem:acc-predictor} is applied for each $t$ with the policy $\pi_t$, label function $L_t$, and $\mu = \hat \mu_t$), we have that for each $t$, the constant function $\hat \mu_{t}$ returned by $\SimulateRegression(\MB_{\pi_t}, \MB_{L_t}, h_t, \MO_{\til x})$ is an $(d_h^{\phi^\st, \pi_t}, G_t, \ep)$-accurate predictor (\cref{def:reg-predictor}) with respect to $\MD_{\phi^\st}$. 
 Let us now choose some fixed decoding function $\til x \in \MX^{\MS\times[N]}$ and $\phi^\st := \Dec_{\til x} \in \Phi_{X,H}$, so that the event $\ME_{\ref{clm:all-good-predictors}}$ occurs with probability at least $3/4$ over the execution of $\Alg$, conditioned on $(\til x, \phi^\st)$.

Now consider the execution of $\Alg$ with the  following regression oracle $\Oregress'$ (and the same sampling oracle $\Osample$ as above): 
\begin{itemize}
\item For each regression oracle query $(\MB_{\pi_t}, \MB_{L_t}, h_t)$, if \SimulateRegression returns a (constant) function $\hat \mu$ that is $(d_h^{\phi^\st,\pi_t}, G_t, \ep)$-accurate with respect to $\MD_{\phi^\st}$, then $\Oregress'$ returns $\hat \mu$. 
\item Otherwise, $\Oregress'$ returns the mapping $x \mapsto G_t(\phi^\st(x))$. 
\end{itemize}
Certainly $\Oregress'$ is an $\ep$-accurate regression oracle for $M^{\phi^\st}_{X,H}$ per \cref{def:regression-oracle}. Moreover, under the event $\ME_{\ref{clm:all-good-predictors}}$ (which occurs with probability at least $3/4$) the transcript of $\Alg$ interacting with $\Oregress'$ is identical to the transcript of $\Alg$ when regression oracles are all implemented by \SimulateRegression as in \cref{alg:simulate-oracle} (instead of $\Oregress'$). Since under $\ME_{\ref{clm:all-good-predictors}}$, the output policy $\pi_T$ is not $1/2$-optimal (in particular, the optimal policy in $M^{\phi^\st}_{X,H}$ has value $1$), 
we have a contradiction to the fact that $\Alg$ is a $(T,\ep)$-reduction per \cref{def:rl-reg-reduction-2}. 
\end{proof}

\begin{proof}[Proof of \cref{lem:phi-oracle-ci}]
  Fix some $t \in [T]$ and some particular value of $\MT_t$. 
  The distribution of $\phi^\st$ conditioned on $\MT_t$ is $\UPhi|_{\MT_t}$, which is the distribution of $\Dec_{\til x}$ when $\til x \in \MX^{\MS \times [N]}$ has all its coordinates independent and uniform on $\MX$ conditioned on $\{ \til x_{s,i} = \MT_t(s,i) \ \forall (s,i) \in \dom(\MT_t)\}$.

  Consider any fixed values for $\til x' := ( \til x_{s,i} \ : \ (s,i) \in \MS \times [N] \backslash \dom(\MT_t) )$. Note that since $\SimulateReduction$ never accesses $\til x'$, the probability of any particular transcript of $\Alg$ (up to but not including the $t$-th regression oracle call) that is consistent with the chosen value for $\MT_t$ is the same, regardless of the value of $\til x'$. In particular, $\BP_{\phi^\st \sim \UPhi, \Alg}((\pi_t, L_t, h_t) \mid \MT_t, \til x')$ does not depend on $\til x'$. Since the marginal distribution of $\til x'$ is uniform on $\MX^{\MS \times [N] \backslash \dom(\MT_t)}$, it follows from Bayes' rule that $\BP_{\phi^\st \sim \UPhi, \Alg}(\til x' \mid \MT_t, (\pi_t, L_t, h_t))$ does not depend on $\til x'$. Hence $\BP_{\phi^\st \sim \UPhi, \Alg}(\til x = \cdot \mid \MT_t) = \BP_{\phi^\st \sim \UPhi, \Alg}(\til x = \cdot \mid \MT_t, (\pi_t, L_t, h_t))$, which completes the proof of the lemma.
\end{proof}

\subsection{Ruling out computational reductions from RL to regression}\label{sec:oracle-lb-second}

In the previous section, we established unconditional exponential lower bounds for reducing RL to regression using a family of Block MDPs $\Moracle$ (\cref{def:oracle-block-mdp-family}). For each $n \in \NN$, the set of MDPs $\MM_n$ corresponding to $\Moracle$ is the set $\{ M_{X_n,H_n}^\phi : \phi \in\Phi_{X_n, H_n}\}$ that arises from allowing $\phi$ to be an arbitrary element of $\Phi_n = \Phi_{X_n, H_n} = \MS_n^{\MX_n}$. Note that $\log |\Phi_{X_n,H_n} | = X_n \log |\MS_n| > N_n$, and $N_n$ itself needed to be exponential in $H_n$ in order for us to derive exponential lower bounds for reducing RL to regression in \cref{thm:oracle-lb}. 

Thus, while \cref{thm:oracle-lb} does rule out computational reductions from RL to regression (which are a subclass of all reductions from RL to regression) for $\Moracle$, it does not establish a gap between computational complexity with respect to the regression oracle and sample complexity, which scales with $\log|\Phi_{X_n,H_n}|$ and therefore is also exponential in $H_n$. In other words, one may argue that $\Moracle$ is ``so hard'' in a statistical sense that one should not even hope for a computationally efficient reduction. As another example, to compute the optimal policy of an MDP $M_{X_n, H_n}^\phi$ in the family $\Moracle$ one needs to evaluate $\phi_n : \MX_n \to \MS_n$, which, being an arbitrary function on an exponentially large domain, is not succinctly describable (i.e. does not have a small circuit representation) in general.


To address these objections, we now construct a \emph{computable} block MDP family (\cref{def:computable-family}) where there are exponential lower bounds against any computational reduction from RL to regression. As discussed at the beginning of the section, any computable block MDP family is statistically tractable (even without a regression oracle), so this result does establish a gap between computational complexity with respect to the regression oracle and sample complexity. Since this is an inherently computational result, it requires making a computational hardness assumption; we assume existence of a pseudorandom permutation family with sub-exponential security (\cref{asm:prf}). 

\paragraph{Construction.} To construct a computable block MDP family,
given parameters $N, H \in \BN$, recall the definition of the state space $\MS = ([H] \times \{0,1,2\}) \cup \perp$ and context space $\MX := [X]$ (where $X = 2^{\lceil \log N^5 \rceil}$) from \cref{sec:oracle-lb-first}, as well as the block MDP $M_{X,H}^\phi$, where $\phi : \MX \to \MS$.  
Suppose that $(3H+1)N = |\MS| N \leq X$, so that there is a natural injection $\MS\times[N] \xhookrightarrow{} [X]$ that identifies $\MS\times[N]$ with $[(3H+1)N] \subseteq [X]$. For any collection $\MF \subset [X]^{[X]}$ of functions $J: [X] \to [X]$, we define a subset $\Phisets$ of $\Phiset$ as
\[\Phisets := \{\phi_J: J \in \MF\}\]
where $\phi_J: \MX\to\MS$ is the \emph{$J$-decoder} defined below. We also define $\UPhis \in \Delta(\Phisets)$ to be the distribution $\Unif\{\phi_J: J \in \MF\}$. 

\begin{definition}[$J$-decoder]\label{def:j-decoder}
For each function $J : [X] \to [X]$, we define a function $\phi_J : \MX \to \MS$, as follows. Recall that we have identified the elements of $\MS \times [N]$ with the first $|\MS|N$ elements of $[X]$. For $x \in \MX$, $\phi_J(x)$ is defined as follows:
\begin{itemize}
\item If some $(s,i) \in \MS \times [N] \simeq [|\MS|N] \subset [X]$ satisfies $J((s,i)) = x$, then set $\phi_J(x) := s$, for the lexicographically first such pair $(s,i)$.
\item Otherwise, set $\phi_J(x) = \perp$. 
\end{itemize}
\end{definition}
The following lemma, which states that $\UPhi$ is the distribution of $\phi_J$ for a \emph{uniformly random} function $J : [X] \to [X]$, is immediate from the definition of $\UPhi$ in \cref{sec:oracle-lb-first}.
\begin{lemma}
  \label{eq:unif-phi}
$\UPhi \in \Delta(\Phiset)$ is exactly the distribution $\til{\mathscr{U}}_{X,H,[X]^{[X]}} = \Unif\{ \phi_J : J \in [X]^{[X]}\}$.
\end{lemma}
We define our computable block MDP family by replacing $[X]^{[X]}$ with a set of pseudorandom permutations:
\begin{defn}
  \label{def:computable-block-mdp-family}
Let $(F_\ell)_{\ell \in \BN}$ be a $(t,q)$-pseudorandom permutation family for some tuple $(t,q)$ satisfying \cref{asm:prf}. Let $\MF_\ell$ denote the class of functions $\{ F_\ell(\rho, \cdot) \ : \ \rho \in \{0,1\}^\ell \}$, so that $F_\ell \in \MF_\ell$ maps $\{0,1\}^\ell \to \{0,1\}^\ell$.
  
We define the block MDP family $\tilMoracle$ to be identical to $\Moracle$ (\cref{def:oracle-block-mdp-family}) with the exception that the class $\Phi_n$ is given by $\til{\Phi}_{X_n, H_n, \MF_{\ell_n}}$. 
\end{defn}

\begin{lemma}
  \label{lem:computable-oracle-true}
  The block MDP family $\tilMoracle$ is  polynomially horizon-computable.
\end{lemma}
\begin{proof}
It is immediate from \cref{def:oracle-block-mdp-family} that $\max\{ |\MS_n|, |\MA_n|, \ell_n \} \leq O(H_n)$. Furthermore, $| \log \Phi_n| = \ell_n \leq O(H_n)$. Finally, the mapping $\Phi_n \times \MX_n \ni (\phi, x) \mapsto \phi(x)$ is efficiently computable by a circuit of size $\poly(\ell_n) \leq \poly(H_n)$, as follows. Each $\phi \in \Phi_n$ may be written as $x \mapsto \phi_{J_\rho}(x)$ for some $J_\rho := F_{\ell_n}(\rho, \cdot) \in \til{\Phi}_{X_n,H_n,\MF_{\ell_n}}$. Since $J_\rho$ is invertible (\cref{it:prp-invertible} of \cref{def:prp}), we have from \cref{def:j-decoder} that \[\phi_{J_\rho}(x) = \begin{cases} J_\rho^{-1}(x) & \text{ if } J_\rho^{-1}(x) \in [|\MS_n|N_n] \\ \perp & \text{ otherwise} \end{cases}.\] Moreover, the map $(\rho,x) \mapsto J_\rho^{-1}(x)$ has a polynomial-time algorithm (\cref{it:prp-eff} of \cref{def:prp}) and hence polynomial-sized circuits, so $(\phi_{J_\rho}, x) \mapsto \phi_{J_\rho}(x)$ does as well (where we are indexing $\Phi_n = \{\phi_{J_\rho}: \rho \in \{0,1\}^{\ell_n}\}$ by $\{0,1\}^{\ell_n}$). 
\end{proof}

\begin{lemma}\label{lemma:sampling-correctness}
Let $J: [X] \to [X]$ be any function so that the restriction to $\MS\times[N] \simeq [|\MS|N] \subset [X]$ is injective. Let $\MO_J$ be an oracle implementing query access to $J$. Then for any circuit $\MB_\pi$ describing a policy $\pi$, the output of $\SimulateSampling(\MB_\pi,\MO)$ (\cref{alg:simulate-oracle}) is distributed according to $\BP^{\phi_J, \pi}$, where $\phi_J$ is the $J$-decoder (\cref{def:j-decoder}).
\end{lemma}

Our main result of this section is stated below; it asserts that some computable block MDP family (which will be taken to be $\tilMoracle$) rules out computational reductions from RL to regression (per \cref{def:rl-reg-reduction-2}). 
\begin{theorem}
  \label{thm:reduction-prp}
  Suppose that \cref{asm:prf} holds. Then there is a polynomially horizon-computable block MDP family $\MM$ (\cref{def:computable-family}) and constants $c_0, \bar C, \bar c > 0$ so that the following holds. 
 For any complexity measures $\Tred : \BN \to \BN, \epred : \BN \to (0,1)$, defining $B : \BN \to \BN$ by $B(n) :=  H_n^{c_0}$,\footnote{Recall that $H_n$ denotes the horizon length of MDPs at index $n$ in the family $\MM$.}, if there is a computational $(\Tred, \epred, B)$-reduction $\Alg$ from RL to regression, 
 then the following holds for $n \geq \bar C$: 
\begin{align}
  \epred(n) \leq 2^{-H_n^{\bar c}} \qquad \mbox{ or } \qquad \Tred(n) \geq 2^{H_n^{\bar c}}\label{eq:computational-reduction-epT}.
    \end{align}
  \end{theorem}
  \begin{proof}[Proof of \cref{thm:reduction-prp}]
    We consider the block MDP family $\MM := \tilMoracle$, which is polynomially horizon-computable by \cref{lem:computable-oracle-true}. Let $(F_\ell)_{\ell \in \BN}$ be the $(t,q)$-pseudorandom permutation family used to define $\tilMoracle$ (so that $(t,q)$ satisfy \cref{asm:prf}), and let $\MF_\ell$ denote the corresponding class of functions $\{ F_\ell(\rho, \cdot ) \ : \ \rho \in \{0,1\}^\ell \}$ mapping $\{0,1\}^\ell \to \{0,1\}^\ell$. Fix a computational $(\Tred, \epred, B)$-reduction $\Alg$ from RL to regression for $\tilMoracle$ (\cref{def:rl-reg-reduction-2}). Suppose for the purpose of contradiction that \cref{eq:computational-reduction-epT} does not hold. 
    
    By \cref{asm:prf} and the fact that $\ell_n \geq H_n$, we have $t(\ell_n) \geq 2^{\ell_n^c} \geq 2^{H_n^c}$ for a sufficiently small constant $c$ (which depends on the choice of constant in \cref{asm:prf}). Also, $N_n = 2^{H_n}$. It follows that, as long as the constant $\bar c>0$ in the theorem statement is sufficiently small, $\epred(n) > \max\{ 2^{-H_n}, N_n^{-1/8}, t(\ell_n)^{-1/8} \}$ and $\Tred(n) < \min \{ 2^{H_n}, N_n^{1/8}, t(\ell_n)^{1/8} \}$. 

    Fix $n \in \BN$. We will write $T := \Tred(n)+1$, $\ep := \epred(n)/3$, $N := N_n$, $H:= H_n$, $X := X_n = 2^{\ell_n}$, $\MS := \MS_n$, $\MA := \MA_n$, $\ell := \ell_n$, $\MF := \MF_\ell$, $\Phi := \Phi_n = \Phisets$, and $\br := \br_n$.  Moreover, set $\delta = 1/(72 T)$. 
    
    
\paragraph{Description of the distinguishing algorithm.}    In \cref{alg:distinguish-regression}, we introduce the procedure $\TestReduction(\MO, T, \ep)$, which takes as input an oracle $\MO$ as well as the parameters $T, \ep$. The oracle $\MO$ implements pointwise queries to either a random function from $\MF_\ell$ (which we will refer to as the case $b=0$) or a random function from $[X]^{[X]}$ (which we will refer to as the case $b=1$). $\TestReduction(\MO,T,\ep)$ 
outputs a random bit $\hat b$, with the goal of guessing $b$ with probability substantially better than $1/2$, which would contradict the security of $F$ (per \cref{def:prp}) if $\TestReduction$ is sufficiently computationally efficient.

Let $\phi^\st := \phi_\MO$ be the decoder associated to $\MO$ (\cref{def:j-decoder}). 
\TestReduction simulates the the execution of $\Alg$ on $M_{X,H}^{\phi^\st}$ by calling $\SimulateReduction(\Alg, X, H, \MO, \ep,\delta)$ (\cref{alg:simulate-oracle}). $\Alg$ makes at most $T-1$ regression oracle calls (as its total running time is at most $T$), whose inputs we denote by $(\MB_{\pi_t}, \MB_{L_t}, h_t)_{t \in [T-1]}$, and whose outputs we denote by $(\hat \mu_t)_{t \in [T-1]}$. If the actual number of regression oracle calls $\hat T$ is less than $T-1$, we choose $\pi_t, h_t$ arbitrarily, and let $L_t \equiv 0$ and $\hat \mu_t = 0$, for each $\hat T < t \leq T-1$. Finally, we denote the output policy of $\Alg$ by $\pi_T$, and choose $L_T$ arbitrarily.

    Next, \TestReduction generates estimates of the function $G_t(s) := \E^{\phi^\st, \pi_t}[L_t(x_{1:H}, a_{1:H}, r_{1:H}) \mid s_{h_t} = s]$, for each $t \in [T]$, by generating samples from the distribution of a trajectory $(s_{1:H}, x_{1:H}, a_{1:H}, r_{1:H}) \sim \BP^{\phi^\st, \pi_t}$, for each $t \in [T]$ (\cref{line:call-simsample}). This is done by using the procedure $\SimulateSampling(\pi_t, \MO)$ (from \cref{alg:simulate-oracle}), and the resulting estimate is denoted by $\hat G_t : \MS \to [0,1]$ (\cref{line:def-hatg}). In  a similar manner, on \cref{line:def-hatv}, \TestReduction computes an estimate $\hat V^t$ of $V_1^{\phi^\st, \pi_t}$ for each $t$. Finally, \TestReduction chooses its output bit as follows:
    \begin{itemize}
    \item If an empirical estimate of $\E^{\phi^\st, \pi_t}[|\hat G_t(s_{h_t}) - \hat \mu_t|]$ is large (\cref{line:disc-large}) or $\hat V^t$ is large (\cref{line:value-large}), then $\hat b := 1$ is returned.
      \item Otherwise, $\hat b := 0$ is returned. 
      \end{itemize}
      As we will show below, the first item above (where $1$ is returned) cannot happen with large probability if $b = 1$ by \cref{clm:all-good-predictors}. Thus if the first item happens when $b=0$, we will have a contradiction to the security of the PRP $F$; but, as we shall show, this will imply that $\Alg$ cannot be a valid reduction. 


\paragraph{Analysis of \TestReduction.} Set $\ep' = \ep/(40H)$. Fix any oracle $\MO: [X]\to[X]$ such that the map $(s,i) \mapsto \MO((s,i))$ is injective (recall the identification of $\MS\times[N]$ with $[|\MS|N] \subset [X]$), and consider the execution of $\TestReduction(\MO, T, \ep)$. Recall the definitions of empirical measures $\hat\BP_t, \hat\BP_t'$ and empirical expectations $\hat\EE_t, \hat\EE'_t$ from \cref{alg:distinguish-regression}. By \cref{lemma:sampling-correctness} and the assumption that $\MO$ is injective on $\MS\times[N]$, $\hat\BP_t$ and $\hat\BP_t'$ are (independent) empirical distributions, each over $m$ independent trajectories drawn from $\BP^{\phi^\st,\pi_t}$, where $m = C_1\log(72T)H^2/\epsilon^2$ as defined on \cref{line:choose-m-c1} of \cref{alg:distinguish-regression}. 

For each $t \in [T]$ and $s \in \MS$, let $p_t(s) := \BP^{\phi^\st, \pi_t}[s_{h_t} = s]$ and $\hat p_t(s) := \hat \BP_t[s_{h_t} = s]$, where $\hat \BP_t$ is defined on \cref{line:call-simsample} of \cref{alg:distinguish-regression}. Thus, for each $t \in [T]$, an application of the Chernoff bound and a union bound (over $s \in \MS$) gives that there is an event $\MG_t$ that occurs with probability $1-\delta$, under which the following property holds: for each $s \in \MS$, 
    \begin{align}
 \left| p_t(s) - \hat p_t(s) \right| = \left| \BP^{\phi^\st, \pi_t}[s_{h_t} = s] - \hat \BP_t[s_{h_t} = s] \right|  \leq \frac{C \sqrt{\log 1/\delta}}{\sqrt{m}} \leq \ep'\nonumber,
    \end{align}
   for some constant $C$. The final inequality holds as long as $C_1$ is chosen sufficiently large as a function of $C$ (in \cref{line:choose-m-c1} of \cref{alg:distinguish-regression}). For each $t$, the number of sampled trajectories $(s_{1:H}^i, x_{1:H}^i, a_{1:H}^i, r_{1:H}^i)$ in \cref{line:call-simsample} of \cref{alg:distinguish-regression} with $s_{h_t}^i = s$ is exactly $\hat p_t(s) \cdot m$. Thus, another application of the Chernoff bound and a union bound over $s \in \MS$ gives that there is an event $\MG_t'$ that occurs with probability $1-\delta$, under which the following holds: for each $s \in \MS$, 
    \begin{align}
      | \hat p_t(s) \cdot  \hat G_t(s) -p_t(s) \cdot  G_t(s)| 
      &=\Big| \hat \E_t[L_t(x_{1:H}, a_{1:H}, r_{1:H}) \cdot \One\{ s_{h_t} = s\} ] \nonumber\\ 
      &\qquad- \E^{\phi^\st, \pi_t}[L_t(x_{1:H}, a_{1:H}, r_{1:H}) \cdot \One\{ s_{h_t} = s\} ] \Big|\nonumber\\
      &\leq \frac{C\sqrt{\log 1/\delta}}{ \sqrt{m}} \leq \ep' \nonumber,
    \end{align}
    for some constant $C$. 
    Thus, under $\MG_t \cap \MG_t'$, for each $s \in \MS$, we have
    \begin{align}
      \E^{\phi^\st, \pi_t}[| \hat G_t(s_{h_t}) - G_t(s_{h_t})|] 
      &= \sum_{s \in \MS} p_t(s) \cdot | \hat G_t(s) - G_t(s)|\nonumber\\
      &\leq \sum_{s \in \MS} |\hat p_t(s) \cdot \hat G_t(s) - p_t(s) \cdot G_t(s)| + |p_t(s) - \hat p_t(s)| \cdot |\hat G_t(s)|\nonumber\\
      &\leq \frac{C\sqrt{\log 1/\delta}}{\sqrt{m}} \sum_{s \in \MS} \left( 1 + | \hat G_t(s)| \right) \leq 8H\ep'
                      \label{eq:phistar-ghat-g}
    \end{align}
    where the final inequality uses that $|\hat G_t(s)| \leq 1$ for all $s \in \MS$, and $|\MS| \leq 4H$.
    
    Finally, by another application of the Chernoff bound, for each $t$, there is some event $\MG_t''$ occurring with probability $1-\delta$, under which we have
    \begin{align}
      \left| \E^{\phi^\st, \pi_t}[| \hat G_t(s_{h_t}) - \hat \mu_t|] - \hat \E_t'[| \hat G_t(s_{h_t}) - \hat \mu_t|]\right|\leq \frac{C \sqrt{\log 1/\delta}}{\sqrt m} &\leq  \ep', \label{eq:phistar-ehatprime}\\
      \left| V_1^{\phi^\st, \pi_t} - \hat V_t \right| \leq \frac{C \sqrt{\log 1/\delta}}{\sqrt m} &\leq \ep'\label{eq:phistar-v},
    \end{align}
    for some constant $C$. Let $\MG := \bigcap_{t \in [T]} (\MG_t \cap \MG_t' \cap \MG_t'')$, so that $\MG$ occurs with probability at least $1-3\delta T = {23/24 > 11/12}$ over the execution of $\TestReduction(\MO, T, \ep)$, for any fixed $\MO:[X]\to[X]$ where $(s,i)\mapsto\MO((s,i))$ is injective. 

    Now let us take $\MO$ to be the random variable $\MO \sim \Unif([X]^{[X]})$ (note that $\phi^\st := \phi_\MO$ is now also a random variable). Then $\MO$ is distributed identically to the oracle $\MO_{\til x}$ in the proof of \cref{thm:oracle-lb}, so \cref{clm:all-good-predictors} gives that there is an event $\ME_{\ref{clm:all-good-predictors}}$ occurring with probability at least $3/4$ (over the choice of $\MO$ and the randomness of \SimulateReduction), under which for all $t \in [T]$, it holds that
 $ 
\E^{\phi^\st, \pi_t}[| G_t(s_{h_t}) - \hat \mu_t|] \leq 3\ep/4
$
and $V_1^{\phi^\st, \pi_T} \le 1/4$. Moreover, in this event, it holds (by definition of $\ME_{\ref{clm:all-good-predictors}}$ in the proof of \cref{clm:all-good-predictors}) that $\phi^\st = \phi_\MO \in \Edistinct$, so $(s,i) \mapsto \MO((s,i))$ is injective. Thus, the event $\ME_{\ref{clm:all-good-predictors}} \cap \MG$ occurs with probability at least $1-1/12-1/4 = 2/3$. Under this event, by \cref{eq:phistar-ghat-g,eq:phistar-ehatprime}, we have
    \begin{align}
      \hat \E_t'[|\hat G_t(s_{h_t}) - \hat \mu_t|] &\leq  \E^{\phi^\st, \pi_t}[|\hat G_t(s_{h_t}) - \hat \mu_t|] + \ep'\nonumber\\
      &\leq \E^{\phi^\st, \pi_t}[| G_t(s_{h_t}) - \hat \mu_t|] + 8H\ep' + \ep' \nonumber \\ 
      &\leq 3\ep/4 + 10 H \ep' \leq \ep\label{eq:beq1},
    \end{align}
    where the first two inequalities use \cref{eq:phistar-ehatprime} and \cref{eq:phistar-ghat-g} (which hold under $\MG$) respectively, the third inequality uses the first guarantee of $\ME_{\ref{clm:all-good-predictors}}$, and the final inequality uses the definition of $\ep' := \ep/(40H)$. Also, under the event $\ME_{\ref{clm:all-good-predictors}} \cap \MG$, we have $\hat V^t \leq V_1^{\phi^\st, \pi_T}  + \ep'\leq 3/8$ (using \cref{eq:phistar-v}) for all $t$. Together with \cref{eq:beq1}, this implies that \begin{equation}\Pr_{\MO\sim\Unif([X]^{[X]})}\left[\TestReduction(\MO, T, \ep) = 0\right] \geq 2/3].\label{eq:unif-case-output}\end{equation}
    Next, we instead take $\MO$ to be the random variable $\MO \sim \Unif(\MF)$, i.e. $\MO = F_\ell(\rho,\cdot) \in [X]^{[X]}$ where $\rho \sim \Unif(\{0,1\}^\ell)$. We distinguish two cases:
    \begin{itemize}
    \item First suppose that, for infinitely many $n \in \NN$, $\TestReduction(\MO, T, \ep)$ outputs $0$ with probability at most $7/12$ over the choice of $\MO$ and the randomness of the execution. Then together with \cref{eq:unif-case-output}, the Turing Machine $\mathscr{A} := \TestReduction(\cdot,T,\epsilon)$ satisfies
    \[\left|\EE_{\MO\sim\Unif([X]^{[X]})}\left[\mathscr{A}(\MO)\right] - \EE_{\MO\sim\Unif(\MF_\ell)}\left[\mathscr{A}(\MO)\right]\right| \geq 1/12 > 1/q(\ell)\]
    where the final inequality holds for sufficiently large $n$. But we claim that \TestReduction runs in time at most $O(T m H)$ with access to the oracle $\MO$. Indeed, \SimulateRegression is an $O(\log m)$-bounded regression oracle, and by definition of $m$ and the bounds $T, \epsilon^{-1} \leq O(H_n)$, we have $O(\log m) \leq H_n^{c_0} =: B(n)$ so long as $c_0$ is chosen to be sufficiently large. Thus by \cref{def:rl-reg-reduction-2}, the simulation of $\Alg$ by \SimulateReduction has running time at most $T$, excluding the implementations of the oracle calls by \SimulateSampling and \SimulateRegression. The total size of the circuits describing $\pi_1,\dots,\pi_T$ can be bounded by $O(T)$, so the total running time of those implementations is at most $O(TmH)$. Hence, the running time of \SimulateReduction in \cref{alg:distinguish-regression} is at most $O(TmH)$. It is then straightforward to check that the execution of the remainder of \TestReduction runs in time at most $O(TmH)$ as well, with access to the oracle $\MO$. But we have \[O(TmH) \leq O(T \log(1/\delta) H^3/\ep^2) \leq  O(T \log(T) \cdot t(\ell)^{1/4} \cdot t(\ell)^{2/8}) \leq O( t(\ell)^{3/4}) \leq t(\ell)\] using our assumptions on $\ep, T$ and  the fact that $t(\ell) = t(\ell_n) \geq H_n^{12} = H^{12}$ as long as $n$ is sufficiently large. 
      Hence, \TestReduction contradicts the fact that $(F_\ell)_{\ell\in\NN}$ is a $(t,q)$-pseudorandom permutation family (see in particular \cref{it:prp-indist} of \cref{def:prp}).
    \item Next suppose that, for all sufficiently large $n \in \NN$, $\TestReduction(\MO, T, \ep)$ outputs $0$ under some event $\ME_0$ that occurs with probability at least $7/12$ over the choice of $\MO$ and the randomness of the execution. From \cref{line:disc-large} of \cref{alg:distinguish-regression}, we know that under event $\ME_0$ it holds that $\hat\E_t'[|\hat G_t(s_{h_t})-\hat\mu_t|] \leq 3\epsilon/4$ for all $t \in [T]$. Thus, under the event $\ME_0 \cap \MG$ we have that, for each $t \in [T]$, 
      \begin{align}
        \E^{\phi^\st, \pi_t}[| G_t(s_{h_t}) - \hat \mu_t|] 
        &\leq \E^{\phi^\st, \pi_t}[|  \hat G_t(s_{h_t}) - \hat \mu_t |] + \ep'\nonumber\\
        &\leq \hat \E_t'[|\hat G_t(s_{h_t}) - \hat \mu_t|] + \ep' + 8H\ep'\nonumber\\
        &\leq 3\ep/4 + 10H\ep' \nonumber\\
        &\leq \ep\nonumber,
      \end{align}
      where we have used \cref{eq:phistar-ghat-g,eq:phistar-ehatprime}. It follows by \cref{lem:acc-predictor} that under $\ME_0 \cap \MG$, for each $t \in [T]$, the constant function $\hat \mu_t$ is a $(d_{h_t}^{\phi^\st, \pi_t}, G_t, 3\ep)$-accurate predictor with respect to $\MD_{\phi^\st}$.  

      Additionally, from \cref{line:value-large} of \cref{alg:distinguish-regression}, we know that $\hat V^T < 3/8$ under event $\ME_0$. Thus, by \cref{eq:phistar-v}, under $\ME_0 \cap \MG$, we have 
      \begin{align}
V_1^{\phi^\st, \pi_T} &\leq \hat V^T + \ep' < 1/2\nonumber.
      \end{align}
      Note that $M^{\phi^\st}_{X,H}$ does have a policy with value $1$, so $\pi_T$ has suboptimality greater than $1/2$.
      
      By the union bound, we know that $\ME_0 \cap \MG$ occurs with probability at least $13/24$ over the choice of $\MO$ and the execution of $\TestReduction(\MO, T, \ep)$. Thus, there exists some $\MO^\st = F_\ell(\rho^\st,\cdot) \in \MF$ so that $\ME_0 \cap \MG$ occurs with probability at least $13/24$ over the execution of $\TestReduction(\MO^\st, T, \ep)$. As in the proof of \cref{thm:oracle-lb}, we now consider modifying the execution of $\SimulateReduction(\Alg, X, H,  \MO^\st, \ep, 1/(16T))$ in \cref{line:simulate-reduction} of \cref{alg:distinguish-regression}, so that all calls to the regression oracle are implemented by the following oracle $\Oregressp$ (instead of by \SimulateRegression in \cref{alg:simulate-oracle}):
      \begin{itemize}
\item For each regression oracle query $(\MB_{\pi_t}, \MB_{L_t}, h_t)$, if \SimulateRegression returns a (constant) function $\hat \mu_t$ which is $(d_h^{\pi_t, \phi^\st}, G_t, 3\ep)$-accurate with respect to $\MD_{\phi^\st}$, then $\Oregressp$ returns $\hat \mu_t$. 
\item Otherwise, $\Oregressp$ returns the mapping $x \mapsto G_t(\phi^\st(x))$. 
\end{itemize}
Certainly $\Oregressp$ is $3\ep$-accurate per \cref{def:regression-oracle}. Since each constant function $\hat \mu_t$ and each function $G_t \circ \phi^\st$ is computable by a circuit of size $H_n^{\cprp}$, where $\cprp$ is a constant that exists by \cref{it:eff-decode} of \cref{def:prp}, the oracle $\Oregressp$ is also $H_n^{\cprp}$-bounded, and thus $B$-bounded if we take $c_0 = \cprp$. Moreover, under the event $\ME_0 \cap \MG$ (which occurs with probability at least $13/24$), the transcript of the execution of $\SimulateReduction(\Alg, X, H, \MO^\st, \ep, 1/(16T))$ is identical when regression oracle calls are implemented by $\Oregressp$ as to when they are all implemented by \SimulateRegression. And as shown above, under $\ME_0 \cap \MG$, the output policy $\pi_T$ of $\Alg$, as simulated by $\SimulateReduction(\Alg, X, H, \MO^\st,\ep, 1/(16T))$, is not $1/2$-optimal.

In all, we've shown that (for all sufficiently large $n \in \NN$), with probability at least $13/24$, the output policy $\pi_T$ of $\Alg$ (as simulated by $\SimulateReduction(\Alg, X, H, \MO^\st,\ep, 1/(16T))$) is not $1/2$-optimal. It follows from \cref{def:rl-reg-reduction-2} that $\Alg$ cannot be a computational $(\Tred, \epred, B)$-reduction. 
\end{itemize}
In both of the above cases, we have arrived at a contradiction, thus proving the theorem.
  \end{proof}

  \begin{algorithm}[t]
    \caption{\TestReduction$(\MO,T,\ep)$}
    \label{alg:distinguish-regression}
    \begin{algorithmic}[1]
      \Require Input oracle $\MO$, $T \in \BN$, $\ep \in (0,1)$. 
      \State Let $m := C_1 \log(72T)H^2/\ep^2$, for a sufficiently large constant $C_1$. \label{line:choose-m-c1}
      \State \label{line:simulate-reduction} Run $ \SimulateReduction(\Alg,X, H, \MO, \ep, 1/(16T))$; for $t \in [T-1]$, let $(\MB_{\pi_t}, \MB_{L_t}, h_t, \MO)$ denote the input to the $t$th call to \SimulateRegression, and let $\pi_T$ denote the output policy.  \Comment{\emph{\cref{alg:simulate-oracle}}}
      \For{$1 \leq t \leq T$}
      \State \label{line:call-simsample} Call $\SimulateSampling(\pi_t, \MO)$ (from \cref{alg:simulate-oracle}) $m$ times, generating $m$ trajectories $(s_{1:H}^i, a_{1:H}^i, x_{1:H}^i, r_{1:H}^i)_{i \in [m]}$, and denote the resulting empirical measure on  $(\MS \times \MA \times \MX \times [0,1])^H$ by $\hat \BP_t$, and the corresponding expectation by $\hat \E_t$.
      \State \label{line:def-hatg} Define $\hat G_t(s) := \hat \E_t[L_t(x_{1:H}, a_{1:H}, r_{1:H}) \mid s_{h_t} = s]$ for each $s \in \MS$.\Comment{\emph{Set $\hat G_t(s) := 0$ if the conditional expectation is undefined.}}
      \State Let $\hat \mu_t$ denote the constant function returned by the $t$th call to \SimulateRegression (i.e., to  $\SimulateRegression(\MB_{\pi_t}, \MB_{L_t}, h_t, \MO)$) within \SimulateReduction above. 
      \State Call $\SimulateSampling(\pi_t, \MO)$ (from \cref{alg:simulate-oracle}) $m$ more times, and denote the resulting empirical measure on  $(\MS \times \MA \times \MX \times [0,1])^H$ by $\hat \BP_t'$, and the corresponding expectation by $\hat \E_t'$  (as in \cref{line:call-simsample}).
      \State \label{line:def-hatv} Define $\hat V^t := \hat \E_t'\left[ \sum_{h=1}^H r_h \right]$. 
      \EndFor
      \If{There is $t \in [T]$ for which ${\hat \E_t'}[|\hat G_t(s_{h_t}) - \hat \mu_{t}|] > 3\ep/4$} \label{line:disc-large}
      \State\Return 1
      \EndIf
      \If{$\hat V^T \geq 3/8$}\label{line:value-large}
      \State\Return 1
      \EndIf
      \State\Return 0
      \end{algorithmic}
  \end{algorithm}



\bibliographystyle{alpha}
\bibliography{bib}

\appendix 

\section{Offline reinforcement learning}\label{sec:offline}
In this section we show that there is an efficient reduction from offline reward-directed RL with all-policy concentrability to regression, closely following a proof from \cite{chen2019information}. Informally, all-policy concentrability means that we have access to data from a policy that covers every state/action pair as well as all other policies, up to a multiplicative factor of $\kappa$. The online RL setting where we are given a succinct description of such a policy is no harder.

To formalize the reduction, we must make several definitions. First, recall that our original definition of a regression oracle for online RL (\cref{def:regression-oracle}) allows the algorithm to query any policy $\pi$ and label function $L$ and step $h$, and the oracle must produce a regressor that is accurate in expectation over the emission distribution induced by $\pi$ at step $h$. This no longer makes sense in offline RL, where the algorithm only has access to data from a fixed policy or distribution. Thus, we introduce the notion of a \emph{fixed-distribution regression oracle}, which is defined by a set of distributions $(\mu_h)_{h \in [H]}$ corresponding to the offline data. The oracle takes as input an action $a$, label function $L$, and step $h$, and it must produce a regressor that is accurate in expectation over $\mu_h$. The error is with respect to the expected label of a partial trajectory $(x,r,x')$ induced by starting at some latent state $s$ at step $h$ and then taking action $a$. We make the formal definition below.

\begin{definition}\label{def:fd-reg-oracle}
Let $M = (H,\MS,\MX,\MA,\til\BP_0,(\til\BP_h)_h,(\til\BO_h)_h,(\til\br_h)_h,\phi^\st)$ be a block MDP. For $B \in \NN$, a $B$-bounded \emph{fixed-distribution regression oracle} for $M$ is a nondeterministic function $\Oregress$ which takes as input a step $h \in [H]$, an action $a \in \MA$, and a (randomized) circuit $\MB_L$ describing a labeling function $L: \MX\times\{0,1\}\times\MX \to \Delta(\{0,1\})$, and which outputs a circuit $\MC_\MR$ describing a mapping $\MR: \MX\to[0,1]$, where $\size(\MC_\MR) \leq B$.

Let $\mu_1,\dots,\mu_H \in \Delta(\MS)$. We say that $\Oregress$ is \emph{$\epsilon$-accurate over $(\mu_h)_{h\in[H]}$} for $M$ if for each tuple $(\MB_L, h, a)$, the output $\MR := \Oregress(\MB_L,h,a)$ is a $(\mu_h, f, \epsilon)$-accurate predictor (\cref{def:reg-predictor}) with respect to $\BO_h(\cdot|s)$, where $f:\MS\to[0,1]$ is the function defined as
\[f(s) := \EE_{\substack{x \sim \til\BO_h(\cdot|s) \\ r \sim \Ber(\til\br_h(s, a)) \\ s' \sim \til \BP_h(\cdot|x,a) \\ x' \sim \til\BO_{h+1}(\cdot|s')}}[L(x,r,x')].\]
\end{definition}

\begin{remark}
We now motivate why this is an appropriate oracle for offline RL. Consider the offline RL setting where for each step $h$, we are given partial trajectories starting at step $h$ with latent state $s \sim \mu_h$ and action $a \sim \Unif(\MA)$ (and subsequently following an arbitrary policy). Then it can be seen that the fixed-distribution regression oracle for distributions $(\mu_h){h \in [H]}$ is statistically tractable, since for any action $a$ and step $h$, one can select the $1/(|\MA|H)$-fraction of the data where action $a$ is chosen at step $h$. More generally, for any data distribution (over trajectories) where every state/action pair is well-covered, there is a choice of $(\mu_h)_{h\in[H]}$ that covers every state (\cref{def:kappa-coverage}) and such that the corresponding fixed-distribution regression oracle is statistically tractable. 
\end{remark}

\begin{remark}[Online RL with exploratory policy]
Consider the online RL setting where we have access to the original regression oracle (\cref{def:regression-oracle}) as well as a succinct description of an exploratory policy, i.e. a policy $\pi$ such that the distributions $(d^\pi_h)_{h \in [H]}$ satisfy $\kappa$-coverage (\cref{def:kappa-coverage}). We can implement a query $(h, a, L)$ to the fixed-distribution regression oracle with distributions $(d^\pi_h)_{h \in [H]}$ as follows. First, define the policy $\pi^{h,a}$ that follows $\pi$ until step $h$ and subsequently plays action $a$. Second, define $L'(x_{1:H}, a_{1:H}, r_{1:H}) := L(x_h, r_h, x_{h+1})$. Third, invoke the original oracle with input $(h, \pi^{h,a}, L')$. Hence, the below reduction also applies in this online, exploratory setting.
\end{remark}


\begin{algorithm}[t]
    \caption{$\FQI(\Oregress)$: Oracle version of Fitted $Q$-Iteration}
    \label{alg:fqi}
    \begin{algorithmic}[1]
      \Require A fixed-distribution regression oracle $\Oregress$.
      \State Set $\hat Q_{H+1} \gets ((x, a) \mapsto 0)$
      \For{$h = H, H-1, \dots, 1$}
        \For{$\overline{a} \in \MA$}
            \Function{$\MB_L$}{$x,r,x'$}
                \State \textbf{return} $\frac{1}{H}\left(r + \max_{a' \in \MA} \hat Q_{h+1}(x',a')\right)$
            \EndFunction
            \State Set $\hat Q_h(\cdot, \overline a) \gets H \cdot \Oregress(\MB_L, h, \overline a)$ 
            \EndFor
            \State Set $\hat \pi_h(\cdot) \gets (x \mapsto \argmax_{a\in\MA} \hat Q_h(x, a))$
      \EndFor
      \State \textbf{return} $(\hat\pi_h)_{h \in [H]}$
      \end{algorithmic}
  \end{algorithm}

\begin{definition}\label{def:kappa-coverage}
Let $M = (H,\MS,\MX,\MA,\til\BP_0,(\til\BP_h)_h,(\til\BO_h)_h,(\til\br_h)_h,\phi^\st)$ be a block MDP, and let $\mu_1,\dots,\mu_H \in \Delta(\MS)$ be distributions. For $\kappa \geq 1$ and $h \in [H]$, we say that $\mu_h$ satisfies \emph{$\kappa$-coverage at step $h$} for $M$ if, for every general policy $(\pi_h)_{h=1}^H$, it holds that
\[\max_{s \in \MS} \frac{d^{M,\pi}_h(s)}{\mu_h(s)} \leq \kappa.\]
We then say that $(\mu_h)_{h\in[H]}$ satisfies $\kappa$-coverage for $M$ if $\mu_h$ satisfies $\kappa$-coverage at step $h$, for each $h \in [H]$.
\end{definition}

We now formally state the reduction, which is simply an implementation of $\FQI$ using the above oracle.

\begin{theorem}[c.f. \cite{chen2019information}]\label{theorem:fqi-guarantee}
Let $M = (H,\MS,\MX,\MA,\til\BP_0,(\til\BP_h)_h,(\til\BO_h)_h,(\til\br_h)_h,\phi^\st)$ be a block MDP with $\MX = \{0,1\}^\ell$, and let $\mu_1,\dots,\mu_H \in \Delta(\MS)$ be distributions. Fix $\kappa, B \geq 1$ and $\epsilon > 0$. Suppose that $(\mu_h)_{h=1}^H$ satisfies $\kappa$-coverage for $M$, and let $\Oregress$ be a $B$-bounded fixed-distribution regression oracle for $M$ that is $\epsilon$-accurate over $(\mu_h)_{h=1}^H$ (\cref{def:fd-reg-oracle}). Then the policy $\hat \pi$ produced by $\FQI(\Oregress)$ satisfies
\[\E^M V^{M,\pi^\st}_1(x_1) - \E^M V_1^{M,\hat\pi}(x_1) \leq 2H^3\sqrt{\kappa|\MA| \epsilon}\]
where $\pi^\st$ is the optimal policy for $M$. Moreover, the time complexity of the algorithm is $\poly(H, |\MA|, d, B)$.
\end{theorem}

The proof is essentially due to \cite{chen2019information}; we simply modify their analysis to the finite-horizon setting and to use the fixed-distribution regression oracle as a black-box (thus avoiding the generalization bounds needed in \cite{chen2019information}). The main ideas of the proof are unchanged.

\begin{lemma}[{c.f. \cite[Lemma 13]{chen2019information}}]\label{lemma:pd-l2}
Let $f_1,\dots,f_H: \MX \times \MA \to \RR$ be functions, and define a policy $(\pi_h)_{h=1}^H$ by $\pi_h(x) = \argmax_{a \in \MA} f_h(x,a)$. Then for any policy $\pi^\st:\MX\to\MA$,
\begin{align*}
\E^M V^{M, \pi^\st}_1(x_1) - \E^M V^{M, \pi}_1(x_1) &\leq \sum_{h=1}^H \sqrt{\E^{M,\pi} (Q^{M,\pi^\st}_h(x_h, \pi_h^\st(x_h)) - f_h(x_h,\pi_h^\st(x_h)))^2} \\
&\qquad+ \sqrt{\E^{M,\pi} (Q^{M,\pi^\st}_h(x_h, \pi_h(x_h)) - f_h(x_h,\pi_h(x_h)))^2}
\end{align*}
\end{lemma}

\begin{proof}
By the performance difference lemma \cite{kakade2002approximately}, we have
\begin{align*}
\E^M V_1^{M,\pi^\st}(x_1) - \E^M V_1^{M,\pi}(x_1)
&= \sum_{h=1}^H \E^{M,\pi}\left[ Q_h^{M,\pi^\st}(x_h, \pi_h^\st(x_h)) - Q_h^{M,\pi^\st}(x_h,\pi_h(x_h))\right] \\ 
&\leq \sum_{h=1}^H \E^{M,\pi}\left[ Q_h^{M,\pi^\st}(x_h, \pi_h^\st(x_h)) - f_h(x_h, \pi_h^\st(x_h)))\right] \\ 
&\qquad+ \sum_{h=1}^H \E^{M,\pi}\left[f_h(x_h,\pi_h(x_h)) - Q_h^{M,\pi^\st}(x_h,\pi_h(x_h))\right] \\ 
&\leq \sum_{h=1}^H \sqrt{\E^{M,\pi} (Q^{M,\pi^\st}_h(x_h, \pi_h^\st(x_h)) - f_h(x_h,\pi_h^\st(x_h)))^2} \\
&\qquad+ \sqrt{\E^{M,\pi} (Q^{M,\pi^\st}_h(x_h, \pi_h(x_h)) - f_h(x_h,\pi_h(x_h)))^2}
\end{align*}
where the first inequality is by definition of $\pi$, and the second inequality is by Jensen's inequality.
\end{proof}

\begin{lemma}[{c.f. \cite[Lemma 15]{chen2019information}}]\label{lemma:fq-error-induction}
Let $\mu_h \in \Delta(\MS)$ be a distribution that satisfies $\kappa$-coverage for $M$ at step $h \in [H]$. Let $\pi$ be a Markovian policy, and let $f_h, f_{h+1}: \MX\times\MA \to \RR$ be functions. Then there is a Markovian policy $\pi'$ such that
\begin{align*}
&\sqrt{\E^{M,\pi}(f_h(x_h,a_h)-Q^{M,\pi^\st}_h(x_h,a_h))^2} \\
&\leq \sqrt{\kappa |\MA| \EE_{\substack{s \sim \mu_h \\ (x,a) \sim \til\BO_h(\cdot|s) \times \Unif(\MA)}} \left(f_h(x,a) - \br_h(x,a) - \EE_{x' \sim \BP_h(\cdot|x,a)} \max_{a' \in \MA} f_{h+1}(x',a')\right)^2} \\ 
&+ \sqrt{\E^{M,\pi'}(f_{h+1}(x_{h+1},a_{h+1}) - Q^{M,\pi^\st}_{h+1}(x_{h+1},a_{h+1}))^2}.
\end{align*}
where $\pi^\st$ is the optimal policy for $M$.
\end{lemma}

\begin{proof}
For notational convenience, for any $h \in [H]$ and $f: \MX\times\MA\to\RR$, let $(\MT_h f): \MX \times\MA\to\RR$ be the function defined by $(\MT_h f)(x_h, a_h) := \br_h(x_h,a_h) + \EE_{x_{h+1}\sim\BP_h(\cdot|x_h,a_h)} \max_{a_{h+1}\in\MA} f(x_{h+1},a_{h+1})$. Observe that by Bellman optimality, $\pi^\st$ satisfies $Q_h^{M,\pi^\st} = \MT_h Q_{h+1}^{M,\pi^\st}$. Thus, using the triangle inequality for the $L_2$ norm, we have
\begin{align}
\sqrt{\E^{M,\pi}(f_h(x_h,a_h) - Q_h^{M,\pi^\st}(x_h,a_h))^2} \nonumber
&\leq \sqrt{\E^{M,\pi}(f_h(x_h,a_h) - (\MT_h f_{h+1})(x_h,a_h))^2} \nonumber \\ 
&\qquad+ \sqrt{\E^{M,\pi}((\MT_h f_{h+1})(x_h,a_h) - (\MT_h Q_{h+1}^{M,\pi^\st})(x_h,a_h))^2}.\label{eq:fq-tri-ineq}
\end{align}
To bound the first term, note that
\begin{align}
\E^{M,\pi}(f_h(x_h,a_h) - (\MT_h f_{h+1})(x_h,a_h))^2
&= \sum_{s \in \MS} d^{M,\pi}_h(s) \EE_{x \sim \til\BO_h(\cdot|s)}  (f_h(x,\pi_h(x)) - (\MT_h f_{h+1})(x,\pi_h(x)))^2 \nonumber\\ 
&\leq \kappa \sum_{s \in \MS} \mu_h(s) \EE_{x \sim \til\BO_h(\cdot|s)}  (f_h(x,\pi_h(x)) - (\MT_h f_{h+1})(x,\pi_h(x)))^2 \nonumber\\ 
&\leq \kappa |\MA| \sum_{s \in \MS} \mu_h(s) \EE_{(x,a) \sim \til\BO_h(\cdot|s)\times\Unif(\MA)}  (f_h(x,a) - (\MT_h f_{h+1})(x,a))^2 \label{eq:ftf-bound}
\end{align}
where the first inequality uses that $\mu_h$ satisfies $\kappa$-coverage for $M$ at step $h$. To bound the second term, let $\pi'$ be the Markovian policy defined by
\[\pi'_k(x) := \begin{cases} \pi_k(x) & \text{ if } k \leq h \\ \argmax_{a \in \MA} \{f_{h+1}(x,a), Q^{M,\pi^\st}_{h+1}(x,a)\} & \text{ if } k = h+1 \\ \bar a & \text{ if } k > h+1 \end{cases}\]
for some arbitrary action $\bar a \in \MA$. Then note that
\begin{align}
\E^{M,\pi}((\MT_h f_{h+1})(x_h,a_h) - (\MT_h Q_{h+1}^{M,\pi^\st})(x_h,a_h))^2
&= \E^{M,\pi} \left(\max_{a \in \MA} f_{h+1}(x_{h+1},a) - \max_{a \in \MA} Q^{M,\pi^\st}_{h+1}(x_{h+1},a)\right)^2 \nonumber\\ 
&\leq \E^{M,\pi} \left(f_{h+1}(x_{h+1},\pi'_{h+1}(x_{h+1})) - Q^{M,\pi^\st}_{h+1}(x_{h+1},\pi'_{h+1}(x_{h+1}))\right)^2 \nonumber\\ 
&= \E^{M,\pi'} \left(f_{h+1}(x_{h+1}, a_{h+1}) - Q^{M,\pi^\st}_{h+1}(x_{h+1},a_{h+1})\right)^2.\label{eq:tftq-bound}
\end{align}
Substituting \cref{eq:ftf-bound,eq:tftq-bound} into \cref{eq:fq-tri-ineq} completes the proof.
\end{proof}

\begin{namedproof}{\cref{theorem:fqi-guarantee}}
The time complexity bound is immediate from the algorithm description and the assumption that $\Oregress$ is $B$-bounded. It remains to bound the suboptimality of $\hat \pi$. For each $h \in [H]$ and $\bar a \in \MA$, the assumption that $\Oregress$ is $\epsilon$-accurate over $(\mu_h)_{h=1}^H$ implies that
\begin{equation}\EE_{\substack{s_h \sim \mu_h \\ x_h \sim \til\BO_h(\cdot|s_h)}} (\hat Q_h(x_h, \bar a) - f(s_h))^2 \leq \epsilon H^2 \label{eq:qhat-error-0}\end{equation}
where
\begin{align*}
f(s_h) &:= \EE_{\substack{x'_h \sim \til\BO_h(\cdot|s_h) \\ r_h \sim \Ber(\tilde\br_h(s_h,\bar a)) \\ s_{h+1} \sim \til\BP_h(\cdot|s_h, \bar a) \\ x_{h+1} \sim \til\BO_{h+1}(\cdot|s_{h+1})}}[r + \max_{a' \in\MA} \hat Q_{h+1}(x_{h+1},a')] \\ 
&= \til\br_h(s_h,\bar a) + \EE_{\substack{x'_h \sim \til\BO_h(\cdot|s_h) \\ x_{h+1} \sim \BP_h(\cdot|x'_h, \bar a)}} \max_{a'\in\MA} \hat Q_{h+1}(x_{h+1},a') \\ 
&= \br_h(x_h, \bar a) + \EE_{x_{h+1} \sim \BP_h(\cdot|x_h, \bar a)} \max_{a'\in\MA} \hat Q_{h+1}(x_{h+1},a')
\end{align*}
and the last equality holds for any $x_h \in \MX$ with $\til\BO_h(x_h|s_h) > 0$, by the definition of a block MDP. Thus, \cref{eq:qhat-error-0} gives that
\begin{equation*}
\EE_{\substack{s_h \sim \mu_h \\ x_h \sim \til\BO_h(\cdot|s_h)}} \left(\hat Q_h(x_h,\bar a) - \br_h(x_h,\bar a) - \EE_{x_{h+1}\sim\BP_h(\cdot|x_h,\bar a)} \max_{a'\in \MA} \hat Q_{h+1}(x_{h+1}, a')\right)^2 \leq \epsilon H^2.
\end{equation*}
Averaging over all $\bar a \in \MA$ gives
\begin{equation}
\EE_{\substack{s_h \sim \mu_h \\ (x_h, \bar a) \sim \til\BO_h(\cdot|s_h)\times \Unif(\MA)}} \left(\hat Q_h(x_h,\bar a) - \br_h(x_h,\bar a) - \EE_{x_{h+1}\sim\BP_h(\cdot|x_h,\bar a)} \max_{a'\in \MA} \hat Q_{h+1}(x_{h+1}, a')\right)^2 \leq \epsilon H^2.\label{eq:qhat-error}
\end{equation}
Now we apply \cref{lemma:fq-error-induction} with distribution $\mu_h$ and functions $\hat Q_h, \hat Q_{h+1}$. By assumption that $\mu_h$ satisfies $\kappa$-coverage for $M$ at step $h$, and by \cref{eq:qhat-error}, we have that for any $h \in [H]$ and Markovian policy $\pi$, there is a Markovian policy $\pi'$ such that
\[\sqrt{\E^{M,\pi}(\hat Q_h(x_h,a_h) - Q^{M,\pi^\st}_h(x_h,a_h))^2} \leq H\sqrt{\kappa|\MA|\epsilon} + \sqrt{\E^{M,\pi'}(\hat Q_{h+1}(x_{h+1},a_{h+1}) - Q^{M,\pi^\st}_{h+1}(x_{h+1},a_{h+1}))^2}.\]
Since $\hat Q_{H+1}(\cdot,\cdot) = Q^{M,\pi^\st}_{H+1}(\cdot,\cdot) = 0$, we get by iterating the above inequality that for any $h \in [H]$ and Markovian policy $\pi$, it holds that
\begin{equation}
\sqrt{\E^{M,\pi}(\hat Q_h(x_h,a_h) - Q^{M,\pi^\st}_h(x_h,a_h))^2} \leq H^2 \sqrt{\kappa |\MA| \epsilon}.\label{eq:hatq-q-error}
\end{equation}
Applying \cref{lemma:pd-l2} with the policy $\hat \pi$ defined by $\hat\pi_h(x) := \argmax_{a \in \MA} \hat Q_h(x,a)$ and the policy $\pi^\st$ that is optimal for $M$, we get that the suboptimality of $\hat\pi$ is at most
\begin{align*}
\E^M V_1^{M,\pi^\st}(x_1) - \E^M V_1^{M,\hat\pi}(x_1)
&\leq \sum_{h=1}^H \sqrt{\E^{M,\hat\pi}(Q^{M,\pi^\st}_h(x_h,\pi^\st_h(x_h)) - \hat Q_h(x_h,\pi^\st_h(x_h)))^2} \\ 
&\qquad+ \sum_{h=1}^H \sqrt{\E^{M,\hat\pi} (Q^{M,\pi^\st}_h(x_h,\hat\pi(x_h)) - \hat Q_h(x_h,\hat\pi(x_h)))^2} \\ 
&\leq 2H^3 \sqrt{\kappa |\MA|\epsilon}
\end{align*}
where the final inequality invokes \cref{eq:hatq-q-error} for each $h \in [H]$, with the policies $\hat \pi$ and $\hat \pi \circ_h \pi^\st$ (i.e. the policy that follows $\hat \pi$ until step $h$, and subsequently follows $\pi^\st$).
\end{namedproof}


\section{Block MDPs with horizon one}\label{sec:bandits}

In this section we show that for the special case of block MDPs with horizon one, there is a reduction from online, reward-directed RL to regression and vice versa. The first reduction is essentially folklore, and can be thought of as a special case of $\FQI$, using the uniformly random policy as the exploratory policy:

\begin{proposition}
Let $\MM = ((\MS_n)_n, (\MA_n)_n, (H_n)_n, (\ell_n)_n, (\Phi_n)_n, (\MM_n)_n)$ be a block MDP family with $H_n = 1$ for all $n \in \NN$. For any function $B: \NN \to \NN$, there is a $(T, \epsilon, B)$-computational reduction from RL to regression for $\MM$, with $T(n) \leq \poly(|\MS_n|,|\MA_n|, H_n, B(n))$ and $\epsilon(n) = \Omega(1/|\MA_n|^2)$.
\end{proposition}

\begin{proof}
We define an oracle Turing Machine $\Alg^{\Oregress}$ with access to oracle $\Oregress$ as follows. For each action $a \in \MA_n$, $\Alg$ constructs the circuit $\MB_{\pi_a}$ computing the constant function $\pi_a(x) = a$, and the circuit $\MB_L$ computing the function $(x,a,r) \mapsto \Ber(r)$. Then $\Alg$ computes $\hat \MR_a \gets \Oregress(\MB_{\pi_a},\MB_L, 1)$. After doing this for each $a \in \MA_n$, the output of $\Alg$ is the circuit $\MC_{\hat \pi}$ computing the function $\hat\pi: \{0,1\}^{\ell_n} \to \MA_n$ with $\pi(x) = \argmax_{a \in \MA_n} \hat \MR_a(x)$.

\paragraph{Analysis.} Fix $n \in \NN$ and an MDP $M = (H_n, \MS_n, \{0,1\}^{\ell_n}, \MA_n, \til\BP_0, (\til \BP_h)_h, (\til \BO_h)_h, (\til \br_h)_h,\phi^\st) \in \MM_n$. By assumption, we have $H_n = 1$, so a policy $\pi$ is a (potentially randomized) map $\{0,1\}^{\ell_n} \to \MA_n$, and the value of $\pi$ on $M$ is
\[V^{M,\pi}_1 = \EE_{\substack{s \sim \til\BP_0 \\ x \sim \til \BO_1(\cdot|s)}} \til\br_1(s, \pi(x)).\]
Suppose that $\Oregress$ is $\epsilon(n)$-accurate for $M$, where $\epsilon(n) := 1/(16 |\MA_n|^2)$. Then for each $a \in \MA_n$, by \cref{def:regression-oracle}, $\hat\MR_a$ is a $(\til\BP_0, f_a, \epsilon(n))$-accurate predictor with respect to $s\mapsto \BO_1(\cdot|s)$, where $f_a: \MS_n \to [0,1]$ is the function $f_a(s) := \EE^{M,\pi_a}[r_1|s_1=s] = \br_1(s,a)$. By \cref{def:reg-predictor}, this means that for each $a \in \MA_n$,
\begin{align}
\EE_{\substack{s \sim \til\BP_0 \\ x \sim \til\BO_1(\cdot|s)}} (\hat\MR_a(x) - \br_1(s,a))^2 \leq \epsilon(n).\label{eq:bandit-oracle-guarantee}
\end{align}
Now let $\pi^\st: \{0,1\}^{\ell_n} \to \MA_n$ be the optimal policy for $M$. For any $s \in \MS_n$ and $x \in \{0,1\}^{\ell_n}$, we have
\begin{align*}
\til\br_1(s,\pi^\st(x)) - \til\br_1(s,\hat\pi(x)) 
&= \til\br_1(s,\pi^\st(x)) - \hat\MR_{\pi^\st(x)}(x) + \hat\MR_{\pi^\st(x)}(x) - \til\br_1(s,\hat\pi(x))\\ 
&\leq \til\br_1(x,\pi^\st(x)) - \hat\MR_{\pi^\st(x)}(x) + \hat\MR_{\hat\pi(x)}(x) - \til\br_1(s,\hat\pi(x)) \\ 
&\leq 2\max_{a \in \MA_n} |\til\br_1(s,a) - \hat\MR_a(x)|
\end{align*}
where the first inequality is by definition of $\hat\pi$. Thus, we can bound the suboptimality of $\hat\pi$ as
\begin{align*}
V_1^{M, \pi^\st} - V_1^{M, \hat\pi}
&= \EE_{\substack{s \sim \til\BP_0 \\ x \sim \til\BO_1(\cdot|s)}} \br_1(s, \pi^\st(x)) - \br_1(s,\hat\pi(x)) \\ 
&\leq 2\EE_{\substack{s \sim \til\BP_0 \\ x \sim \til\BO_1(\cdot|s)}} \max_{a \in \MA_n} |\til\br_1(s,a) - \hat\MR_a(x)| \\ 
&\leq 2|\MA_n| \cdot \max_{a \in \MA_n} \EE_{\substack{s \sim \til\BP_0 \\ x \sim \til\BO_1(\cdot|s)}} |\til\br_1(s,a) - \hat\MR_a(x)| \\ 
&\leq 2|\MA_n| \sqrt{\epsilon(n)}  \leq 1/2
\end{align*}
where the third inequality is by \cref{eq:bandit-oracle-guarantee} and Jensen's inequality, and the final inequality is by choice of $\epsilon(n)$.

Moreover, if $\Oregress$ is $B(n)$-bounded, then by definition each circuit $\hat\MR_a$ has size at most $B(n)$. It follows that the time complexity of $\Alg^\Oregress$ is at most $\poly(|\MA_n|, \ell_n, B(n))$.

We conclude that $\Alg$ is a computational $(T,\epsilon,B)$-reduction from RL to regression for $\MM$, as specified by \cref{def:rl-reg-reduction-2}.
\end{proof}

We next show that solving regression is also a \emph{prerequisite} for solving reinforcement learning, even in horizon-$1$ block MDPs. The reduction is simple and does not contain novel ideas (see e.g. \cite[Appendix D]{golowich2023exploring} for a very similar reduction in the context of sparse linear MDPs) but we include it for completeness nonetheless.

\begin{proposition}
Let $\MM = ((\MS_n)_n,(\MA_n)_n,(H_n)_n,(\ell_n)_n,(\Phi_n)_n,(\MM_n)_n)$ be a block MDP family where $\MM_n$ consists of all $\Phi_n$-realizable block MDPs with state space $\MS_n$, action space $\MA_n$, horizon $H_n = 1$, and context space $\{0,1\}^{\ell_n}$. Let $\Obandits$ be an oracle that is given interactive access to an MDP and outputs a policy.

Let $\epsilon: \NN \to \RR_{>0}$ and $B: \NN \to \NN$. There is a oracle Turing Machine $\Alg^\Obandits$ with the following property. For any $n \in \NN$ and $M \in \MM_n$, suppose that the output of $\Obandits$, with interactive access to $M$, is a circuit of size at most $B(n)$, that with probability $1-o(1)$ describes a policy with sub-optimality at most $\epsilon(n)$ for $M$. Also suppose that $\Obandits$ requires at most $S(n)$ episodes of interaction. Then $\Alg^\Obandits$ is a $(S, T, \epsilon', B)$-realizable regression algorithm (\cref{def:regression-algorithm}) for $\MM$, where $T(n) \leq \poly(S(n), \ell_n, B(n))$ and $\epsilon'(n) := \epsilon(n) + 1/|\MA_n|^2$.


\end{proposition}

Formally speaking, \cref{def:regression-algorithm} only applies to Turing Machines, whereas $\Alg$ has oracle access, but the definition naturally generalizes. In particular, the time complexity parameter $T(n)$ is measuring the runtime of $\Alg$ assuming all oracle calls are efficient (e.g. take time $1$), and the success of $\Alg$ is conditional on the stated assumptions about the oracle $\Obandits$. The proposition can be alternatively stated with no reference to oracles as follows: if $\Obandits$ can be implemented efficiently, then $\Alg^\Obandits$ (which can therefore also be implemented efficiently) is a realizable regression algorithm in the strict sense of \cref{def:regression-algorithm}.

\begin{proof}
First, without loss of generality we may identify $\MA_n$ with the set $\{0,\delta(n), 2\delta(n),\dots, 1-\delta(n)\} \subset [0,1]$ where $\delta(n) := 1/|\MA_n|$. We now describe the behavior of $\Alg$. Given samples $(x^i,y^i)_{i=1}^{S'(n)}$ and the parameter $n$, $\Alg$ invokes $\Obandits$. To simulate an episode of interaction, it picks an unused sample $(x^i,y^i)$ and passes the context $x^i$ to the oracle. The oracle returns an action $a^i \in \MA_n$, and $\Alg$ passes the reward $r^i \sim \Ber(1-(a^i - y^i)^2)$ to the oracle. After some number of episodes of interaction (at most $S(n)$), $\Obandits$ outputs a circuit $\MC$. Then $\Alg$ simply outputs $\MC$ as well.

\paragraph{Analysis.} Let $\MD \in \Delta(\{0,1\}^{\ell_n})$ be any distribution and let $\phi \in \Phi_n$ be any decoding function. Let $(x^i, y^i)_{i=1}^{S'(n)}$ be i.i.d. samples where $x^i \sim \MD$ and $y^i \in \{0,1\}$ satisfies $y^i \perp x^i | \phi(x^i)$. Then there is a function $f: \MS_n \to [0,1]$ so that $y^i \sim \Ber(f(\phi(x^i)))$ conditioned on any $x^i$.

Consider the block MDP $M$ with state space $\MS_n$, action space $\MA_n$, horizon $1$, observation space $\{0,1\}^{\ell_n}$, initial distribution $\til \BP_0$ defined as the distribution of $\phi(x^1)$, observation distribution $\til\BO_1(\cdot|s)$ defined as the conditional distribution $x^1 | \phi(x^1) = s$, and reward function $\til \br_1(s, a) = 1 - (a - f(s))^2 - f(s) + f(s)^2$. Note that $M$ is $\phi$-decodable. Then we claim that each episode of interaction between $\Obandits$ and $\Alg$ is distributed as in $M$: indeed, $\phi(x^i)$ is distributed as $\til \BP_0$ by definition, $x^i|\phi(x^i)$ is distributed as $\til \BO_1(\cdot|\phi(x^i))$ (again by definition), and for $s \in \MS_n$ and $a \in \MA_n$, we have
\begin{align*}
\EE[1-r^i|\phi(x^i) = s, a^i = a] 
&= \EE[(a - y^i)^2|\phi(x^i) = s] \\ 
&= a^2 + (1 - 2a) \EE[y^i|\phi(x^i) = s] \\ 
&= a^2 + (1-2a)f(s) \\ 
&= 1-\til \br_1(s,a)
\end{align*}
where the second inequality uses that $y^i$ is $\{0,1\}$-valued, and the third inequality uses the definition of $f$.

By construction, we have $M \in \MM_n$. Thus, by assumption on $\Obandits$, the output of $\Obandits$ (and thus $\Alg$), after at most $S(n)$ simulated episodes of interaction with $M$, is a circuit $\MC_{\hat \pi}$ of size at most $B(n)$. Moreover, with probability $1-o(1)$, $\MC_{\hat\pi}$ computes a policy $\hat\pi: \{0,1\}^{\ell_n} \to \MA_n$ with suboptimality at most $\epsilon(n)$. Since the policy $\pi$ that maps $x \mapsto \delta(n) \lfloor f(\phi(x))/\delta(n)\rfloor \in \MA_n$ has value 
\[\E^{M,\pi}[\til\br_1(s_1,a_1)] \geq \EE_{s \sim \til \BP_0}[1 - \delta(n)^2 - f(s) + f(s)^2],\]
it follows that
\[\EE_{\substack{s\sim\til\BP_0 \\ x \sim \til\BO_1(\cdot|s)}}[(\hat\pi(x) - f(s))^2] = \EE_{s\sim\til\BP_0}[1-f(s)+f(s)^2] - \E^{M,\hat\pi}[\til\br_1(s_1,a_1)] \leq \delta(n)^2 + \epsilon(n).\]
Thus, $\hat\pi$ is a $(\til\BP_0, f, \epsilon'(n))$-predictor with respect to $\til\BO_1$, where $\epsilon'(n) := \delta(n)^2 + \epsilon(n)$. By assumption on $\Obandits$, the circuit $\MC$ describing $\hat\pi$ has size at most $B(n)$. The sample complexity $S'(n)$ of $\Alg$ is the same as the number of episodes of interaction needed by $\Obandits$, which is at most $S(n)$ by assumption. Finally, the time complexity of $\Alg$ (up to the oracle calls) is at most $S(n) \cdot \poly(\ell_n, |\MA|_n)$. So $\Alg$ is indeed a $(S, T, \epsilon', B)$-realizable regression algorithm for $\MM$ as specified by \cref{def:regression-algorithm}.

\end{proof}

\section{Block MDPs with deterministic dynamics}\label{sec:deterministic}
The following reduction is heavily inspired by the $\PPE$ algorithm from \cite{efroni2021provable}; the main difference is that their oracle is a multi-class maximum likelihood oracle, whereas ours is the regression oracle, but the analysis is roughly the same.

\begin{algorithm}[t]
    \caption{$\PPE(\delta, \Osample, \Oregress)$: Modification of Predictive Path Elimination \cite{efroni2021provable}}
    \label{alg:ppe}
    \begin{algorithmic}[1]
      \Require Parameters $\delta \in (0,1)$, a sample oracle $\Osample$, and a regression oracle $\Oregress$.
      \State Set $m \gets 128\log(2H|\MA|^2|\MS|^2/\delta)$
      \State Set $\Psi_1 \gets \{\perp\}$
      \For{$2 \leq h \leq H$}
            \State Initialize $\Psi_h \gets \emptyset$
            \For{$\ba^0 \in \Psi_{h-1}\times \MA$}
                \State $\mathsf{flagRedundant} \gets \False$
                \For{$\ba^1 \in \Psi_{h-1} \times \MA$ such that $\ba^1$ strictly precedes $\ba^0$ lexicographically}
                    \State Set $\bar\ba^0 \gets (\ba^0, a, \dots, a) \in \MA^H$ and $\bar \ba^1 \gets (\ba^1, a, \dots, a) \in \MA^H$ for a default action $a \in \MA$
                    \Function{$\MB_\pi$}{$k, \tau$}
                        \Comment{$k \in [H]$ and $\tau = (x_{1:k}, a_{1:k-1})$}
                        \State $b \sim \Unif(\{b' \in \{0,1\}: \bar\ba^{b'}_{1:k-1} = a_{1:k-1}\})$
                        \State \textbf{return} $\ba^b_k$
                    \EndFunction
                    \Function{$\MB_L$}{$x_{1:H}, a_{1:H}$}
                        \State \textbf{return} $\mathbbm{1}[a_{1:H} = \bar\ba^1_{1:H}]$
                    \EndFunction
                    \State $\MR \gets \Oregress(\MB_\pi,\MB_L,h)$
                    \For{$1 \leq i \leq m$}
                        \State $y^i \gets \Ber(1/2)$
                        \State $(x^i_{1:H}, a^i_{1:H}) \gets \Osample(\bar\ba^{y^i})$
                    \EndFor
                    \If{$\frac{1}{m}\sum_{i=1}^m (\MR(x^i_h) - y^i)^2 > \frac{1}{8}$}
                        \State $\mathsf{flagRedundant} \gets \True$
                    \EndIf
                \EndFor
                \If{not $\mathsf{flagRedundant}$}
                    \State Add $\ba^0$ to $\Psi_h$
                \EndIf
            \EndFor
      \EndFor
      \State \textbf{return} $(\Psi_h)_{h \in [H]}$
      \end{algorithmic}
  \end{algorithm}

\begin{definition}\label{def:det-dynamics}
We say that a block MDP $M = (H, \MS, \{0,1\}^\ell, \MA, \til\BP_0, (\til\BP_h)_h, (\til \BO_h)_h, (\til \br_h)_h,\phi^\st)$ \emph{has deterministic dynamics} if the distributions $\til\BP_0$ and $\til\BP_h(\cdot|s,a)$ are deterministic for all $h \in [H]$, $s \in \MS$, and $a \in \MA$.
\end{definition}

Note that a block MDP with deterministic dynamics may still have general, stochastic emission distributions.

\begin{lemma}\label{lemma:ppe-guarantee}
Let $\delta \in (0,1)$ and $B, \ell \in \NN$. Let $M = (H, \MS, \{0,1\}^\ell, \MA, \til\BP_0, (\til\BP_h)_h, (\til \BO_h)_h, (\til \br_h)_h,\phi^\st)$ be a block MDP with deterministic dynamics (\cref{def:det-dynamics}). Let $\Osample$ be a sampling oracle for $M$, and let $\Oregress$ be a $B$-bounded, $1/16$-accurate regression oracle for $M$. Then $\PPE(\delta,\Osample,\Oregress)$ has time complexity $\poly(H, \ell, |\MA|, |\MS|, B, \log(1/\delta))$, and the output $(\Psi_h)_{h \in [H]}$ satisfies the following properties with probability at least $1-\delta$: for every $h \in [H]$, it holds that $|\Psi_h| \leq |\MS|$. Also, for every $h \in [H]$ and $s \in \MS$,
\begin{equation} \max_{\ba \in \Psi_h} d^{M,\ba}_h(s) = \max_{\pi \in \Pi} d^{M,\pi}_h(s).\label{eq:det-cover}\end{equation}
\end{lemma}

\begin{proof}
Observe that since the dynamics and initialization are deterministic, for every $h \in [H]$ and $s \in \MS$ we have $\max_{\pi\in\Pi} d^{M,\pi}_h(s) \in \{0,1\}$. Moreover, the maximum is achieved by some fixed action sequence $\ba^\st(s)$. 

We now prove the lemma statement by induction on $h$. In particular, the induction hypothesis at step $h$ is that \cref{eq:det-cover} holds for all $s \in \MS$, and also $|\Psi_h| \leq \MS$. Observe that \cref{eq:det-cover} holds with probability $1$ for $h=1$ and all $s \in \MS$. Moreover $|\Psi_1| = 1$. So the induction hypothesis at step $1$ is clear.

Now fix any $2 \leq h \leq H$ and suppose that $\cref{eq:det-cover}$ holds at $h-1$ (for all $s \in \MS$). Fix any $s \in \MS$ with $\max_{\pi\in\Pi} d^{M,\pi}_h(s) = 1$. Then there is some action sequence $\ba$ that reaches $s$ at step $h$ with probability $1$. Let $s'$ be the state reached by this action sequence at step $h-1$. By the induction hypothesis, there is some $\ba' \in \Psi_{h-1}$ that reaches $s'$ at step $h-1$, with probability $1$. Thus, the sequence $(\ba'_1,\dots,\ba'_{h-2}, \ba_{h-1})$ reaches $s$ at step $h$ with probability $1$, and moreover is contained in $\Psi_{h-1} \times \MA$. Hence, the set $K(s) \subseteq \Psi_{h-1}\times\MA$ of action sequences in $\Psi_{h-1}\times\MA$ that reach $s$ at step $h-1$ is non-empty. Let $\ba^0$ denote the lexicographically first element of $K(s)$. For any lexicographically earlier $\ba^1 \in \Psi_{h-1}\times\MA$, the sequence $\ba^1$ reaches some state $\til s$ at step $h$, with probability $1$, and $\til s \neq s$. 

Now consider the policy $\pi$ and label function $L$ defined by $\MB_\pi$ and $\MB_L$ respectively, at the step of the execution of the algorithm corresponding to $\ba^0$ and $\ba^1$. Observe that $\pi$ is the policy that flips a coin $b \sim \Unif(\{0,1\})$ and follows policy $\ba^b$ for the entire episode. Thus, $d_h^{M,\pi}$ is uniform on $\{s, \til s\}$. Moreover, any trajectory $\tau = (x_{1:H}, a_{1:H})$ with $\phi^\st(x_h) = s$ must have $L(\tau) = 0$, and if $\phi^\st(x_h) = \til s$ then $L(\tau) = 1$. Thus, the function $f := \EE^{M,\pi}[L(x_{1:H},a_{1:H})|\phi^\st(x_h) = \cdot]$ has $f(s) = 0$ and $f(\til s) = 1$. This means that \[\EE_{\substack{s \sim d^{M,\pi}_h \\ y\sim\Ber(f(s))}} (f(s)-y)^2 = 0\]
and hence, since $\MR \gets \Oregress(\MB_\pi,\MB_L,h)$ is a $(d^{M,\pi}_h,f,1/16)$-accurate predictor with respect to $\til\BO_h(\cdot|s)$, 
\[\EE_{\substack{s\sim d^{M,\pi}_h \\ x \sim \til\BO_h(\cdot|s) \\ y\sim \Ber(f(s))}} (\MR(x) -y)^2 \leq \frac{1}{16}.\]
Observe that the samples $(x_h^i,y^i)_{i=1}^m$ are independent, and the distribution of each sample is identical to that of $(x,y)$ where $s\sim d^{M,\pi}_h$, $x \sim \til\BO_h(\cdot|s)$, and $y \sim \Ber(f(s))$. Thus $\EE (\MR(x_h^i) - y^i)^2 \leq 1/16$ for each $i \in [m]$. Since $m \geq 128\log(2H|\MA||\MS|^2/\delta)$, Hoeffding's inequality gives that $\frac{1}{m}\sum_{i=1}^m (\MR(x_h^i) - y^i)^2 \leq \frac{1}{8}$ with probability at least $1-\delta/(2H|\MA||\MS|^2)$. By the induction hypothesis, $|\Psi_{h-1}| \leq |\MS|$, so by a union bound over all $\ba^1 \in \Psi_{h-1}\times\MA$ that lexicographically precede $\ba^0$, we get that with probability at least $1-\delta/(2H|\MS|)$, the action sequence $\ba^0$ is not marked redundant, and thus is added to $\Psi_h$. Thus, by another union bound, \cref{eq:det-cover} is satisfied at step $h$ for all $s \in \MS$ with probability at least $1-\delta/(2H)$.

Next we check that $|\Psi_h| \leq |\MS|$. Indeed, consider any $\ba^0$ and $\ba^1$ that reach the same state $s \in \MS$ at step $h$, such that $\ba^1$ is lexicographically earlier than $\ba^0$. Then in fact for each $i \in [m]$, the label $y^i$ is independent of $x_h^i$. So $\EE(\MR(x_h^i) - y^i)^2 \geq \frac{1}{4}$. By Hoeffding's inequality we get that $\frac{1}{m}\sum_{i=1}^m (\MR(x_h^i) - y^i)^2 \geq \frac{1}{8}$ with probability at least $1-\delta/(2H|\MA|^2|\MS|^2)$. So $\ba^0$ is not added to $\Psi_h$. By a union bound over $\ba^0, \ba^1 \in \Psi_{h-1}\times \MA$, we get that with probability at least $1-\delta/(2H)$, only the lexicographically first action sequence to reach any given state is added to $\Psi_h$, so indeed $|\Psi_h| \leq |\MS|$. This completes the induction.
\end{proof}

\begin{proposition}
Let $\MM = ((\MS_n)_n,(\MA_n)_n,(H_n)_n,(\ell_n)_n,(\Phi_n)_n,(\MM_n)_n)$ be a block MDP family in which each MDP has deterministic dynamics (\cref{def:det-dynamics}). For any function $B:\NN\to\NN$, there is a $(T,\epsilon,B)$-computational reduction from RL to regression for $M$ (\cref{def:rl-reg-reduction-2}), with $T(n) \leq \poly(|\MS_n|,|\MA_n|,H_n,\ell_n,B(n))$ and $\epsilon(n) := 1/(16H_n^6 |\MS_n||\MA_n|)$.
\end{proposition}

\begin{proof}
The reduction $\Alg$ has access to oracles $\Oregress$ and $\Osample$ and takes as input a parameter $n \in \NN$. First, $\Alg$ computes $(\Psi_h)_{h \in [H]} \gets \PPE(1/2, \Osample, \Oregress)$. Next, for each $\bar a\in\MA$ and $h \in [H]$, let $\pi^{\bar a, h}$ be the policy that picks $\ba \sim \Unif(\Psi_h)$, and then follows the action sequence $(\ba_1,\dots,\ba_{h-1}, \bar a, \dots,\bar a)$. Let $\MO$ be the procedure that takes as input a (randomized) circuit $\MB_L$, a step $h \in [H]$, and an action $a \in \MA$, and returns $\Oregress(\MB_{\pi^{a,h}}, \MB'_L, h)$ where $\MB_{\pi^{a,h}}$ is a circuit describing $\pi^{a,h}$, and $\MB'_L$ is the circuit defined by $\MB'_L(x_{1:H},a_{1:H}, r_{1:H}) = \MB_L(x_h, r_h, x_{h+1})$. Then $\Alg$ computes and outputs $\hat \pi \gets \FQI(\MO)$. 

\paragraph{Analysis.} Fix $n \in \NN$ and $M \in \MM_n$. For notational simplicity we will write $\MA = \MA_n$, $\MS = \MS_n$, $\MX = \MX_n = \{0,1\}^{\ell_n}$, and $H = H_n$ for the majority of the analysis. Suppose that $\Oregress$ is a $B(n)$-bounded, $\epsilon(n)$-accurate regression oracle for $M$, and $\Osample$ is a sampling oracle for $M$. By \cref{lemma:ppe-guarantee} and the fact that $\epsilon(n) \leq 1/16$, there is an event $\ME$ that occurs with probability at least $1/2$ under which $\max_{\ba \in \Psi_h} d^{M,\ba}_h(s) = \max_{\pi\in\Pi} d^{M,\pi}_h(s)$ for all $h \in [H]$ and $s \in \MS$, and also $|\Psi_h| \leq |\MS|$ for all $h \in [H]$. Condition on $\ME$ henceforth. For each $h \in [H]$, let $\mu_h \in \Delta(\MS)$ be the distribution with density $\mu_h(s) = \EE_{\ba\in\Psi_h} d^{M,\ba}_h(s)$. Observe that $d^{M,\pi}_h(s) \leq |\Psi_h| \mu_h(s) \leq |\MS| \mu_h(s)$ for every $s \in \MS$ and $\pi \in \Pi$, so $\mu_h$ satisfies $|\MS|$-coverage at step $h$ for $M$ (\cref{def:kappa-coverage}).

Next, we claim that $\MO$ is a $B(n)$-bounded fixed-distribution regression oracle for $M$ that is $\epsilon(n)$-accurate over $(\mu_h)_{h\in[H]}$ (\cref{def:fd-reg-oracle}). Indeed, fix any tuple $(\MB_L, h, a)$ where $\MB_L$ is a randomized circuit describing a function $L:\MX\times\{0,1\}\times\MX\to\Delta(\{0,1\})$, and $h \in [H]$ and $a \in \MA$. By the definition of $\MO$, the oracle output $\MR \gets \MO(\MB_L, h, a)$ is $(d^{M,\pi^{a,h}},f,\epsilon(n))$-accurate with respect to $\til\BO_h$, where $f(s) = \EE^{M, \pi^{a,h}}[L'(x_{1:H},a_{1:H},r_{1:H}) | s_h = s]$. Since $\pi^{a,h}$ (restricted to the first $h-1$ steps) is the uniform mixture policy over $\Psi_h$, we have $d^{M,\pi^{a,h}}_h = \mu_h$. Since policy $\pi^{a,h}$ plays action $a$ deterministically from step $h$ onwards, we have (by definition of $L'$) that
\[f(s) = \E^{M,\pi^{a,h}}[L(x_h, r_h, x_{h+1})|s_h=s] = \EE_{\substack{x_h \sim \til\BO_h(\cdot|s) \\ r_h \sim \Ber(\til\br_h(s, a)) \\ x_{h+1} \sim \BP_h(\cdot|x_h, a)}} L(x_h, r_h, x_{h+1}).\]
Moreover, since $\Oregress$ is $B(n)$-bounded, it's immediate that $\size(\MR) \leq B(n)$. Thus, \cref{def:fd-reg-oracle} is indeed satisfied.

By \cref{theorem:fqi-guarantee}, the sub-optimality of the policy $\hat\pi$ produced by $\FQI(\MO)$ is at most $2H_n^3\sqrt{|\MS_n||\MA_n|\epsilon(n)}$, which is at most $1/2$ by definition of $\epsilon$. Moreover, since $\Oregress$ is $B(n)$-bounded, it's clear from the algorithm description and \cref{lemma:ppe-guarantee,theorem:fqi-guarantee} that the time complexity of $\Alg$ is at most $\poly(H_n,|\MS_n|,|\MA_n|,\ell_n,B(n))$.
\end{proof}


\end{document}